\documentclass[12pt]{article}
\usepackage[margin=0.9in]{geometry}

\usepackage{makecell,colortbl,skak}
\usepackage{microtype}
\usepackage{graphicx}
\usepackage{subfigure}
\usepackage{booktabs} 
 
\usepackage{times,rotating}
\usepackage{url}
\usepackage{booktabs, multicol, multirow}
\usepackage{morenotations2}
\usepackage{caption,fontawesome5}

\usepackage{hyperref}
\usepackage[capitalize,noabbrev]{cleveref}

\usepackage{yfonts}
\usepackage{tgadventor}

\title{\papertitle\footnote{\faHandPointRight[regular] This paper is subsumed by paper "Hyperbolic Embeddings of Supervised Models", by Richard Nock, Ehsan Amid, Frank Nielsen, Alexander Soen and Manfred K. Warmuth, appearing at NeurIPS'24.}}

\author{Richard Nock$^\dagger$ \quad Ehsan Amid$^\ddagger$ \quad Frank Nielsen$^\circ$\\
  Alexander Soen$^\ast$\quad Manfred K. Warmuth$^\dagger$\\\\
  $^\dagger$Google Research, $^\ddagger$Google DeepMind, $^\circ$Sony CS Labs, Inc.,\\ $^\ast$The Australian National University $\&$ RIKEN AIP\\\\
{\normalsize \texttt{$\{$richardnock,eamid,manfred$\}$@google.com}}\\{\normalsize \texttt{frank.nielsen@acm.org,alexander.soen@anu.edu.au}}
}
\begin{document}

\date{}

\maketitle

\begin{abstract}

Most mathematical distortions used in ML are fundamentally integral in nature: $f$-divergences, Bregman divergences, (regularized) optimal transport distances, integral probability metrics, geodesic distances, etc. 
In this paper, we unveil a grounded theory and tools which can help improve these distortions to better cope with ML requirements. 
We start with a generalization of Riemann integration that also encapsulates functions that are not strictly additive but are, more generally, $t$-additive, as in nonextensive statistical mechanics. Notably, this recovers Volterra's product integral as a special case. We then generalize the Fundamental Theorem of calculus using an extension of the (Euclidean) derivative. This, along with a series of more specific Theorems, serves as a basis for results showing how one can specifically design, alter, or change fundamental properties of distortion measures in a simple way, with a special emphasis on geometric- and ML-related properties that are the metricity, hyperbolicity, and encoding. We show how to apply it to a problem that has recently gained traction in ML: hyperbolic embeddings with a ``cheap" and accurate encoding along the hyperbolic vs Euclidean scale. We unveil a new application for which the Poincar\'e disk model has very appealing features, and our theory comes in handy: \textit{model} embeddings for boosted combinations of decision trees, trained using the log-loss (trees) and logistic loss (combinations).

\end{abstract}

\section{Introduction}
\label{sec:intro}

The design, use, and refinement of mathematical objects comes with the special purpose in machine learning to tie them to solving data-related problems. Consider one of the most ubiquitous of those: integration. Integration is the cornerstone of most distortions used in ML, from $f$-divergences~\cite{fdiv-AliSilvey-1966,Csiszar-1967} to Bregman divergences~\cite{Bregman-1967}, (regularized) optimal transport~\cite{cuturi2013sinkhorn}, integral probability metrics~\cite{muller1997integral,birrell2022f} and geodesic distances~\cite{Rao-1945}, etc.\footnote{This comes from the definition for all these distortion measures, except for the Bregman divergence, for which a lesser known property states that a Bregman divergence is a path integral in disguise, see \textit{e.g.} \cite[Slide 99]{wOLA}.}. Integration is analytically additive, a trivial consequence of its definition as in \textit{e.g.} Riemann integral for scalar functions. Some divergences are also additive as a consequence of the \textit{properties} of the systems they study. For example, the KL divergence, which is both an $f$-divergence and a Bregman divergence~\cite{amari2009alpha}, is additive over the direct product of measurable spaces \cite[Section G]{vehRD}. Sometimes, however, additivity does not -- and cannot -- hold because of properties of the underlying system, as \textit{e.g.} in nonextensive statistical mechanics \cite{TsallisBook-2022}. A more general property than additivity, called $t$-additivity, then holds. For example, Tsallis divergence (an $f$-divergence~\cite{fdiv-AliSilvey-1966}) is $t$-additive on the simplex \cite{nnORAT}, the tempered relative entropy (a Bregman divergence) is $t$-additive on the co-simplex \cite{anwCA,nawBW}, etc. Pure mathematics has also early embraced non-additive integrations: Italian mathematician Vito Volterra (1860 - 1940) designed a product integral that can be motivated by the study of a particular ordinary differential equation \cite{dfPI,vSFD}, which later found a variety of applications in statistics \cite{gjAS,sPII}.

\textbf{Our first contribution} starts with a Theorem that brings together Riemann, Tsallis, and Volterra: we show that any $t$-additive (scalar) function behaves in the limit like a non-linear function of Riemann integral. The result, which is interesting in itself, roughly states if we let $S^{(t)}_n$ denote the $n$-sum of a $t$-additive function, that ($S^{(t)}_{\infty} \defeq \lim_{n\rightarrow +\infty} S^{(t)}_n$)
\begin{eqnarray}
S^{(t)}_{\infty} & = & \log_t \exp S^{(1)}_{\infty}, \forall t\in \mathbb{R}, \quad \mbox{ with }\log_t(z) \defeq \left\{
\begin{array}{ll}
  \frac{z^{1-t}-1}{1-t}, & t\neq 1,\\
\log z & t=1.
\end{array}
\right., \label{eq-link-t-1}
\end{eqnarray}
and the particular case $S^{(1)}_{\infty}$ ($t=1$) is Riemann's integral (conditions for integration are the same for all $t$s and thus match Riemann's). Case $t=0$ is Volterra's integral up to additive constant 1. We then embed this Theorem in a generalization of Leibniz-Newton's fundamental theorem of calculus, which involves a generation of the classical derivative. We believe this tool can be of broad appeal to ML. We make this more concrete in the context of distortion measures.

\textbf{Our second contribution} shows how this simple additional parameter $t$ can be useful to design, alter, or change the properties of distortions according to key notions, including hyperbolicity, metricity, encoding, convexity, etc. To appreciate how our first two contributions naturally extend to the many properties and use of classical integration, we provide (in an Appendix, Sections \ref{proof_add-res}, \ref{proof_add-statinf}) a range of extensions orbiting in the context of our paper, ranging from basic ones like Chasles relationship, additivity, monotonicity, to intermediate ones like the mean-value Theorem and more specific ones like the hyperbolic Pythagorean Theorem, the data processing inequality and the nature of population minimizers. Some of these results are used to address a very specific ML problem that has recently gained traction.

\textbf{Our third contribution} shows how $t$ can be used in the context of hyperbolic geometry to tackle the problem of numerical accuracy in hyperbolic embeddings, with a special focus on the Poincar\'e disk model, for which this problem is crucial \cite{sdgrRT}. Specifically, we show how tuning $t$ can allow to improve encoding requirements while keeping hyperbolicity under control, \textit{without} altering geodesics. Our main application is one we believe has never been investigated despite its attractivity for ML, in particular regarding hyperbolic embeddings.

\textbf{Our main application} is the problem of \textbf{model} embedding in hyperbolic geometry, and in particular embedding sets of boosted decision trees (DTs), to accurately capture both the symbolic (tree representation, branching logical rules, etc.) and numeric part of DT classification, the latter covering both the \textit{local} level of each DT node (classification \textit{and} confidence) and the \textit{global} level of the DT's leverage in a boosted ensemble (leveraging coefficient). This application comes with two side contributions of independent interest: (i) we show a link between training with the log-loss (posterior estimation) and logistic losses (real-valued classification) and hyperbolic distance computation in Poincar\'e disk which makes this particular model of hyperbolic geometry very appealing for our purpose; (ii) we show how to overcome a technical difficulty for a ``clean" embedding of DTs, via a new class of tree-shaped models we nickname \textit{monotonic decision trees}, used to model monotonic classification paths in DTs -- itself providing an explainable alternative to DTs via monoticity~\cite{sbTI}. Experiments complete this application.

To ease reading, all proofs of results appearing in the main body are given in the Appendix.

\section{Related work}
\label{sec:related}

The understanding and formalization of ML has grown from numerous mathematical fields, in particular over the past decade. Mathematical objects are built from a subjective truth: first principles \cite{bLTF}. Refinements can be necessary to properly ``fit" into ML shoes, such as to take into consideration an accurate encoding of such objects. Doing otherwise can lead to catastrophic approximations, as witnessed, among others, in hyperbolic geometry \cite{mwwyTN}. Sometimes also, basic properties need to be altered as a consequence of the system studied. For example, decomposing a system as the product of two independent systems decomposes its Shannon entropy as a sum of the two systems' entropies. In nonextensive statistical mechanics -- whose theory has found many applications in ML -- \cite{TsallisBook-2022}, this cannot hold: the decomposition satisfies a more general notion known as $t$-additivity. What happens if we carry this decomposition \textit{ad infinitum}? To our knowledge, this fundamental question does not yet have an answer. Sticking to a simple mathematical perspective, this can be understood by replacing additivity by $t$-additivity in \textit{e.g.} Riemann integration. Solving it in such a context ideally imposes generalizing the Fundamental Theorem of calculus, \textit{i.e.} eliciting the corresponding generalization of a derivative.

While this question is already interesting from a pure mathematical standpoint, putting it in ML-context brings a quite fascinating perspective linked to its origins. Nonextensive statistical mechanics allow to substantially expand the realm of Boltzmann-Gibbs theory \cite[Chapter 1]{TsallisBook-2022}. This theory emphasizes a key thermodynamic measure, entropy \cite{jGVB}, whose links with information theory were later highlighted in Shannon's own work \cite{sAMT}, with the considerable footprint on ML that we know for this latter one. On top of it, integration is the basic tool allowing to compute such quantities in the limit, beyond the realm of distortion measures summarized in the introduction.

Concerning our specific application, models of hyperbolic geometry have been mainly useful to embed hierarchies \cite{gbhHE,nkPE,yltdSC,yzyckHR}, with a sustained emphasis on coding size and numerical accuracy \cite{mwwyTN,sdgrRT,ydNA}. In unsupervised learning and clustering, some applications have sought a simple representation of data on the form of a tree or via hyperbolic projections \cite{cgnrHH,cgcrFT,llszLD,sgTI}. Approaches dealing with supervised learning assume the \textit{data} lies in an hyperbolic space: the output visualized is therefore an embedding of the data itself, with additional details linked to the classification method, either support vector machines \cite{cdpbLM}, neural nets, logistic regression \cite{gbhHN} or (ensembles of) decision trees \cite{ctkpFH,dmsmHR}. We insist on the fact that those latter methods do not represent the \textit{models} in the hyperbolic space, even when those models are indeed tree-shaped, which makes hyperbolic geometry an especially appealing framework. This question of embedding classifiers is potentially important to improve the state of the art visualization: in the case of decision trees, popular packages stick to a topological layer (the tree graph) to which various additional information about classification are superimposed but without link to the ``embedding"\footnote{\url{https://github.com/parrt/dtreeviz}.}.

\section{T-calculus and a fundamental theorem}
\label{sec:t-riem}

At the core of our generalization is the replacement of the addition by the tempered addition~\cite{nlwGA}, $z \oplus_t z' \defeq z + z' + (1-t)zz'$. The additional term is a scaled saddle point curve $zz'$, which is positive when the sign of $z$
and $z'$ agree and negative otherwise.
We let $[n] \defeq \{1, 2, ..., n\}$ for $n \in \mathbb{N}_{>0}$. Hereafter, $f$ is a function defined on an interval $[a,b]$. We define a generalization of Riemann integration to the tempered algebra (see~\cite{amid2023tempered} for a preliminary development) and for this objective, given an interval $[a,b]$ and a division $\Delta$ of this interval using $n+1$ reals $x_0 \defeq a <x_1 < ... < x_{n-1} < x_n \defeq  b$, we define the Riemann $t$-sum of $f$ over $[a,b]$ using $\Delta$, for a set $\{\xi_i \in [x_{i-1}, x_i]\}_{i\in [n]}$,
\begin{eqnarray}
S^{(t)}_\Delta(f) & \defeq & (\mbox{\Large $\oplus$}_t)_{i\in [n]} (x_i - x_{i-1}) \cdot f(\xi_i),\quad \left(a \oplus_t b \defeq a + b + (1-t)ab\right).\label{deftRiem}
\end{eqnarray}
($S^{(1)}_\Delta(f)$ is the classical Riemann summation) Let $s(\Delta) \defeq \max_i |\mathbb{I}_i|$ denote the step of division $\Delta$. The conditions for $t$-Riemann integration are the same as for $t=1$.
\begin{definition}\label{tint}
  Fix $t\in \mathbb{R}$. A function $f$ is $t$-(Riemann) integrable over $[a,b]$ iff there exists $L\in \mathbb{R}$ such that
\begin{eqnarray}
  \forall \epsilon > 0, \exists \delta > 0 : \forall \mbox{ division } \Delta \mbox{ with }  s(\Delta) < \delta, \nonumber\\
  \left|S^{(t)}_\Delta(f) - L\right| < \varepsilon. \label{eq-tRiem}
  \end{eqnarray}
  When this happens, we note
  \begin{equation}
\label{eq:t-int-def}
    \riemannint{t}{a}{b} f(x) \mathrm{d}_t x = L.
  \end{equation}
\end{definition}
The case $t=1$, Riemann integration, is denoted using classical notations. We now prove our first main results for this Section, namely the link between $t$-Riemann integration and Riemann integration.
\begin{theorem}\label{thtRIEMANN}
  Any function is either $t$-Riemann integrable for all $t\in \mathbb{R}$ simultaneously, or for none. In the former case, we have the relationship
  \begin{eqnarray}
\riemannint{t}{a}{b} f(u) \mathrm{d}_t u & = & \liftt_t \left( \int_a^b f(u) \mathrm{d} u \right), \forall t \in \mathbb{R} ,\label{eq-t-riemann} \quad\mbox{ with } \liftt_t(z) \defeq \log_t \exp z. \label{eq-liftt}
        \end{eqnarray}
  \end{theorem}
  (Proof in \supplement, Section \ref{proof_thtRIEMANN}) Interestingly, for the case $t=0$, we get $1 + \riemannint{0}{a}{b} f(u) \mathrm{d}_t u = \exp \int_a^b f(u) \mathrm{d} u$, which is Volterra's integral \cite[Theorem 5.5.11]{sPII}. Classical Riemann integration and derivation are fundamental inverse operations. Classical derivation is sometimes called ``Euclidean derivative" \cite{bmTH} by opposition to other forms of derivative, such as hyperbolic derivatives: the former uses the Euclidean difference quotient, the latter the hyperbolic difference quotient. We now elicit the notion of derivative which is the ``inverse" of $t$-Riemann integration. Unsurprisingly, it generalizes Euclidean derivative. The Theorem stands as a generalization of the classical fundamental Theorem of calculus. 
  \begin{theorem}\label{thtRIEM-DER}
Let
  \begin{equation}
    \label{eq:t-derivative}
    \mathrm{D}_t f(z) \defeq \lim_{\delta \rightarrow 0} \frac{f(z + \delta) \ominus_t f(z)}{\delta} ,\quad \left(a \ominus_t b \defeq \frac{a-b}{1+(1-t)b}\right).
\end{equation}
Suppose $f$ $t$-Riemann integrable. Then function
  \begin{eqnarray}
    \begin{array}{rcl}
      F : [a,b] & \rightarrow & \mathbb{R}\\
      z & \mapsto & \riemannint{t}{a}{z} f(u) \mathrm{d}_t u
      \end{array}\label{defFT}
  \end{eqnarray}
  is such that $\mathrm{D}_t F = f$. We call $F$ a $t$-primitive of $f$ (which zeroes when $z=a$) and $\mathrm{D}_t F$ the $t$-derivative of $F$.
\end{theorem}
The function $\liftt_{t}$ (Figure \ref{fig:logtexp}) is key to many of our results; it has quite remarkable properties summarized below which follow from Theorem \ref{thtRIEMANN}, $\mathrm{D}_t$ \eqref{eq:t-derivative} and the properties of $\log_t$.
\begin{lemma}\label{lem-prop-tlift}
  $\liftt_{t}$ satisfies the following properties:
  \begin{itemize}
  \item [1.] $\liftt_{t}(z)$ is strictly increasing for any $t\in \mathbb{R}$, strictly concave for any $t>1$, strictly convex for any $t<1$ and such that $\mathrm{sign}(\liftt_{t}(z)) = \sign(z), \forall t \in \mathbb{R}$; 
  \item [2.] $\mathrm{D}_t\liftt_t(z)  = 1$;
  \item [3.] ($t$-integral mean-value) Suppose $f$ Riemann integrable over $[a,b]$. Then there exists $c \in (a,b)$ such that
  \vspace{-0.1cm}
    \begin{eqnarray*}
\hspace{-0.7cm} (b-a) \cdot \liftt_{t'} \circ f(c) = \hspace{-0.4cm} \riemannint{t}{a}{b} f(u)\mathrm{d}_tu \:\: (t' \defeq 1-(1-t)(b-a)).
    \end{eqnarray*}
    \end{itemize}
  \end{lemma}
  Looking at [2.] and [3.], one can remark that replacing $\liftt_t$ by the identity function $z\mapsto z$ brings two well known properties of classical derivation and integration. In that sense, $\liftt_t$ is somehow the ``identity function" of tempered calculus.
\begin{figure}
  \centering 
\includegraphics[trim=0bp 0bp 0bp 0bp,clip,width=0.6\columnwidth]{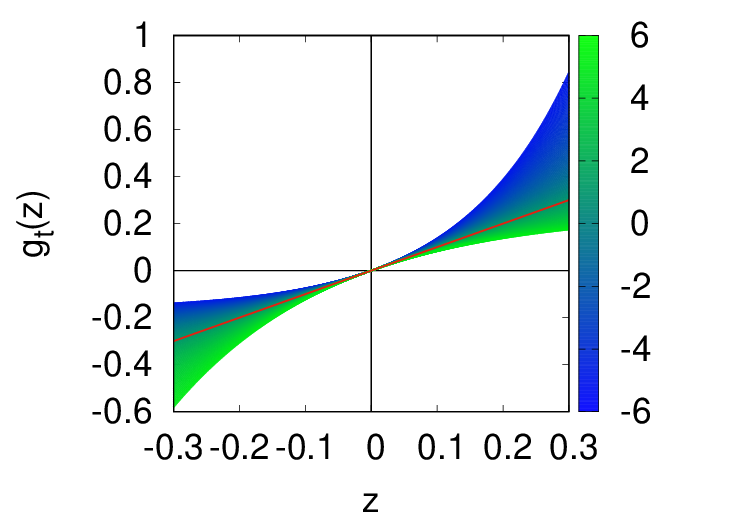} 
    \caption{Plot of $\liftt_t(z)$ \eqref{eq-liftt} for different values of $t$ (color map on the right bar), showing where it is {\color{blue} convex} / {\color{darkgreen} concave}. The $t=1$ case ($\liftt_1(z) = z$) is emphasized in {\color{red} red}.}
    \label{fig:logtexp}
  \end{figure}
Many more results are provided in \supplement~(Section \ref{proof_add-res}) showing how Theorems \ref{thtRIEMANN} and \ref{thtRIEM-DER} naturally ``percolate" through many properties known for classical integration and the fundamental Theorem of calculus. 

\section{Sets endowed with a distortion and their t-self}
\label{sec:t-hyp}

The basic element of this Section is a set $\mathcal{X}$ endowed with some function $d : \mathcal{X} \times \mathcal{X} \rightarrow [0,+\infty]$ that serves as a comparison measure between its elements.
\begin{definition}
Let $\mathcal{X}$ be endowed with function $d : \mathcal{X} \times \mathcal{X} \rightarrow [0,+\infty]$. The t-self of $\mathcal{X}$ is the set (implicitly) endowed with  $d^{(t)} \defeq \liftt_t \circ d$.
\end{definition}
When $d$ is derived from an integral, we have the link between $d$ and $d^{(t)}$ of Theorem \ref{thtRIEMANN}: the purpose of this Section is to demonstrate how, by just sliding $t$, for classical distortion measures new properties can appear and existing ones can change. Due to the sheer number of examples that can be found in ML and the focus of our paper, we concentrate here on geometric examples (so it is useful to think that $d$ is a geodesic length) and defer to Section \ref{proof_add-statinf} in \supplement~additional examples where statistical information replaces geometric information. Let us first start with an important result, trivial because $\liftt_t$ is strictly increasing for any $t\in \mathbb{R}$: if one curve has smaller length than another one according to $d^{(t)}$, it remains smaller if we change $t$ for any $t' \neq t$. Hence,
\begin{lemma}
Geodesics are invariant to the change of $t$: they are the same in the t-self of $\mathcal{X}$ and in $\mathcal{X}$.
  \end{lemma}
  To be very specific, invariance is defined using \textit{pregeodesics} on the t-self because \eqref{eq-liftt} allows to define a geodesic length parameterized by a diffeomorphism on arc length. In terms of visualization, there is no difference, which is why we stick to the term "geodesic" to avoid ladening our narrative. In the case of geometric information, classical properties for $d$ are \textbf{(R)}eflexivity ($d(\ve{x},\ve{x}) = 0$), \textbf{(S)}ymmetry ($d(\ve{x},\ve{y}) = d(\ve{y},\ve{x})$), the \textbf{(I)}dentity of indiscernibles ($d(\ve{x},\ve{y}) = 0 \Rightarrow \ve{x} = \ve{y}$) and of course the \textbf{(T)}riangle inequality ($d(\ve{x}, \ve{z}) \leq d(\ve{x}, \ve{y}) + d(\ve{y}, \ve{z})$), for any $\ve{x}, \ve{y}, \ve{z} \in \mathcal{X}$. Another property that is important for our purpose, that we state keeping aside some technical details that can be found in \cite{bACOG}.
  \begin{definition}\label{def-hyperbolicity}
  \textbf{(H)} $d$ is $\tau$-hyperbolic for some $\tau\geq 0$ iff for any three geodesic curves $\upgamma^1, \upgamma^2, \upgamma^3$ linking three points, there exists $\ve{x} \in \mathcal{X}$ such that $\max_i d(\ve{x}, \upgamma^i) \leq \tau$, where $d(\ve{x}, \upgamma) \defeq \inf_j d(\ve{x}, \upgamma_j)$.
\end{definition}
Let us first show how we can create hyperbolicity from $d$ being solely non-negative.
\begin{lemma}\label{lem-create-hyperbolicity}
  For any $\tau > 0$, if we let $t = 1 + (1/\tau)$, then the t-self of $\mathcal{X}$ satisfies \textbf{(H)} for hyperbolic constant $\tau$.
\end{lemma}
The proof consists in remarking that since $d$ is non negative, $d^{(t)}$ satisfies \textbf{(H)} for hyperbolic constant $\tau$ iff
\begin{eqnarray*}
\frac{1-\exp(-(t-1) z)}{t-1} & \leq & \tau,
  \end{eqnarray*}
for any $z\geq 0$, which is clearly achieved for $t = 1 + (1/\tau)$. We can also tune hyperbolicity, which is a trivial consequence of Lemma \ref{lem-prop-tlift} (point [1.]).
\begin{lemma}
If $d$ is $\tau$-hyperbolic then $d^{(t)}$ is $\liftt_t(\tau)$-hyperbolic, for any $t \in \mathbb{R}$.
\end{lemma}
We now investigate two models of hyperbolic geometry, the Lorentz model and Poincar\'e disk model, with a special emphasis on the latter one because of our application.

\paragraph{Lorentz model} We consider the Lorentz model for a simple illustration in which one creates approximate metricity in the t-self. This $d$-dimensional manifold is embedded in $\mathbb{R}^{d,1}$ via the hyperboloid with constant $-c<0$ curvature and defined by $\mathbb{H}_c \defeq \{\ve{x}\in \mathbb{R}^{d,1}: x_0 > 0 \wedge \ve{x} \circ \ve{x} = -1/c\}$, with $\ve{x} \circ \ve{y} \defeq -x_0y_0 + \sum_{i=1}^d x_iy_i$ the Lorentzian inner product \cite[Chapter 3]{rFOH}. In ML, two distortions are considered  on $\mathbb{H}_c$ \cite{llszLD}, one of which is the Lorentzian ``distance" $d_L$:
    \begin{eqnarray*}
      d_L(\ve{x}, \ve{y}) & \defeq & -\frac{2}{c} - 2 \cdot \ve{x} \circ \ve{y}.
    \end{eqnarray*}
    $d_L$ is notoriously not a distance because it does not satisfy the triangle inequality, yet we show that we can pick $t$ and the curvature in such a way that the t-self is arbitrarily close to a metric space.
    \begin{lemma}\label{lem-d-l-metric}
      For any $\delta > 0$, choose parameter $t$ and curvature $c$ as:
      \begin{eqnarray*}
t \defeq 1 + \frac{1}{\delta} & ; & c = \frac{2}{\delta}.
      \end{eqnarray*}
      Then the t-self of $\mathbb{H}_c$ is approximately metric: $d_L^{(t)}$ satisfies \textbf{(R)}, \textbf{(S)}, \textbf{(I)}, and the $\delta$-triangle inequality,
      \begin{eqnarray}
d_L^{(t)}(\ve{x}, \ve{z}) \leq d_L^{(t)}(\ve{x}, \ve{y}) + d_L^{(t)}(\ve{y}, \ve{z}) + \delta, \forall \ve{x}, \ve{y}, \ve{z} \in \mathbb{H}_c. \label{deltaT}
        \end{eqnarray}
      \end{lemma}
      (Proof in \supplement, Section \ref{proof_lem-d-l-metric})
      \paragraph{Poincar\'e disk model} We now switch to Poincar\'e disk model, $\mathbb{B}_1$ (negative curvature, -1). It is well-known that this model has the advantage to have a small hyperbolic constant $\tau$ (Definition \ref{def-hyperbolicity}) but requires very high precision for a numerical encoding to prevent fatal error, in particular, near the border (which will be a critical region for our application) \cite{mwwyTN,sdgrRT}. Our objective is to show show how one can balance hyperbolicity and encoding by just tuning $t$ in the t-self \(\mathbb{B}^{(t)}_1\). 
      The following defines the critical region where high numeric error can occur~\cite{mwwyTN,sdgrRT}.
  \begin{definition}\label{def-close-enc}
A point $\ve{x}$ in the Poincar\'e disk is said $k$-close to the boundary if $\|\ve{x}\| = 1 - 10^{-k}$. It is said encodable iff $\|\ve{x}\|<1$ in machine encoding (it is not ``pushed to the boundary''). 
\end{definition}
Machine encoding constrains the maximum possible $k$: in the double precision floating-point representation (Float64), $k\approx 16$ \cite{mwwyTN}. The question is then what is the corresponding maximal distance $d_*$ from the origin $\ve{0}$ that this brings, because numerical error prevents us from encoding points between this zone and the boundary of the disk. 
In the case of Poincar\'e disk, this distance is a small affine order in $k$ \cite{mwwyTN}:
  \begin{eqnarray}
d_* & \leq & \log(2) + \log(10)\cdot k + O(10^{-k}),\label{cont-d-poinc}
  \end{eqnarray}
  which means in practice that only a ball of radius $d^* \approx 38$ around the origin can be accurately represented. This is in deep contrast with the Euclidean representation, where $d_* = \Omega(2^k)$. 
  The following Lemma explains how to get a middle ground, whereby careful selection of \( t \) results in a \( t \)-self which provides a new $d_*$ anywhere in the ``hyperbolic vs Euclidean" scale, while keeping a guaranteed (finite) hyperbolic constant.
  \begin{lemma}\label{lem-t-self-poincare}
    Pick any increasing function $g(k)\geq 0$. For the choice $t = 1 - f(k)$ where $f(k)\in \mathbb{R}$ is any function satisfying
    \begin{eqnarray}
\frac{\log\left(1+f(k)g(k)\right)}{f(k)} & \leq & \log(10) \cdot k,\label{eq-cond-f-g-dist}
    \end{eqnarray}
    we get in lieu of \eqref{cont-d-poinc} the maximal $d_*^{(t)}$ satisfying
    \begin{eqnarray}
d_*^{(t)} & \geq & g(k),\label{cont-d-poinc-t}
    \end{eqnarray}
    and the new hyperbolic constant $\tau_t$ of the t-self satisfies
    \begin{eqnarray}
\tau_t & = & \frac{\exp \left(f(k) \tau\right) - 1}{f(k)}. \label{eq-cond-f-g-tau}
    \end{eqnarray}
  \end{lemma}
  (Proof in \supplement, Section \ref{proof_lem-t-self-poincare}) Since both $\log(1+x) = x + o(x)$ and $\exp(x)-1 = x + o(x)$ in a neighborhood of $0$, we check that \eqref{eq-cond-f-g-dist} and \eqref{eq-cond-f-g-tau} get back to properties of the Poincar\'e model as $f(k) \rightarrow 0$ \cite{mwwyTN}. Depending on the sign of $f(k)$, the triangle inequality \textbf{(T)} can still hold or be replaced by a weaker version that parallels \eqref{deltaT}. We indeed have:
    \begin{eqnarray*}
d_{\mathbb{B}_1}^{(t)}(\ve{x}, \ve{z}) \leq d_{\mathbb{B}_1}^{(t)}(\ve{x}, \ve{y}) + d_{\mathbb{B}_1}^{(t)}(\ve{y}, \ve{z}) + \max\{0,f(k)\} \cdot d_{\mathbb{B}_1}^{(t)}(\ve{x}, \ve{y}) \cdot d_{\mathbb{B}_1}^{(t)}(\ve{y}, \ve{z}),
    \end{eqnarray*}
    and $d_{\mathbb{B}_1}^{(t)}$ has a particularly simple expression, shown here for the distance between the origin and some $\ve{z} \in \mathbb{B}_1$ with Euclidean norm $\|\ve{z}\|$:
    \begin{eqnarray}
d_{\mathbb{B}_1}(\ve{z}, \ve{0}) = \log \left(\frac{1+r}{1-r}\right) & \Rightarrow & d_{\mathbb{B}_1}^{(t)}(\ve{z}, \ve{0}) = \log_t \left(\frac{1+r}{1-r}\right), \quad r \defeq \|\ve{z}\|. \label{hypdistPD}
    \end{eqnarray}
    We then have two cases. We can easily approach back the Euclidean $d_*$ by picking $f(k) > 0$: choosing $f(k) = 1$ gets there and we still keep a finite hyperbolic constant, albeit exponential in the former one. If however we want to improve further the hyperbolic constant, we can pick some admissible $f(k) < 0$, but then \eqref{eq-cond-f-g-dist} will constrain $g(k)$ to a very small value.

    Let us now dig into another convenient property of the t-self regarding the mapping of the region close to the border $\partial \mathbb{B}_1$. Indeed, in our case (nodes of a decision tree), the better the models (ML-wise), the higher portion of the model (nodes) is embedded close to $\partial \mathbb{B}_1$. Given the risk of numerical instability in this region, we ideally want a t-self where (i) distortions are \textit{fair} to the embedding in the Poincar\'e disk model, \textit{i.e.} if some $\ve{z} \in \mathbb{B}_1$ (\textit{e.g.} the coordinate of a node) gets mapped to $\ve{z}^{(t)} \in \mathbb{B}^{(t)}_1$ (the t-self), then
     \begin{eqnarray}
d_{\mathbb{B}_1}^{(t)}(\ve{z}^{(t)}, \ve{0}) = d_{\mathbb{B}_1}(\ve{z}, \ve{0}); \label{eqIDDDT}
     \end{eqnarray}
     we also want (ii) that $\ve{z}^{(t)}$ be mapped sufficiently within $\mathbb{B}^{(t)}_1$ when $\ve{z}$ is close to $\partial \mathbb{B}_1$, and finally (iii) because the risk of numerical instabilities is minimal near the center of $\mathbb{B}_1$, we want this region to be mapped with minimal or null non-linear distortion in the t-self $\mathbb{B}^{(t)}_1$. Conditions (ii) and (iii) impose a non-linear embedding 
     \begin{eqnarray}
     \embedt_t :\mathbb{B}_1 \rightarrow \mathbb{B}^{(t)}_1. \label{defembedpt}
\end{eqnarray}
     It turns out that the design of $\embedt_t$ is trivial: it suffices to just change the (Euclidean) norm $r^{(t)} \defeq \|\ve{z}^{(t)}\|$ as a function of the norm $r \defeq \|\ve{z}\|$ using \eqref{eqIDDDT}. Figure \ref{fig:rtorprime} displays the corresponding relationship between $r$ and $r^{(t)}$. While it is clear for (iii) above that there is indeed minimal non-linear distortion near the center ($r=0$) for \textit{any} $t\in [0,1]$, condition (ii) is also conveniently satisfied: even for $t<1$ very close to 1 and $\ve{z}$ close to $\partial \mathbb{B}_1$ ($r$ close to 1), the mapping can send $\ve{z}^{(t)}$ substantially ``back in" the t-self $\mathbb{B}^{(t)}_1$ (for $t=0.7$ and $r = 1 - 10^{-4}$ -- \textit{i.e.} $k=4$ in Definition \ref{def-close-enc} --, we get $r^{(t)} \approx 0.96$).

     To start connecting this with our forthcoming application, Table \ref{tab:isolines-t-self} provides a visual assessment of what the t-self achieves in the disks themselves, from which we can spot clearly the achievement of both (ii) and (iii). We also remark that for the parameterization that matters to our application (probabilities), isolines corresponding to regularly spaced probabilities are also equidistant from each other in $\mathbb{B}_1$, which is highly convenient for visualization.

\begin{figure}
  \centering 
\includegraphics[trim=0bp 0bp 0bp 0bp,clip,width=0.4\columnwidth]{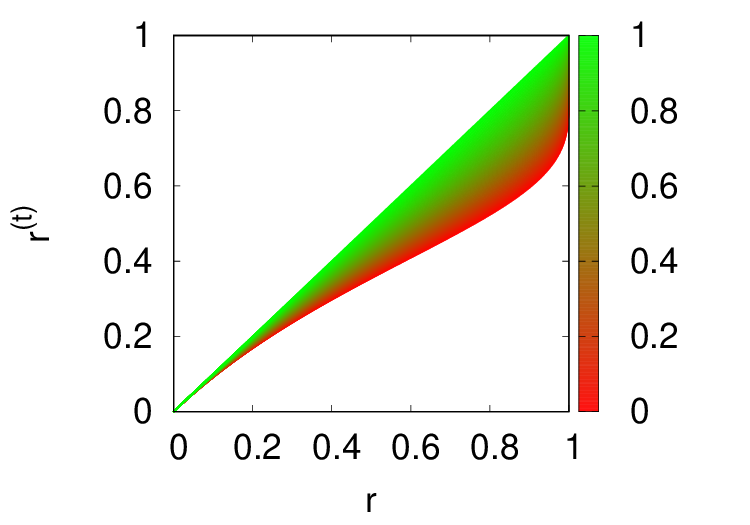} 
    \caption{Suppose $r \defeq \|\ve{z}\|$ is the norm a point $\ve{z}$ in Poincar\'e disk $\mathbb{B}_1$. Fix $t\in [0,1]$ (color bar). The plot gives the norm $r^{(t)}$ of a point $\ve{z}^{(t)}$ in the t-self such that $d_{\mathbb{B}_1}^{(t)}(\ve{z}^{(t)}, \ve{0}) = d_{\mathbb{B}_1}(\ve{z}, \ve{0})$.}
    \label{fig:rtorprime}
  \end{figure}
  
  \begin{table}
  \centering
  \resizebox{\textwidth}{!}{\begin{tabular}{cccc}\Xhline{2pt}
                             \includegraphics[trim=0bp 0bp 0bp 0bp,clip,width=0.3\columnwidth]{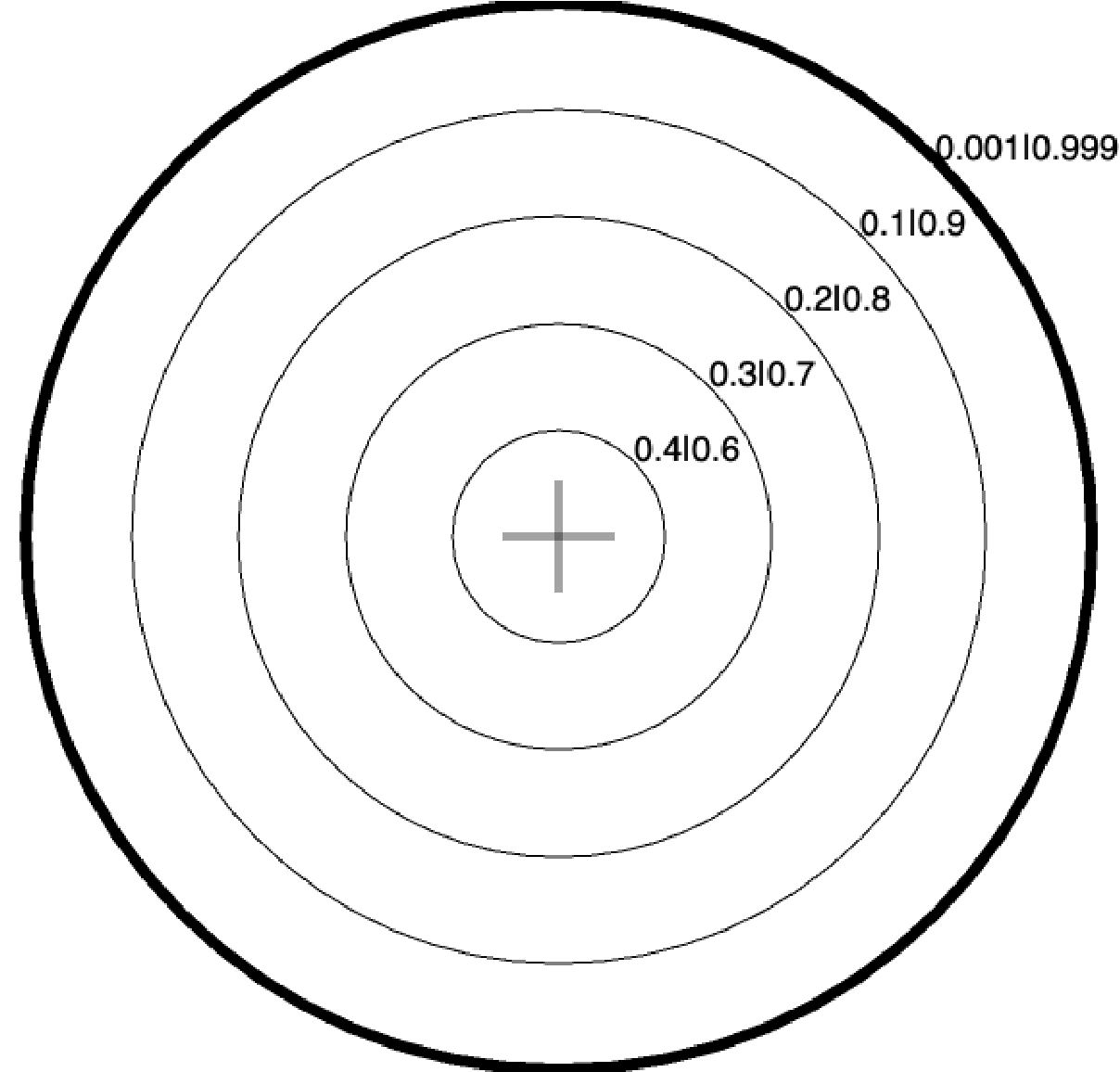} & \includegraphics[trim=0bp 0bp 0bp 0bp,clip,width=0.3\columnwidth]{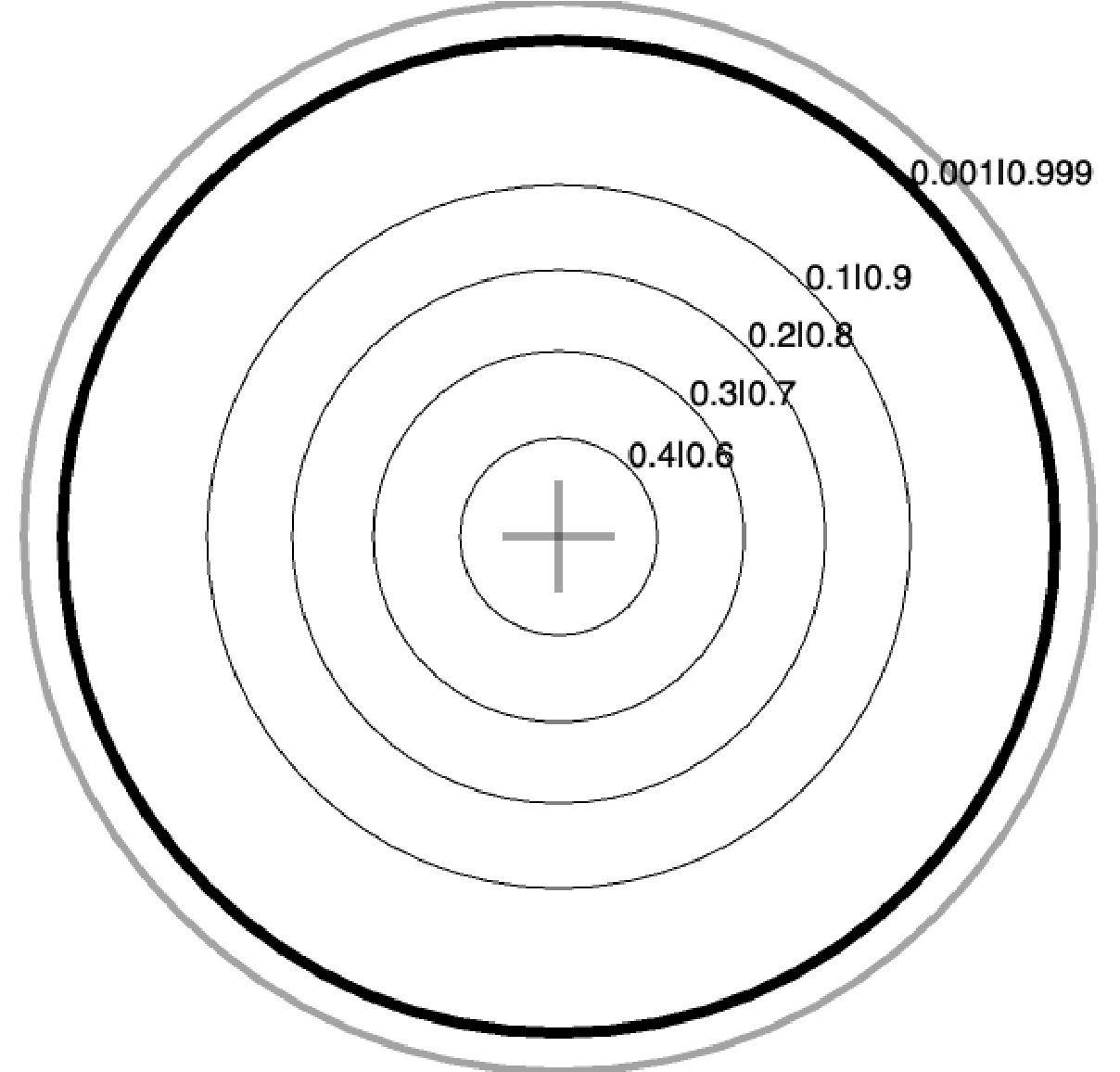} & \includegraphics[trim=0bp 0bp 0bp 0bp,clip,width=0.3\columnwidth]{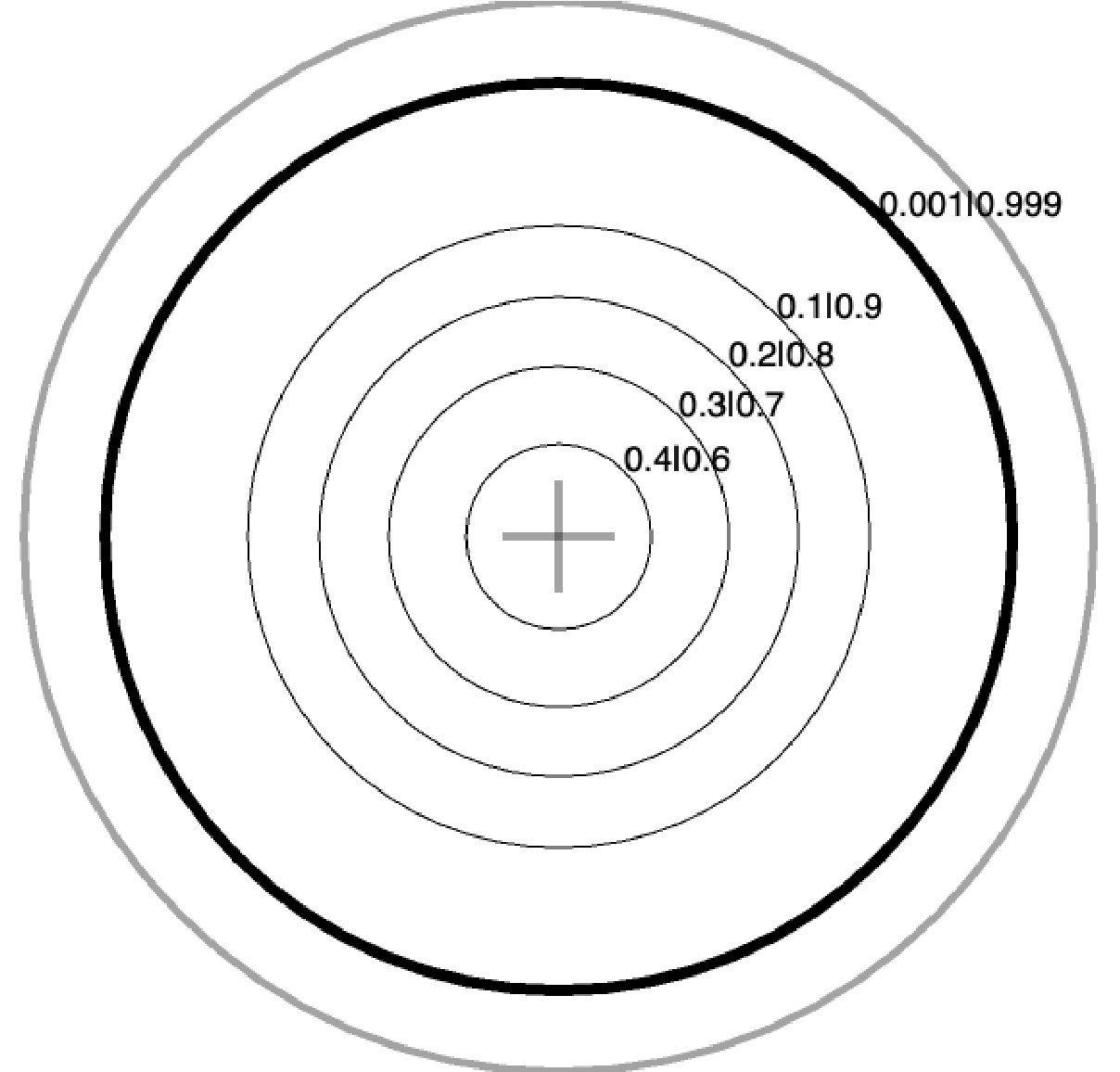}  &  \includegraphics[trim=0bp 0bp 0bp 0bp,clip,width=0.3\columnwidth]{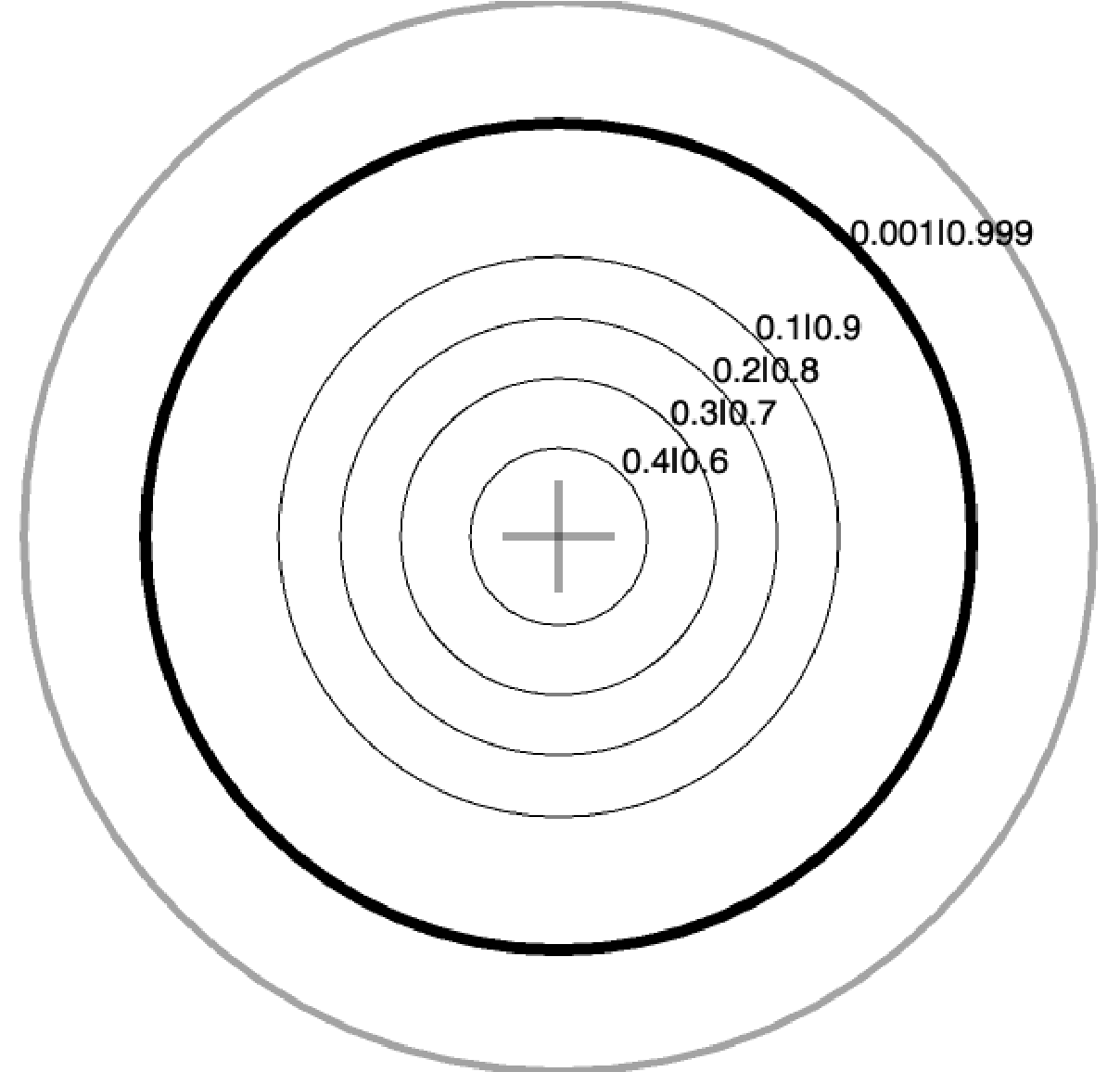} \\
                              $t=1.0$ ($\mathbb{B}_1$) & $t=0.6$ & $t=0.3$ & $t=0.0$\\  \Xhline{2pt} 
  \end{tabular}}
\caption{Starting from $t=1$ (Poincar\'e disk $\mathbb{B}_1$, left), we compute the t-self $\mathbb{B}^{(t)}_1$ in which we progressively decrease $t$ to $t=0$ (right). We display the same set of isolines among plots when invariance \eqref{eqIDDDT} holds. Isolines parameterized by a probability $p \in [0,1]$ which gives the norm $r \defeq |2p - 1|$ in $\mathbb{B}_1$. Remark \textit{e.g.} that if we regularly space $p \in \{0.6, 0.7, 0.8, 0.9, 1.0\}$, the corresponding isolines in $\mathbb{B}_1$ are also regularly spaced (this holds symmetrically for $p<1/2$). More, equidistance approximately remains even for $t<1$ for many isolines (up to isoline $p \in \{0.1,0.9\}$) while the distortion we need clearly happens near the border: outermost isoline $p \in \{0.001,0.009\}$ (plotted with \textbf{bigger} width) is smoothly and substantially ``moved" within the t-self as $t$ decreases, guaranteeing good readability and coding convenience for the most accurate model parts of our embedding application (see text).}
    \label{tab:isolines-t-self}
  \end{table}

\section{Model Embedding in Hyperbolic Geometry}
\label{sec:t-dt}

$\mathcal{S} \defeq \{(\ve{x}_i,y_i), i \in [m]\}$ denotes a training sample of $m$ examples, where $\ve{x}_. \in \mathcal{X}$ and $y_. \in \{-1,1\}$ (labels or classes). A decision tree (DT) consists of a binary rooted tree \( H \), where at each internal nodes the two outgoing arcs define a Boolean test over observation variables (see Figure \ref{fig:dt-vs-mdt}, left, for an example); $\nodeset(H)$ is the set of nodes of $H$, $\leafset(H) \subseteq \nodeset(H)$ is the set of leaf nodes of $H$; $\leaf_i$ is the leaf reached by example $(\ve{x}_i,y_i)$ and $p^+_{\leaf}$ is the local relative proportion of positive examples at leaf $\leaf$ (a \textit{posterior}). 

\paragraph{Proper losses and the log loss} We briefly present the theory of \textit{losses for class probability estimation} (CPE, which has been instrumental in DT induction). We follow the notations of \cite{mnwRC}: A CPE $\loss : \mathcal{Y} \times [0,1]
\rightarrow \mathbb{R}$, is
\begin{eqnarray}
\loss(y,p) & \defeq & \iver{y=1}\cdot \partialloss{1}(p) +
                     \iver{y=-1}\cdot \partialloss{-1}(p), \label{eqpartialloss}
\end{eqnarray}
where $\iver{.}$ is Iverson's bracket \cite{kTN}. Functions $\partialloss{1}, \partialloss{-1}$ are called \textit{partial} losses. A CPE loss is
\textit{symmetric} when $\partialloss{1}(p) = \partialloss{-1}(1-p),
\forall p \in [0,1]$. The \textit{log-loss} is a symmetric CPE loss with $\partiallogloss{1}(p) \defeq -\log(p)$. The fundamental property of a CPE loss is \textit{properness}, which enforces the optimality of Bayes rule for classification: letting $\cbr(\pi) \defeq \inf_p \E_{\Y\sim \pi} \loss(\Y, p)$ denote the minimal pointwise achievable risk (called Bayes risk) where we sample labels according to the \textit{prior} $\Pr[\Y = 1] = \pi$, properness imposes $\cbr(\pi) = \E_{\Y\sim \pi} \properloss(\Y, \pi), \forall \pi
\in [0,1]$ and \textit{strict} properness imposes in addition $\E_{\Y\sim \pi} \properloss(\Y, p) > \E_{\Y\sim \pi} \properloss(\Y, \pi), \forall p \neq \pi$. The log-loss is strictly proper; its Bayes risk is very familiar in DT induction: it is the splitting criterion used in C4.5 \cite{qC4}:
\begin{eqnarray}
\cbrlog(p) & = & -p\cdot \log p -(1-p) \cdot \log(1-p).\label{cbrlog}
\end{eqnarray}
DT induction usually proceeds with a top-down induction of a big tree following the minimization of the Bayes risk, thus relying on a model that predicts a \textit{posterior} at the leaves, $\hat{\Pr}[\Y = 1 | \X]$. There exists a connection between CPE classification and the real-valued classification setting familiar to \textit{e.g.} deep learning: optimizing Bayes risk for the posterior is equivalent to minimizing its \textit{convex surrogate} using the \textit{canonical link} of the loss to compute real classification \cite[Theorem 1]{nmSL}. The canonical link of the log-loss, $\canolog : [0,1] \rightarrow \mathbb{R}$ is the inverse sigmoid, which has a very convenient form for our purpose, 
\begin{eqnarray}
\canolog(p) = \log \left(\frac{p}{1-p}\right). \label{defCANO}
\end{eqnarray}
Notably, the absolute value $|\canolog(p)|$ is a \textit{confidence} and the sign the class predicted. In the case of the log-loss, the convex surrogate is hugely popular in ML: it is the logistic loss. 
\paragraph{Top-down induction of a DT, the log-loss and hyperbolic distance} the top-down induction of a DT in many popular packages for DT induction boils down to repeatedly minimizing the expected Bayes risk at the leaves of the current DT $H$ \cite{bfosCA,qC4}, a loss we define as
\begin{eqnarray}
  \cbr(H, \mathcal{S}) & \defeq & \expect_{i\sim [m]}[\cbr(H_{[0,1]}(\ve{x}_i))] \quad \left(= \expect_{\leaf \sim \leafset(H)} [\cbr(p^+_{\leaf})]\right), \quad H_{[0,1]}(\ve{x}_i) \defeq p^+_{\leaf_i} \label{ecbr}
\end{eqnarray}
(the index notation in $H$ emphasizes its range). \eqref{ecbr} computes a loss that relies on estimated posteriors at the leaves, but quite remarkably, as already briefed, it can also be formulated using real-valued classification using the canonical link, which gives, for the log-loss, 
  \begin{eqnarray}
    \cbrlog(H, \mathcal{S}) = \expect_i [\log(1+\exp(-y_i \cdot H_{\mathbb{R}}(\ve{x}_i)))], \quad H_{\mathbb{R}}(\ve{x}_i) \defeq \canolog(p^+_{\leaf_i}) = \log \left(\frac{p^+_{\leaf_i}}{1-p^+_{\leaf_i}}\right).\label{ecbrReal}
  \end{eqnarray}
Let us put next to each other the absolute confidence of a leaf prediction and the hyperbolic distance to the origin of some $\ve{z} \in \mathbb{B}_1$, which hopefully clarifies the choice of parameterization made in Table \ref{tab:isolines-t-self}: 
  \begin{eqnarray}
|\canolog(p^+_{\node})| = \log\left(\frac{1+r_{\node}}{1-r_{\node}}\right), \quad r_{\node} \defeq |2p^+_{\node}-1| \quad ; \quad d_{\mathbb{B}_1}(\ve{0}, \ve{z}) = \log\left(\frac{1+\|\ve{z}\|}{1-\|\ve{z}\|}\right).\label{defCANOLOGB1}
  \end{eqnarray}
  
  \paragraph{Objective for hyperbolic geometric embedding of a DT}  At the end of the top-down induction of a DT, \textit{all} nodes have participated at some point during induction in the classification of examples, the root being the first leaf of a DT without arcs. A full mapping of this training history may help to interpret substantial local improvements, relate different such events if they appear in the same path, visualize options for subsequent pruning in their tree context \cite{kmOTj}, etc. . So, a good embedding should map the largest number of nodes in a clean and readable way, representing, for each of them, how good or confident they are / were. Looking at \eqref{defCANOLOGB1}, we see that the best nodes should then be mapped close to the border $\partial \mathbb{B}_1$. Here comes another side advantage of hyperbolic representation: for such nodes, the hyperbolic distance between them gives an approximate estimation of their joint absolute confidence \cite[Figure 1]{sdgrRT}. Let us go straight to the embedding objective, in three parts: 
  \begin{mdframed}[style=MyFrame]
    {\vspace{-0.30cm}\colorbox{gray!20}{\fbox{Embedding objective}}}\\
    \noindent \textbf{(A)} the largest part of the tree in DT $H$ is represented and easily readable for \textit{semantics} about $H$ (tree organization, embedding of Boolean branching tests, etc. ),\\
    \noindent \textbf{(B)} \textit{locally} speaking, each embedded node $\nu$ of $H$ gets mapped at some $\ve{z}_\nu$ such that $r_{\node}$ is close to $\|\ve{z}_\node\|$ \eqref{defCANOLOGB1},\\
    \noindent \textbf{(C)} \textit{globally} speaking, the whole embedding remains convenient and readable to compare, in terms of confidence in classification, different subtrees in the same tree or even between different trees. This includes their leveraging coefficient in an ensemble.
  \end{mdframed}
  Satisfying accurately all three objectives \textbf{(A-C)} mean a very easy way to visually cross classification, confidence in classification and the semantic of the model. But it is not hard to show that this is hard for a \textit{whole} DT. There is, fortunately, a simple workaround which involves embedding the \textit{monotonic} part of the DT classification on the form of a new class of models we now introduce. 

  \paragraph{Monotonic Decision Trees and hyperbolic embeddings} Let us start by explaining why a complete DT embedding cannot even remotely guarantee \textbf{(A-C)} in general: the sequence of absolute confidences in a path from the root to a leaf is in general \textit{not} monotonic. For example, one can have a root mapped to the center of the disk ($p^+_{\mbox{\tiny{root}}} = 0.50$), with a child with $p^+_{\node} \neq 0.50$ and a grandchild with $p^+_{\node'} = 0.50$, creating a double arc $\leftrightarrow$ in the embedding. In the same way, loops can happen for longer paths, making tricky even approximate solutions to \textbf{(A-C)} of good quality. Note that the source of the disorder is the presence of nodes of poor predictive quality in the tree. Our solution is simple in principle:
  \begin{center}
\textit{Embed only the monotonic classification part of a DT}
\end{center}
(meaning for each path from the root to a leaf in $H$, we end up embedding the subsequence of nodes whose absolute confidences are strictly increasing); to do so, we replace the DT by a path-monotonic approximation using a new class of models we introduce, called \textit{Monotonic Decision Trees} (MDT). 

  \begin{definition}\label{defMDT}
A \textit{Monotonic Decision Tree} (MDT) is a rooted tree with a Boolean test labeling each arc and a real valued prediction at each node. Any sequence of nodes in a path starting from the root is strictly monotonically increasing in absolute confidence. At any internal node, no two Boolean tests at its outgoing arcs can be simultaneously satisfied. The classification of an observation is obtained by the bottommost node's prediction reachable from the root.
\end{definition}
We now introduce an algorithm which takes as input a DT $H$ and outputs a MDT $H'$ satisfying the following invariant:
\begin{itemize}
\item [\textbf{(M)}] for any observation $\ve{x}\in \mathcal{X}$, the prediction $H'(\ve{x})$ is equal to the prediction in the path followed by $\ve{x}$ in $H$ of its deepest node in the strictly monotonic subsequence starting from the root of $H$.
\end{itemize}
Figure \ref{fig:dt-vs-mdt} (right) presents an example of MDT $H'$ that would be built for some DT $H$ (left) and satisfying \textbf{(M)} (unless observations have missing values, $H'$ is unique). Figure \ref{fig:dt-vs-mdt} adopts some additional conventions to ease parsing of $H'$:
\begin{itemize}
\item [\textbf{(D1)}] some internal nodes of $H'$ are also tagged with labels corresponding to the leaves of $H$.
If a node in $H'$ is tagged with a label of one of $H$'s leaves $\leaf$, it indicates that examples (and predictions) which reach $\leaf$ in the original $H$ are being `rerouted' to $H'$'s tagged node, where the original prediction at $\leaf$ may change in $H'$
(we emphasize that some are internal nodes of $H'$);
\item [\textbf{(D2)}] arcs in $H'$ have a width proportional to the number of boolean tests it takes to reach its tail from its head in $H$. A large width thus indicates a long path in $H$ to improve classification confidence.
\end{itemize}
It is also worth remarking that $H'$ satisfying \textbf{(M)} has in general a smaller number of vertices than $H$ (but never more), \textit{but} it always has the same depth. Finally, any pruning of $H$ is a subtree of $H$ and its corresponding MDT is a subtree of the MDT of $H$. The aforementioned troubles to embed the DT in the Poincar\'e disk considering \textbf{(A-C)} do not exist anymore for the corresponding MDT because the best embeddings necessarily have all arcs going outwards in the Poincar\'e disk. We now present algorithm \createmdt, in Algorithm \ref{alg-initsampling}. To produce $H'$, after having initialized it to a root = single leaf, we just run
\begin{center}
  \createmdt($\mbox{root}(H),\textbf{true}, \mbox{root}(H'), [\min\{p^+_{\mbox{\tiny{root}}}, 1-p^+_{\mbox{\tiny{root}}}\}, \max\{p^+_{\mbox{\tiny{root}}}, 1-p^+_{\mbox{\tiny{root}}}\}]$)
\end{center}
At the end of the algorithm, the tree rooted at $\mbox{root}(H')$ is the MDT sought.
\begin{algorithm}[H]
\caption{\createmdt($\node,\textbf{b},\node', \mathbb{I}$)}\label{alg-initsampling}
\begin{algorithmic}
  \STATE  \textbf{Input:} Node $\node\in \nodeset(H)$ ($H$ = DT), Boolean test $\textbf{b}$, Node $\node' \in \nodeset(H')$ ($H'$ = MDT being build from $H$), interval of forbidden posteriors $\mathbb{I} \subset [0,1]$;
  \STATE  1 : \textbf{if} $\node \in \leafset(H)$ \textbf{then}
  \STATE  2 :\hspace{0.5cm} \textbf{if} $p^+_\node \in \mathbb{I}$ \textbf{then}
  \STATE  3 :\hspace{1.0cm} $\node' \leftarrow \tagdtleaf(\node)$; \hfill // tags node $\node' \in \nodeset(H')$ with info from leaf $\node \in \leafset(H)$
  \STATE  4 :\hspace{0.5cm} \textbf{else}
    \STATE  5 :\hspace{1.0cm} $\node'' \leftarrow H'.\createnode(\node)$; \hfill // $\node''$ will be a new leaf in MDT $H'$
  \STATE  6 :\hspace{1.0cm} $H'.\addarc(\node', \textbf{b}, \node'')$; \hfill // adds arc $\node' \rightarrow_{\textbf{b}} \node''$ in $H'$ 
    \STATE  7 :\hspace{0.5cm} \textbf{endif} 
  \STATE  8 : \textbf{else}
  \STATE  9 :\hspace{0.5cm} \textbf{if} $p^+_{\node} \not\in \mathbb{I}$ \textbf{then}
    \STATE  10 :\hspace{1.0cm} $\node'' \leftarrow H'.\createnode(\node)$; \hfill // $\node''$ will be a new internal node in MDT $H'$
    \STATE  11 :\hspace{1.0cm} $H'.\addarc(\node', \textbf{b}, \node'')$; \hfill // adds arc $\node' \rightarrow_{\textbf{b}} \node''$ in $H'$ 
    \STATE  12 :\hspace{1.0cm} $\mathbb{I}_{\mbox{\tiny{new}}} \leftarrow [\min\{p^+_{\node}, 1-p^+_{\node}\}, \max\{p^+_{\node}, 1-p^+_{\node}\}]$; \hfill // $\mathbb{I} \subset \mathbb{I}_{\mbox{\tiny{new}}}$
    \STATE  13 :\hspace{1.0cm} $\node'_{\mbox{\tiny{new}}} \leftarrow \node''$;
    \STATE  14 :\hspace{1.0cm} $\textbf{b}_{\mbox{\tiny{new}}} \leftarrow \textbf{true}$;
    \STATE  15 :\hspace{0.5cm} \textbf{else}
    \STATE  16 :\hspace{1.0cm} $\mathbb{I}_{\mbox{\tiny{new}}} \leftarrow \mathbb{I}$; $\node'_{\mbox{\tiny{new}}} \leftarrow \node'$; $\textbf{b}_{\mbox{\tiny{new}}} \leftarrow \textbf{b}$; \hfill // node $\node$ does not produce any change in $H'$
    \STATE  17 :\hspace{0.5cm} \textbf{endif} 
    \STATE  18 :\hspace{0.5cm} \createmdt($\node.\mbox{leftchild},\textbf{b}_{\mbox{\tiny{new}}}\wedge \node.\textsc{test}(\mbox{leftchild}),\node'_{\mbox{\tiny{new}}}, \mathbb{I}_{\mbox{\tiny{new}}}$);
    \STATE  19 :\hspace{0.5cm} \createmdt($\node.\mbox{rightchild},\textbf{b}_{\mbox{\tiny{new}}}\wedge \node.\textsc{test}(\mbox{rightchild}),\node'_{\mbox{\tiny{new}}}, \mathbb{I}_{\mbox{\tiny{new}}}$);
  \STATE  20 : \textbf{endif}
\end{algorithmic}
\end{algorithm}
We complete the description of \createmdt: When a leaf of $H$ does not have sufficient confidence and ends up being mapped to an internal node of the MDT $H'$, \tagdtleaf~is the procedure that tags this internal node with information from the leaf (see \textbf{(D1)} above). We consider a tag being just the name of the leaf (Figure \ref{fig:dt-vs-mdt}, leaves $\#5, 9, 11$), but other conventions can be adopted. The other two methods we use grow $H'$ by creating a new node via \createnode~(and passing the information of the corresponding node in the DT it "copies") and adding a new arc between existing nodes via \addarc.

\begin{theorem}
When run with DT $H$, the MDT $H'$ built by \createmdt~satisfies \textbf{(M)}.
\end{theorem}
\begin{proofsketch}
Take any observation $\ve{x} \in \mathcal{X}$. Trivially, $H'(\ve{x})$ is a real value that belongs to the path followed by $\ve{x}$ in $H$. Furthermore, if we consider any path from the root to a leaf in $H$, \textit{all} the vertices in its strictly monotonically increasing subsequence of absolute confidences appear in $H'$ (conditions 4, 8 in \createmdt), which guarantees the condition on prediction in \textbf{(M)}. Hence \textbf{(M)} is satisfied.
  \end{proofsketch}
  
\begin{figure}
  \centering
  \resizebox{\textwidth}{!}{\begin{tabular}{c?c}\Xhline{2pt}
                                 \includegraphics[trim=55bp 50bp 470bp 115bp,clip,height=0.4\columnwidth]{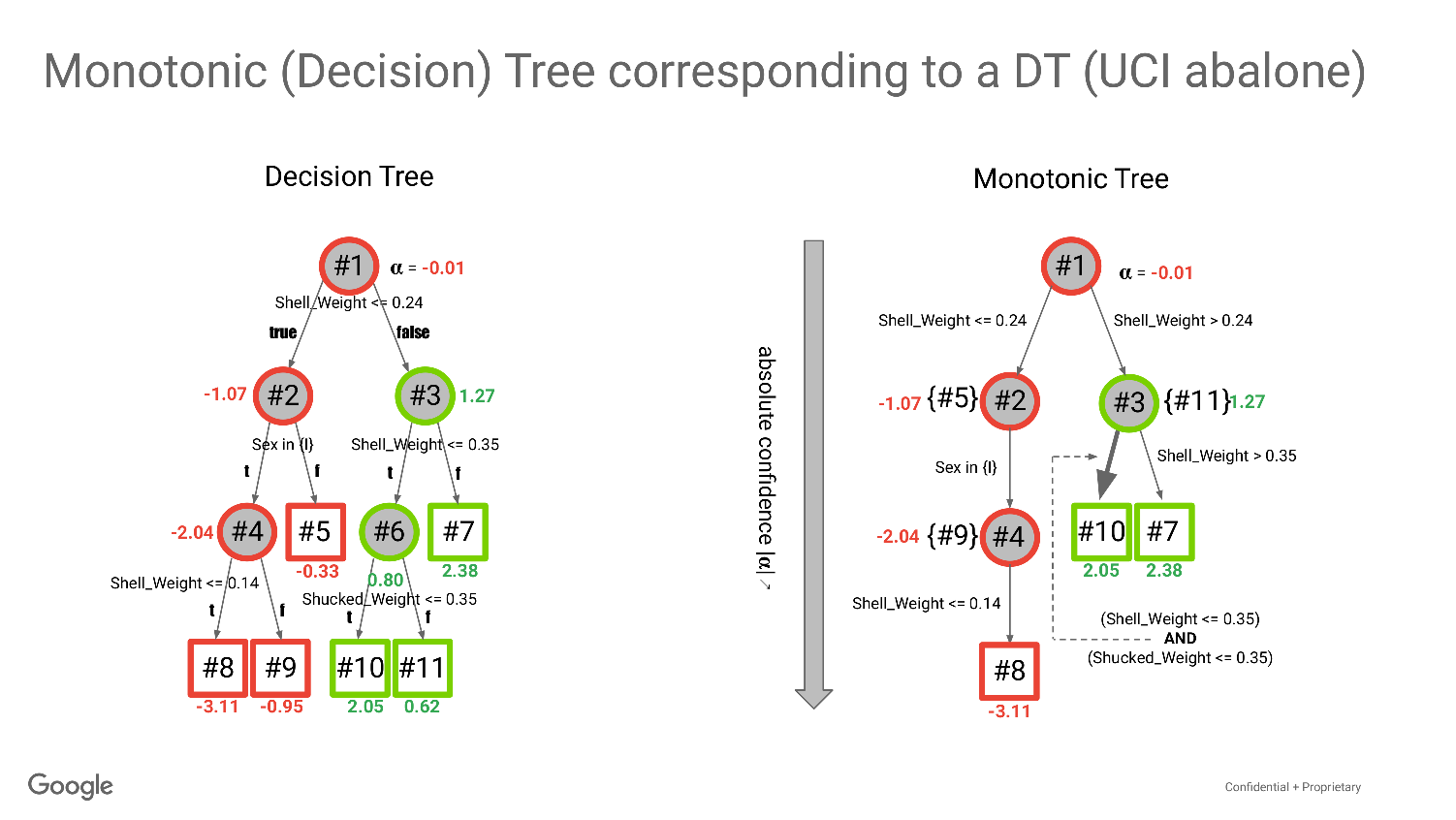} & \includegraphics[trim=370bp 50bp 70bp 115bp,clip,height=0.4\columnwidth]{Figs/slides-shared/dt-vs-mdt-example-uci-abalone.eps}\\
                                 Decision Tree $H$ & Monotonic Decision Tree $H'$ \\ \Xhline{2pt}
                            \end{tabular}}
                          \caption{A small decision tree (DT) learned on UCI \domainname{abalone} (left) and its corresponding monotonic decision tree (MDT, right) learned using \createmdt. Colors ({\color{red} red}, {\color{darkgreen} green}) denote the majority class in each node. In each node, the real-valued prediction \eqref{ecbrReal} is indicated, also in color. Observe that, indeed, $H$ does not grant path-monotonic classification but $H'$ does (Definition \ref{defMDT}). Observe also that in $H'$, some nodes have outdegree 1; also, internal node $\#$6 in the DT, whose prediction is worse than its parent, disappears in $H'$. One arc in $H'$ is represented in double width because its Boolean test aggregates both tests it takes to go from $\#$3 to $\#$10 in $H$. Finally, the depths of $H$ and $H'$ are the same.}
    \label{fig:dt-vs-mdt}
  \end{figure}

  \paragraph{Embedding a sequence of trees and leveraging coefficients: boosting} The interesting link between distances in the Poincar\'e model of hyperbolic geometry and tree prediction in the CPE loss theory \eqref{defCANOLOGB1} can be extended to linear combinations of tree predictions learned using one of ML's most celebrated optimization setting: boosting \cite{kTO,ssIBj}. What we mean is there exists an efficient boosting algorithm for the logistic loss which produces leveraging parameters that \textit{also} comply with the convenient form of DT predictions (with respect to the Poincar\'e hyperbolic distance) making them easily and reliably representable on the disk, in addition to the embedding of the DT they leverage. To get there, we need to rely on the optimization of the logistic loss \eqref{ecbrReal} \textit{and} we need to design a tailored boosting algorithm which parallels all tricks for the derivation and analysis of AdaBoost's confidence rated prediction boosting in \cite{ssIBj} -- note that this is different from LogitBoost \cite{fhtAL}, which applies Friedman's functional gradient boosting prescription \cite{fGFA}.

  To preserve readability, we fully derive and describe the boosting algorithm, called \logisticboost, in the \supplement~(Section \ref{algo_logisticboost}). Assuming basic knowledge of boosting algorithms, we just give here the main ingredients. First, the \textit{unnormalized} weight update after adding DT $H_j$ to the combination is
  \begin{eqnarray}
w_{(j+1)i} & \leftarrow & \frac{w_{ji}}{w_{ji} + (1-w_{ji}) \cdot \exp(\alpha_j y_i H_j(\ve{x}_i))} \label{defweightupdateMAIN},
  \end{eqnarray}
  all weights being initialized at value $1/2$ at the beginning. $\alpha_j$ is the leveraging coefficient learned during boosting for ``weak" classifier $H_j$. In the case of the log-/logistic-losses, the final model $\mbox{\hcomb}_T$ (linear combination of DTs) produces a classification which depends only on $\canolog$ and thus can be easily embedded in Poincar\'e disk \eqref{defCANOLOGB1}:
\begin{eqnarray}
\mbox{\hcomb}_T(\ve{x}) & \defeq & \sum_{j=1}^{T} \frac{({\canolog})_j }{({\canolog})^*_j} \cdot H_j(\ve{x}) \quad = \sum_{j=1}^{T} \frac{({\canolog})_j }{({\canolog})^*_j} \cdot \canolog\left(p^+_{j\leaf(\ve{x})}\right), \label{defBOOSTHMAIN}
\end{eqnarray}
where $({\canolog})^*_j$ is the maximal absolute confidence of $H_j$:
\begin{eqnarray}
({\canolog})^*_j & \defeq & \max_{\leaf \in \leafset(H_j)} \left|\log\left(\frac{p^+_{j\leaf}}{1-p^+_{j\leaf}}\right)\right|\label{defCANOLOGTSTAR}
\end{eqnarray}
(index `$j$' in $p^+_.$ emphasize the use of normalized boosting weights to compute all posteriors) and $p^+_{j\leaf(\ve{x})}$ is the posterior estimation at leaf reached by observation $\ve{x}$. Notice that $({\canolog})^*_j$ can be read directly from the Poincar\'e embedding, since it is not hard to show that the maximal absolute confidence for the MDT is also that of the DT. The second part in leveraging in \eqref{defBOOSTHMAIN} is
\begin{eqnarray}
({\canolog})_j & \defeq & \log\left(\frac{p^+_{j}}{1-p^+_{j}}\right) , \quad p^+_j \defeq \frac{1+r_j}{2} , \quad r_j \defeq \expect_{i\sim \tilde{w}_j} \left[\frac{1}{({\canolog})^*_j} \cdot \log\left(\frac{p^+_{j\leaf(\ve{x}_i)}}{1-p^+_{j\leaf(\ve{x}_i)}}\right)\right],\label{defCANOLOGT}
\end{eqnarray}
where $\tilde{\ve{w}}_j$ denotes weights \eqref{defweightupdateMAIN} normalized. It is not hard to show that $p^+_{j} \geq 1/2$ and so $({\canolog})_j \geq 0$, so it can easily be displayed and read from Poincar\'e disk as the confidence of an imaginary vertex with posterior $p^+_j$ (we choose to represent it with a circle, see Figure \ref{fig:hyperbolicsum}).

At this stage, we have (i) detailed the relationships between training a classifier with the log-/logistic-loss and Poincar\'e model of hyperbolic geometry, (ii) explained how to design an embedding of the monotonic part of a DT classification on the form of a MDT tackling objectives \textbf{(A-C)}, (iii) briefed training a boosted ensemble of DTs using the logistic loss and with leveraging coefficients easily embeddable in Poincar\'e disk. There remains to present how we embed the MDT corresponding to a DT. This last part is the simplest.

\paragraph{A modified Sarkar algorithm to embed trees and their leveraging information} Sarkar's algorithm \cite{sLDD} gives a clean low-distortion embedding when the tree is binary or the arc length is constant \cite{sdgrRT}. Things are different in our case: MDT nodes have arbitrary out-degree and lengths depend on the absolute confidence at the corresponding MDT nodes. Plus, a direct implementation of Sarkar's algorithm would violate the constraint for strict path-monotonicity in a MDT that nodes in a path from the root need to progressively come closer to the border -- equivalently, it would create substantial embedding errors for confidences and violate \textbf{(B)}. Without focusing on an optimal solution (that we leave for future work), one can remark that all these problems can be heuristically addressed by changing one step of Sarkar's algorithm, replacing the use of the total $2\pi$ fan out angle for mapping children (Step 5: in Algorithm 1 of \cite{sdgrRT}) by a variable angle with a special orientation in the disk.   We defer to \supplement~for a more detailed exposure of the tweak (Section \ref{sarkar-mod}). We evaluate the quality of the MDT embedding using an expected error for the embedding of each node:
\begin{eqnarray}
\rho(H') & \defeq & \expect_{\node \sim \nodeset(H')} \left[\left|\frac{\left|\alpha_{\node}\right| - d_{\mathbb{B}_1}(\ve{0}, \ve{z}_{\node})}{\left|\alpha_{\node}\right|}\right|\right], \label{errorEmbed}
\end{eqnarray}
where $\alpha_{\node}$ refers to the relevant $\canolog\left(p^+_{j\leaf}\right)$ in \eqref{defBOOSTHMAIN}, \textit{i.e.} the confidence computed at the iteration $t$ when $\node$ was a leaf, $\leaf$, in $H$ (and $\ve{z}_{\node} \in \mathbb{B}_1$ is its embedding in Poincar\'e disk). Hence, the error covers all nodes of MDT $H'$. The final representation of a MDT is shown on an example in Figure \ref{fig:hyperbolicsum}, learned on an UCI domain\footnote{\url{https://archive.ics.uci.edu/dataset/468/online+shoppers+purchasing+intention+dataset}.}. Notice the small embedding error $\rho$.

  \begin{figure}
  \centering
\includegraphics[trim=20bp 25bp 160bp 60bp,clip,width=\columnwidth]{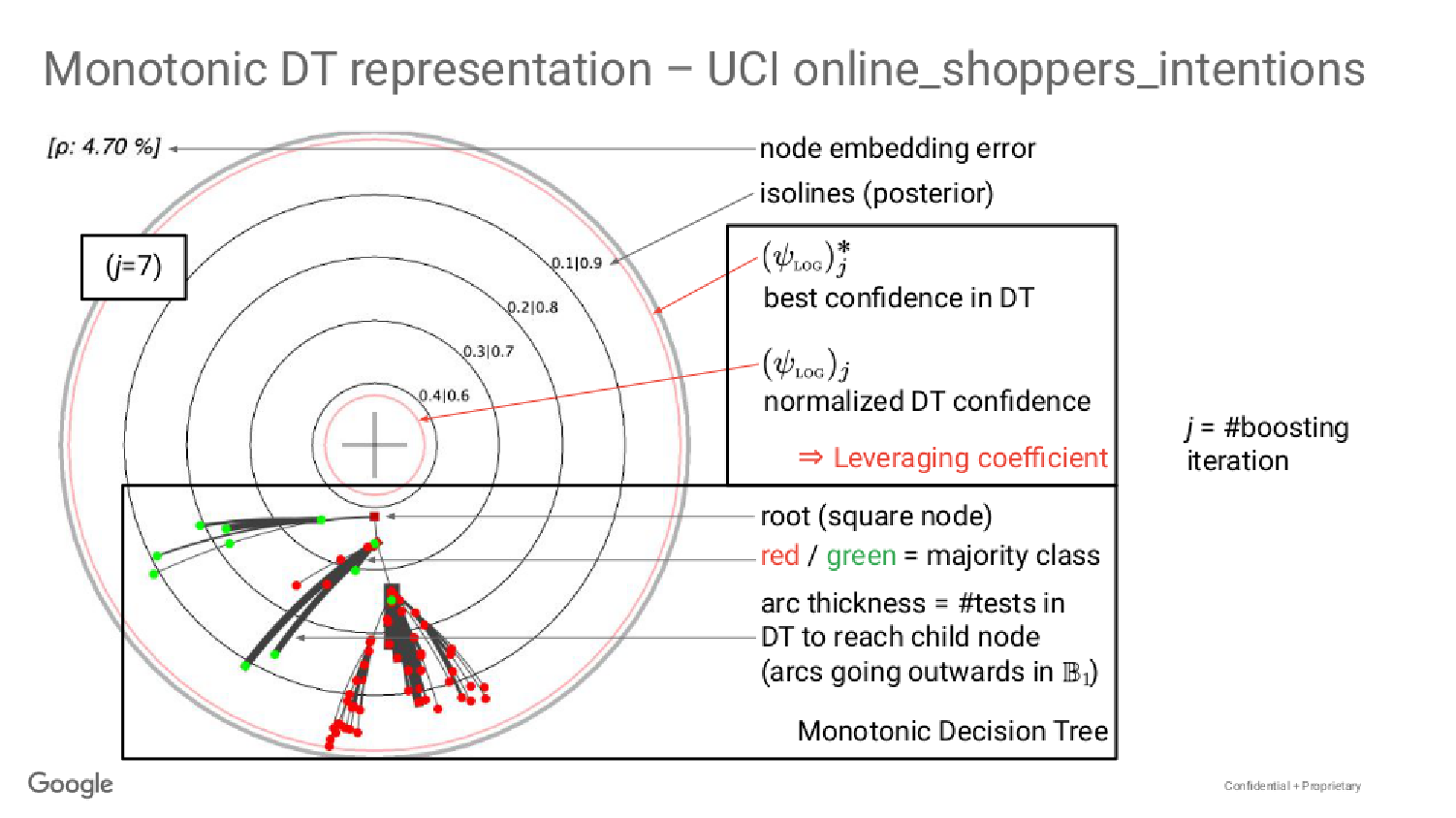} 
    \caption{Summary of the representation adopted for the $j=7^{th}$ DT in a boosted ensemble, using its MDT embedded in Poincar\'e disk $\mathbb{B}_1$ (UCI domain \domainname{online$\_$shoppers$\_$intention}, see Section \ref{sec:t-exp} for the complete experiment). See \eqref{errorEmbed} for the definition of $\rho$, \eqref{defCANOLOGTSTAR} for the defintion of $({\canolog})^*_j$, \eqref{defCANOLOGT} for the defintion of $({\canolog})_j$.}
    \label{fig:hyperbolicsum}
  \end{figure}

  \paragraph{Using the t-self} We implement the simple non-linear scaling mapping $\embedt_t$ \eqref{defembedpt} satisfying \eqref{eqIDDDT} (otherwise detailed at the end of Section \ref{sec:t-hyp}). In short, we scale the norm of all relevant parameters in $\mathbb{B}_1$ (nodes' embeddings, boosting's leveraging coefficients, isolines) via \eqref{eqIDDDT} to preserve the hyperbolic distance. For example, any node $\node$ of a MDT $H'$, which is embedded using $\ve{z}_\node \in \mathbb{B}_1$, becomes in the t-self $\ve{z}^{(t)}_\node \in \mathbb{B}^{(t)}_1$ by changing the norm to satisfy \eqref{eqIDDDT}, given that in the t-self we have
  \begin{eqnarray}
d^{(t)}_{\mathbb{B}_1}(\ve{0}, \ve{z}^{(t)}_\node) = \log_t \exp d_{\mathbb{B}_1}(\ve{0}, \ve{z}^{(t)}_\node) = \log_t \left(\frac{1+\|\ve{z}^{(t)}_\node\|}{1-\|\ve{z}^{(t)}_\node\|}\right) \label{defCANOLOGBT1}.
  \end{eqnarray}
  All parameters depending on $\canolog$ in \eqref{defBOOSTHMAIN} get the same non-linear mapping. Notice the interest of the t-self in the context of numerical accuracy and readability, since for example $({\canolog})^*_t$ \eqref{defCANOLOGTSTAR} is computed from the node that is the closest from the border of the disk. 
\section{Experiments}
\label{sec:t-exp}

We detail experiments carried out, with the complete set of experiments and additional ones in \supplement, Section \ref{sec-sup-exp}. In the top-down induction scheme for DT, the split is the one that minimizes $\cbrlog(., \mathcal{S})$ \eqref{ecbr} among all choices. The leaf chosen to be split is the heaviest non pure leaf, \textit{i.e.} the one among the leaves containing both classes with the largest proportion of training examples, as measured by the initial distribution (first tree) or the boosting updated distribution (next trees). Induction is stopped when the tree reaches a fixed size, or when all leaves are pure, or when no further split allows to decrease $\cbrlog(., \mathcal{S})$. We do not prune DTs. The computation of the corresponding MDTs, the embeddings in Poincar\'e disk are done as described in Section \ref{sec:t-dt}. The domains we consider are all public domains, from the UCI repository of ML datasets \cite{dgUM}, OpenML or Kaggle (details in \supplement). All test errors are estimated from a 10-folds stratified cross-validation.

\paragraph{Visualizing a DT \textit{via} its MDT} One can always choose to directly learn Monotonic DTs instead of DTs, given their natural fit for embedding in the Poincar\'e disk. In this case, the hyperbolic representation of the MDT can be immediately used for assessment. 
Suppose we stick to learning DTs (because \textit{e.g.} they have been standard in ML for decades) \textit{and} wish to use the visualization of its corresponding MDT to make inferences on the original DT itself. 
Can we make reliable conclusions? 
We explore this question by observing that the prediction from DT to MDT can only change if the leaf $\leaf$ that an example reaches in the DT satisfies two conditions:
(a) $\leaf$ does not appear as a leaf in the MDT and (b) it is tagged to an internal node of the MDT with a confidence of the opposite sign (this is \textbf{(D1)}, Step 3: in \createmdt).
Without tackling the formal aspect of this question, we have designed a simple experiment for a simple assessment of whether / when this is reasonable. We have trained a set of $T$=200 boosted DTs, each with maximal size 201 (total number of nodes). After having computed the corresponding MDTs, we have compared the test errors of the boosted set of trees and that of the ensemble where each DT is replaced by its MDT, \textit{but} the leveraging coefficients do not change. Intuitively, if test errors are on par, the variation in classification of DTs vs MDTs (including confidences) is negligible and we can ``reduce" the interpretation of the DT to that of its MDT in the Poincar\'e disk. The results are summarized in Table \ref{tab:dt-vs-mdt-XAI}.
\begin{table}
  \centering
\begin{tabular}{crrc}\Xhline{2pt}
  domain & \multicolumn{1}{c}{DT} & \multicolumn{1}{c}{MDT} & p-val\\ \hline
  \domainname{breastwisc} & $4.15\pm 2.18$ & $5.00\pm 2.38$ & \textbf{0.2408}\\
  \domainname{ionosphere} & $5.71\pm 2.70$ & $9.11\pm 5.34$ & 0.0237\\
  \domainname{tictactoe} & $2.61\pm 1.85$ & $2.09\pm 1.10$ & \textbf{0.1390}\\
  \domainname{winered} & $18.39 \pm 2.02$ & $18.32 \pm 2.27$ & \textbf{0.9227}\\
  \domainname{german} & $24.00\pm 4.37$ & $23.90\pm 4.25$ & \textbf{0.8793}\\
  \domainname{analcatdata$\_$supreme} & $23.40 \pm 1.85$ & $22.56 \pm 1.75$ & 0.0134\\
  \domainname{abalone} & $22.15 \pm 2.20$ & $21.14 \pm 2.44$ & \textbf{0.1267}\\
  \domainname{qsar} & $12.98 \pm 2.63$ & $12.89 \pm 4.05$ & \textbf{0.8772}\\
  \domainname{hillnoise} & $37.04 \pm 2.73$ & $45.88 \pm 4.72$ & 0.0008\\
  \domainname{firmteacher} & $6.87 \pm 1.23$ & $7.04 \pm 1.17$ & \textbf{0.4292}\\
  \domainname{online$\_$shoppers$\_$intention} & $9.90 \pm 0.78$ & $10.67 \pm 0.54$ & 0.0057\\
  \domainname{hardware} & $1.33 \pm 0.27$ & $1.44 \pm 0.19$ & \textbf{0.0534}\\\Xhline{2pt} 
  \end{tabular}
  \caption{Results of the experiments checking whether "reducing" the interpretation of a DT to that of its MDT (using its hyperbolic embedding) can give accurate information about the tree as well. Numerical columns, from left to right, give the average$\pm$std dev error for DTs, for MDTs and provide the p-value of a Student paired t-test with H0 being the identity of the average errors. Entries in \textbf{bold faces} correspond to \textit{keeping H0} for a 0.05 first-order risk. See text for details.}
    \label{tab:dt-vs-mdt-XAI}
  \end{table}
   From Table \ref{tab:dt-vs-mdt-XAI}, we can safely say that our hypothesis reasonably holds in many cases, with two important domains for which it does not: \domainname{ionosphere} and \domainname{hillnoise}. For the former domain, we attribute it to the small size of the domain, which prevents training big enough trees; for the latter domain, we attribute it to the fact that the domain contains substantial noise, which makes it difficult to substantially improve posteriors by splitting and thus make many DT nodes, including leaves, ``disappear" in the MDT conversion (Steps 2: and 14: in \createmdt).

  \paragraph{Poincar\'e embeddings of MDTs} We first summarize in Table \ref{tab:summary-dt-exerpt} a number of experiments on the domains of Table \ref{tab:dt-vs-mdt-XAI} (conventions and notations in Section \ref{sec:t-dt}; the training setting is same as Table \ref{tab:dt-vs-mdt-XAI}). The complete set of plots is in \supplement, Section \ref{sec-exp-pdemb}. On each of these plots, the ratio between the hyperbolic radii of the smallest over the largest red circle gives boosting's leveraging coefficient of the DTs. Node colors are labels and the thicker an arc $\node \rightarrow_{\mbox{\tiny MDT}} \node'$ appears, the more ``digging" we need in the DT to go from node $\node$ down to another ($\node'$) with larger absolute confidence (see also Figure \ref{fig:hyperbolicsum} for a summary). Hyperbolic distances from the origin directly approximate a node's absolute confidence (error $\rho$ (\ref{errorEmbed}, indicated), so the closer a node is from the border, the better it is for classification. The center of the disk represents ``poor" classification (confidence 0), or, alternatively in the context of boosting, random classification.

  All experiments display (expectable) common patterns, seen in Table \ref{tab:summary-dt-exerpt}: as boosting iterations increase, low-depth tree nodes tend to converge to the center of the disk, indicating increased hardness in classification. This is also an effect of boosting, which increases the weight of ``hard" examples. Highly imbalanced domains get a root predictably initially embedded ``far" from the origin (\domainname{online$\_$shoppers$\_$intention}, \domainname{analcatdata$\_$supreme}) while more balanced domains get their roots embedded near the origin (\domainname{abalone}).

  Stark differences also emerge between domains: Table \ref{tab:dt-vs-mdt-XAI} displays that \domainname{online$\_$shoppers$\_$intention} yields small test error. The plots display that for many iterations the bottom-most nodes indeed get embeddings close to the border, \textit{and} there is an early good split, at the root, which makes a good partition between positive and negative examples (one of the edge is the longest in MDT $\#1$ in Table \ref{tab:summary-dt-exerpt}). This property still survives to some extent in the 10$^{th}$ tree. For \domainname{abalone}, we keep a reasonably good split at the root for the first tree, but very soon this nice property disappears: many nodes are aggregated near the origin in the 10$^{th}$ tree, with subtrees almost equally balanced between classes. The case of \domainname{analcatdata$\_$supreme} is worse: very soon during boosting iterations \textit{all nodes} get dragged close to the origin, and many nodes keep negative classification, which displays the hardness of classifying the positive ({\color{darkgreen} green}) class. Those visible embedding differences translate in differences in performances (Table \ref{tab:dt-vs-mdt-XAI}) that align quite well with the ``observed" hardnesses.

\begin{table}
  \centering
  \resizebox{\textwidth}{!}{\begin{tabular}{c|c?c?c?c}\Xhline{2pt}
                              & {\scriptsize \domainname{online$\_$shoppers$\_$intention}} & {\scriptsize \domainname{analcatdata$\_$supreme}} & {\scriptsize \domainname{abalone}}& {\scriptsize \domainname{hardware}}\\ \hline
                              \rotatebox[origin=l]{90}{MDT $\#$ 1} & \includegraphics[trim=0bp 0bp 0bp 0bp,clip,width=0.3\columnwidth]{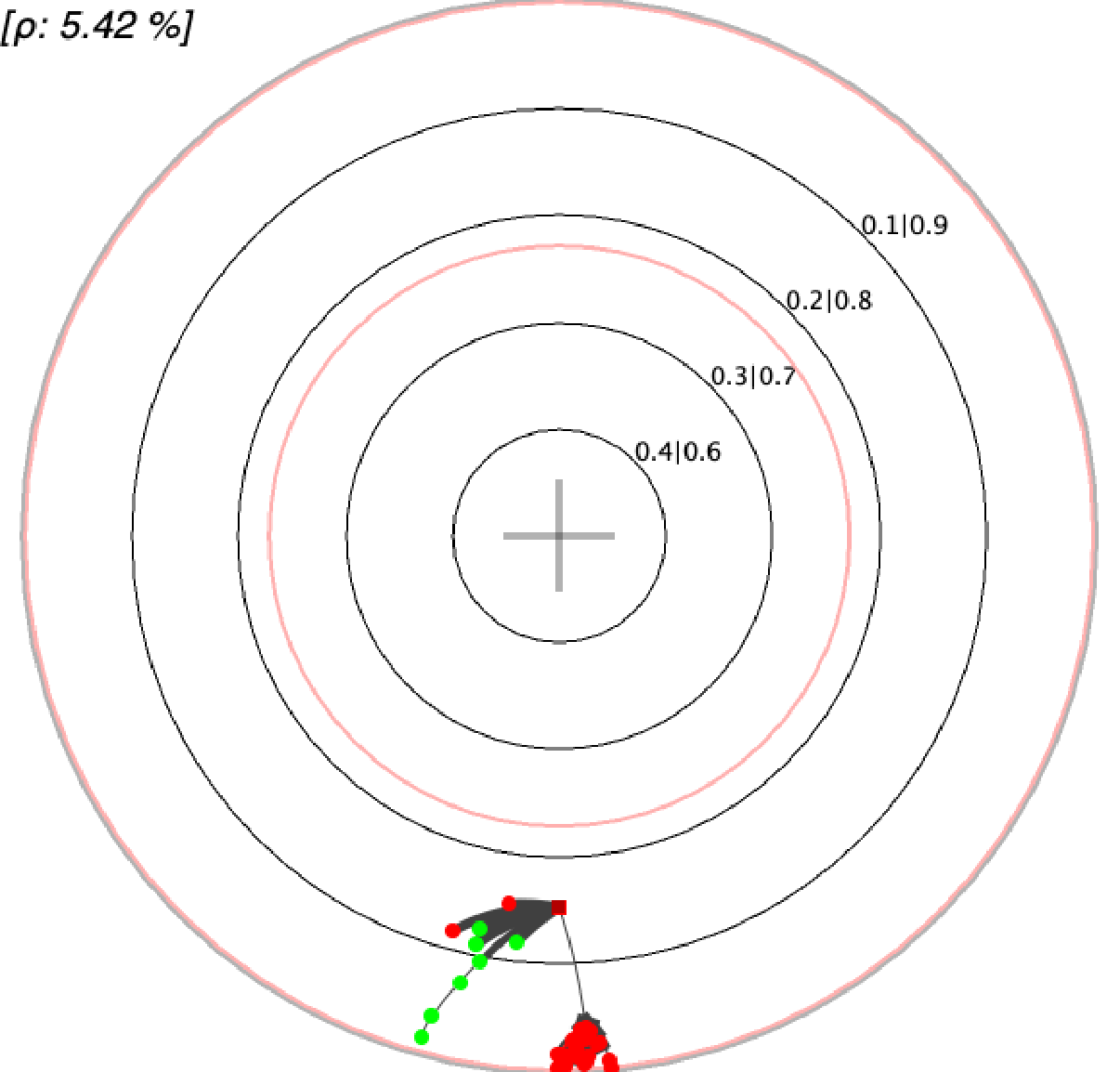} & \includegraphics[trim=0bp 0bp 0bp 0bp,clip,width=0.3\columnwidth]{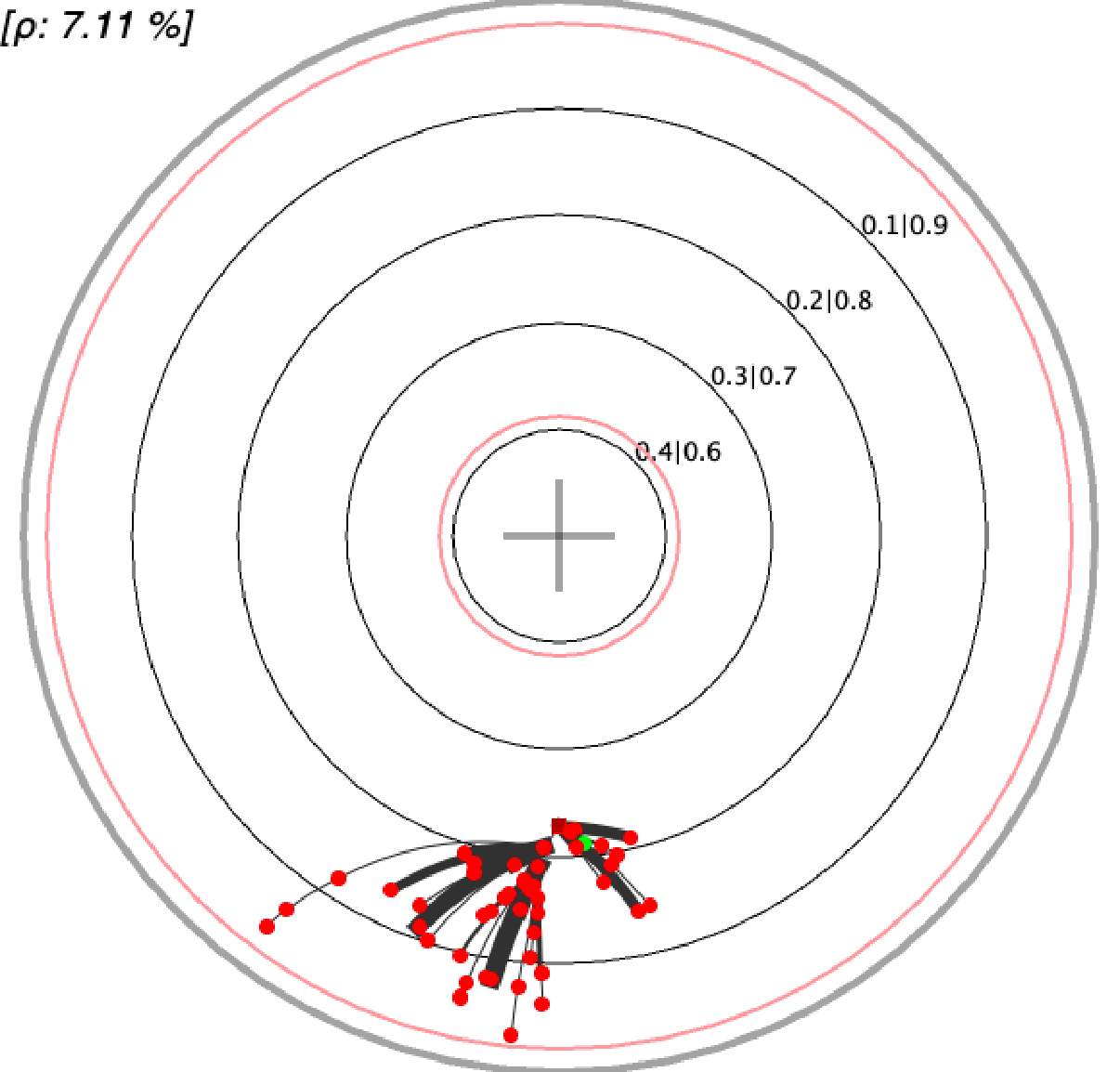} & \includegraphics[trim=0bp 0bp 0bp 0bp,clip,width=0.3\columnwidth]{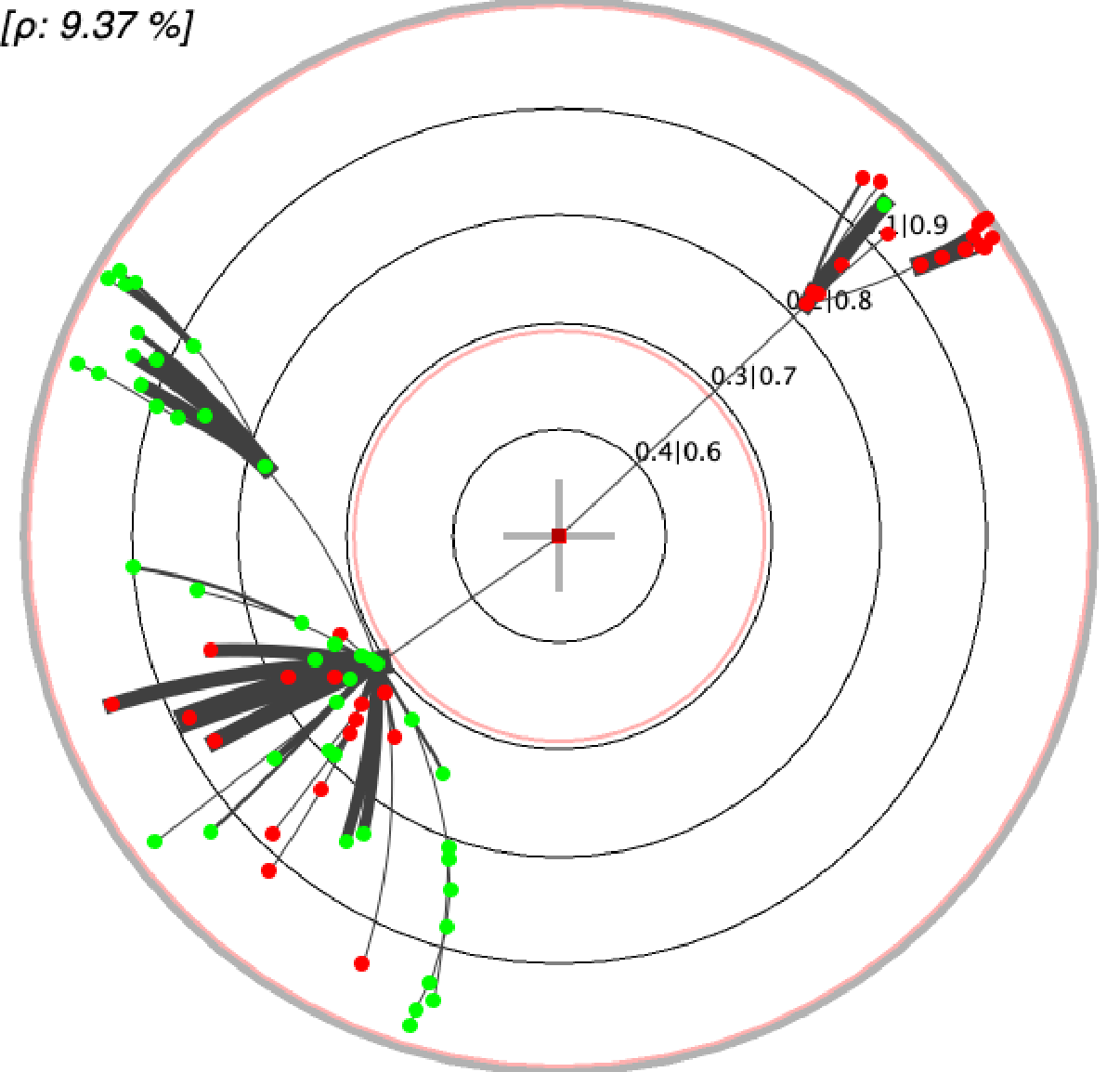} & \includegraphics[trim=0bp 0bp 0bp 0bp,clip,width=0.3\columnwidth]{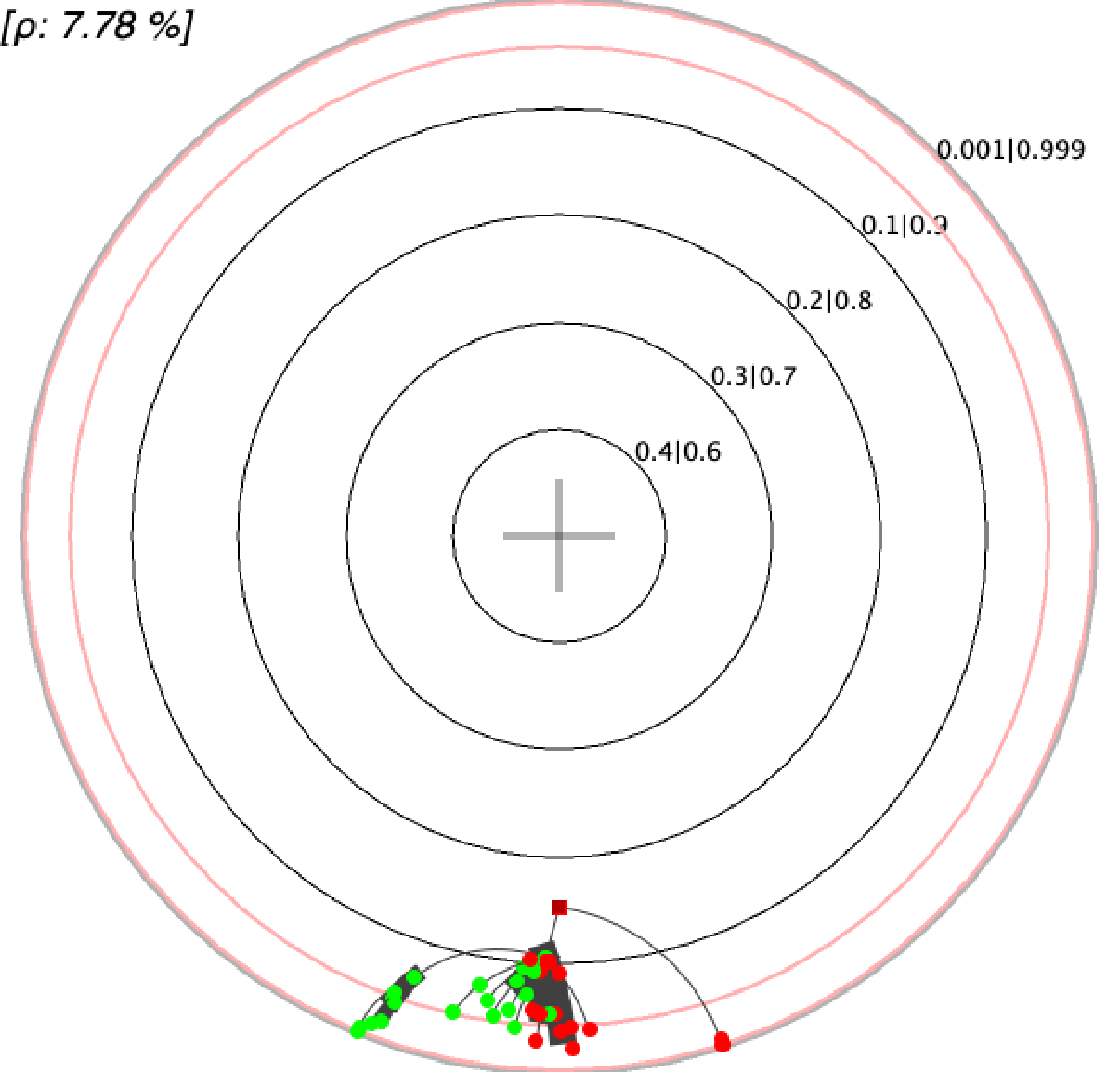} \\ \hline
                             \rotatebox[origin=l]{90}{MDT $\#$ 10} & \includegraphics[trim=0bp 0bp 0bp 0bp,clip,width=0.3\columnwidth]{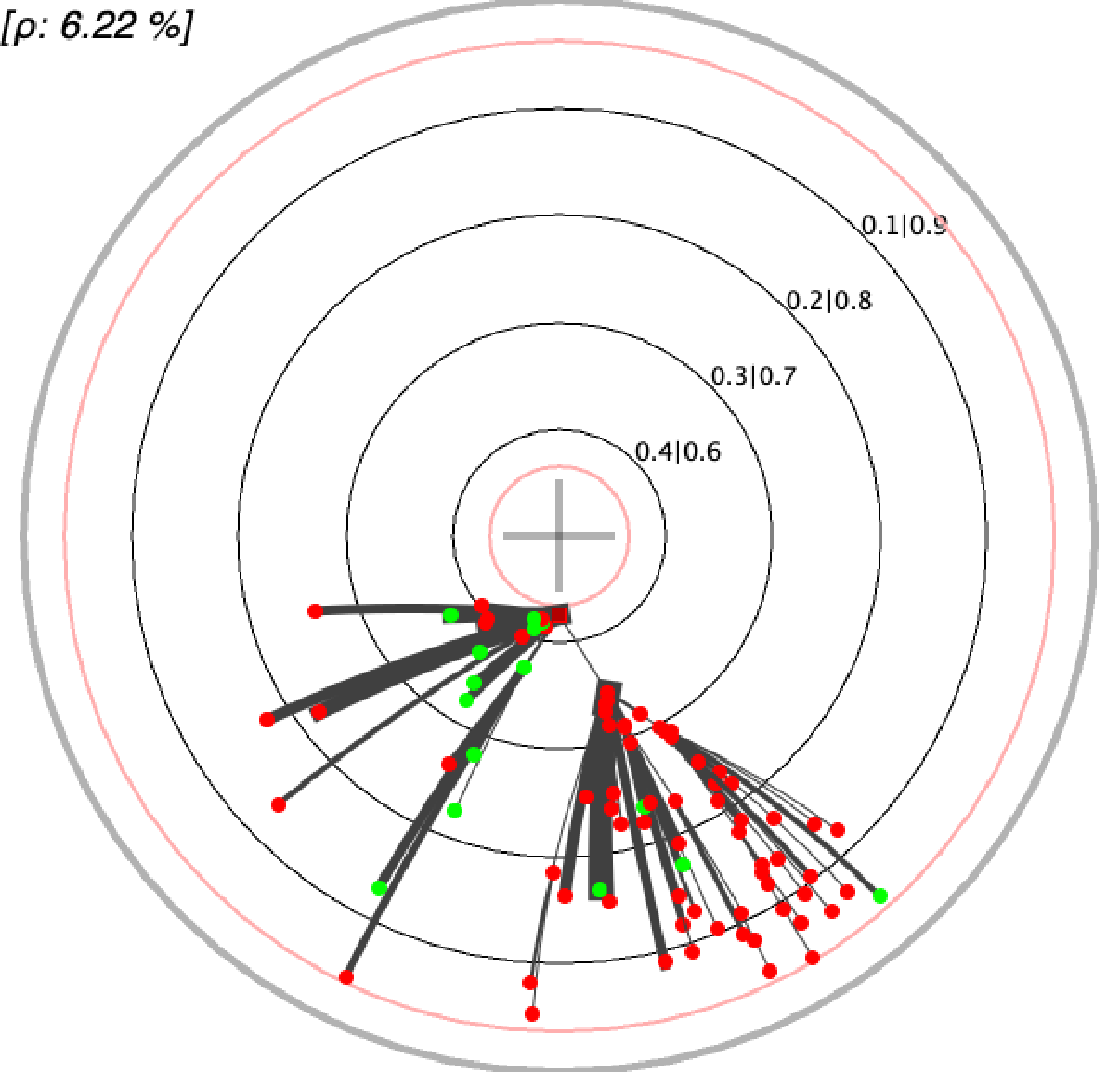} & \includegraphics[trim=0bp 0bp 0bp 0bp,clip,width=0.3\columnwidth]{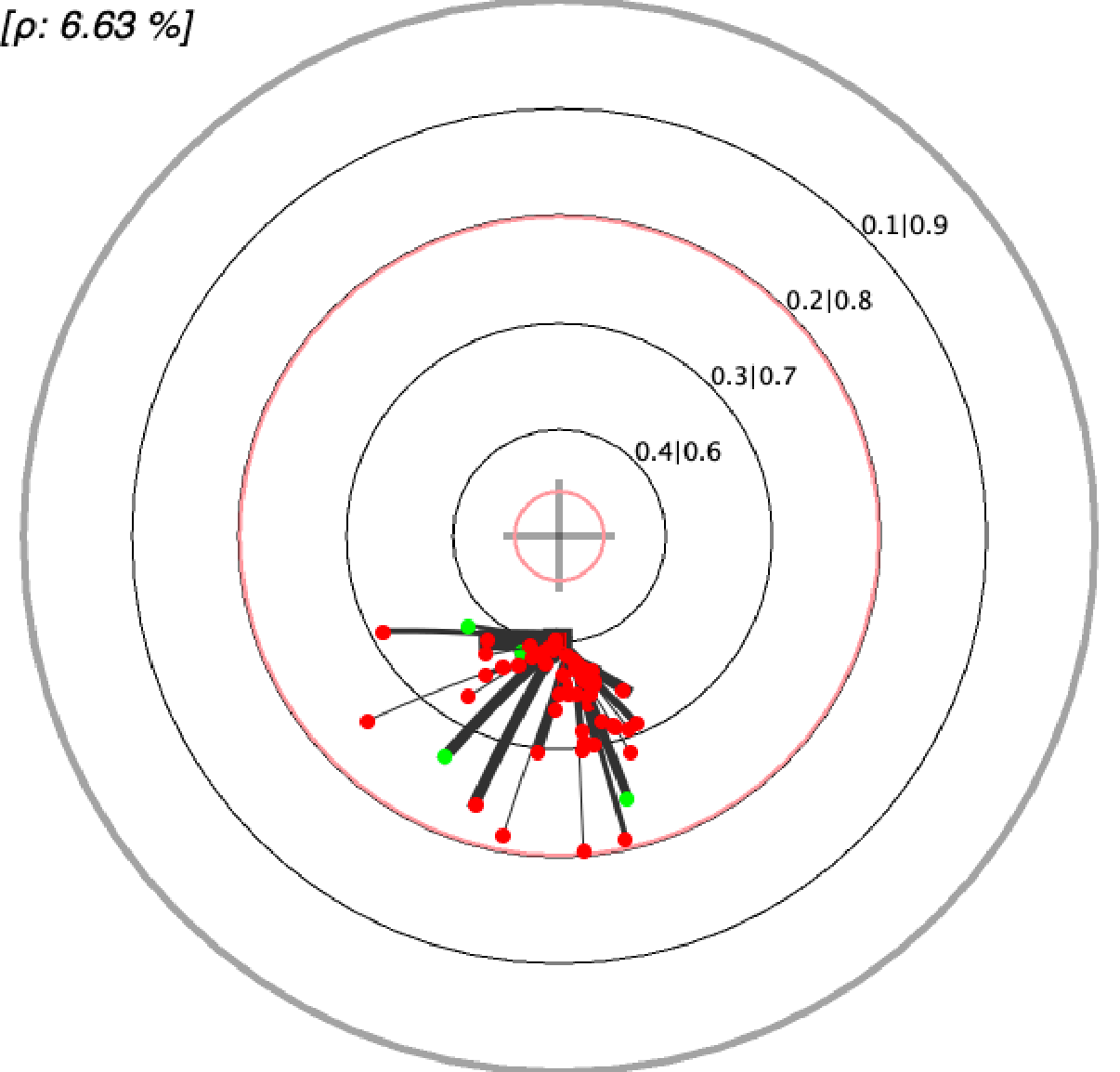} & \includegraphics[trim=0bp 0bp 0bp 0bp,clip,width=0.3\columnwidth]{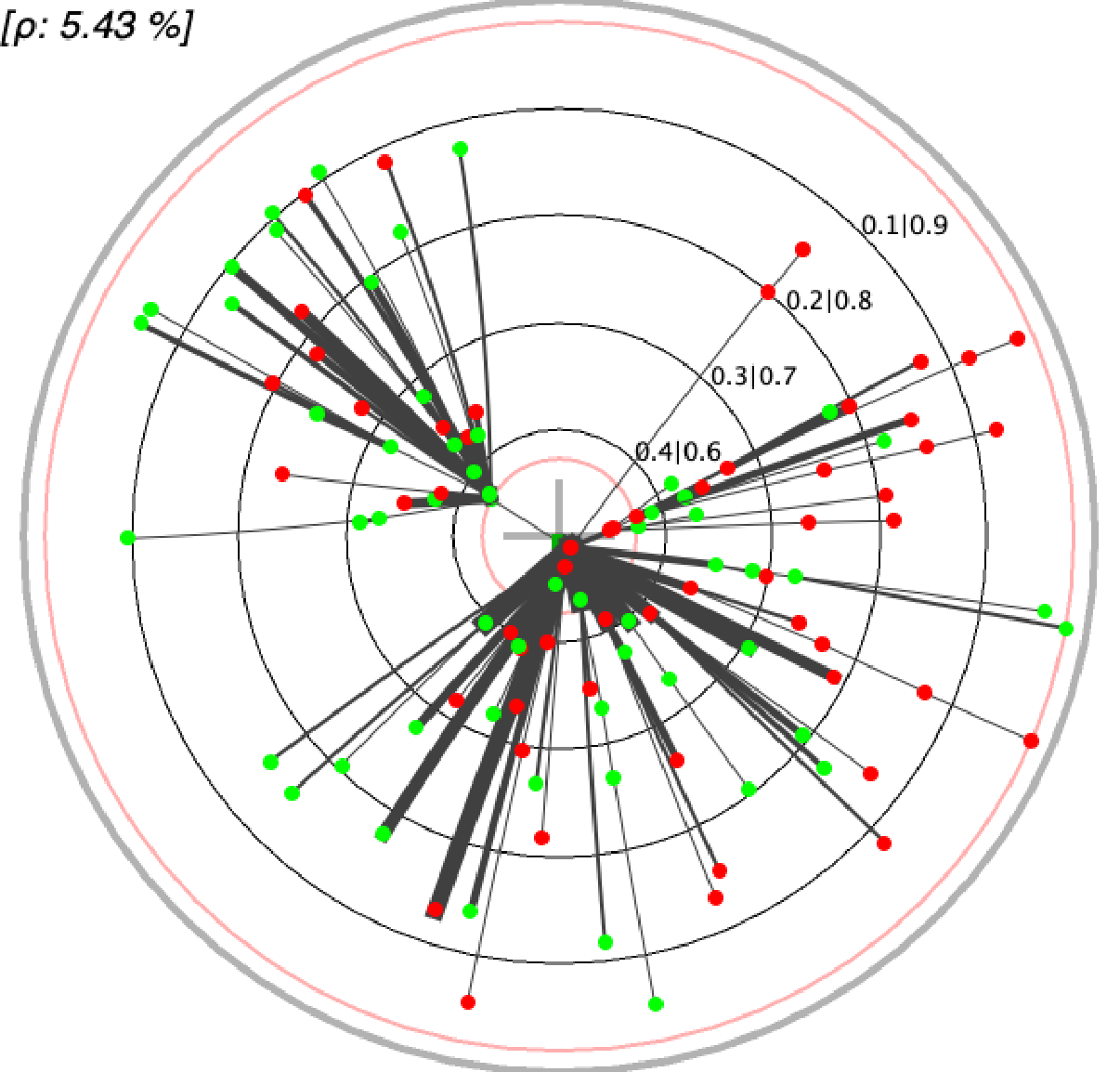}& \includegraphics[trim=0bp 0bp 0bp 0bp,clip,width=0.3\columnwidth]{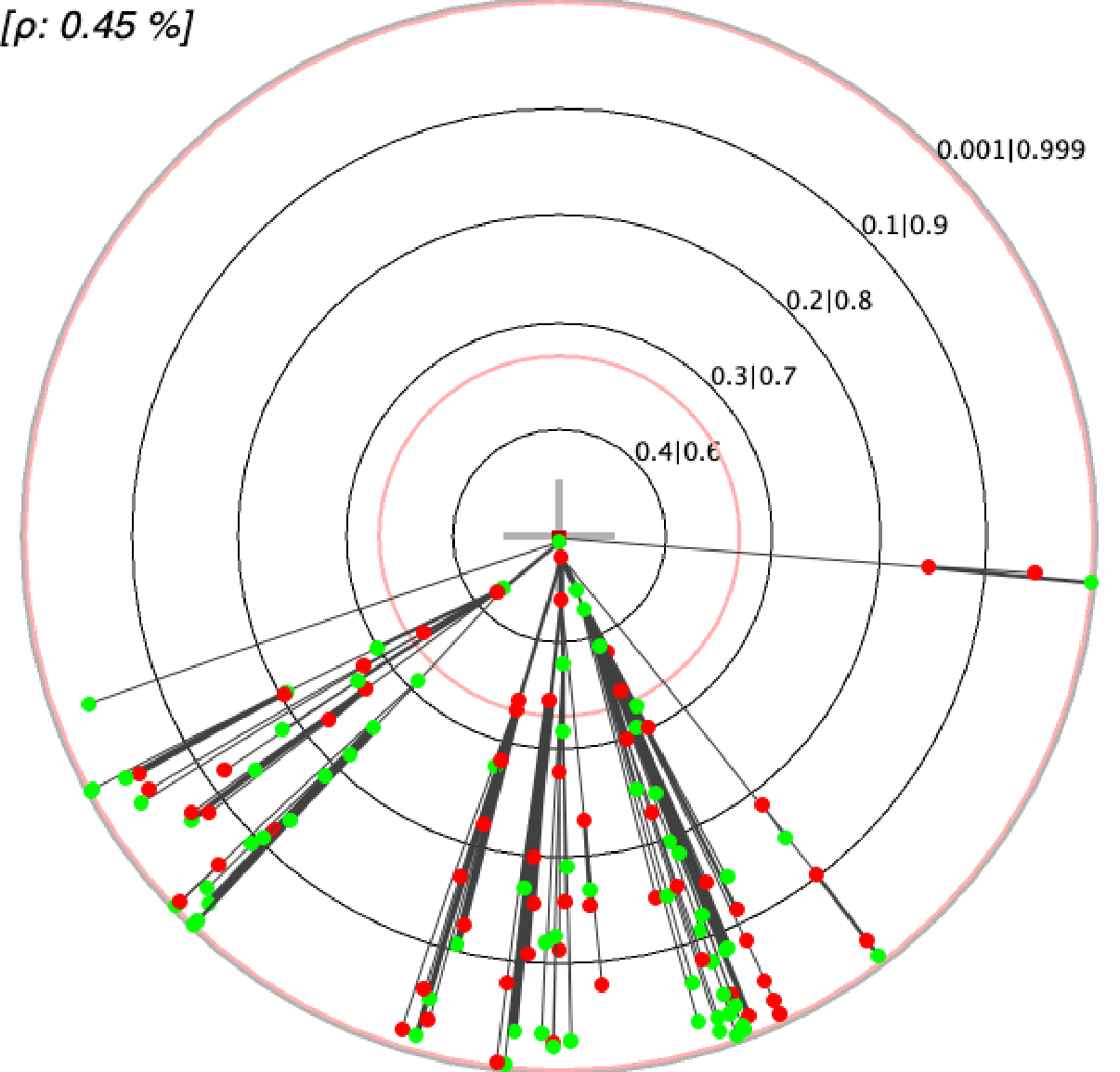} \\\Xhline{2pt} 
  \end{tabular}}
\caption{Embedding in Poincar\'e disk of the Monotonic Decision Trees corresponding to the DTs learned at the first (top row) and 10$^{th}$ (bottom row) boosting iteration, on four UCI domains. Stark differences emerges between these domains from the plots alone, see text for details.}
    \label{tab:summary-dt-exerpt}
  \end{table}

  \paragraph{Interest of the t-self for visualization} We now test the experimental impact of switching to the t-self in Poincar\'e disk model (Section \ref{sec:t-hyp}). Recall that switching to the t-self moves way model parts close to $\partial \mathbb{B}_1$ but keeps a low non-linear distortion around the center, which is thus roughly only affected by a scaling factor. Table \ref{tab:self-vs-t-self} presents a few results. For domain \domainname{online$\_$shoppers$\_$intention}, we note that the part of the tree that is within the isoline defined by posterior $p^+_. \in \{0.1,0.9\}$ gets indeed just scaled: both plots lots quite identical. Very substantial differences appear near the border: the best parts of the model could easily be misjudged as equivalent from $\mathbb{B}_1$ alone ({\color{orange} orange} rectangle) but there is little double from $\mathbb{B}^{(t)}_1$ that one of them, which crosses the $p^+_. \in \{0.001,0.009\}$ isoline, is in fact much better than the others. When many subtrees seem to be aggregating near the border as in \domainname{buzz$\_$in$\_$social$\_$media}, stark differences can appear on the t-self: the best subtrees are immediately spotted from the t-self ({\color{orange} orange} rectangles). In between, the t-self makes a much more visible ordering between the best nodes and subtrees, compared to Poincar\'e disk. \domainname{hardware} demonstrate that such very good nodes that are hard to differentiate from the others in $\mathbb{B}_1$ can appear at any iteration.

  \begin{figure}
    \centering
    \resizebox{\textwidth}{!}{\begin{tabular}{c?c}\Xhline{2pt}
                                {\small \domainname{online$\_$shoppers$\_$intention}} $(j=0)$ & {\small \domainname{twitter}} $(j=0)$\\ \hline
                                \includegraphics[trim=35bp 50bp 50bp 60bp,clip,width=0.8\columnwidth]{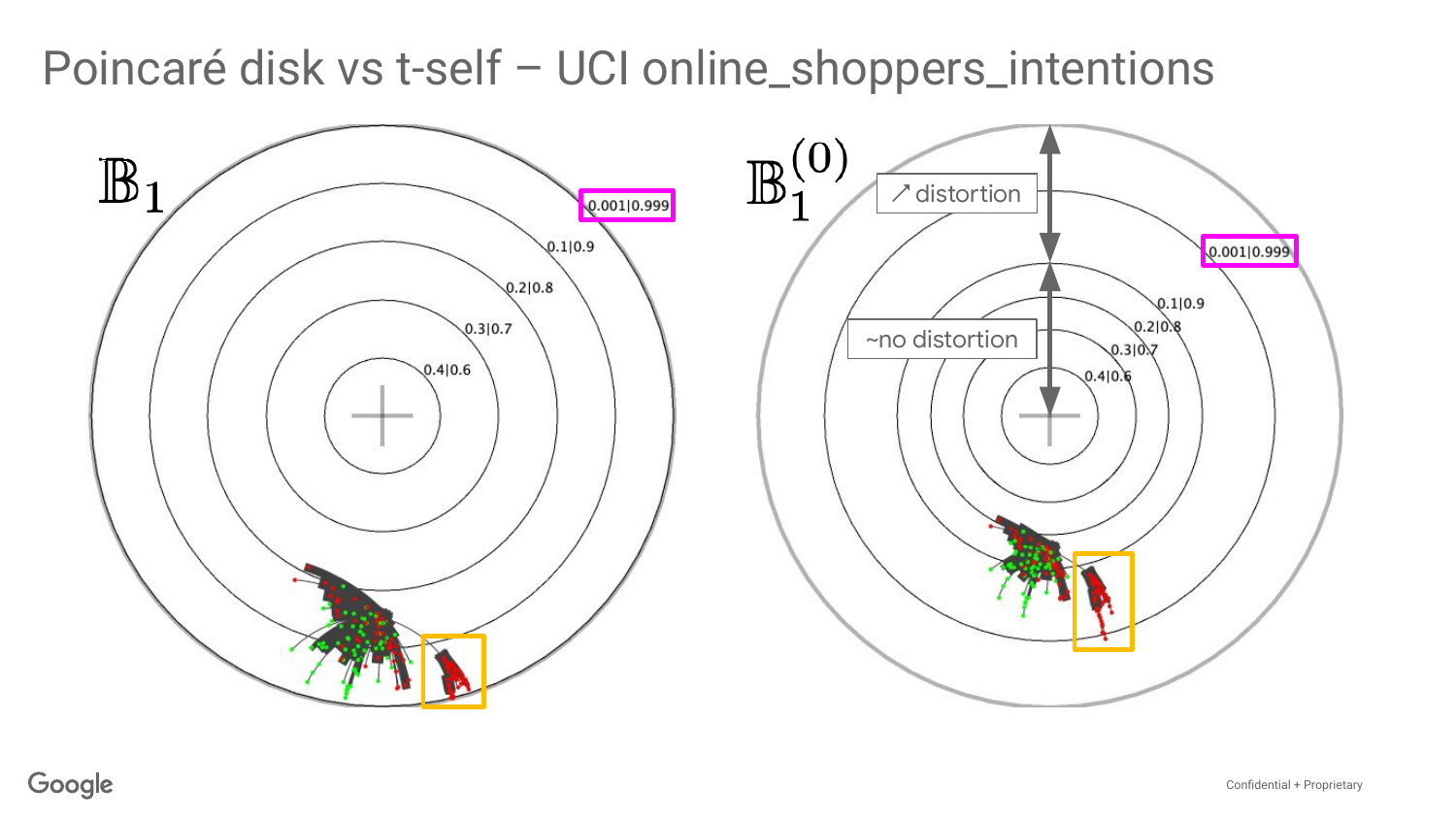} & \includegraphics[trim=35bp 50bp 50bp 60bp,clip,width=0.8\columnwidth]{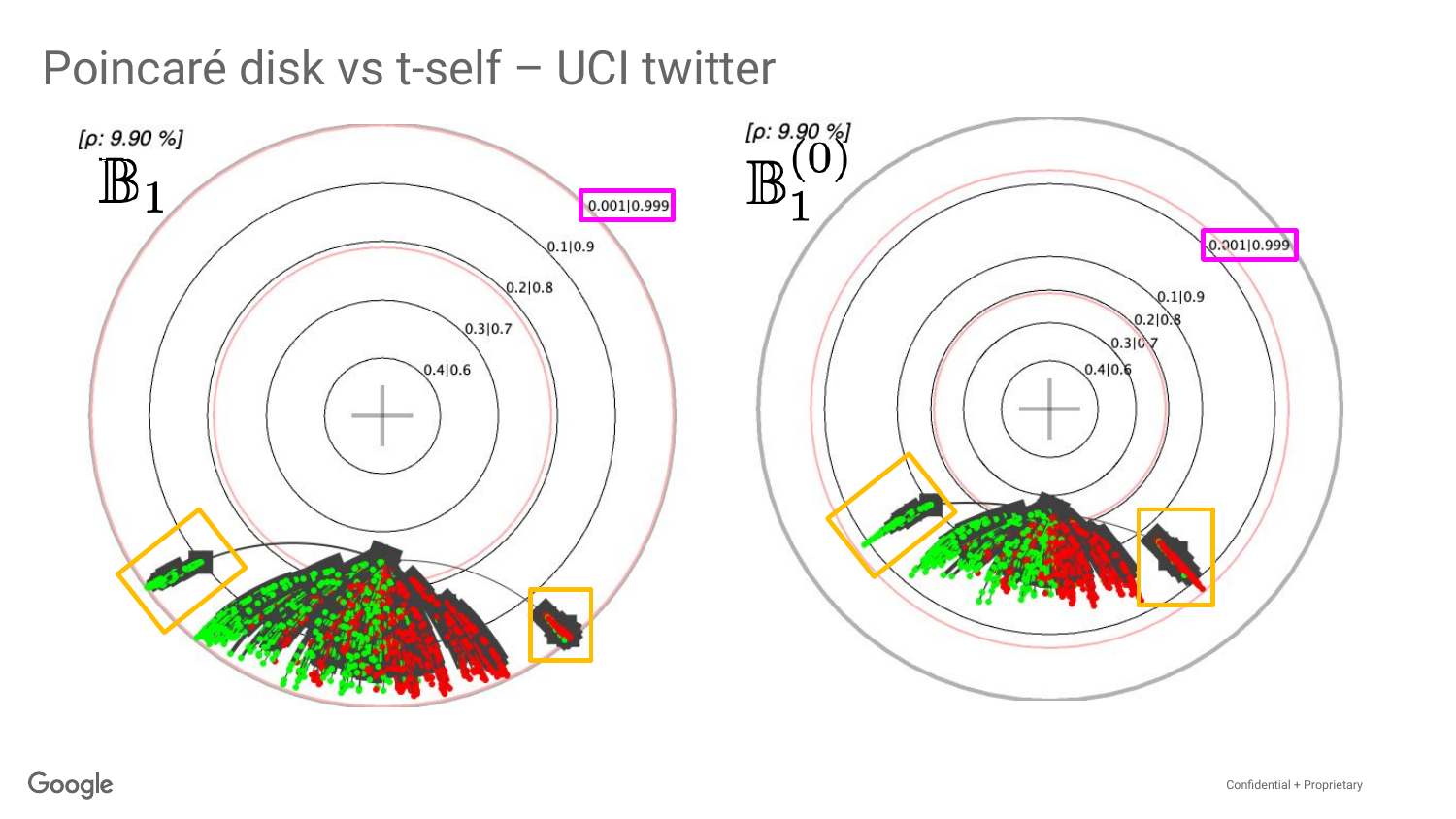} \\\Xhline{2pt} 
                               {\small \domainname{hardware}} $(j=70)$& {\small \domainname{hardware}} $(j=139)$\\ \hline
                                \includegraphics[trim=35bp 50bp 50bp 60bp,clip,width=0.8\columnwidth]{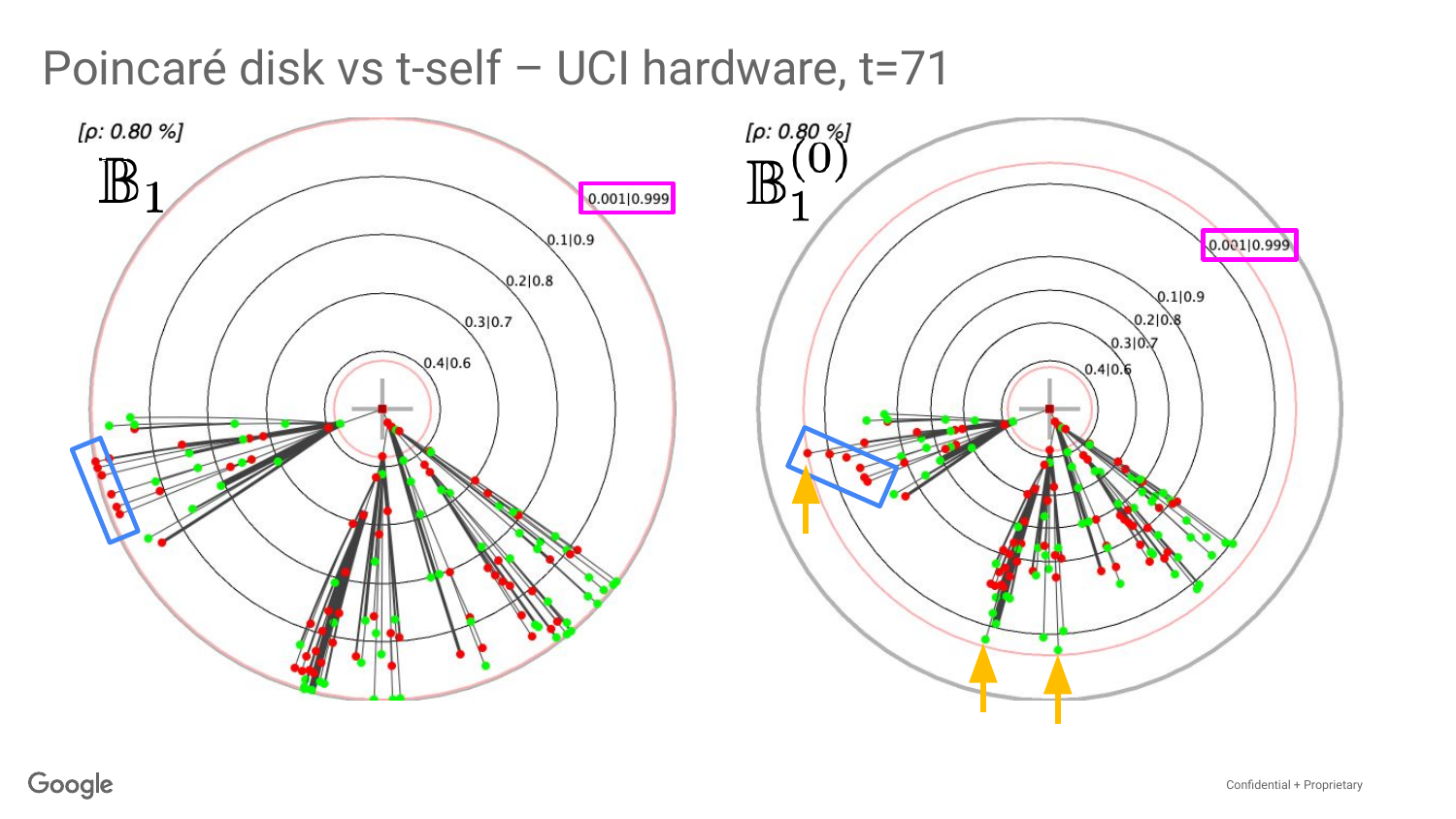} & \includegraphics[trim=35bp 50bp 50bp 60bp,clip,width=0.8\columnwidth]{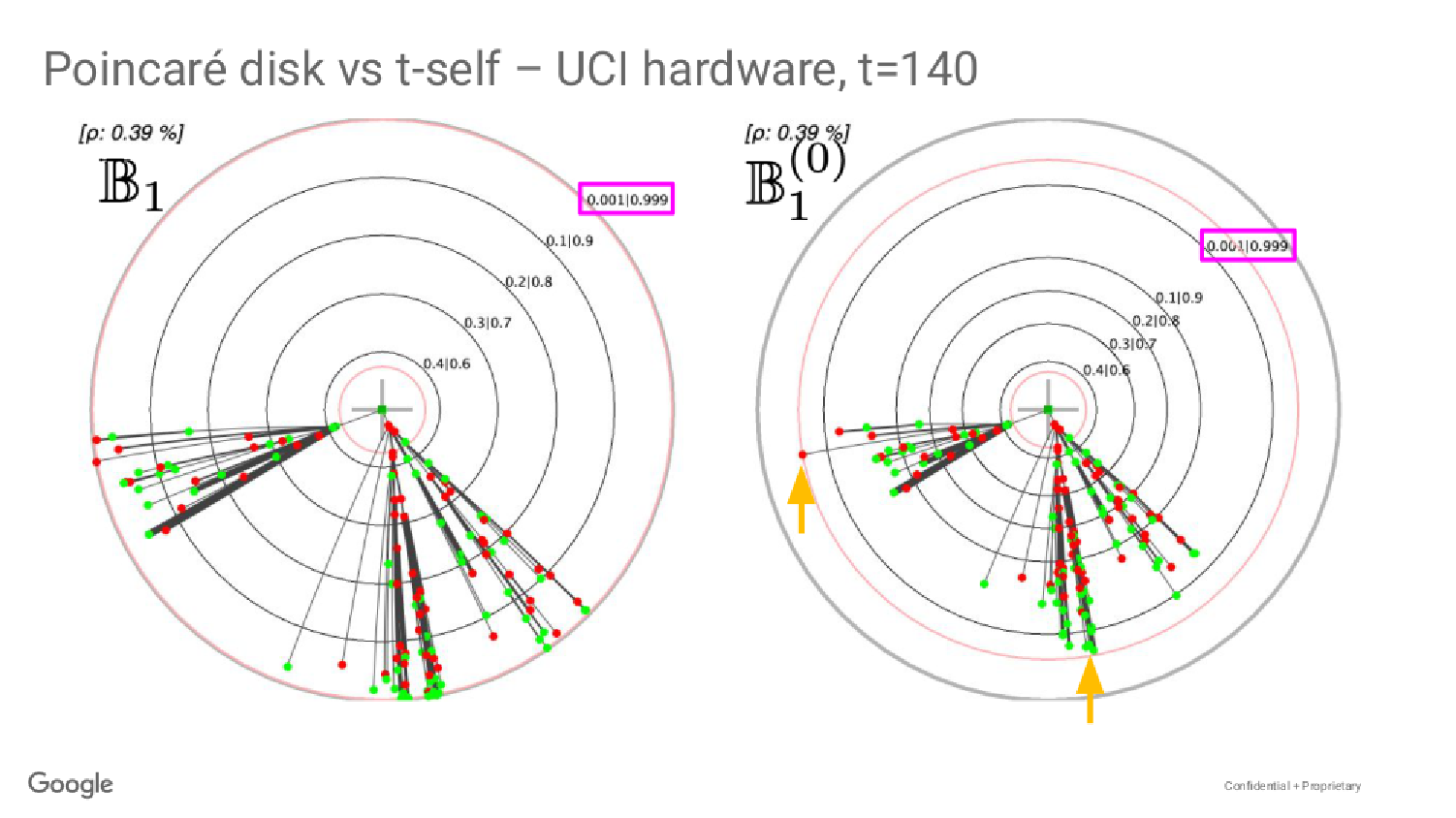} \\ \Xhline{2pt} 
  \end{tabular}}
\caption{\textit{Top pane}: comparison between Poincar\'e disk embedding and its t-self for $t=0$ for a MDT learned on UCI \domainname{online$\_$shoppers$\_$intention} (left) and \domainname{twitter} (right). On the left panel, we have not plotted the boosting coefficients information. The isoline distinguished in {\color{magenta} magenta} is the big \textbf{width} one in Table \ref{tab:isolines-t-self}. Note the difference in (non-linear) distortion created in the t-self, in which the central part just enjoys a scaling. \textit{Bottom pane}: stark differences in visualization between $\mathbb{B}_1$ and the t-self do not just appear initially: they can appear at any iteration. We display two MDTs learned on UCI \domainname{hardware} at two different iterations (indicated). It would be very hard to differentiate the best leaves from just the Poincar\'e disk embedding, while it becomes obvious from the t-self ({\color{orange} orange} arrows). Note also the set of {\color{red} red} nodes in the Poincar\'e disk for $t=70$ that mistakenly look aligned, but not in the t-self ({\color{blue} blue} rectangle). See text for details.}
    \label{tab:self-vs-t-self}
  \end{figure}

\section{Discussion and conclusion}
\label{sec:disc}

Our contributions start with a generalization of classical integration and derivation to a context encompassing the concept of additivity, upon which integration is built. This generalization is rooted in the theory governing nonextensive statistical mechanics, in which additivity needs to be replaced by a more general notion of $t$-additivity. The standard properties of integration and derivation that we have investigated are extended in a natural way, for many of them involving a generalization of the arithmetic over the reals also tied to the field of nonextensive statistical mechanics. That  such properties can be obtained without additional technical ``effort" offers a perspective of developments and applications in other subfields of ML where such tools could be used, not just in the context of hyperbolic embeddings or more broadly distortion measures on which we have focused. We note the recent developments of tools in connection with nonextensive statistical mechanics, that have been used in ML problems as diverse as clustering \cite{anwCA}, boosting \cite{nawBW}, optimal transport \cite{annwOT}, etc. .

More specific avenues for future work are related to our application. First, we experimentally observed that Mototonic Decision Trees (MDTs) can be substantially better than DTs even when using the DTs' leveraging coefficients in boosting. This could be attributable to the fact that being smaller in general (and sometimes much smaller), they reduce the risk of overfitting on big models. This could make them a good alternative to DT for ML and deserves further investigation. Second, our embedding algorithm for MDTs, based on Sarkar's approach, is heuristic and, while it does seem to provide good visual results in general (modulo some fine tuning of some of its parameters), it does not give any formal guarantees with respect to the embedding quality; this problem would deserve a full treatment of its own to get to a provably good embedding. This seems to be a very promising direction since popular packages to represent decision trees rely on a simple topological layer provided by the tree graph, to which various additional information are superimposed without link to the embedding. We believe our approach is the first to provide all layers of a full geometric embedding with a geometry directly interpretable in terms of classification performances, considering trees alone or in a boosted ensemble.

\section*{Acknowledgments}

Many thanks to Mathieu Guillame-Bert for discussions and suggestions around the material presented.

\section*{Code availability}

The code and an example public dataset to process can be obtained from 
\begin{center}
    \scriptsize{\url{http://users.cecs.anu.edu.au/~rnock/code/boosting-tempered-and-hyperbolic.html}}
\end{center} 

\bibliographystyle{plain}
\bibliography{bibgen}






\newpage
\appendix
\onecolumn

\renewcommand\thesection{\Roman{section}}
\renewcommand\thesubsection{\thesection.\arabic{subsection}}
\renewcommand\thesubsubsection{\thesection.\thesubsection.\arabic{subsubsection}}

\renewcommand*{\thetheorem}{\Alph{theorem}}
\renewcommand*{\thelemma}{\Alph{lemma}}
\renewcommand*{\thecorollary}{\Alph{corollary}}

\renewcommand{\thetable}{A\arabic{table}}

\begin{center}
{\Huge Appendix}
\end{center}

This is the Appendix to Paper "\papertitle". To
differentiate with the numberings in the main file, the numbering of
Theorems, etc. is letter-based (A, B, ...).

\section*{Table of contents}

\noindent \textbf{$t$-algebra and $t$-additivity} \hrulefill Pg
\pageref{sec:t-al-t-add}\\

\noindent \textbf{Supplementary material on proofs} \hrulefill Pg
\pageref{sec-sup-pro}\\
\noindent $\hookrightarrow$ Proof of Theorem \ref{thtRIEMANN} \hrulefill Pg \pageref{proof_thtRIEMANN}\\
\noindent $\hookrightarrow$ Proof of Theorem \ref{thtRIEM-DER} \hrulefill Pg \pageref{proof_thtRIEM-DER}\\
\noindent $\hookrightarrow$ Additional and helper results for Theorems \ref{thtRIEMANN} and \ref{thtRIEM-DER} \hrulefill Pg \pageref{proof_add-res}\\
\noindent $\hookrightarrow$ Sets endowed with a distortion and their $t$-self: statistical information \hrulefill Pg \pageref{proof_add-statinf}\\
\noindent $\hookrightarrow$ Proof of Lemma \ref{lem-d-l-metric} \hrulefill Pg \pageref{proof_lem-d-l-metric}\\
\noindent $\hookrightarrow$ Proof of Lemma \ref{lem-t-self-poincare} \hrulefill Pg \pageref{proof_lem-t-self-poincare}\\
\noindent $\hookrightarrow$ Boosting with the logistic loss \textit{\`a-la} AdaBoost \hrulefill Pg \pageref{algo_logisticboost}\\
\noindent $\hookrightarrow$ Modifying Sarkar's embedding in Poincar\'e disk \hrulefill Pg \pageref{sarkar-mod}\\

\noindent \textbf{Supplementary material on experiments} \hrulefill Pg
\pageref{sec-sup-exp}\\

\newpage 

\section{$t$-algebra and $t$-additivity}
\label{sec:t-al-t-add}

We provide here a few more details on the basis of our paper, the $t$-algebra and the $t$-additivity of some divergences.

\subsection{Primer on $t$-algebra}

Classical arithmetic over the reals can be used to display duality relationships between operators using the $\log, \exp$ functions, such as for example $\log(a / b) = \log a - \log b$, $\exp(a + b) = \exp(a) \cdot \exp(b)$, and so on. They can also be used to define one operator from another one. There is no difference between the operators appearing inside and outside functions. In the $t$-algebra, a difference appears and such relationships can be used to define the $t$-operators from those over the reals, as indeed one can define the tempered addition
\begin{eqnarray*}
x \oplus_t y & \defeq & \log_t (\exp_t(x) \cdot \exp_t(y)),
\end{eqnarray*}
the tempered subtraction,
\begin{eqnarray*}
x \ominus_t y & \defeq & \log_t (\exp_t(x) / \exp_t(y)),
\end{eqnarray*}
(both simplifying to the expressions we use), and of course the tempered product and division,
\begin{eqnarray*}
x \otimes_t y \defeq \exp_t(\log_t(x) + \log_t(y)) & ; & x \oslash_t y \defeq \exp_t(\log_t(x) - \log_t(y)),
\end{eqnarray*}
whose simplified expression appears \textit{e.g.} in \cite{annwOT}. See also \cite{nlwGA}.

\subsection{Functional $t$-additivity}
As is well-known, Boltzman-Gibbs (and so, Shannon) entropy is additive while Tsallis entropy is $t$-additive. Note that additivity for BG requires being on the simplex, but $t$-additivity of Tsallis technically requires only positive measures -- the simplex restriction ensures the limit exists for $t\rightarrow 1$ and then T$\rightarrow$BG.\\
\noindent \textbf{Divergences can also be $t$-additive} The KL divergence
\begin{eqnarray*}
D_{\textsc{kl}} (\ve{p}\|\ve{q}) & \defeq & \sum_k p_k \log(p_k/q_k)
\end{eqnarray*}
is additive on the simplex but not $t$-additive: using a decomposition of $\ve{p}, \ve{q}$ as product of (independent) distributions ($\ve{p}_1,\ve{p}_2$ and $\ve{q}_1, \ve{q}_2$) using the cartesian product of their support, we indeed check $D_{\textsc{kl}} (\ve{p}\|\ve{q}) = D_{\textsc{kl}} (\ve{p}_1\|\ve{q}_1) + D_{\textsc{kl}} (\ve{p}_2\|\ve{q}_2)$ with $p_{ij} = p_{1i} p_{2j}$ and $q_{ij} = q_{1i} q_{2j}$. On the other hand, Tsallis divergence \cite{nnORAT} is $t$-additive on positive measures with such a decomposition, with 
\begin{eqnarray*}
D_{\textsc{T}} (\ve{p}\|\ve{q}) & \defeq & \frac{\sum_k p_k (p_k/q_k)^{1-t}-1}{1-t},
\end{eqnarray*}
and we check that $D_{\textsc{T}} (\ve{p}\|\ve{q}) = D_{\textsc{T}} (\ve{p}_1\|\ve{q}_1) + D_{\textsc{T}} (\ve{p}_2\|\ve{q}_2) + (1-t)\cdot D_{\textsc{T}} (\ve{p}_1\|\ve{q}_1) \cdot D_{\textsc{T}} (\ve{p}_2\|\ve{q}_2)$ (additional requirement to be on the simplex for convergence to $D_{\textsc{kl}}$ as $t\rightarrow 1$). For $t\not\in (1,2)$ and on the simplex, Tsallis divergence is an $f$-divergence with generator $z\mapsto (z^{2-t}-1)/(1-t)$. Similarly, the tempered relative entropy is $t$-additive \textit{on the co-simplex}, where in this case
\begin{eqnarray*}
D_{t} (\ve{p}\|\ve{q}) & \defeq & \frac{1 - \sum_k p_k q^{1-t}_{k}}{1-t}, \ve{p}, \ve{q}\in \tilde{\Delta}_m.
\end{eqnarray*}
The tempered relative entropy is a Bregman divergence with generator $z\mapsto z\log_t z - \log_{t-1} z$.

\section{Supplementary material on proofs}\label{sec-sup-pro}

\subsection{Proof of Theorem \ref{thtRIEMANN}}\label{proof_thtRIEMANN}

The Theorem is tautological for $t=1$ so we prove it for $t\neq 1$. Denote $\mathcal{P}(\mathcal{S})$ the set of subsets of set $\mathcal{S}$ and $\mathcal{P}_*(\mathcal{S}) \defeq \mathcal{P}(\mathcal{S}) \backslash \{\emptyset\}$. We first transform $S^{(t)}_{\Delta_n}(f)$ (index $n$ shown for readability) in a better suited expression:
\begin{eqnarray}
  S^{(t)}_{\Delta_n}(f) & \defeq & (\mbox{\Large $\oplus$}_t)_{i=1}^n |\mathbb{I}_i| f(\xi_i)\nonumber\\
                    & = & \sum_{P \in \mathcal{P}_*([n])} (1-t)^{|P|-1} \cdot \prod_{i \in P} |\mathbb{I}_i| f(\xi_i)\nonumber\\
                    & = & \frac{1}{1-t} \cdot \sum_{P \in \mathcal{P}_*([n])} \prod_{i \in P} (1-t) |\mathbb{I}_i| f(\xi_i)\nonumber\\
  & = & \frac{1}{1-t} \cdot \left(\prod_{i =1}^n \left(1+(1-t) |\mathbb{I}_i| f(\xi_i)\right) - 1\right),\label{simplSdelta}
\end{eqnarray}
where we have used $\mathbb{I}_i \defeq [x_{i-1}, x_{i} ]$ for conciseness. We now need a technical Lemma.
\begin{lemma}\label{lemRAT1}
  Fix any $0\leq v < 8/10$ and consider any $n$ reals $q_i, i\in[n]$ such that $|q_i| \leq v, \forall i\in [n]$. Then
  \begin{eqnarray}
1 \leq \frac{\left(1+\frac{1}{n} \cdot \sum_i q_i\right)^n}{\prod_i (1+q_i)} \leq \exp\left(n v \cdot v\right).\label{ineqGend3}
\end{eqnarray}
  \end{lemma}
\begin{proof}
Suppose all the $q_i, i\in[n]$ satisfy $q_i \in [u,v]$ and let
$\varphi$ be strictly convex differentiable and defined over $[u,v]$. Then it comes from \cite[Lemma 9]{cnBD} that we get the right-hand side of
\begin{eqnarray}
0 \leq \expect_i [\varphi(q_i)] - \varphi(\expect_i[q_i]) \leq D_{\varphi} \left(w \left\| (\varphi')^{-1}\left(\frac{\varphi(u) - \varphi(v)}{u-v}\right) \right.\right),\label{ineqGend}
\end{eqnarray}
where we can pick $w \in \{u, v\}$ and $D_\varphi$ is the Bregman divergence with generator $\varphi$ (the left-hand side is Jensen's inequality). Picking $u \defeq -v$ (assuming wlog $v>0$) and letting
\begin{eqnarray}
Y_v & \defeq & \frac{1-v}{2v}\cdot \log\left(1+\frac{2v}{1-v}\right) \label{defXb},
  \end{eqnarray}
for the choice $\varphi(z) \defeq -\log(1+z)$ and $w \defeq -v$, we obtain after simplification
\begin{eqnarray}
0 \leq \log\left(1+\frac{1}{n} \cdot \sum_i q_i\right) - \frac{1}{n}\cdot \sum_i \log(1+q_i) \leq Y_v - \log(Y_v) - 1.\label{ineqGend2}
\end{eqnarray}
Analizing function $v \mapsto Y_v - \log(Y_v) - 1$ reveals that it is upperbounded by $v\mapsto v^2$ if $v\in [-8/10, 8/10]$. Hence, multiplying by $n$ all sides and passing to the exponential, we get that
\begin{eqnarray*}
1 \leq \frac{\left(1+\frac{1}{n} \cdot \sum_i q_i\right)^n}{\prod_i (1+q_i)} \leq \exp\left(n v^2\right), \forall 0\leq v<\frac{8}{10},
\end{eqnarray*}
which leads to the statement of the Lemma.
\end{proof}
Let us come back to our Riemannian summation setting and let
\begin{eqnarray*}
  q_i & \defeq & (1-t) \cdot |\mathbb{I}_i| f(\xi_i), \forall i\in [n],\\
  v & \defeq & \max_i |q_i|.
\end{eqnarray*}
Assume now that $f$ is Riemann integrable, which allows to guarantee that, at least for $n$ large enough and a step of division $\Delta_n$ not too large, we have $|q_i| \leq 8/10, \forall i\in [n]$ and so we can use Lemma \ref{lemRAT1}. We get from \eqref{simplSdelta} and Lemma \ref{lemRAT1}, if $t<1$
\begin{eqnarray}
  \lefteqn{S^{(t)}_{\Delta_n}(f)}\nonumber\\
  & = & \frac{1}{1-t} \cdot \left(\prod_{i =1}^n \left(1+q_i\right) - 1\right)\nonumber \\
                        & \in & \left[\frac{1}{1-t} \cdot \left(\left(1+\frac{1}{n} \cdot \sum_i q_i\right)^n\cdot \exp(-nv \cdot v) - 1\right) , \frac{1}{1-t} \cdot \left(\left(1+\frac{1}{n} \cdot \sum_i q_i\right)^n - 1\right)\right],\label{boundST}
\end{eqnarray}
and we permute the bounds if $t>1$. Importantly, we note that
\begin{eqnarray}
  \sum_i q_i & = & (1-t) \cdot S^{(1)}_{\Delta_n}(f).\label{defSQi}
\end{eqnarray}
So suppose $\lim_{n\rightarrow +\infty} S^{(1)}_{\Delta_n}(f) \defeq L_1 \defeq \int_{a}^{b} f(u)\mathrm{u}$ is finite and choose the step of division $\Delta_n$ not too large so that
\begin{eqnarray*}
n \cdot \max_i |q_i| & = & (1-t) \cdot \max_i (n|\mathbb{I}_i|) |f(\xi_i)| \leq K \cdot ((1-t) |b-a| \max f([a,b])),
\end{eqnarray*}
for some constant $K\geq 1$ (this is possible because $f$ is (1-)Riemann integrable; we also note that if the division is regular then we can choose $K=1$). The value of $K$ is not important: what is important is that $nv$ remains finite in \eqref{boundST} as $n$ increases, while $\lim_{n\rightarrow +\infty} v = 0$. Hence, $\exp(-nv \cdot v) \rightarrow 1$ as $n\rightarrow +\infty$ and \eqref{boundST} implies, because $\lim_{n \rightarrow +\infty} (1+a/n)^n = \exp a$, the two first identities in
\begin{eqnarray}
  \lim_{n\rightarrow +\infty} S^{(t)}_{\Delta_n}(f) & = & \frac{1}{1-t} \cdot \left(\lim_{n\rightarrow +\infty}\left(1+\frac{1}{n} \cdot \sum_i q_i\right)^n - 1\right)\nonumber\\
  & = & \frac{1}{1-t} \cdot \left(\exp\left(\sum_i q_i\right) - 1\right)\nonumber\\
                                                    & = & \frac{1}{1-t} \cdot \left(\exp\left((1-t) L_1\right) - 1\right)\label{eqDEFSQi}\\
                                                    & = & \log_t \exp(L_1)\nonumber\\
  & = & \log_t \exp \left(\int_{a}^{b} f(u)\mathrm{d}u \right)\nonumber
\end{eqnarray}
(\eqref{eqDEFSQi} follows from \eqref{defSQi}) and by definition $\lim_{n\rightarrow +\infty} S^{(t)}_{\Delta_n}(f)
\defeq \riemannint{t}{a}{b} f(u) \mathrm{d}_t u$. So we get that Riemann integration ($t=1$) grants
\begin{eqnarray*}
\riemannint{t}{a}{b} f(u) \mathrm{d}_t u & = & \log_t \exp \int_a^b f(u) \mathrm{d} u,
\end{eqnarray*}
which, in addition to showing \eqref{eq-t-riemann} (main file) also shows that $t=1$-Riemann integration is equivalent to all $t\neq 1$-Riemann integration, ending the proof of Theorem \ref{thtRIEMANN}.

\begin{remark}
  The absence of affine terms in upperbounding $v \mapsto Y_v - \log(Y_v) - 1$ in the neighborhood of 0 is crucial to get to our result.
\end{remark}

\subsection{Proof of Theorem \ref{thtRIEM-DER}}\label{proof_thtRIEM-DER}

Using Theorem \ref{thtRIEMANN}, we just have to analyze the limit in relationship to the Riemannian case:
\begin{eqnarray}
  \mathrm{D}_t F(z) & \defeq &  \lim_{\delta \rightarrow 0} \frac{\riemannint{t}{a}{z+\delta} f(u) \mathrm{d}_t u \ominus_t \riemannint{t}{a}{z} f(u) \mathrm{d}_t u}{\delta} \nonumber\\
  & = & \lim_{\delta \rightarrow 0} \frac{\log_t \exp \int_a^{z+\delta} f(u) \mathrm{d} u \ominus_t \log_t \exp \int_a^{z} f(u) \mathrm{d} u}{\delta} \label{eqStep1}\\
  & = & \lim_{\delta \rightarrow 0} \frac{1}{\delta} \cdot \log_t \left(\frac{\exp \int_a^{z+\delta} f(u) \mathrm{d} u}{\exp \int_a^{z} f(u) \mathrm{d} u}\right) \label{eqStep2}\\
  & = & \lim_{\delta \rightarrow 0} \frac{1}{\delta} \cdot \log_t \exp \left(\int_a^{z+\delta} f(u) \mathrm{d} u- \int_a^{z} f(u) \mathrm{d} u\right) \nonumber\\
  & = & \lim_{\delta \rightarrow 0} \frac{1}{\delta} \cdot \left(\int_a^{z+\delta} f(u) \mathrm{d} u- \int_a^{z} f(u) \mathrm{d} u\right) \label{eqStep4}\\
  & \defeq & \mathrm{D}_1 F(z) = f,\nonumber
\end{eqnarray}
where the last identity is the classical Riemannian case, \eqref{eqStep1} follows from Theorem \ref{thtRIEMANN}, \eqref{eqStep2} follows from a property of $\log_t$ ($\log_t a \ominus \log_t b = \log_t(a/b)$), \eqref{eqStep4} follows from the fact that $ \log_t \exp (z) =_0 z + o(z)$. This completes the proof of Theorem \ref{thtRIEM-DER}.

\subsection{Additional and helper results for Theorems \ref{thtRIEMANN} and \ref{thtRIEM-DER}}\label{proof_add-res}

We list a series of consequences to both theorems.
\paragraph{General properties of $t$-integrals} Some properties generalize those for classical Riemann integration.
\begin{theorem}\label{thproptriemann}
  The following relationships hold for any $t\in \mathbb{R}$ and any functions $f,g$ $t$-Riemann integrable over some interval $[a,b]$:
  \begin{eqnarray*}
  \begin{array}{rcll}
    \riemannint{t}{a}{b} f(u) \left\{+ \mbox{ or } -\right\} g(u) \mathrm{d}_t u & = & \left(\riemannint{t}{a}{b} f(u) \mathrm{d}_t u\right)\left\{\oplus_t\mbox{ or }\ominus_t\right\} \left(\riemannint{t}{a}{b} g(u) \mathrm{d}_t u\right) & (\mbox{additivity}),\\
    \riemannint{t}{a}{b} \lambda f(u) \mathrm{d}_t u & = & \lambda \cdot \riemannint{1-(1-t)\lambda}{a}{b}  f(u) \mathrm{d}_t u \quad(\lambda \in \mathbb{R}) & (\mbox{dilativity}),\\
\riemannint{t}{a}{b} f(x) \mathrm{d}_t u & = & \riemannint{t}{a}{c} f(u) \mathrm{d}_t u \oplus_t \riemannint{t}{c}{b} f(u) \mathrm{d}_t u \quad(c\in [a,b]) & (\mbox{Chasles' relationship}),\\
    \left|\riemannint{t}{a}{b} f(u) \mathrm{d}_t u\right| & \leq & \riemannint{1-|1-t|}{a}{b} |f(u)| \mathrm{d}_t u & (\mbox{triangle inequality}),\\
    \riemannint{t}{a}{b} f(u) \mathrm{d}_t u & \leq & \riemannint{t}{a}{b} g(u) \mathrm{d}_t u \quad (f \leq g) & (\mbox{monotonicity}).
\end{array}
  \end{eqnarray*}
  \end{theorem}
\begin{proof}
  We show additivity for $\oplus_t / +$ (the same path shows the result for $\ominus_t / -$):
  \begin{eqnarray*}
    \riemannint{t}{a}{b} f(u) \mathrm{d}_t u \oplus_t \riemannint{t}{a}{b} g(u) \mathrm{d}_t u & = & \log_t \exp \int_a^b f(u) \mathrm{d} u \oplus_t \log_t \exp \int_a^b g(u) \mathrm{d} u\\
    & = & \log_t \left(\exp \int_a^b f(u) \mathrm{d} u \cdot \exp \int_a^b g(u) \mathrm{d} u \right)\\
    & = & \log_t  \exp \int_a^b (f+g)(u) \mathrm{d} u \\
    & = & \riemannint{t}{a}{b} (f+g)(u) \mathrm{d}_t u.
  \end{eqnarray*}
  We show dilativity, using $t' \defeq 1 - (1-t)\lambda$:
  \begin{eqnarray*}
    \riemannint{t}{a}{b} \lambda f(u) \mathrm{d}_t u & = & \log_t \exp \int_a^b \lambda f(u) \mathrm{d} u \\
                                                     & = & \log_t \exp \lambda \int_a^b f(u) \mathrm{d} u \\
                                                     & = & \frac{1}{1-t} \cdot \left(\exp\left(\lambda(1-t) \int_a^b f(u) \mathrm{d} u\right)-1\right)\\
    & = & \frac{1-t'}{1-t} \cdot \frac{1}{1-t'} \cdot \left(\exp\left((1-t') \int_a^b f(u) \mathrm{d} u\right)-1\right)\\
    & = & \lambda \cdot \frac{1}{1-t'} \cdot \left(\exp\left((1-t') \int_a^b f(u) \mathrm{d} u\right)-1\right)\\
    & = & \lambda \cdot \log_{t'} \exp \int_a^b  f(u) \mathrm{d} u\\
    & = & \lambda \cdot \riemannint{t'}{a}{b}  f(u) \mathrm{d}_t u = \lambda \cdot \riemannint{1-(1-t)\lambda}{a}{b}  f(u) \mathrm{d}_t u .
  \end{eqnarray*}
  The triangle inequality follows from the relationship:
  \begin{eqnarray*}
|\log_t \exp(z)| & \leq & \log_{1-|1-t|} \exp |z|, \forall z,t \in \mathbb{R},
  \end{eqnarray*}
  from which we use the fact that Riemann integration satisfies the triangle inequality and $\log_{1-|1-t|}$ is monotonically increasing on $\mathbb{R}_+$ in the penultimate line of:
  \begin{eqnarray*}
    \left|\riemannint{t}{a}{b} f(u) \mathrm{d}_t u\right| & \defeq &  \left|\log_t \exp \int_a^b f(u) \mathrm{d} u\right|\\
                                                          & \leq & \log_{1-|1-t|} \exp \left|\int_a^b f(u) \mathrm{d} u\right|\\
    & \leq & \log_{1-|1-t|} \exp \int_a^b |f(u)| \mathrm{d} u\\
    & = & \riemannint{1-|1-t|}{a}{b} |f(u)| \mathrm{d}_t u .
    \end{eqnarray*}
  We show Chasles relationship:
  \begin{eqnarray*}
    \riemannint{t}{a}{c} f(u) \mathrm{d}_t u \oplus_t \riemannint{t}{c}{b} f(u) \mathrm{d}_t u & \defeq & \log_t \exp \int_a^c f(u) \mathrm{d} u \oplus_t \log_t \exp \int_c^b f(u) \mathrm{d} u\\
                                                                                               & = & \log_t \left(\exp \int_a^c f(u) \mathrm{d} u \cdot \exp \int_c^b f(u) \mathrm{d} u\right)\\
    & = & \log_t \exp \left(\int_a^c f(u) \mathrm{d} u + \int_c^b f(u) \mathrm{d} u\right)\\
                                                                                               & = & \log_t \exp \int_a^b f(u) \mathrm{d} u\\
    & = & \riemannint{t}{a}{b} f(u) \mathrm{d}_t u,
  \end{eqnarray*}
  where the second identity uses the property $\log_t a \oplus \log_t b = \log_t(ab)$. Monotonicity follows immediately from the fact that $z\mapsto \log_t \exp z$ is strictly increasing:
  \begin{eqnarray*}
    \riemannint{t}{a}{b} f(u) \mathrm{d}_t u & = & \log_t \exp \int_a^b f(u) \mathrm{d} u\\
    & \leq & \log_t \exp \int_a^b g(u) \mathrm{d} u \defeq \riemannint{t}{a}{b} g(u) \mathrm{d}_t u.
  \end{eqnarray*}
  This ends the proof of Theorem \ref{thproptriemann}.
\end{proof}
\paragraph{Computing $t$-integrals} Next, classical relationships to compute integrals do generalize to $t$-integration. We cite the case of integration by part.
\begin{lemma}\label{lem-int-by-part}
  Integration by part translates to $t$-integration by part as:
  \begin{eqnarray*}
\riemannint{t}{a}{b} fg' \mathrm{d}u & = & \mathbin{^{(t)}\left[fg\right]_a^b} \ominus_t \riemannint{t}{a}{b} f'g \mathrm{d}u,
\end{eqnarray*}
where we let
\begin{eqnarray*}
\mathbin{^{(t)}\left[h\right]_a^b} & \defeq & \log_t \exp(h(b)) \ominus_t \log_t \exp(h(a)).
\end{eqnarray*}
\end{lemma}
(Proof immediate from Theorem \ref{thtRIEMANN})
\paragraph{Geometric properties based on $t$-integrals} This is a more specific result, important in the context of hyperbolic geometry: the well-known Hyperbolic Pythagorean Theorem (HPT) does translate to a tempered version with the same relationship to the Euclidean theorem. Consider a hyperbolic right triangle with hyperbolic lengths $a, b, c$, $c$ being the hyperbolic length of the hypothenuse. Let $a_t, b_t, c_t$ denote the corresponding tempered lengths, which are therefore explicitly related using $\liftt_t$ as
\begin{eqnarray*}
a_t = \log_t \exp a, \quad b_t = \log_t \exp b, \quad c_t = \log_t \exp c.
\end{eqnarray*}
Define the tempered generalization of $\cosh$:
\begin{eqnarray}
\cosh_t z & \defeq & \frac{\exp_t z + \exp_t(-z)}{2} \label{def-cosht}.
\end{eqnarray}
The HPT tells us that $\cosh c = \cosh a \cosh b$. It is a simple matter of plugging $\liftt_t$, using the fact that $\log_t$ and $\exp_t$ are inverse of each other \noteRNoff{Check composition, see NeurIPS: OK because $\exp_t \log_t$ and not the inverse} and simplifying to get the tempered HPT, which we call $t$-HPT for short.
\begin{lemma}\label{lem-t-hpt}($t$-HPT)
  For any hyperbolic triangle described as above, the tempered lengths are related as
  \begin{eqnarray*}
\cosh_t c_t = \cosh_t a_t \cosh_t b_t.
    \end{eqnarray*}
  \end{lemma}
  Now, remark that \textit{for any $t\neq 0$}, a series expansion around 0 gives
  \begin{eqnarray*}
\exp_t(z) & = & 1 + \frac{t z^2}{2} + o(z^3)
  \end{eqnarray*}
($\exp_t$ is always infinitely differentiable around $0$, for any $t$) So the $t$-HPT gives
  \begin{eqnarray*}
1 + \frac{t c_t^2}{2} + o(c_t^3) & = & \left(1 + \frac{t a_t^2}{2} + o(a_t^3)\right)\cdot\left(1 + \frac{t b_t^2}{2} + o(b_t^3)\right),
  \end{eqnarray*}
  which simplifies, if we multiply both sides by $2/t$ and simplify, into
  \begin{eqnarray*}
c_t^2 + o(c_t^3) & = & a_t^2 + b_t^2 + o(a_t^3) + o(b_t^3),
  \end{eqnarray*}
  which for an infinitesimal right triangle gives $c_t^2 \approx a_t^2 + b_t^2 $, \textit{i.e.} Pythagoras Theorem, as does the HPT one gives in this case ($c^2 \approx a^2 + b^2$), which is equivalent of the particular $t=1$-HPT case.
  
\paragraph{$t$-mean value Theorem} The $t$-derivative yields a generalization of the Euclidean mean-value theorem.
\begin{lemma}\label{lem-tTAF}
  Let $t\in \mathbb{R}$ and $f$ be continuous over an interval $[a,b]$, differentiable on $(a,b)$ and such that $-1/(1-t) \not\in f([a,b])$. Then $\exists c \in (a,b)$ such that
  \begin{eqnarray*}
\mathrm{D}_t f(c) & = & \frac{(f(b) \ominus_t f(c)) - (f(a) \ominus_t f(c))}{b-a}.
    \end{eqnarray*}
  \end{lemma}
  \begin{proof}
We can obtain a direct expression of $\mathrm{D}_t f$ by using the definition of $\ominus_t$ and the classical derivative: 
\begin{equation}
\label{eq:t-derivative-simple}
    \mathrm{D}_t f(x) = \lim_{\delta \rightarrow 0} \frac{1}{\delta} \cdot \frac{f(x+\delta) - f(x)}{1+(1-t)f(x)} = \frac{1}{1+(1-t)f(x)} \cdot \lim_{\delta \rightarrow 0} \frac{f(x+\delta) - f(x)}{\delta} = \frac{f'(x)}{1+(1-t)f(x)}.
  \end{equation}
  From here, the mean-value theorem tells us that there exists $c\in [a,b]$ such that
  \begin{eqnarray*}
f'(c) & = & \frac{f(b) - f(a)}{b-a}
  \end{eqnarray*}
  Dividing by $1+(1-t)f(c)$ (assuming $f(c) \neq -1/(1-t)$) and reorganising, we get
  \begin{eqnarray*}
    \mathrm{D}_t f(c) & = & \frac{1}{b-a} \cdot \frac{f(b) - f(a)}{1+(1-t)f(c)}\\
                      & =  & \frac{1}{b-a} \cdot \frac{f(b) - f(c)}{1+(1-t)f(c)} - \frac{1}{b-a} \cdot \frac{f(a) - f(c)}{1+(1-t)f(c)} \\
    & = & \frac{(f(b) \ominus_t f(c)) - (f(a) \ominus_t f(c))}{b-a},
  \end{eqnarray*}
  which completes the proof of the Lemma.
    \end{proof}
    In fact, the $t$-derivative of $f$ at some $c$ is "just" an Euclidean derivative for an affine transformation of the function, namely $z \mapsto f(z) \ominus_t f(c)$, also taken at $z=c$. This "proximity" between $t$-derivation and derivation is found in the tempered chain rule (proof straightforward).
    \begin{lemma}
      Suppose $g$ differentiable at $z$ and $f$ differentiable at $g(z)$. Then
      \begin{eqnarray*}
\mathrm{D}_t (f\circ g)(z) & = & \mathrm{D}_t (f)(g(z)) \cdot g'(z).
  \end{eqnarray*}
\end{lemma}

\subsection{Sets endowed with a distortion and their $t$-self: statistical information}\label{proof_add-statinf}

Here, $\mathcal{X}$ contains probability distributions or the parameters of probability distributions: $f$ can then be an $f$-divergence (information theory) or a Bregman divergence (information geometry). Tsallis divergence and the tempered relative entropy are examples of $t$-additive information theoretic and information geometric divergences. A key property of information theory is the data processing inequality \textbf{(D)}, $\mathcal{X}$ being a probability space, which says that passing random variables through a Markov chain cannot increase their divergence as quantified by $f$ \cite{pvAD,vehRD}. A key property of information geometry is the population minimizer property \textbf{(P)}, which elicits a particular function of a set of points as the minimizer of the expected distortion to the set, as quantified by $f$ \cite{bmdgCW}. We let \textbf{(J)} denote the joint convexity property, which would state for $f$ and any $\ve{x}_1, \ve{x}_2, \ve{y}_1, \ve{y}_2$ that
\begin{eqnarray}
f(\lambda\cdot \ve{x}_1 + (1-\lambda)\cdot \ve{x}_2, \lambda\cdot \ve{y}_1 + (1-\lambda)\cdot \ve{y}_2) \leq \lambda f(\ve{x}_1, \ve{y}_1) + (1-\lambda) f(\ve{x}_2, \ve{y}_2), \label{eq-jointc}
\end{eqnarray}
and convexity \textbf{(C)}, which amounts to picking $\ve{y}_1 = \ve{y}_2$ (convexity in the left parameter) xor $\ve{x}_1 = \ve{x}_2$ (in the right parameter).
\begin{lemma}\label{lem-inf-prop}
  For any $t\in \mathbb{R}$, the following holds true:
  \begin {itemize}
  \item [\textbf{(D)}] $f$ satisfies the data processing inequality iff $f^{(t)}$ satisfies the data processing inequality;
  \item [\textbf{(P)}] $\ve{\mu}_* \in \arg\min_{\ve{\mu}} \sum_i f(\ve{x}_i, \ve{\mu})$ iff
    \begin{eqnarray}
      \ve{\mu}_* & \in & \arg\min_{\ve{\mu}} \left(\oplus_t\right)_i f^{(t)}(\ve{x}_i, \ve{\mu}); \label{eq-popmint}
      \end{eqnarray}
  \item [\textbf{(J)}] $f$ satisfies \eqref{eq-jointc} iff the following $(t,t',t'')$-joint convexity property holds:
    \begin{eqnarray}
f^{(t)}(\lambda\cdot \ve{x}_1 + (1-\lambda)\cdot \ve{x}_2, \lambda\cdot \ve{y}_1 + (1-\lambda)\cdot \ve{y}_2) \leq \lambda f^{(t')} (\ve{x}_1, \ve{y}_1) + (1-\lambda) f^{(t'')} (\ve{x}_2, \ve{y}_2), \label{eq-jointc2}
    \end{eqnarray}
    with $t' \defeq \min\{t,1-\lambda + \lambda t\}$ and $t'' \defeq \min\{t,\lambda + (1-\lambda)t\}$.
    \end{itemize}
  \end{lemma}
  \begin{proof}
\textbf{(D)} and \textbf{(P)} are immediate consequences of Lemma \ref{lem-prop-tlift} (point [1.], main file) and properties of $\log_t$. We prove \textbf{(J)}.  $\liftt_t$ being strictly increasing for any $t$, we get for $t\leq 1$
    \begin{eqnarray}
      f^{(t)}(\lambda\cdot \ve{x}_1 + (1-\lambda)\cdot \ve{x}_2, \lambda\cdot \ve{y}_1 + (1-\lambda)\cdot \ve{y}_2) & \leq & \liftt_t \left( \lambda f(\ve{x}_1, \ve{y}_1) + (1-\lambda) f(\ve{x}_2, \ve{y}_2)\right)\nonumber\\
                                                                                                                    & \leq & \lambda \cdot \liftt_t \circ f(\ve{x}_1, \ve{y}_1) + (1-\lambda) \cdot \liftt_t \circ f(\ve{x}_2, \ve{y}_2)\nonumber\\
      & & = \lambda f^{(t)}(\ve{x}_1, \ve{y}_1) + (1-\lambda)f^{(t)}(\ve{x}_2, \ve{y}_2), \label{eq-tsmaller}
    \end{eqnarray}
    because $\liftt_t$ is convex. If $t>1$, we restart from the first inequality and remark that
    \begin{eqnarray}
      \liftt_t \left( \lambda f(\ve{x}_1, \ve{y}_1) + (1-\lambda) f(\ve{x}_2, \ve{y}_2)\right) & = & \liftt_t \left( \lambda f(\ve{x}_1, \ve{y}_1)\right) \oplus_t \liftt_t\left((1-\lambda) f(\ve{x}_2, \ve{y}_2)\right)\nonumber\\
                                                                                               & \leq & \liftt_t \left( \lambda f(\ve{x}_1, \ve{y}_1)\right) + \liftt_t\left((1-\lambda) f(\ve{x}_2, \ve{y}_2)\right)\nonumber\\
                                                                                               & & = \lambda \cdot \liftt_{1-\lambda + \lambda t)} \circ f(\ve{x}_1, \ve{y}_1) + (1-\lambda) \cdot \liftt_{\lambda + (1-\lambda)t} \circ f(\ve{x}_2, \ve{y}_2)\nonumber\\
      & = & \lambda f^{(1-\lambda + \lambda t)} (\ve{x}_1, \ve{y}_1) + (1-\lambda) f^{(\lambda + (1-\lambda)t)} (\ve{x}_2, \ve{y}_2).\label{eq-tlarger}
    \end{eqnarray}
    The inequality is due to the fact that $a \oplus_t b = a + b + (1-t)ab \leq a+b$ if $ab\geq 0$ and $t\geq 1$. The last equality holds because, for $t' \defeq 1-(1-t)b$, we have
    \begin{eqnarray*}
      \log_{t} \left(a^b\right) = \frac{a^{(1-t)b}-1}{1-t} = \frac{1-t'}{1-t} \cdot \frac{a^{1-t'}-1}{1-t'} = b \cdot \log_{t'} (a).
    \end{eqnarray*}
    Putting altogether \eqref{eq-tsmaller} and \eqref{eq-tlarger}, we get that for any $t\in \mathbb{R}$,
    \begin{eqnarray}
      \lefteqn{f^{(t)}(\lambda\cdot \ve{x}_1 + (1-\lambda)\cdot \ve{x}_2, \lambda\cdot \ve{y}_1 + (1-\lambda)\cdot \ve{y}_2) }\nonumber\\
      & \leq & \lambda f^{(\min\{t,1-\lambda + \lambda t\})} (\ve{x}_1, \ve{y}_1) + (1-\lambda) f^{(\min\{t,\lambda + (1-\lambda)t\})} (\ve{x}_2, \ve{y}_2),
    \end{eqnarray}
    as claimed for \textbf{(J)}. This ends the proof of Lemma \ref{lem-inf-prop}.
    \end{proof}
    We note that \eqref{eq-jointc2} also translates into a property for \textbf{(C)}; also, if $t\leq 1$, then $t=t'=t''$ in \eqref{eq-jointc2}.
    
    \subsection{Proof of Lemma \ref{lem-d-l-metric}}\label{proof_lem-d-l-metric}

    $d_L^{(t)}$ still obviously satisfies reflexivity, the identity of indiscernibles and non-negativity, so we check the additional property it now satisfies, the weaker version of the triangle inequality. Given any $\ve{x}, \ve{y}, \ve{z}$ in $\mathbb{H}_c$, condition $d_L^{(t)}(\ve{x},\ve{z}) \leq d_L^{(t)}(\ve{x},\ve{y}) + d_L^{(t)}(\ve{y},\ve{z}) + \delta$ for $t > 1$ is 
    \begin{eqnarray*}
      \frac{1-\exp\left(-(t-1) \left(-\frac{2}{c} - 2 \cdot \ve{x} \circ \ve{z}\right)\right)}{t-1} & \leq & \frac{1-\exp\left(-(t-1) \left(-\frac{2}{c} - 2 \cdot \ve{x} \circ \ve{y}\right)\right)}{t-1} \\
      & & + \frac{1-\exp\left(-(t-1) \left(-\frac{2}{c} - 2 \cdot \ve{y} \circ \ve{z}\right)\right)}{t-1} + \delta,
    \end{eqnarray*}
    which simplifies to
    \begin{eqnarray*}
      \exp\left(2(t-1) \cdot \ve{x} \circ \ve{y}\right) + \exp\left(2(t-1) \cdot \ve{y} \circ \ve{z}\right) & \leq & \exp\left(-\frac{2(t-1)}{c}\right) + (t-1)\delta \exp\left(-\frac{2(t-1)}{c}\right) \\
      & & + \exp\left(2(t-1) \cdot \ve{x} \circ \ve{z}\right).
    \end{eqnarray*}
    We have $\ve{x} \circ \ve{y} \leq -1/c$ by definition, so a sufficient condition to get the inequality is to have
    \begin{eqnarray}
\exp\left(2(t-1) \cdot \ve{y} \circ \ve{z}\right) & \leq & (t-1)\delta \exp\left(-\frac{2(t-1)}{c}\right). \label{eq-sc1}
    \end{eqnarray}
    Function $h(z) \defeq z \delta \exp(-2z/c)$ is maximum over $\mathbb{R}_+$ for $z_* = c/2$, for which it equals $h(z_*) = c\delta / (2e)$. Fix $t = 1 + (c/2)$. We then have $\exp\left(2(t-1) \cdot \ve{y} \circ \ve{z}\right) \leq 1/e$ so to get \eqref{eq-sc1} for this choice of $t$, it is sufficient to pick curvature $c=2/\delta$, yielding relationship $t = 1 + (1/\delta)$.

    \subsection{Proof of Lemma \ref{lem-t-self-poincare}}\label{proof_lem-t-self-poincare}

    We start by proving \eqref{cont-d-poinc-t}. Using the proof of \cite[Proposition 3.1]{mwwyTN}, we know that any point $\ve{x}$ $k$-close to the boundary (Definition \ref{def-close-enc}) satisfies
\begin{eqnarray*}
d^{(t)}(\ve{x}, \ve{0}) & = & \log_t \left(\frac{1+\|\ve{x}\|}{1-\|\ve{x}\|}\right) = \log_t \left(2\cdot 10^{k}-1\right) ,
    \end{eqnarray*}
so to get \eqref{cont-d-poinc-t}, we want $\log_t (2\cdot 10^k-1) \geq g(k)$. Letting $t = 1-f(k)$ with $f(k) \in \mathbb{R}$, we observe $\log_t (2\cdot 10^k-1) =  \left((2\cdot 10^k-1)^{f(k)}-1\right)/f(k)$, so we want, after taking logs,
    \begin{eqnarray}
\log(2\cdot 10^k-1) & \geq & \frac{\log\left(1+f(k)g(k)\right)}{f(k)} \label{eq-logfk}
    \end{eqnarray}
    (this also holds if $f(k) < 0$ because $1+f(k)g(k) < 1$), and there remains to observe $\log(2\cdot 10^k-1) =  k\log(10) + \log(2-1/10^{k})$ with $\log(2-1/10^{k}) \geq 0, \forall k\geq 0$. Hence, to get \eqref{eq-logfk} it is sufficient to request
    \begin{eqnarray*}
\frac{\log\left(1+f(k)g(k)\right)}{f(k)} & \leq & \log(10) \cdot k,
    \end{eqnarray*}
    which is \eqref{cont-d-poinc-t}. For such $t$, the new hyperbolic constant satisfies
    \begin{eqnarray*}
\tau_t & = & \log_t \exp\tau = \frac{1}{f(k)} \cdot \left(\exp \left(f(k) \tau\right) - 1\right)
      \end{eqnarray*} 

\subsection{Boosting with the logistic loss \textit{\`a-la} AdaBoost}\label{algo_logisticboost}

\begin{algorithm}[H]
\caption{\logisticboost($\mathcal{S}$)}\label{alg-logisticboost}
\begin{algorithmic}
  \STATE  \textbf{Input:} Labeled sample $\mathcal{S} \defeq \{(\ve{x}_i,y_i), i \in [m]\}$, $T\in \mathbb{N}_{>0}$;
  \STATE  1 : $w_{1i} \leftarrow 1/2, \forall i\in [m]$; \hfill // initialize all weights (equivalent to maximally \textit{un}confident prediction)
  \STATE  2 :\hspace{0.5cm} \textbf{for} $j = 1, 2, ..., T$
  \STATE  3 :\hspace{1.0cm} $H_j \leftarrow \weaklearn(\mathcal{S}, \ve{w}_j)$; \hfill // request a DT as a "weak hypothesis"
  \STATE  4 :\hspace{1.0cm} $\alpha_j \in \mathbb{R}$; \hfill // picks leveraging coefficient for $H_j$
  \STATE  5 :\hspace{1.0cm} \textbf{for} $i = 1, 2, ..., m$ \hfill // weight update, \textit{not} normalized
  \begin{eqnarray}
w_{(j+1)i} & \leftarrow & \frac{w_{ji}}{w_{ji} + (1-w_{ji}) \cdot \exp(\alpha_j y_i H_j(\ve{x}_i))} ; \label{defweightupdate}
  \end{eqnarray}
  \STATE  \textbf{Output:} Classifier $\mbox{\hcomb}_T  \defeq \sum_j \alpha_j H_j(.)$;
\end{algorithmic}
\end{algorithm}

We want a boosting algorithm for the liner combination of DTs which displays classification that can be easily and directly embedded in the Poincar\'e disk. Algorithm \logisticboost~is provided above. For the weight update, we refer to \cite{nbanbGN}. We already know how to embed DTs via their Monotonic DTs. What a boosting algorithm of this kind does is craft
\begin{eqnarray}
\mbox{\hcomb}_T & \defeq & \sum_{j=1}^{T} u_j(.), \quad u_j(\ve{x}) \defeq \alpha_j \cdot H_j(\ve{x}); \label{defBOOSTH}
\end{eqnarray}
Remark that we have merged the leveraging coefficient and the DTs' outputs $H_.$, on purpose. To compute the leveraging coefficients $\alpha_.$ in Step 4, we use AdaBoost's secant approximation trick\footnote{Explained in \cite[Section 3.1]{ssIBj}.}, applied not to the exponential loss but to the logistic loss: for any $z \in [-R,R]$ and $\alpha\in \mathbb{R}$,
\begin{eqnarray}
\log(1+\exp(-\alpha z)) & \leq & \frac{1+u}{2}\cdot \log(1+\exp(-\alpha R)) + \frac{1-u}{2}\cdot \log(1+\exp(\alpha R)) . \label{eqAPPROXALPHA}
\end{eqnarray}
Hence, to compute the leveraging coefficient $\alpha^*_j$ that approximately minimizes the current loss, $\sum_i w_{ji} \log(1+\exp(-\alpha y_i H_j(\ve{x}_i)))$, we minimize instead the upperbound using \eqref{eqAPPROXALPHA}. Letting
\begin{eqnarray}
({\canolog})^*_j & \defeq & \max_{\leaf \in \leafset(H_j)} \left|\log\left(\frac{p_{j\leaf}}{1-p_{j\leaf}}\right)\right|\label{defNt}
\end{eqnarray}
(we remind that $\leafset(.)$ denotes the set of leaves of a tree; the index in the local proportion of positive example $p^+_{j\leaf}$ reminds that weights used need to be boosting's weights) which we note can be directly approximated on the Poincar\'e disk by looking at the leaf nearest to the border of the disk because the maximal absolute confidence in the DT is also the maximal absolute confidence in its MDT, we obtain the sought minimum,
\begin{eqnarray*}
\alpha^*_j & = & \frac{1}{({\canolog})^*_j} \cdot \log\left(\frac{1+r_j}{1-r_j}\right), 
\end{eqnarray*}
where $r_j \in [-1,1]$ is the normalized edge
\begin{eqnarray}
  r_j & \defeq & \frac{1}{\sum_i w_{ji}} \cdot \sum_{i\in [m]}  w_{ji} \cdot \frac{y_i H_j(\ve{x}_i)}{\max_k |H_j(\ve{x}_k)|}\nonumber\\
  & = & \expect_{i\sim \tilde{w}_j} \left[\frac{1}{({\canolog})^*_j} \cdot \log\left(\frac{p^+_{j\leaf(\ve{x}_i)}}{1-p^+_{j\leaf(\ve{x}_i)}}\right)\right],
\end{eqnarray}
where $\tilde{\ve{w}}_j$ indicates normalized weights. It is not hard to show that because we use the local posterior $p^+_{j\leaf}$ at each leaf, $\alpha^*_j \geq 0$ and also $r_j \geq 0$. Hence, everything is like if we had an imaginary node $\node_j$ with $p^+_{j\node_j} \defeq (1+r_j)/2$ ($\geq 1/2$) and positive confidence
\begin{eqnarray*}
  ({\canolog})_j & \defeq & {\canolog}(p^+_{j\node_j}) = \log\left(\frac{p^+_{j\node_j}}{1-p^+_{j\node_j}}\right)
\end{eqnarray*}
that we can display in Poincar\'e disk (we choose to do it as a circle, see Figure \ref{fig:hyperbolicsum}, main file). We deduce from \eqref{defBOOSTH} that
\begin{eqnarray*}
u_j(\ve{x}) & = & \frac{({\canolog})_j }{({\canolog})^*_j} \cdot \log\left(\frac{p^+_{j\leaf(\ve{x})}}{1-p^+_{j\leaf(\ve{x})}}\right),
\end{eqnarray*}
and note that \textit{all} three key parameters can easily be displayed or computed directly from the Poincar\'e disk.

\subsection{Modifying Sarkar's embedding in Poincar\'e disk}\label{sarkar-mod}

We refer to the concise and neat description of Sarkar's embedding in \cite{sdgrRT} for the full algorithm. Our modification relies on changing one step of the algorithm, as described in Figure \ref{fig:sarkar}. The key step that we change is step 5: in the description of \cite[Algorithm 1]{sdgrRT}. This step embeds the children of a given node (and then the algorithm proceeds recursively until all nodes are processed). Sarkar's algorithm corresponds to the simple case where all arcs to/from a node define a fixed angle, which does not change during reflection because Poincar\'e model is conformal. Hence, if the tree is binary, this angle is $2\pi/3$, which provides a very clean display of the tree. In our case however, some children may have just one arc to a leaf while others may support big subtrees. Also, arc lengths can vary substiantially. We thus design the region of the disk into which the subtrees are going to be embedded by choosing an angle proportional to the number of leaves reachable from the node, and of course lengths have to match the difference between absolute confidence between a node and its children. There is no optimization step to learn a clean embedding, so we rely on a set of hyperparameters to effectively compute this new step of the algorithm.

\begin{figure}
  \centering 
\includegraphics[trim=20bp 30bp 150bp 60bp,clip,width=0.8\columnwidth]{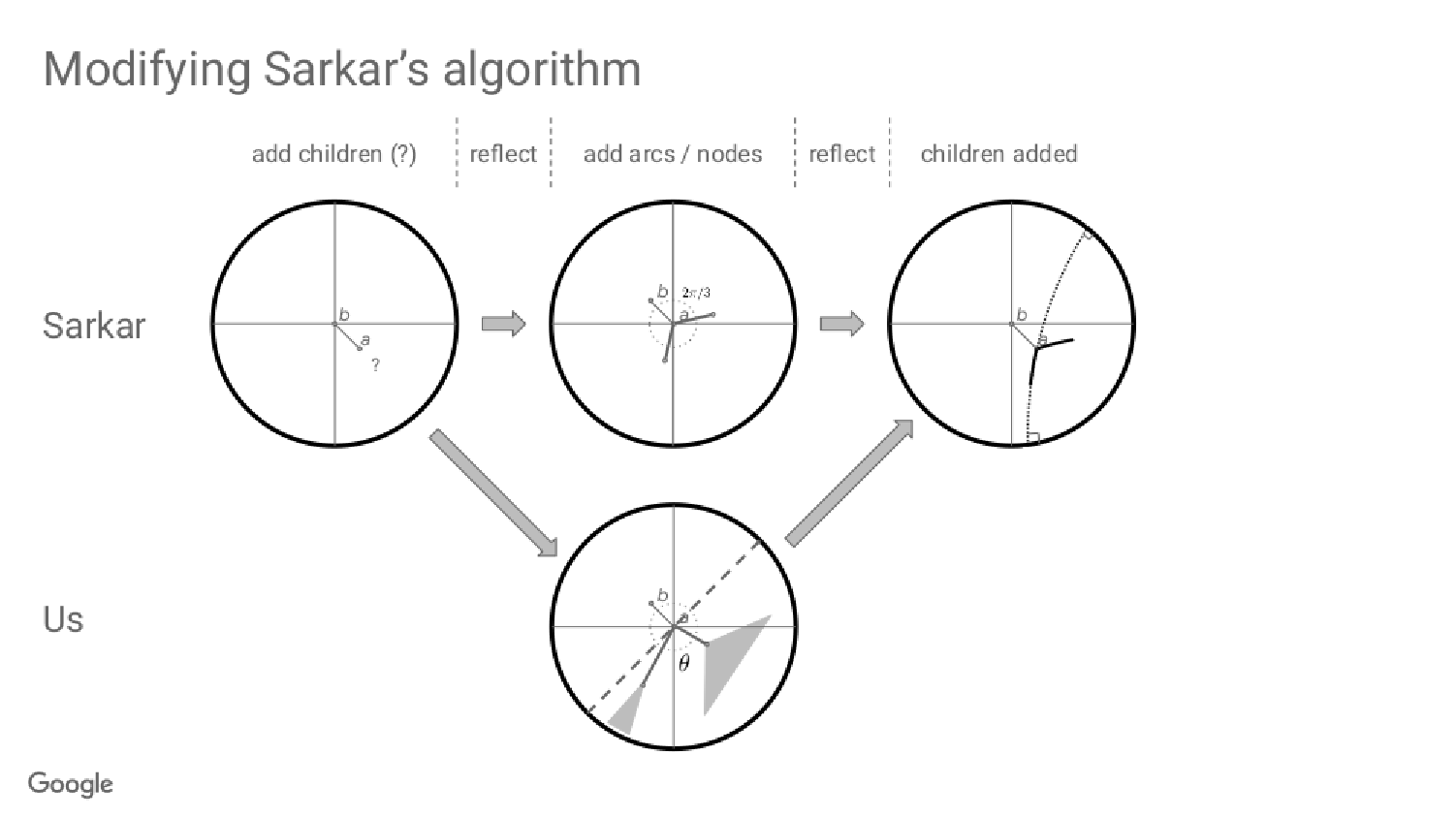} 
    \caption{Schematic description of our modification (bottom) of Sarkar's algorithm (up). Our modification solely changes the step in which the children of a node (here, $a$) are computed, this node having been reflected back to the origin. Instead of using a fixed angle and length to position arcs (and thus children of the node), the angle depends on the number of leaves that can be reached from a children and the length depends on the difference between absolute confidence between the node and the corresponding child. We also define an angle which represents the domain (before reflecting back) in which the embedding is going to take place, shown with the thick dashed line (here, this angle is $\pi$).}
    \label{fig:sarkar}
  \end{figure}

\section{Supplementary material on experiments}\label{sec-sup-exp}

\subsection{Domains}\label{sec-doms}

  \setlength\tabcolsep{5pt}

\begin{table}[h]
\begin{center}
\begin{tabular}{|crr|}
\hline \hline
Domain & \multicolumn{1}{c}{$m$} & \multicolumn{1}{c|}{$d$}  \\ \hline 
\domainname{breastwisc}  & 699 & 9 \\
\domainname{ionosphere} & 351 & 33  \\
\domainname{tictactoe} & 958 & 9 \\
\domainname{winered} & 1 599 & 11 \\
\domainname{german} & 1 000 & 20  \\
\domainname{analcatdata$\_$supreme}  & 4 053 & 8  \\
 \domainname{abalone}  & 4 177 & 8  \\
\domainname{qsar} & 1 055 & 41\\
\domainname{hillnoise} & 1 212 & 100\\
\domainname{firmteacher} & 10 800 & 16\\ 
\domainname{online$\_$shoppers$\_$intention} & 12 330 & 17\\ 
\domainname{give$\_$me$\_$some$\_$credit} & 120 269 & 11\\ 
Buzz$\_$in$\_$social$\_$media (Tom's \domainname{hardware}) & 28 179 & 96\\
Buzz$\_$in$\_$social$\_$media (\domainname{twitter}) & 583 250 & 78\\ \hline\hline       
\end{tabular}
\end{center}
\caption{UCI, OpenML (Analcatdata$\_$supreme) and Kaggle (Give$\_$me$\_$some$\_$credit) domains considered in our experiments ($m=$ total number
  of examples, $d=$ number of features), ordered in
  increasing $m \times n$.}
  \label{t-s-uci}
\end{table}
The domains we consider are all public domains, from the UCI repository of ML datasets \cite{dgUM}, OpenML, or Kaggle, see Table \ref{t-s-uci}.

\setlength\tabcolsep{0pt}

\subsection{Supplement on experiments -- Poincar\'e disk embeddings}\label{sec-exp-pdemb}

\begin{table}
  \centering
  \resizebox{\textwidth}{!}{\begin{tabular}{ccccc}\Xhline{2pt}
                              \includegraphics[trim=0bp 0bp 0bp 0bp,clip,width=0.3\columnwidth]{Experiments/poincare_DT/online-shopping-intentions/treeplot_Jan_15th__17h_51m_26s_USE_BOOSTING_WEIGHTS_Algo0_SplitCV4_Tree0.eps} & \includegraphics[trim=0bp 0bp 0bp 0bp,clip,width=0.3\columnwidth]{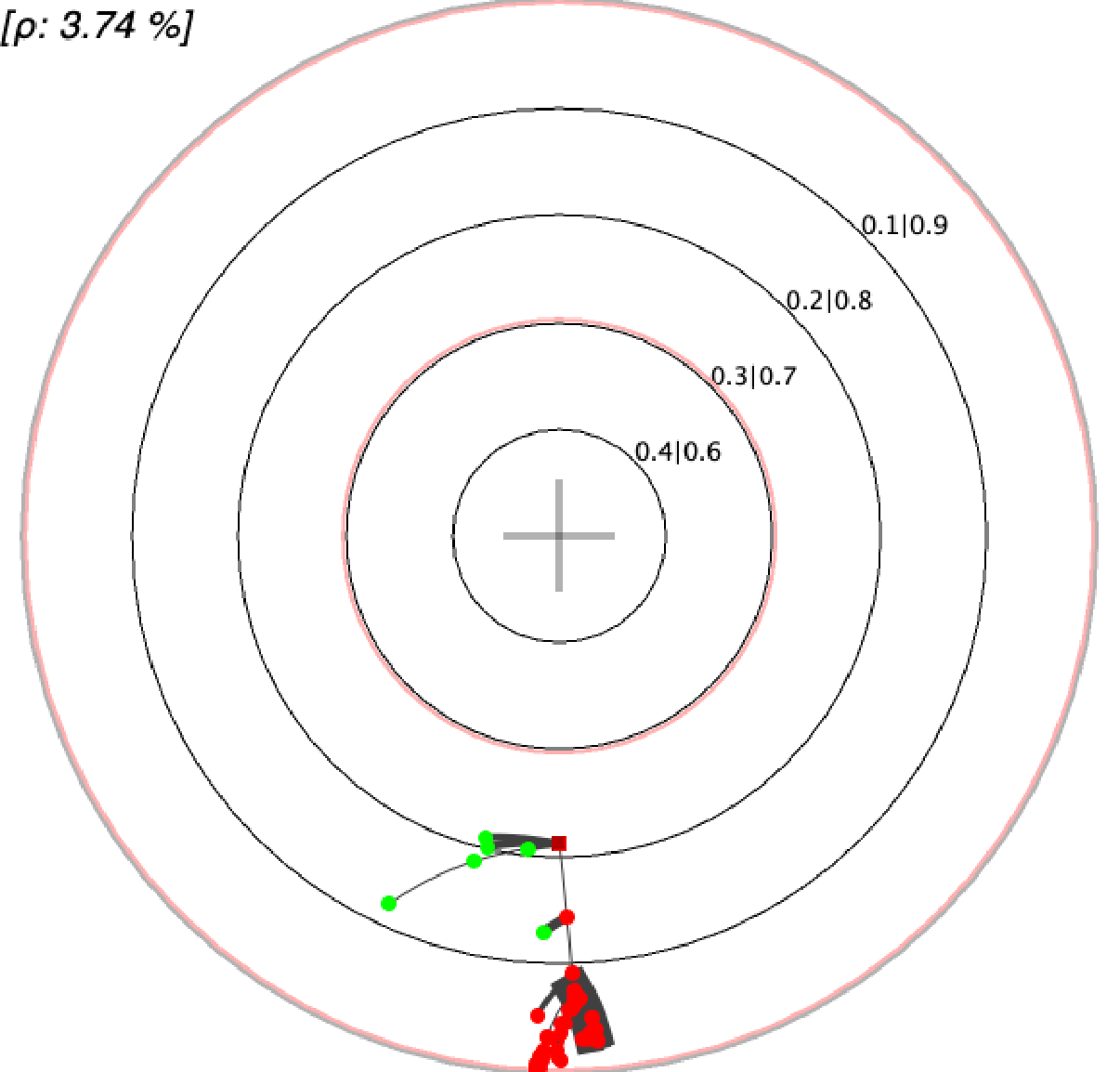} & \includegraphics[trim=0bp 0bp 0bp 0bp,clip,width=0.3\columnwidth]{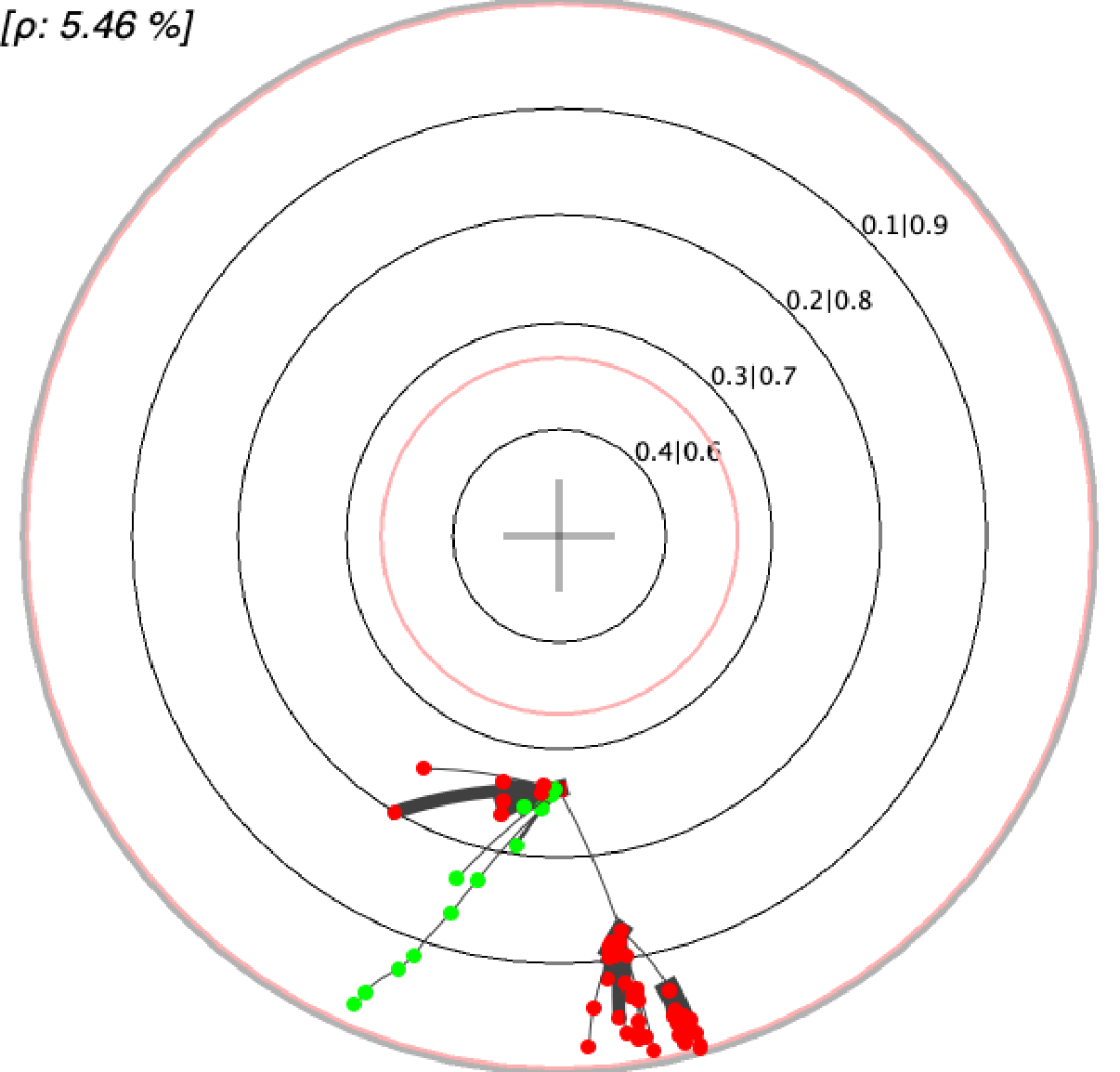}  & \includegraphics[trim=0bp 0bp 0bp 0bp,clip,width=0.3\columnwidth]{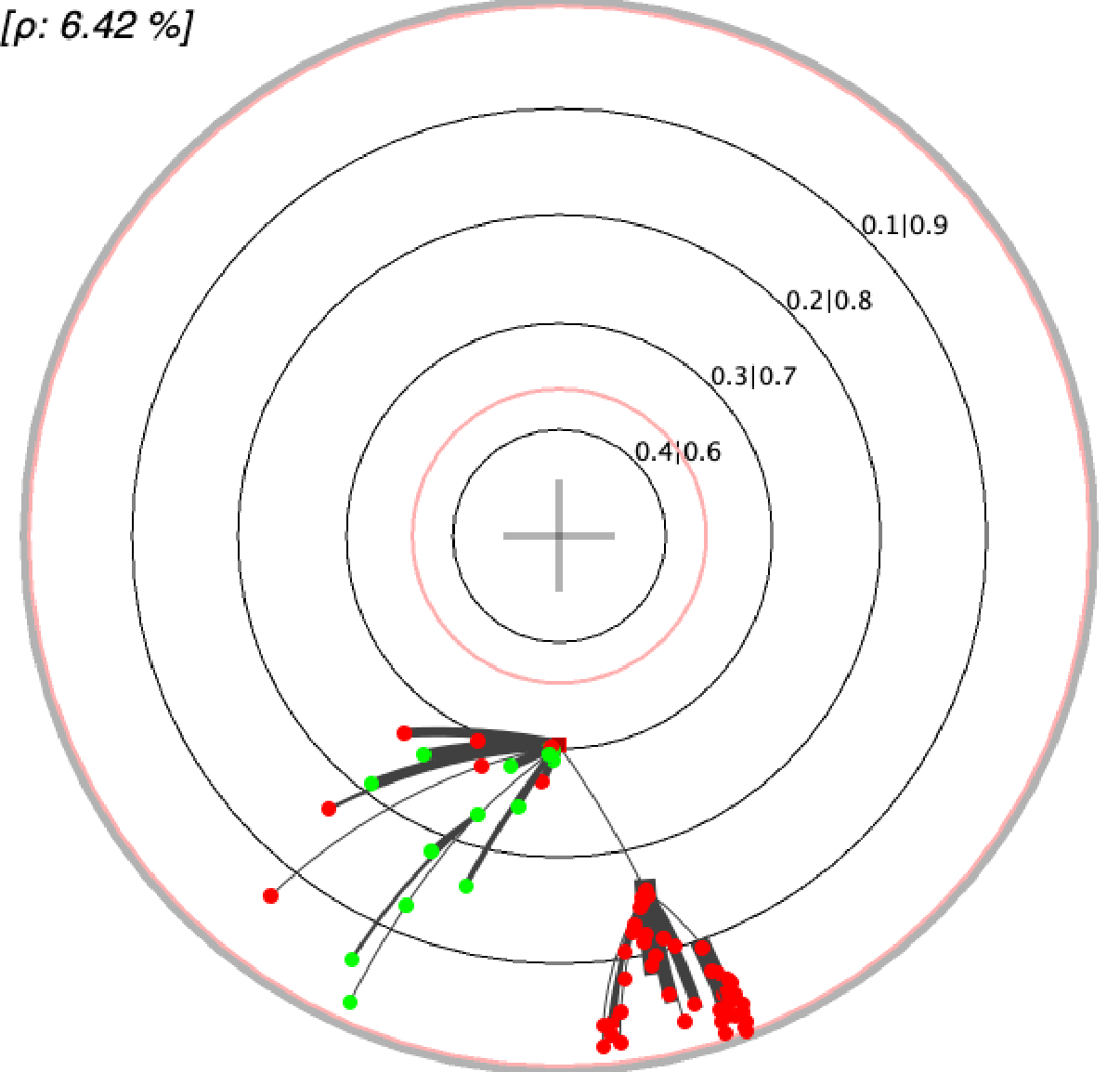}  & \includegraphics[trim=0bp 0bp 0bp 0bp,clip,width=0.3\columnwidth]{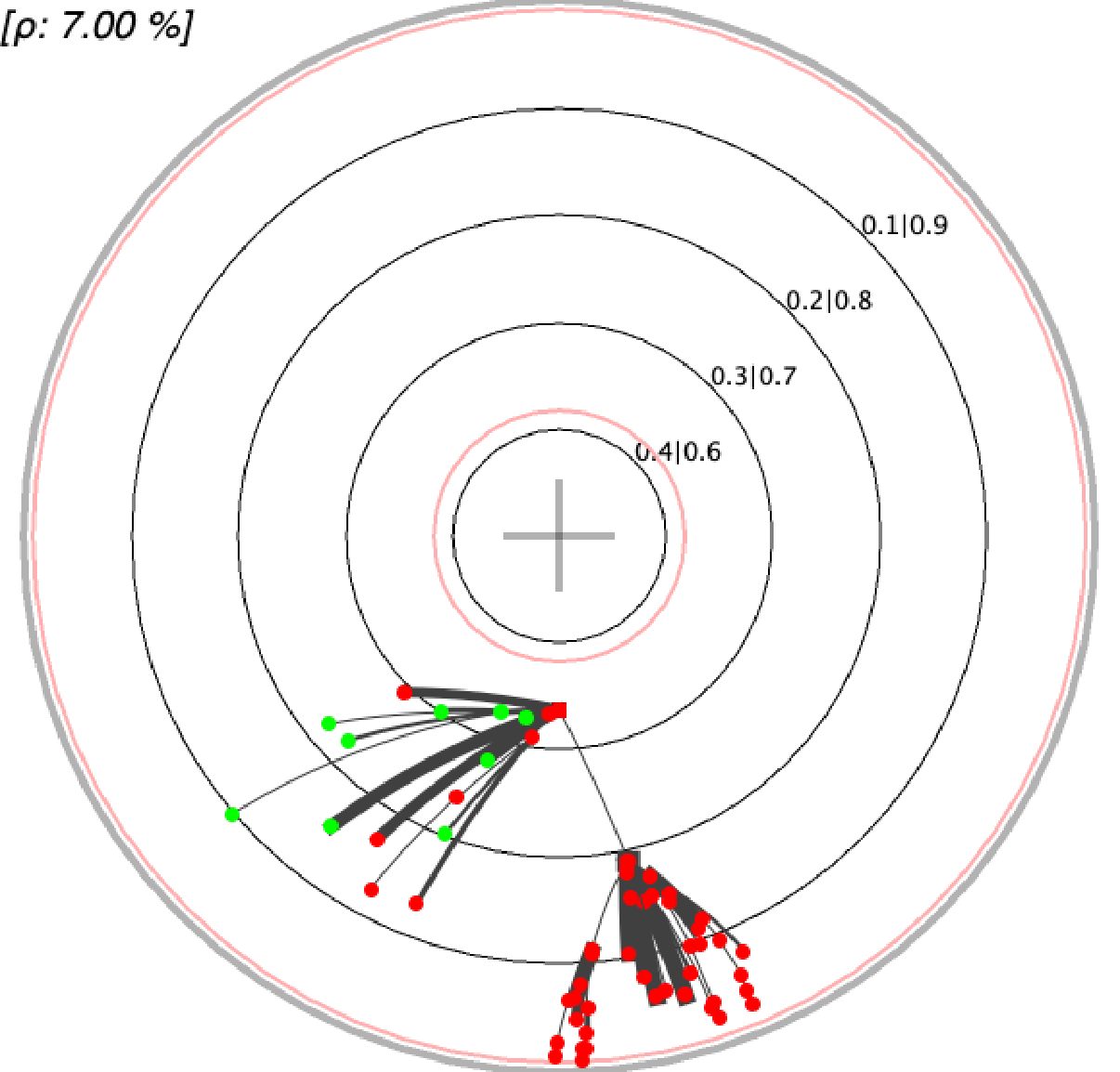} \\
                              MDT $\#1$ & MDT $\#2$& MDT $\#3$& MDT $\#4$& MDT $\#5$\\
                             \includegraphics[trim=0bp 0bp 0bp 0bp,clip,width=0.3\columnwidth]{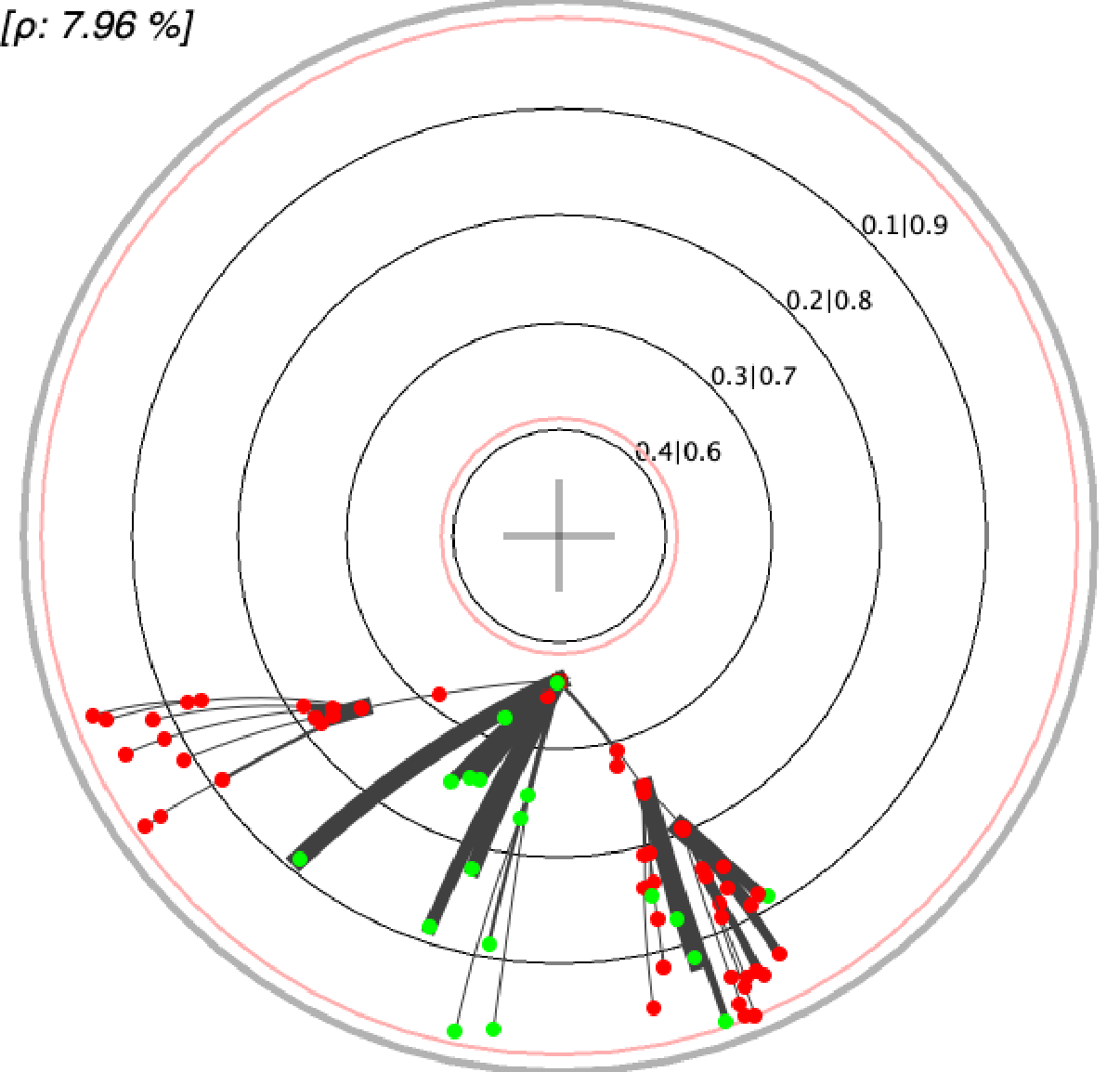} & \includegraphics[trim=0bp 0bp 0bp 0bp,clip,width=0.3\columnwidth]{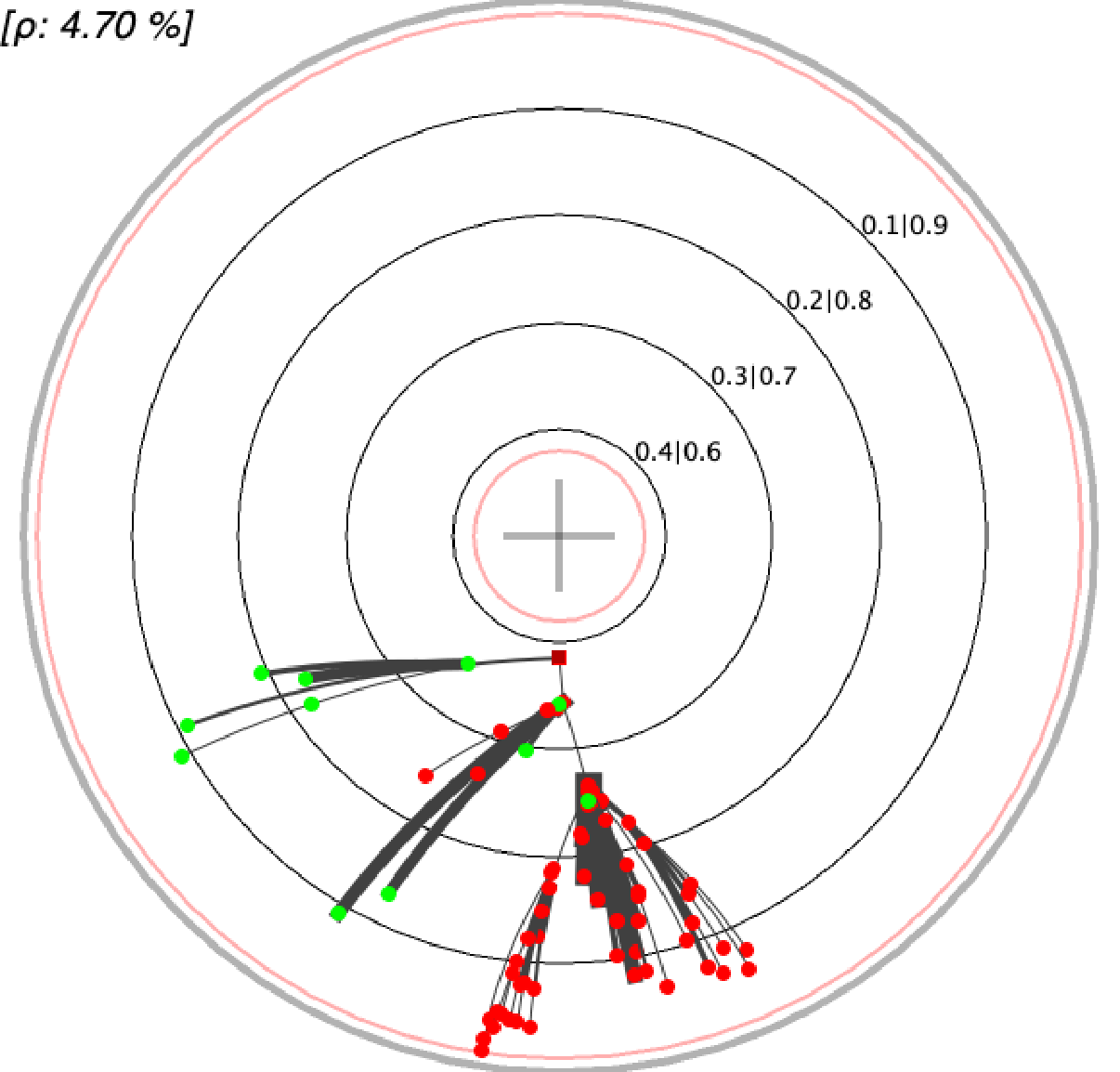} & \includegraphics[trim=0bp 0bp 0bp 0bp,clip,width=0.3\columnwidth]{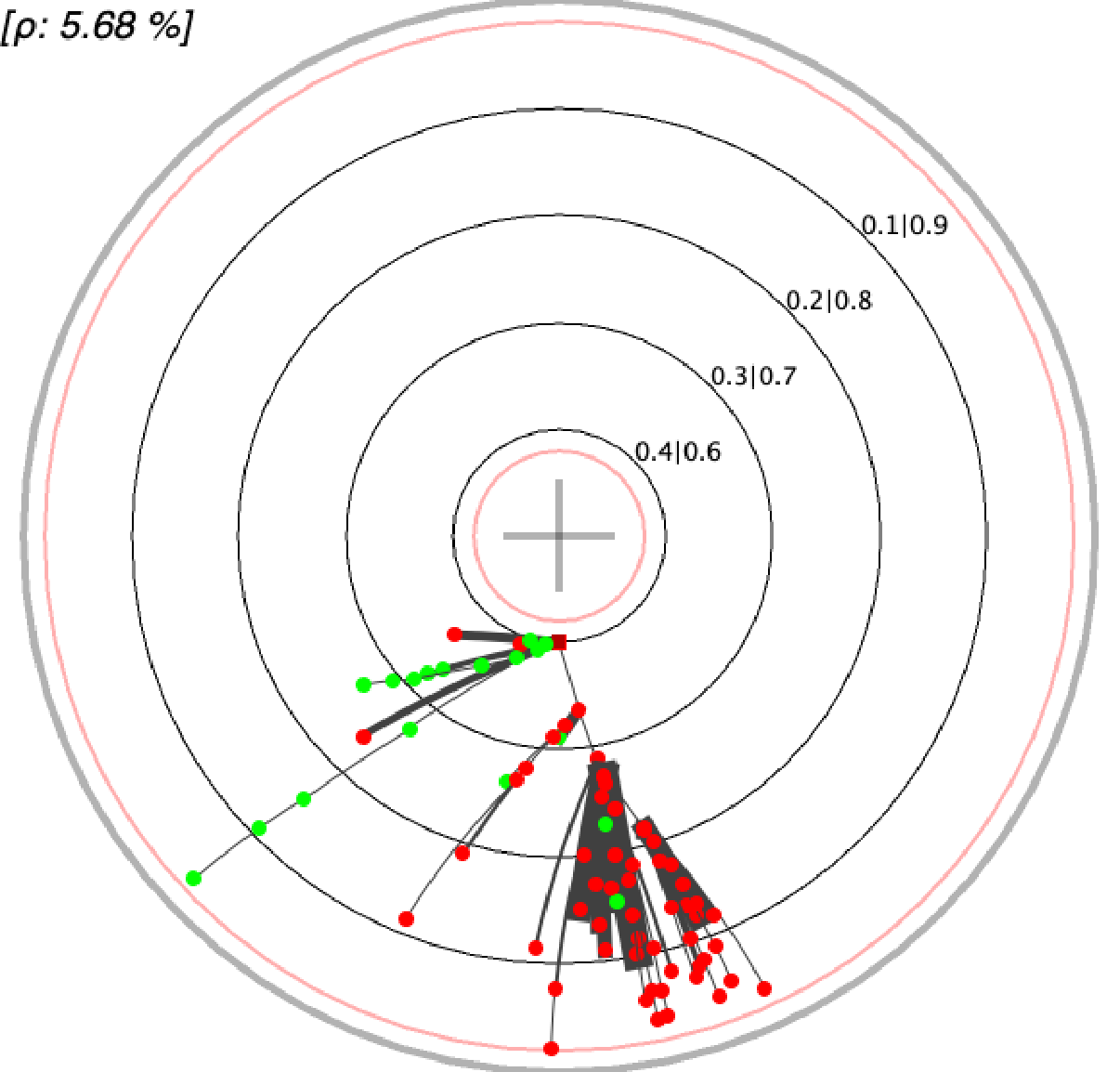}  & \includegraphics[trim=0bp 0bp 0bp 0bp,clip,width=0.3\columnwidth]{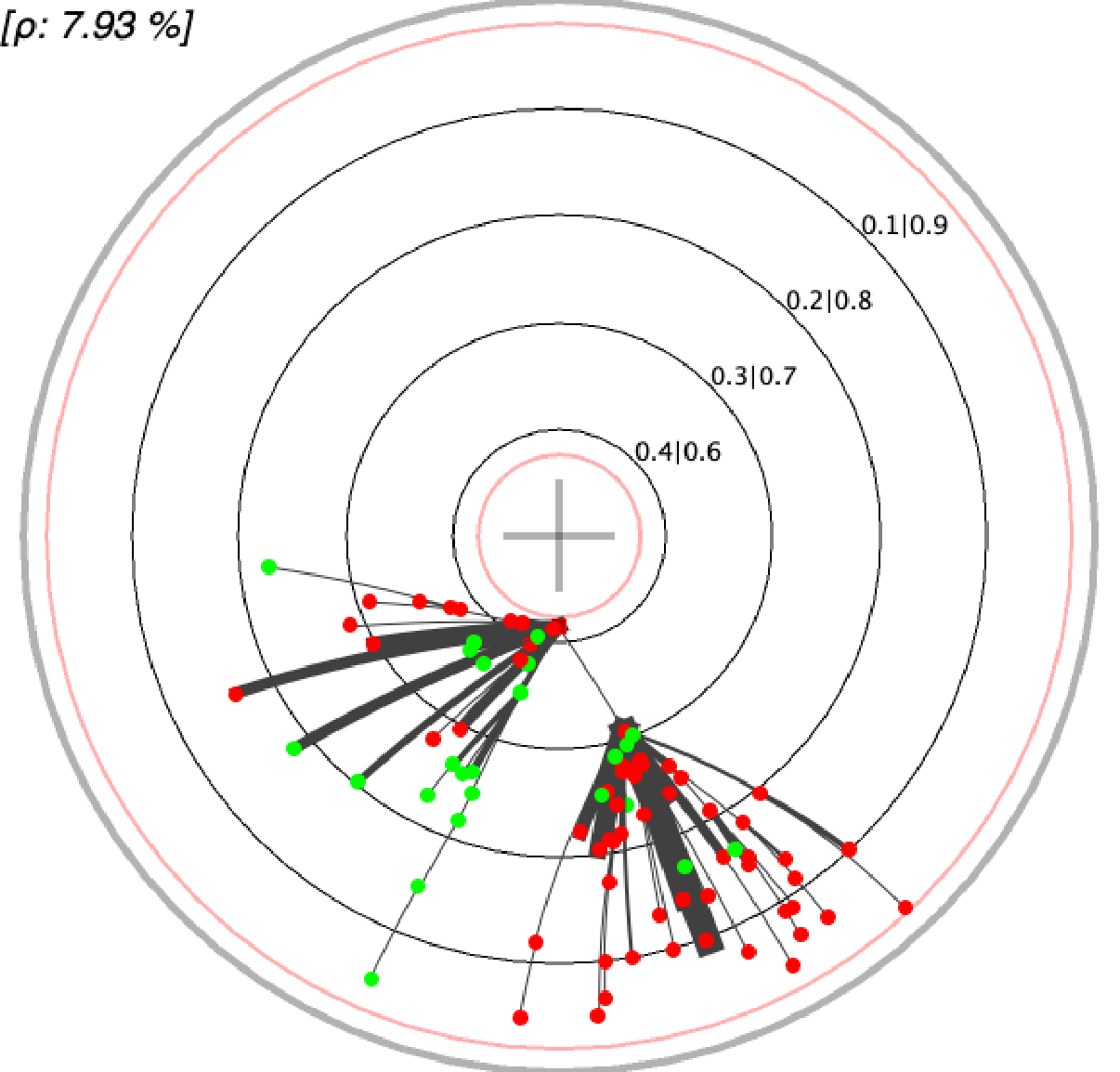}  & \includegraphics[trim=0bp 0bp 0bp 0bp,clip,width=0.3\columnwidth]{Experiments/poincare_DT/online-shopping-intentions/treeplot_Jan_15th__17h_51m_26s_USE_BOOSTING_WEIGHTS_Algo0_SplitCV4_Tree9.eps} \\
                              MDT $\#6$ & MDT $\#7$& MDT $\#8$& MDT $\#9$& MDT $\#10$\\  \Xhline{2pt} 
  \end{tabular}}
\caption{First 10 Monotonic Decision Trees (MDTs) embedded in Poincar\'e disk, corresponding to several Decision Trees (DTs) with 201 nodes each, learned by boosting the log / logistic-loss on UCI \domainname{online$\_$shoppers$\_$intention}. Isolines correspond to the node's prediction confidences, and geometric embedding errors ($\rho \%$) are indicated. See text for details.}
    \label{tab:online-shopping-intentions-dt-exerpt}
  \end{table}

  \begin{table}
  \centering
  \resizebox{\textwidth}{!}{\begin{tabular}{ccccc}\Xhline{2pt}
                              \includegraphics[trim=0bp 0bp 0bp 0bp,clip,width=0.3\columnwidth]{Experiments/poincare_DT/analcatdata-supreme/treeplot_Jan_17th__8h_49m_55s_USE_BOOSTING_WEIGHTS_Algo0_SplitCV3_Tree0.eps} & \includegraphics[trim=0bp 0bp 0bp 0bp,clip,width=0.3\columnwidth]{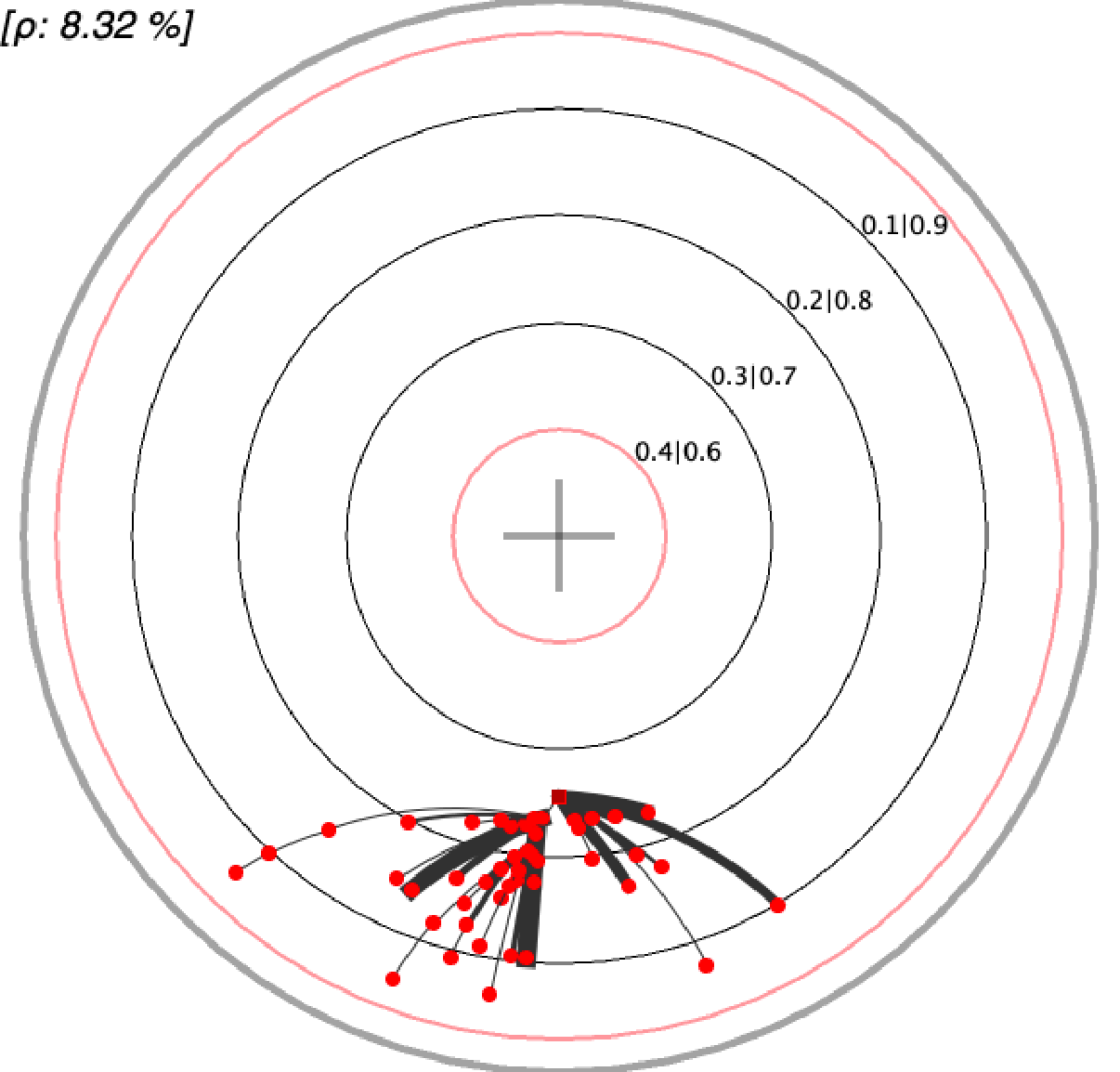} & \includegraphics[trim=0bp 0bp 0bp 0bp,clip,width=0.3\columnwidth]{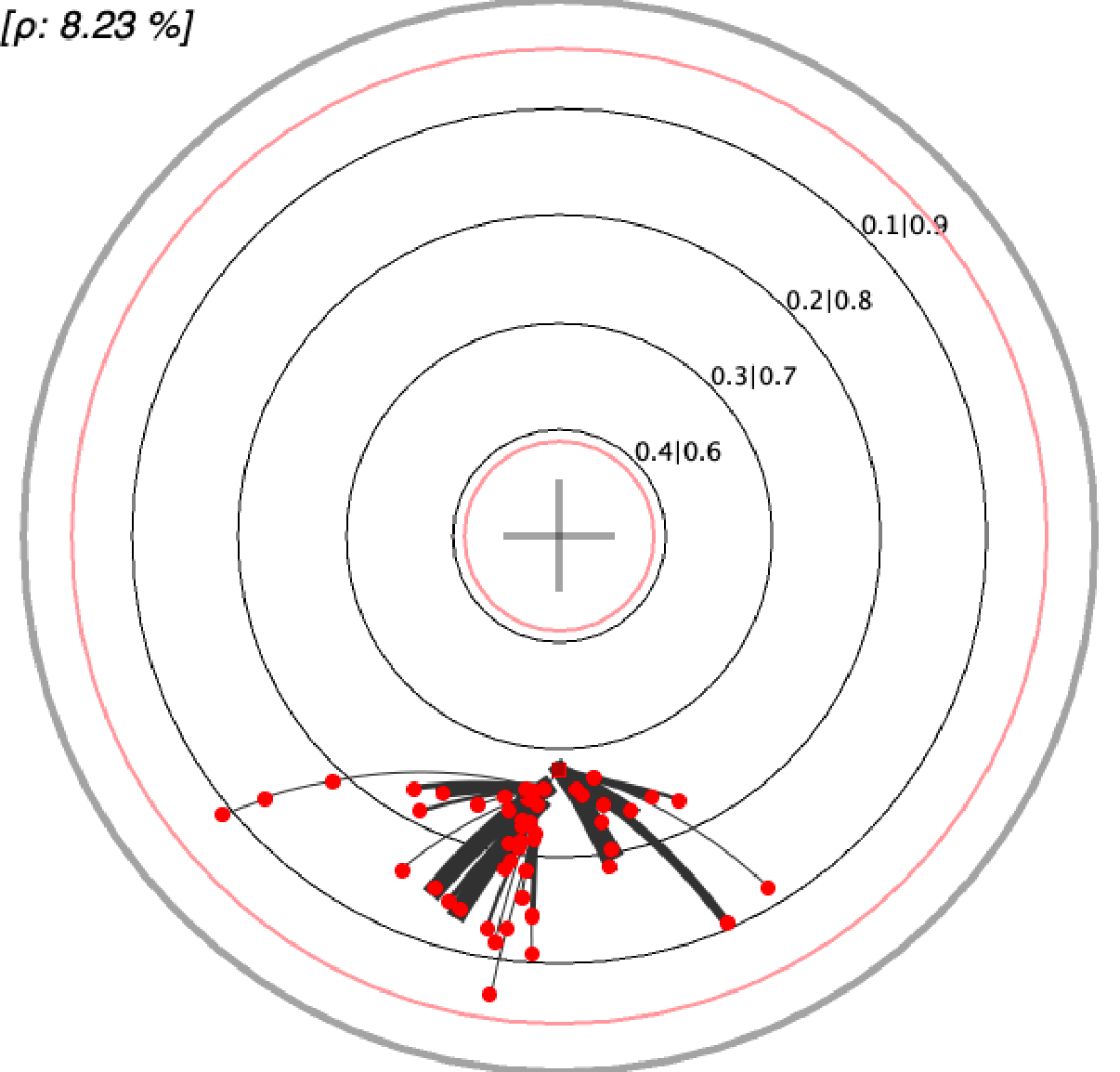}  & \includegraphics[trim=0bp 0bp 0bp 0bp,clip,width=0.3\columnwidth]{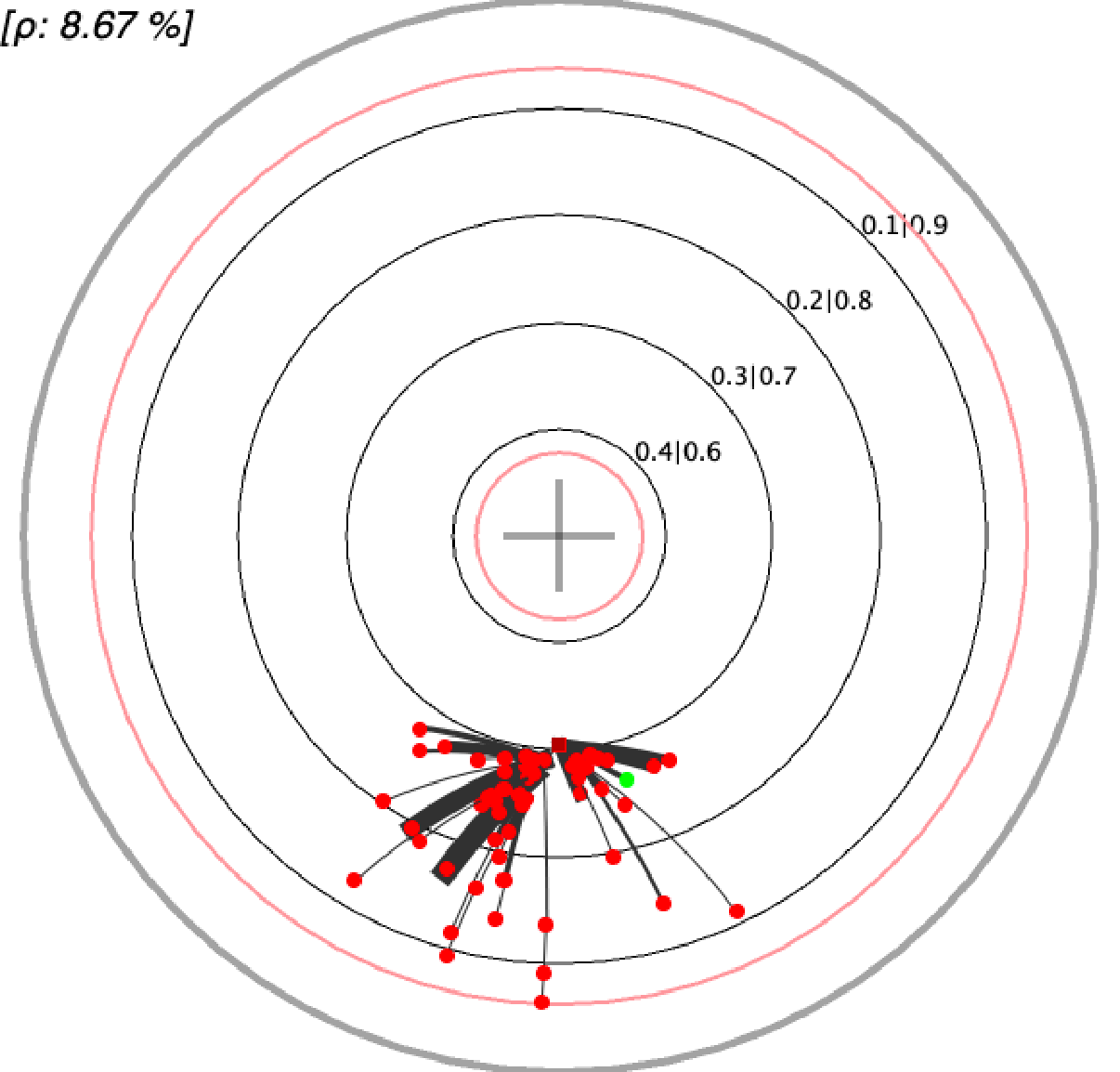}  & \includegraphics[trim=0bp 0bp 0bp 0bp,clip,width=0.3\columnwidth]{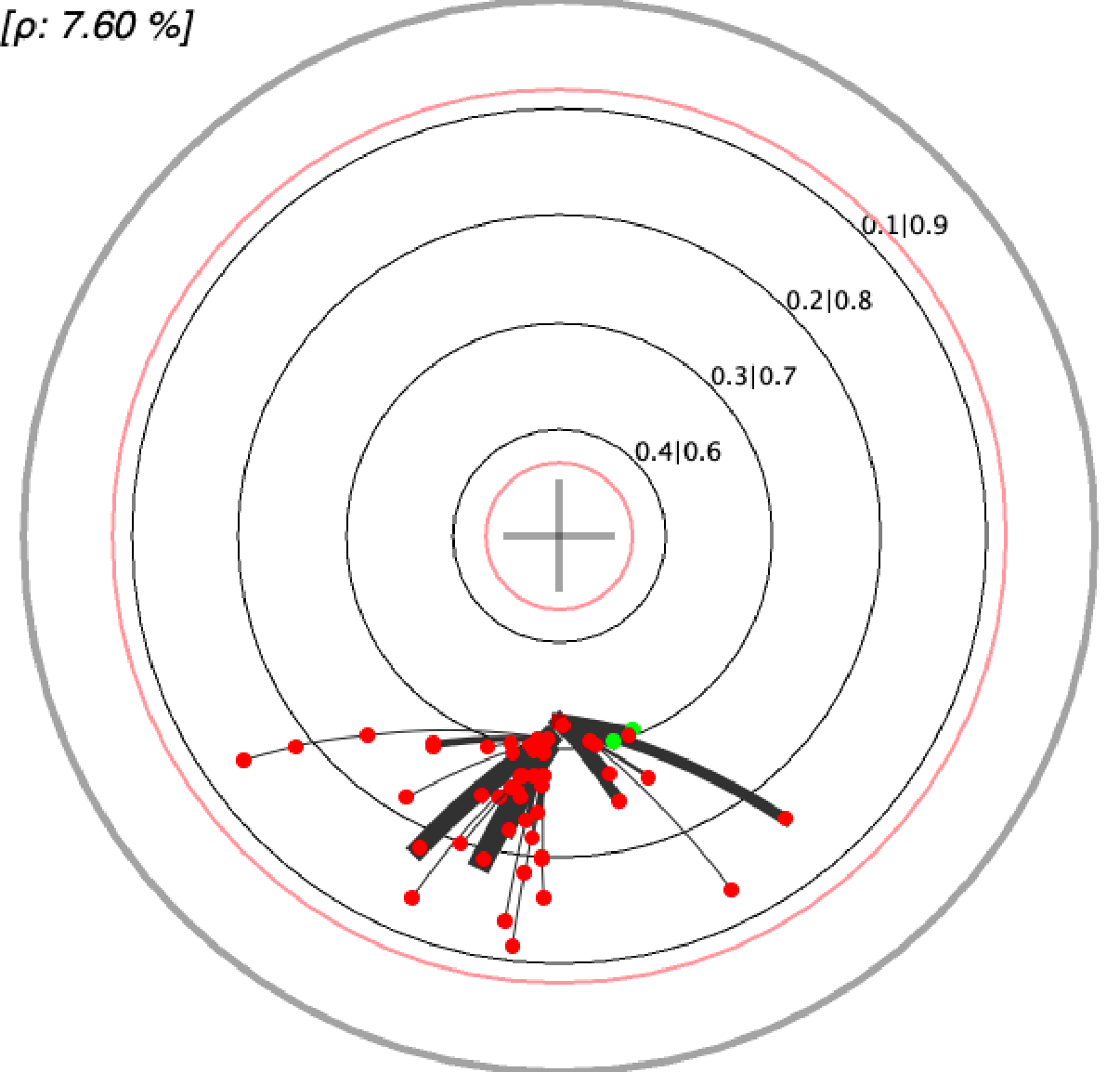} \\
                              MDT $\#1$ & MDT $\#2$& MDT $\#3$& MDT $\#4$& MDT $\#5$\\
                             \includegraphics[trim=0bp 0bp 0bp 0bp,clip,width=0.3\columnwidth]{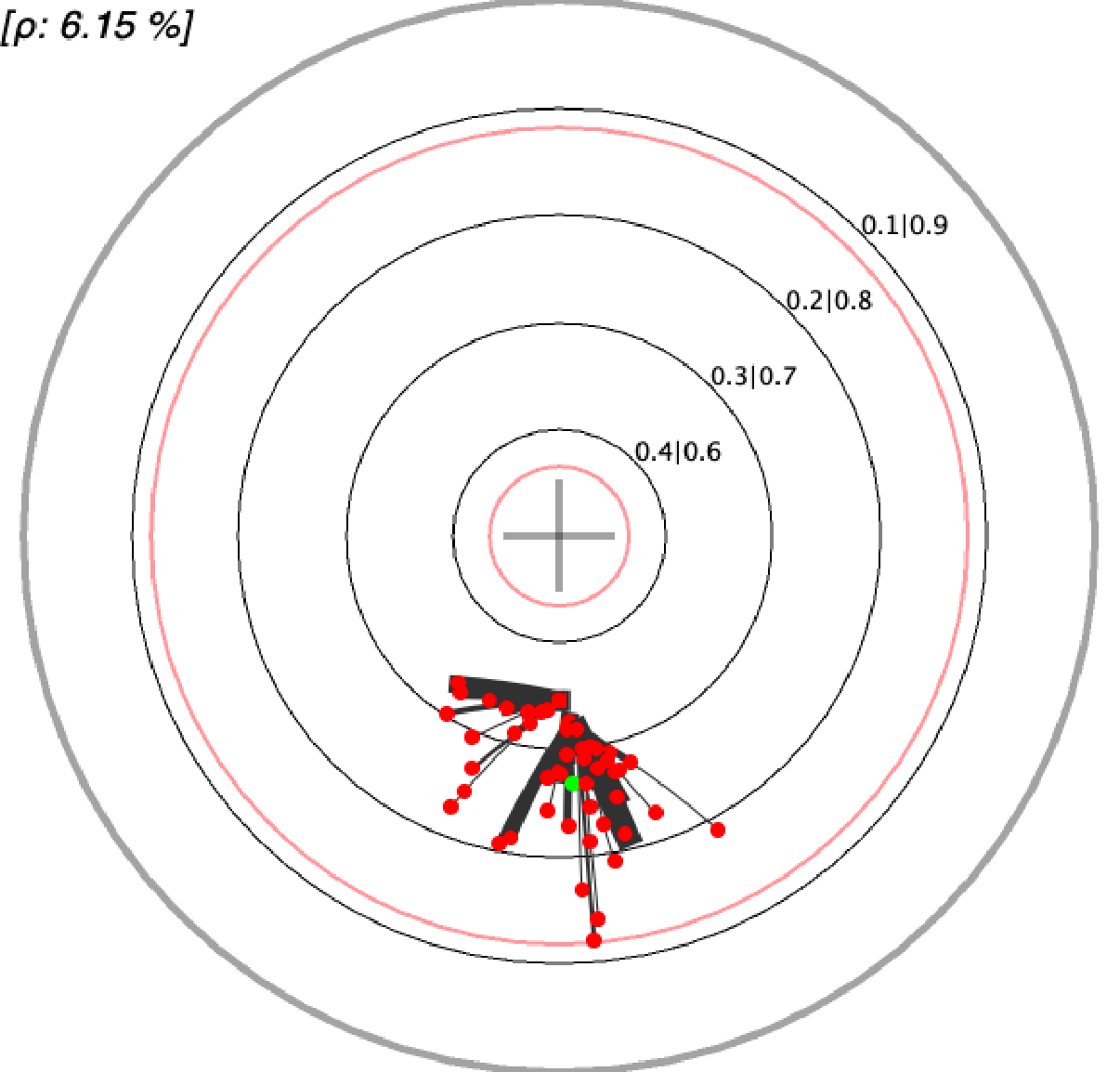} & \includegraphics[trim=0bp 0bp 0bp 0bp,clip,width=0.3\columnwidth]{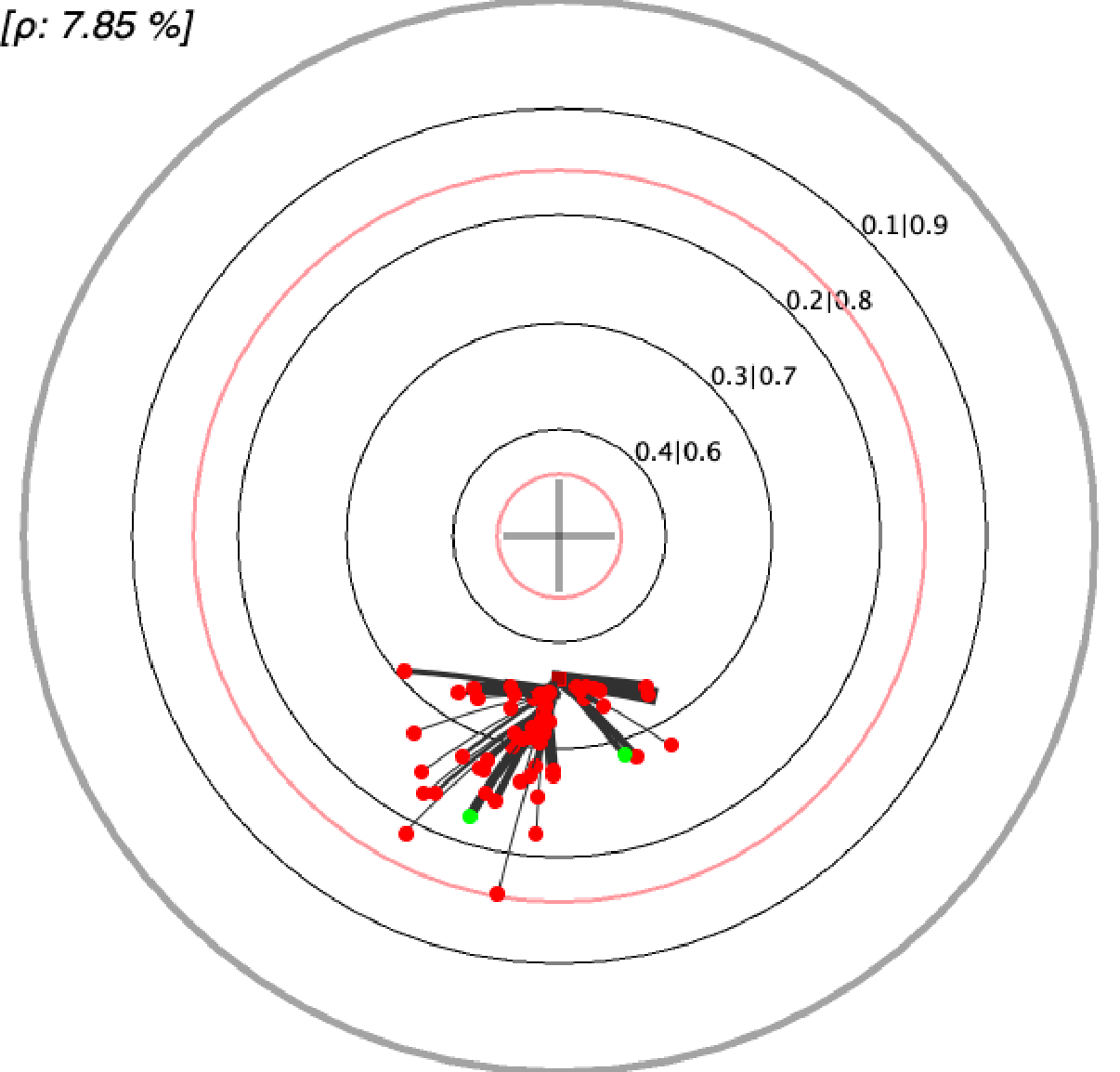} & \includegraphics[trim=0bp 0bp 0bp 0bp,clip,width=0.3\columnwidth]{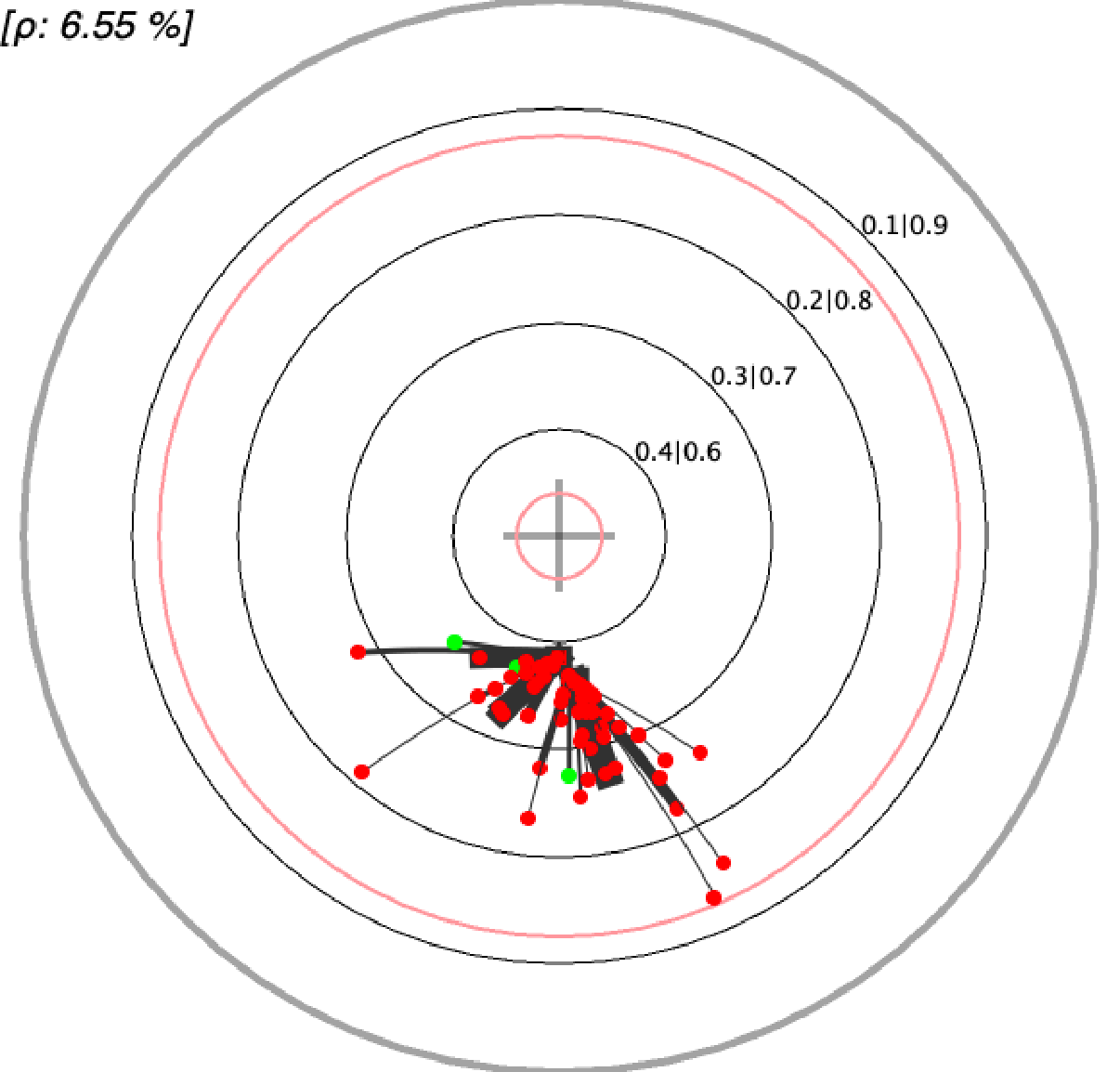}  & \includegraphics[trim=0bp 0bp 0bp 0bp,clip,width=0.3\columnwidth]{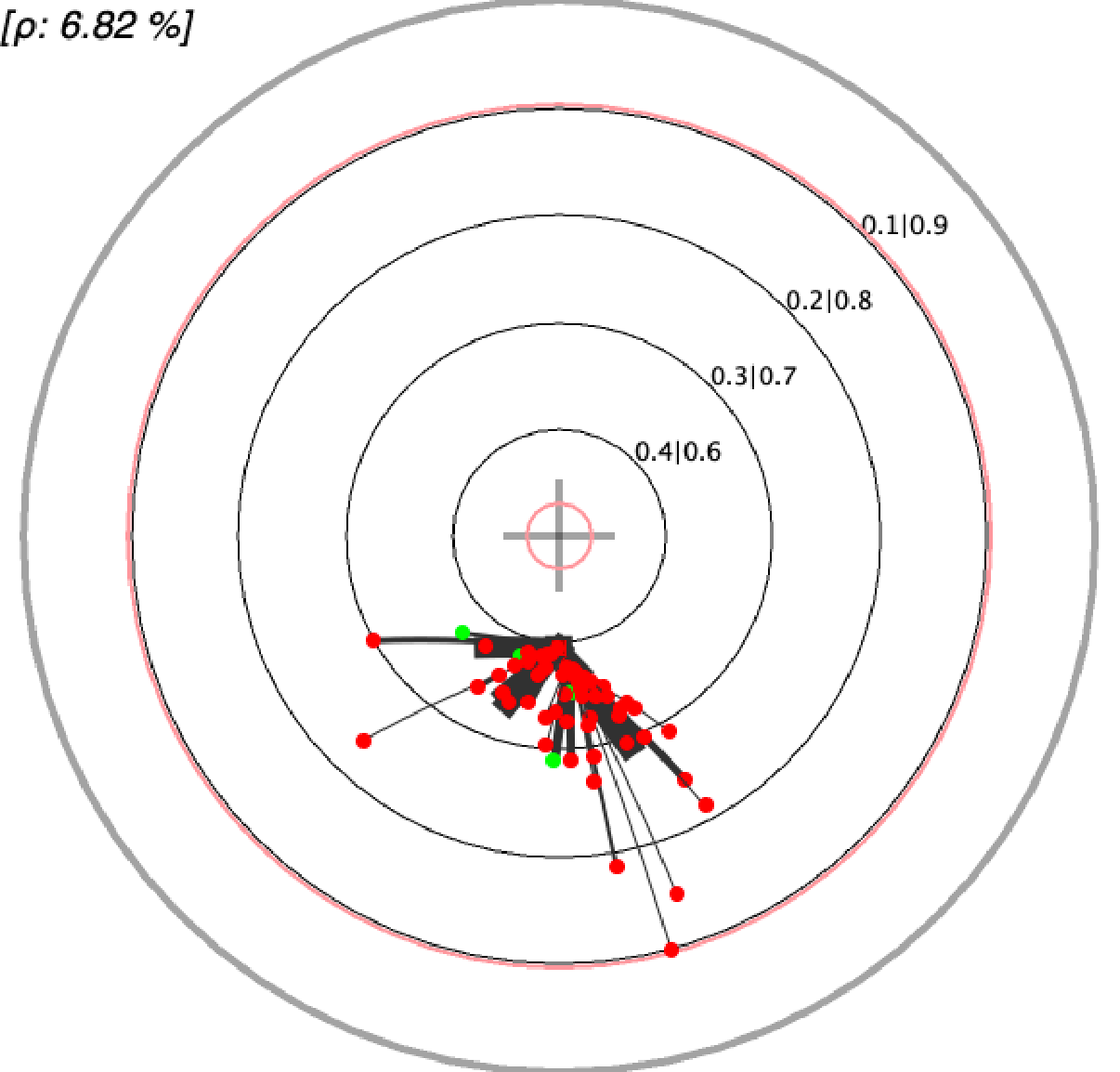}  & \includegraphics[trim=0bp 0bp 0bp 0bp,clip,width=0.3\columnwidth]{Experiments/poincare_DT/analcatdata-supreme/treeplot_Jan_17th__8h_49m_55s_USE_BOOSTING_WEIGHTS_Algo0_SplitCV3_Tree9.eps} \\
                              MDT $\#6$ & MDT $\#7$& MDT $\#8$& MDT $\#9$& MDT $\#10$\\  \Xhline{2pt} 
  \end{tabular}}
\caption{First 10 MDTs for UCI \domainname{analcatdata$\_$supreme}. Convention follows Table \ref{tab:online-shopping-intentions-dt-exerpt}.}
    \label{tab:analcatdata-supreme-dt-exerpt}
  \end{table}
  
  \begin{table}
  \centering
  \resizebox{\textwidth}{!}{\begin{tabular}{ccccc}\Xhline{2pt}
                              \includegraphics[trim=0bp 0bp 0bp 0bp,clip,width=0.3\columnwidth]{Experiments/poincare_DT/abalone/treeplot_Jan_17th__8h_54m_58s_USE_BOOSTING_WEIGHTS_Algo0_SplitCV5_Tree0.eps} & \includegraphics[trim=0bp 0bp 0bp 0bp,clip,width=0.3\columnwidth]{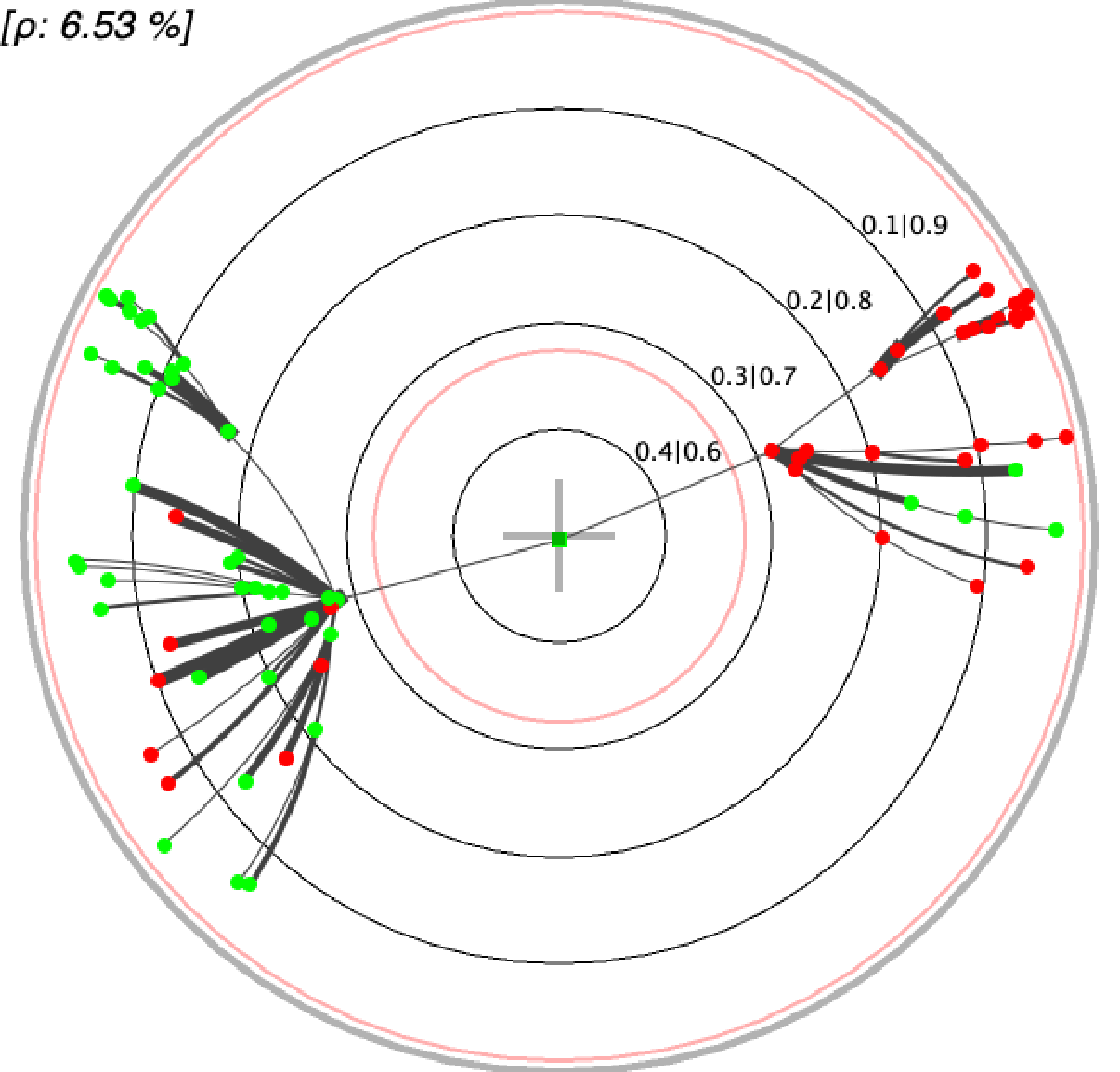} & \includegraphics[trim=0bp 0bp 0bp 0bp,clip,width=0.3\columnwidth]{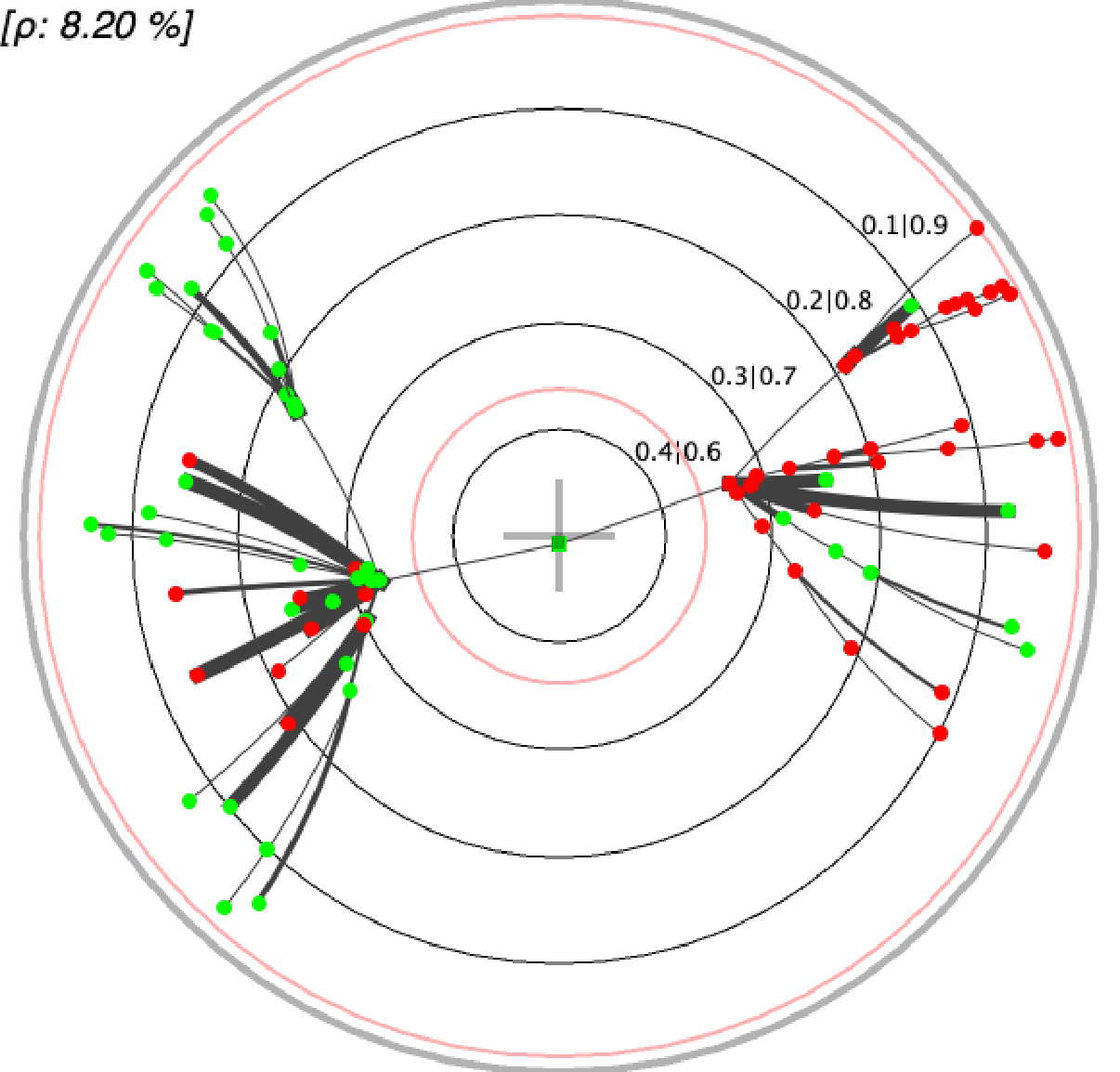}  & \includegraphics[trim=0bp 0bp 0bp 0bp,clip,width=0.3\columnwidth]{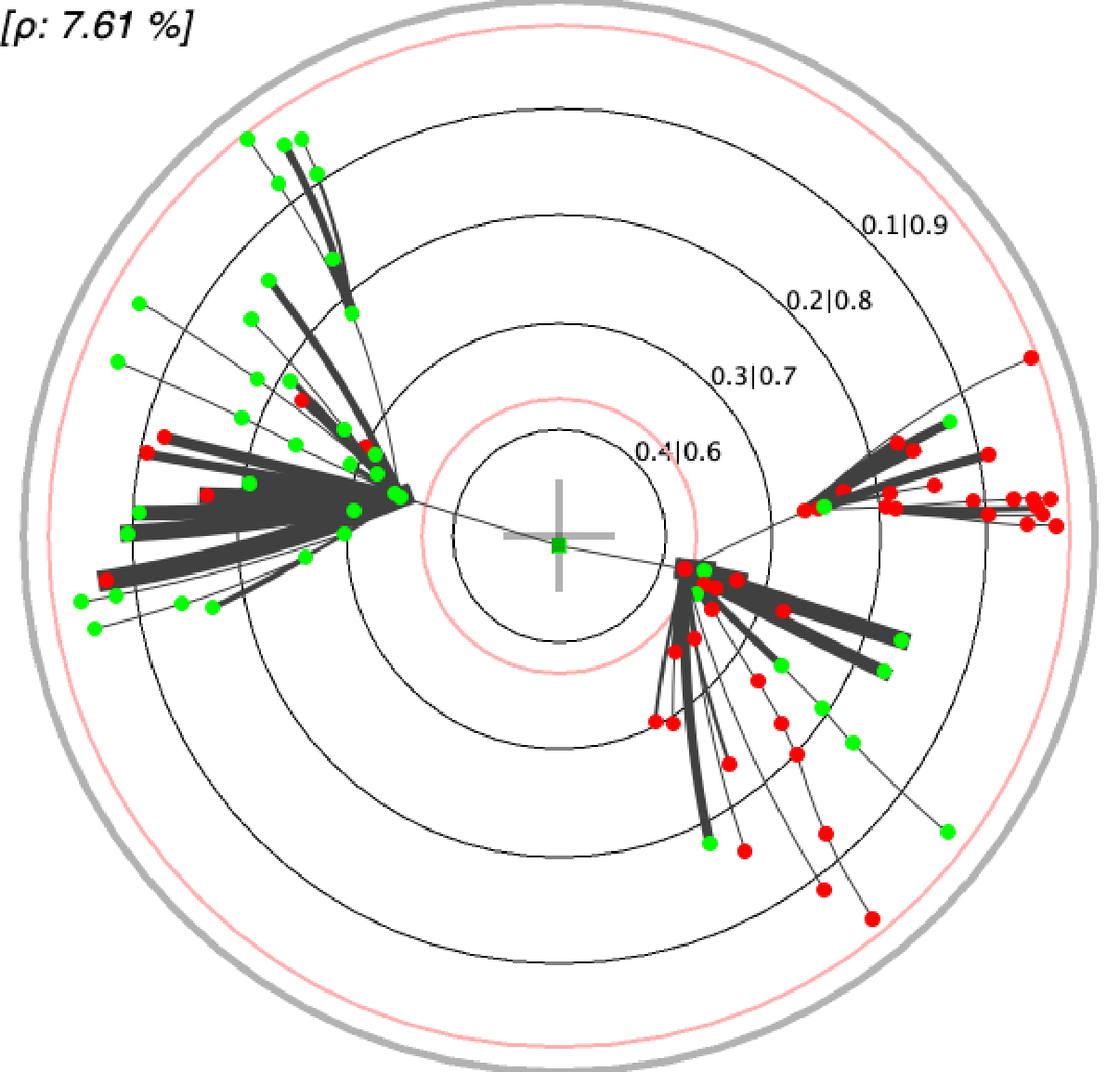}  & \includegraphics[trim=0bp 0bp 0bp 0bp,clip,width=0.3\columnwidth]{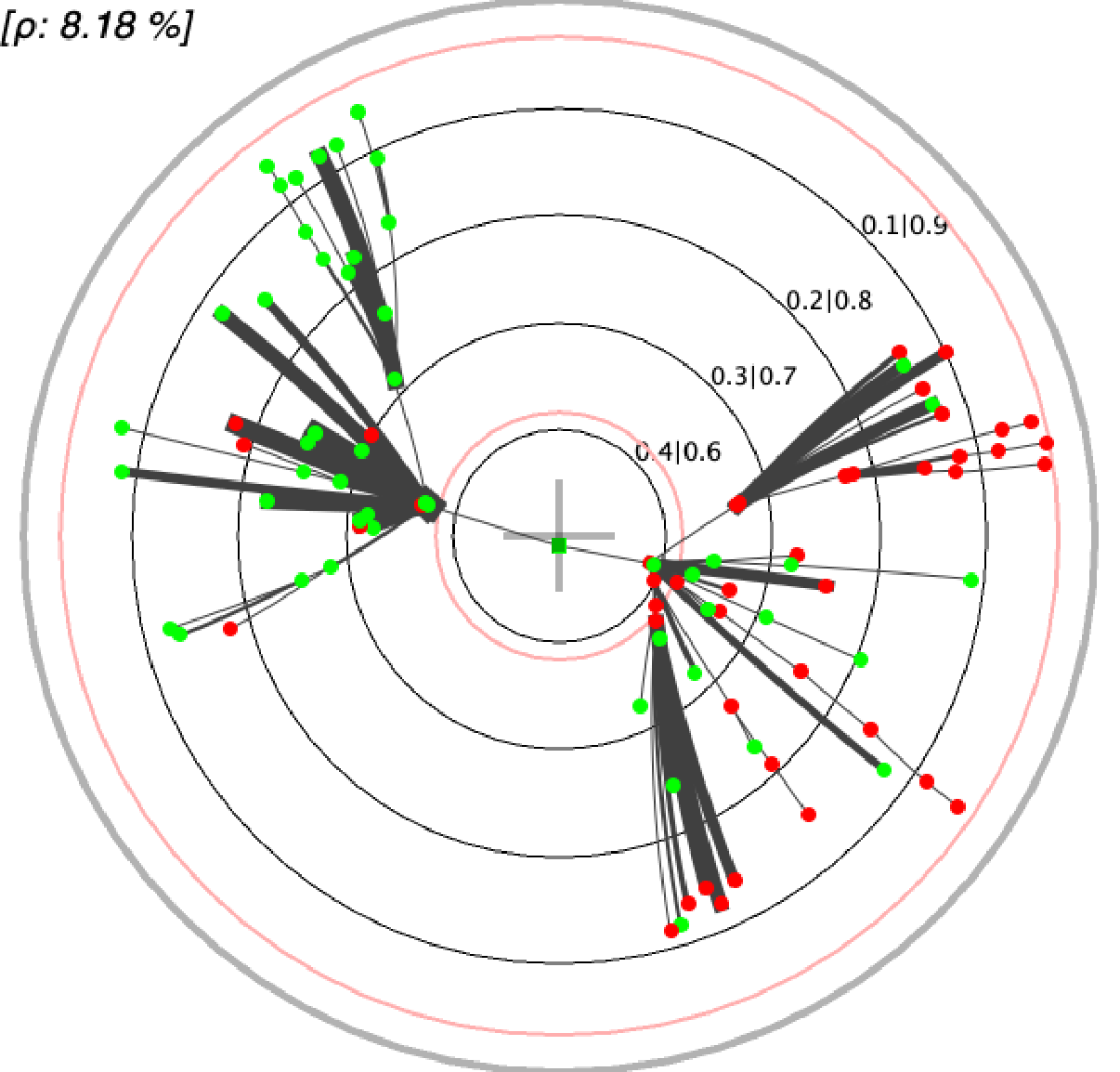} \\
                              MDT $\#1$ & MDT $\#2$& MDT $\#3$& MDT $\#4$& MDT $\#5$\\
                             \includegraphics[trim=0bp 0bp 0bp 0bp,clip,width=0.3\columnwidth]{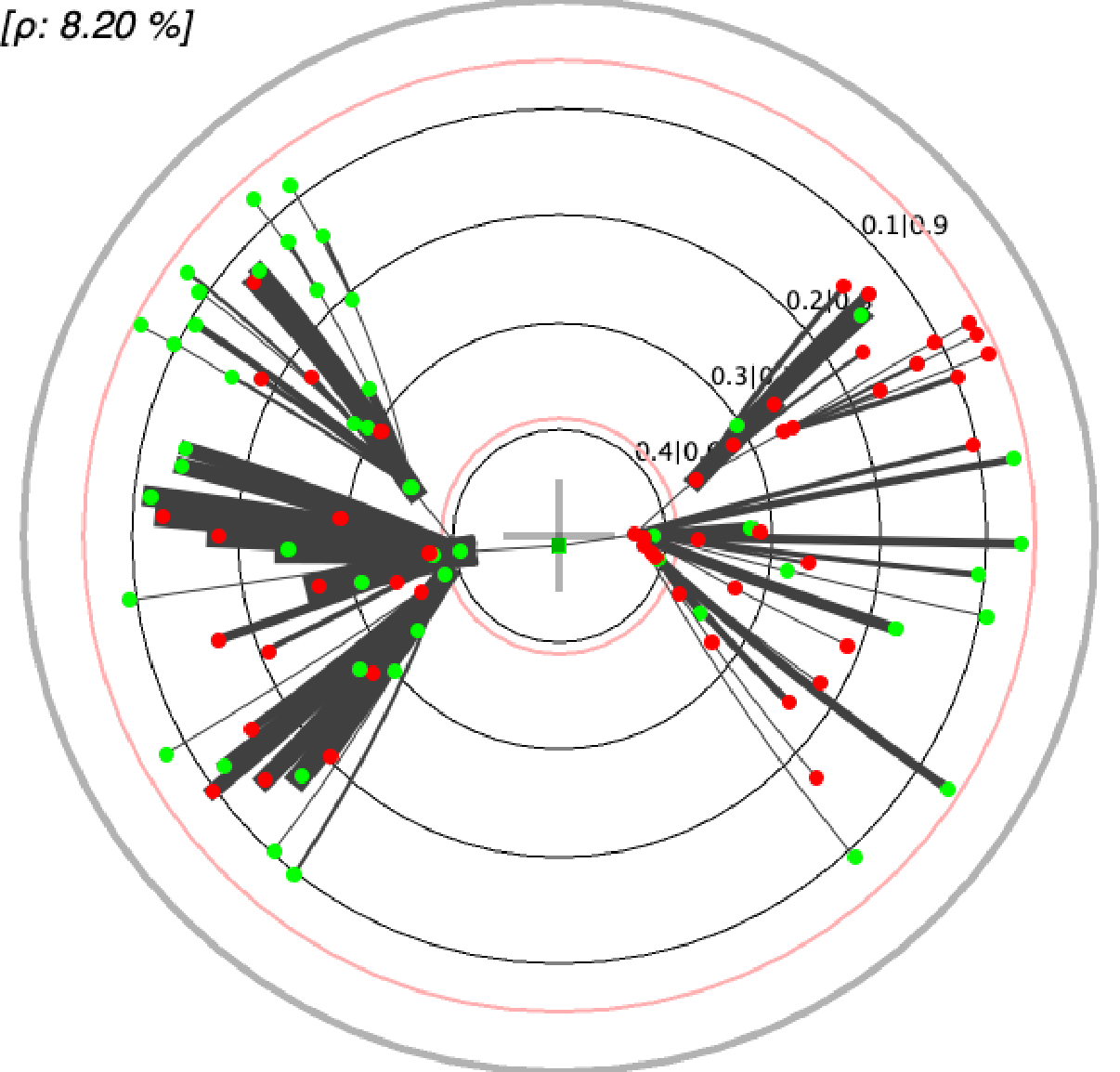} & \includegraphics[trim=0bp 0bp 0bp 0bp,clip,width=0.3\columnwidth]{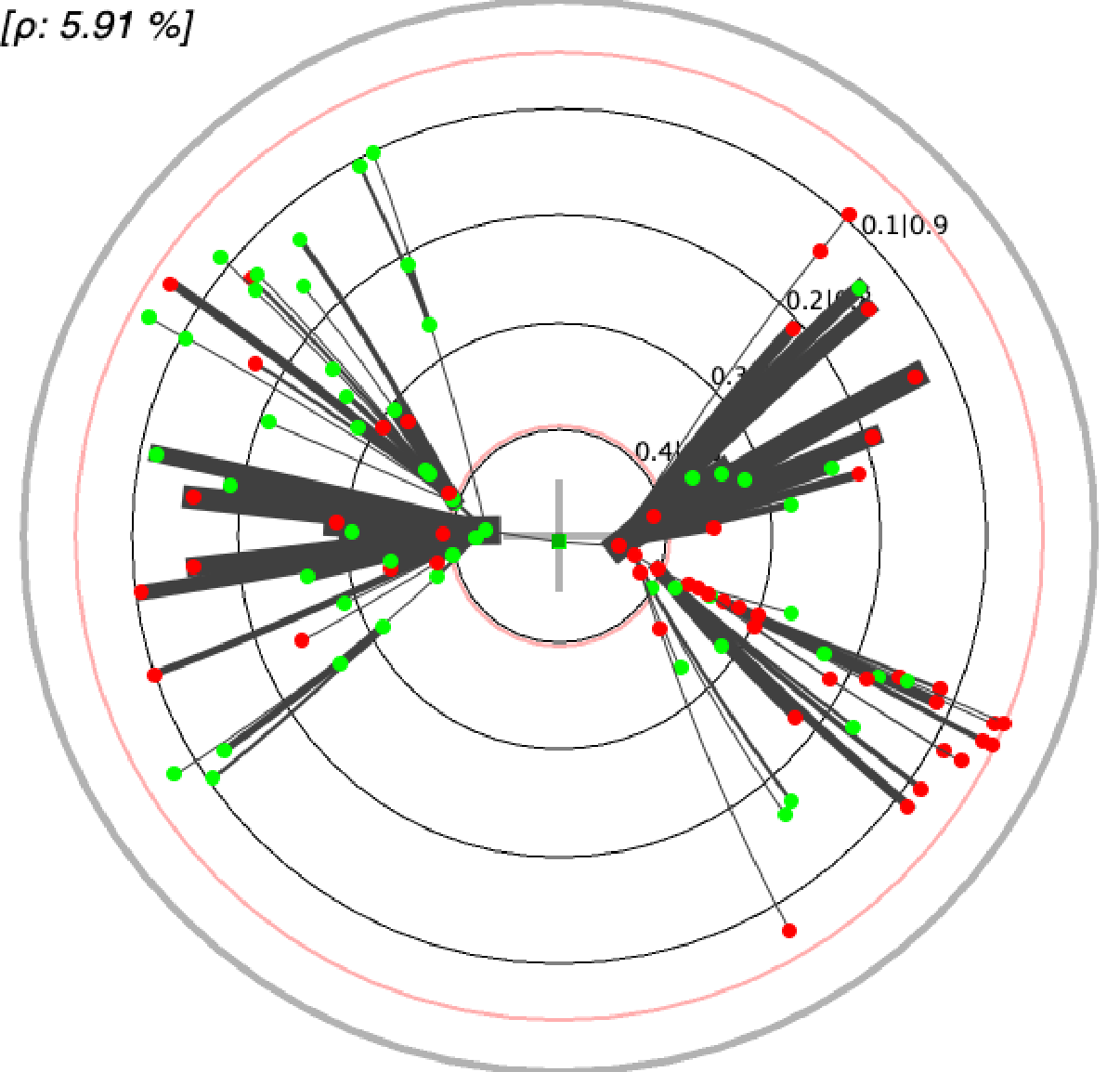} & \includegraphics[trim=0bp 0bp 0bp 0bp,clip,width=0.3\columnwidth]{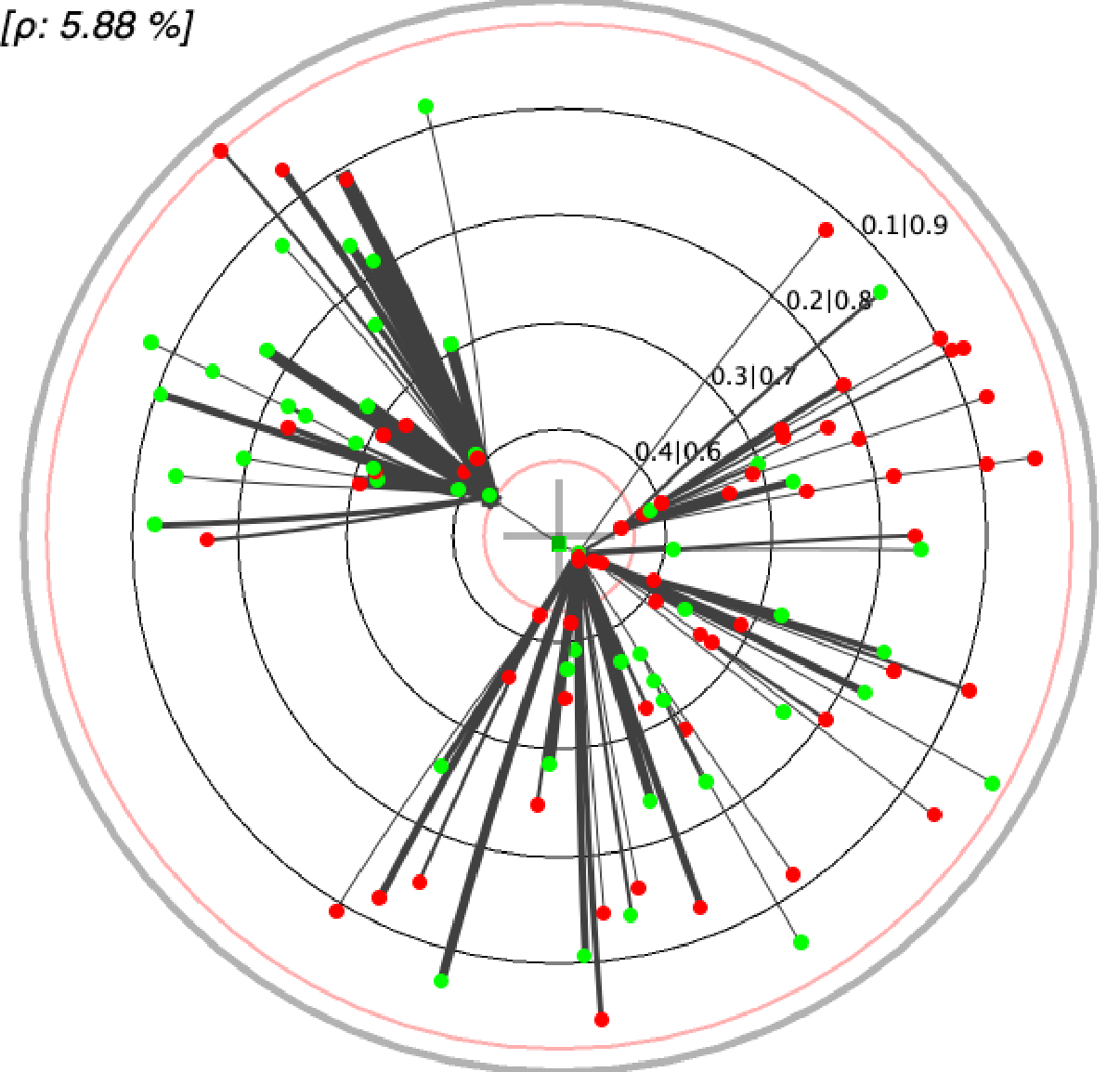}  & \includegraphics[trim=0bp 0bp 0bp 0bp,clip,width=0.3\columnwidth]{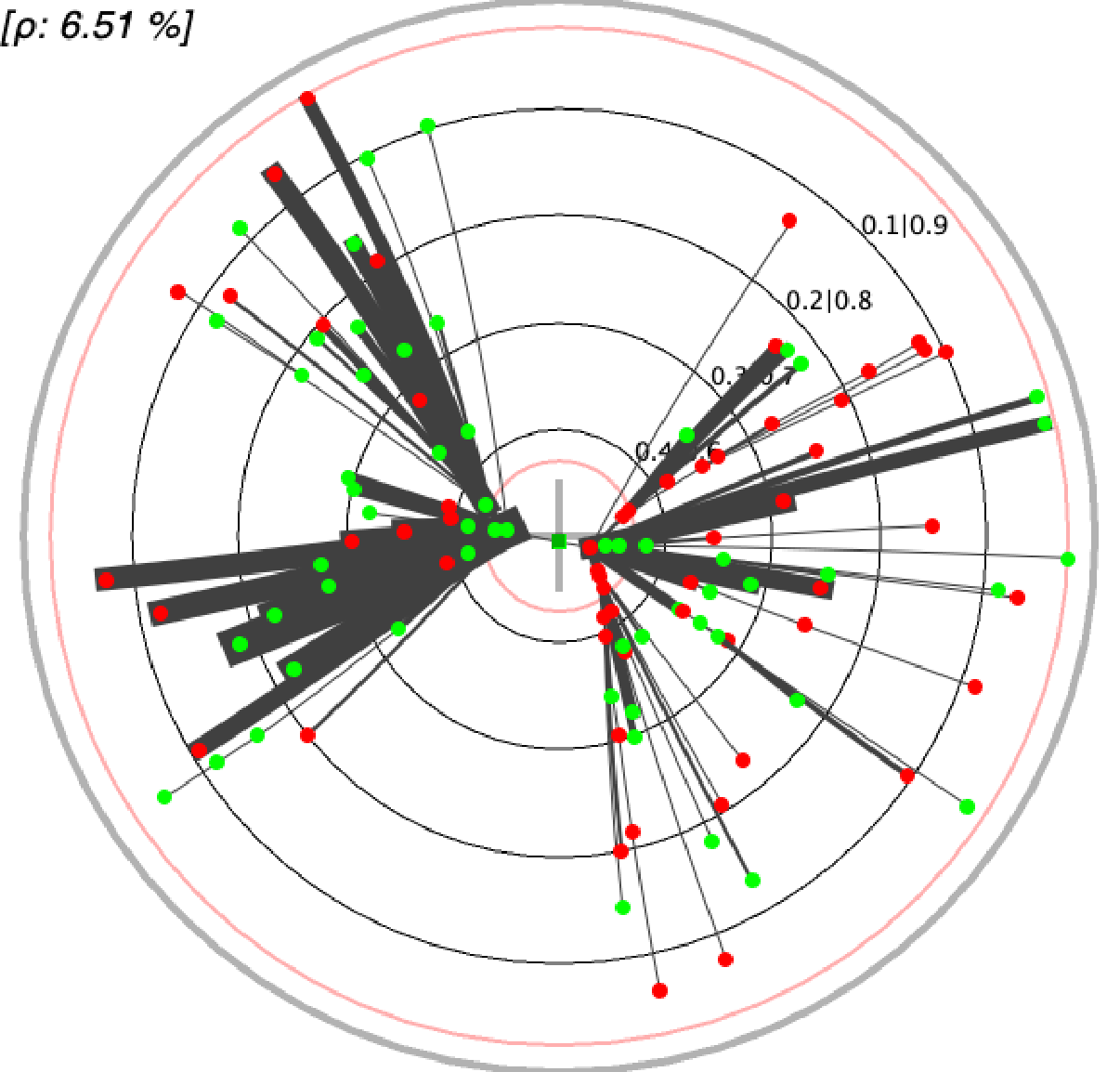}  & \includegraphics[trim=0bp 0bp 0bp 0bp,clip,width=0.3\columnwidth]{Experiments/poincare_DT/abalone/treeplot_Jan_17th__8h_54m_58s_USE_BOOSTING_WEIGHTS_Algo0_SplitCV5_Tree9.eps} \\
                              MDT $\#6$ & MDT $\#7$& MDT $\#8$& MDT $\#9$& MDT $\#10$\\  \Xhline{2pt} 
  \end{tabular}}
\caption{First 10 MDTs for UCI \domainname{abalone}. Convention follows Table \ref{tab:online-shopping-intentions-dt-exerpt}.}
    \label{tab:abalone-dt-exerpt}
  \end{table}

   \begin{table}
  \centering
  \resizebox{\textwidth}{!}{\begin{tabular}{ccccc}\Xhline{2pt}
                              \includegraphics[trim=0bp 0bp 0bp 0bp,clip,width=0.3\columnwidth]{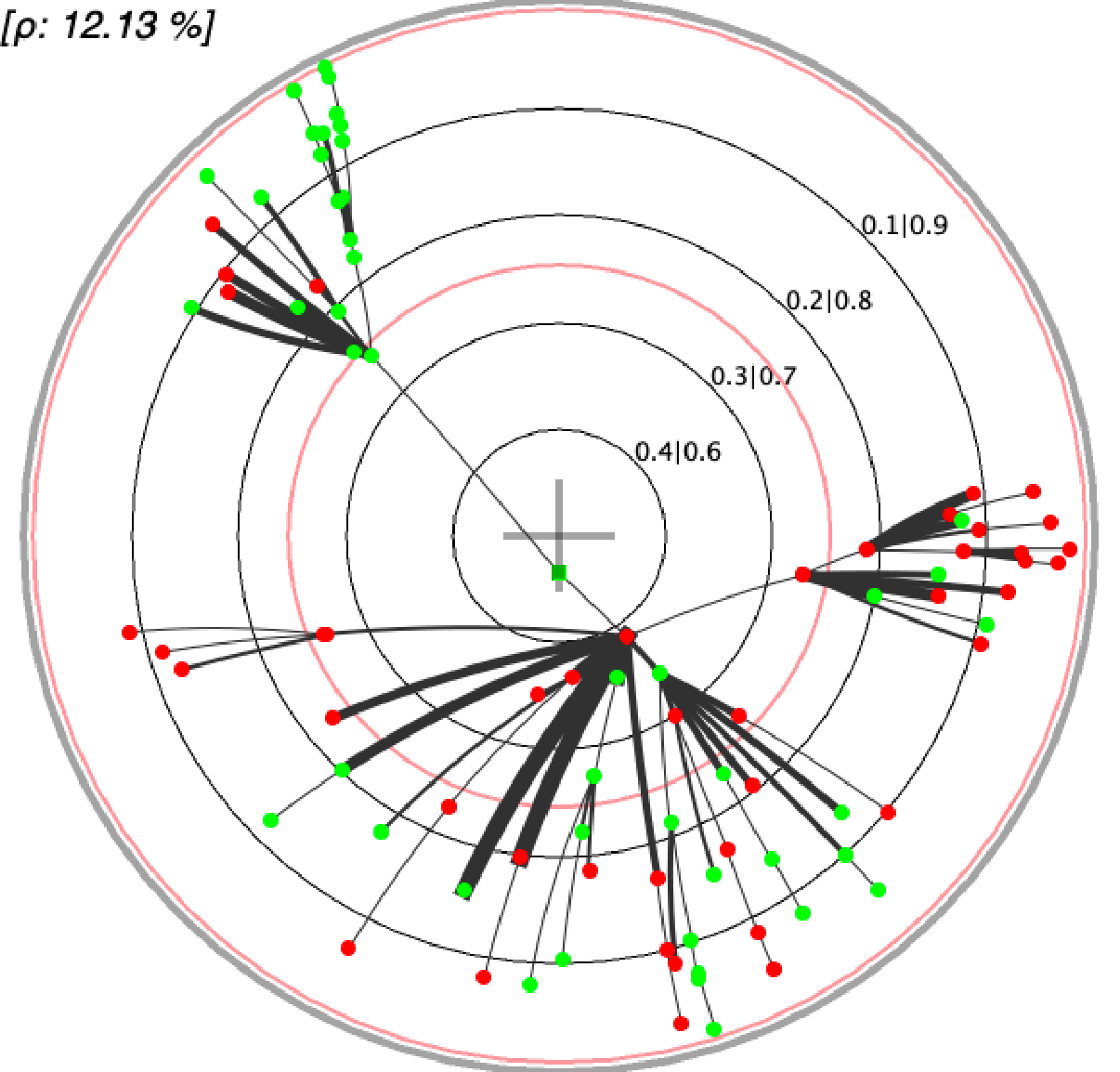} & \includegraphics[trim=0bp 0bp 0bp 0bp,clip,width=0.3\columnwidth]{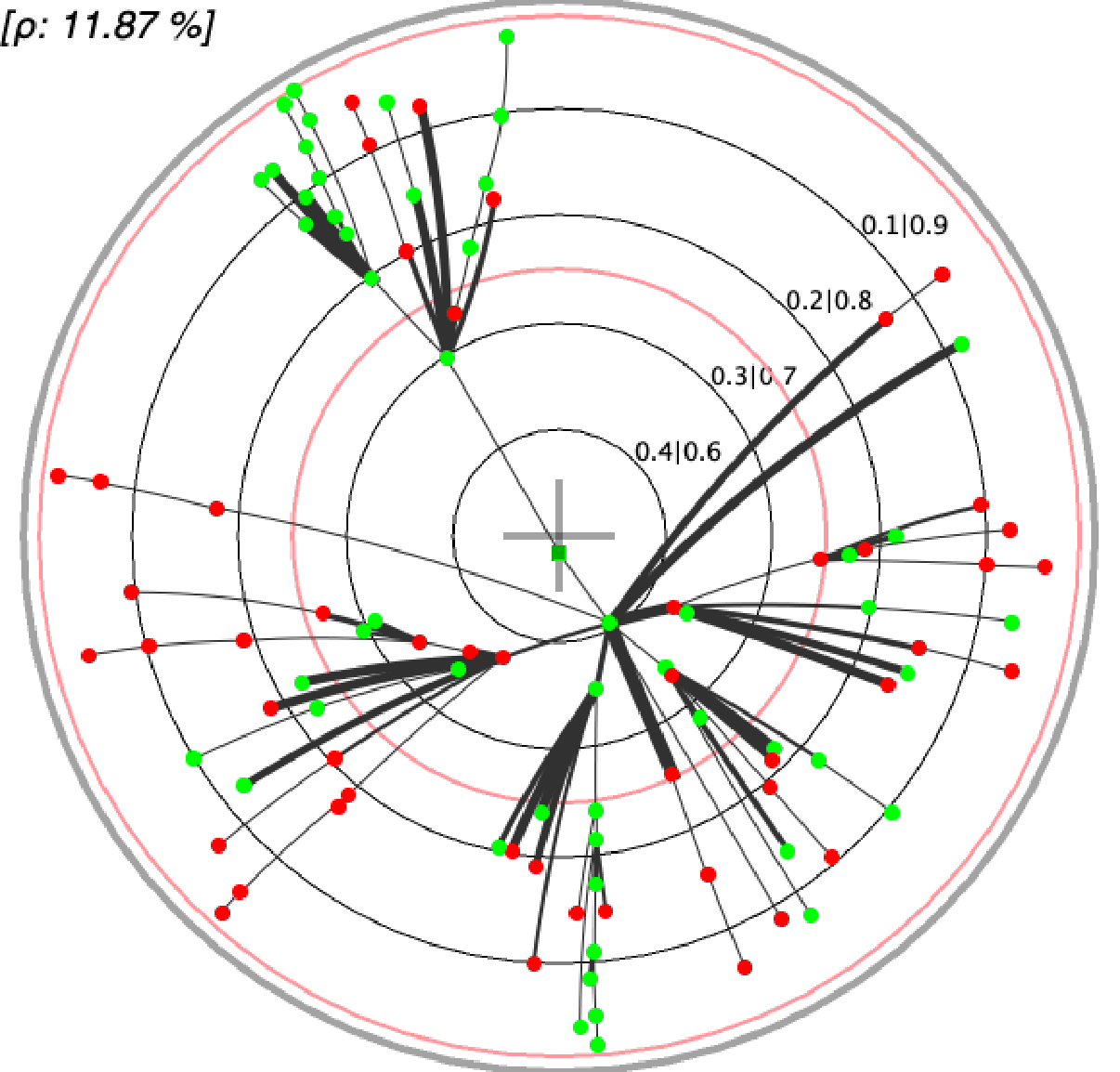} & \includegraphics[trim=0bp 0bp 0bp 0bp,clip,width=0.3\columnwidth]{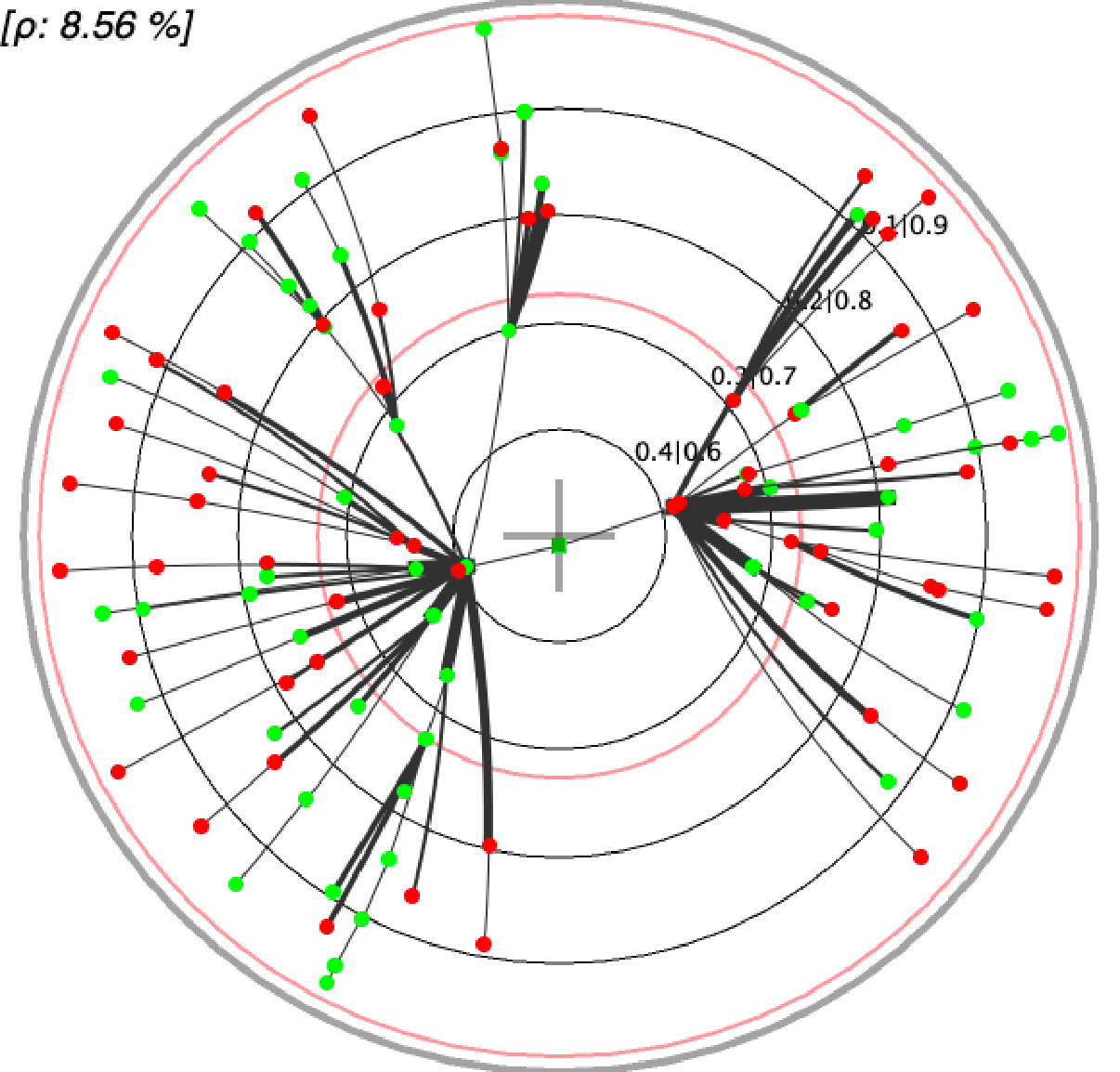}  & \includegraphics[trim=0bp 0bp 0bp 0bp,clip,width=0.3\columnwidth]{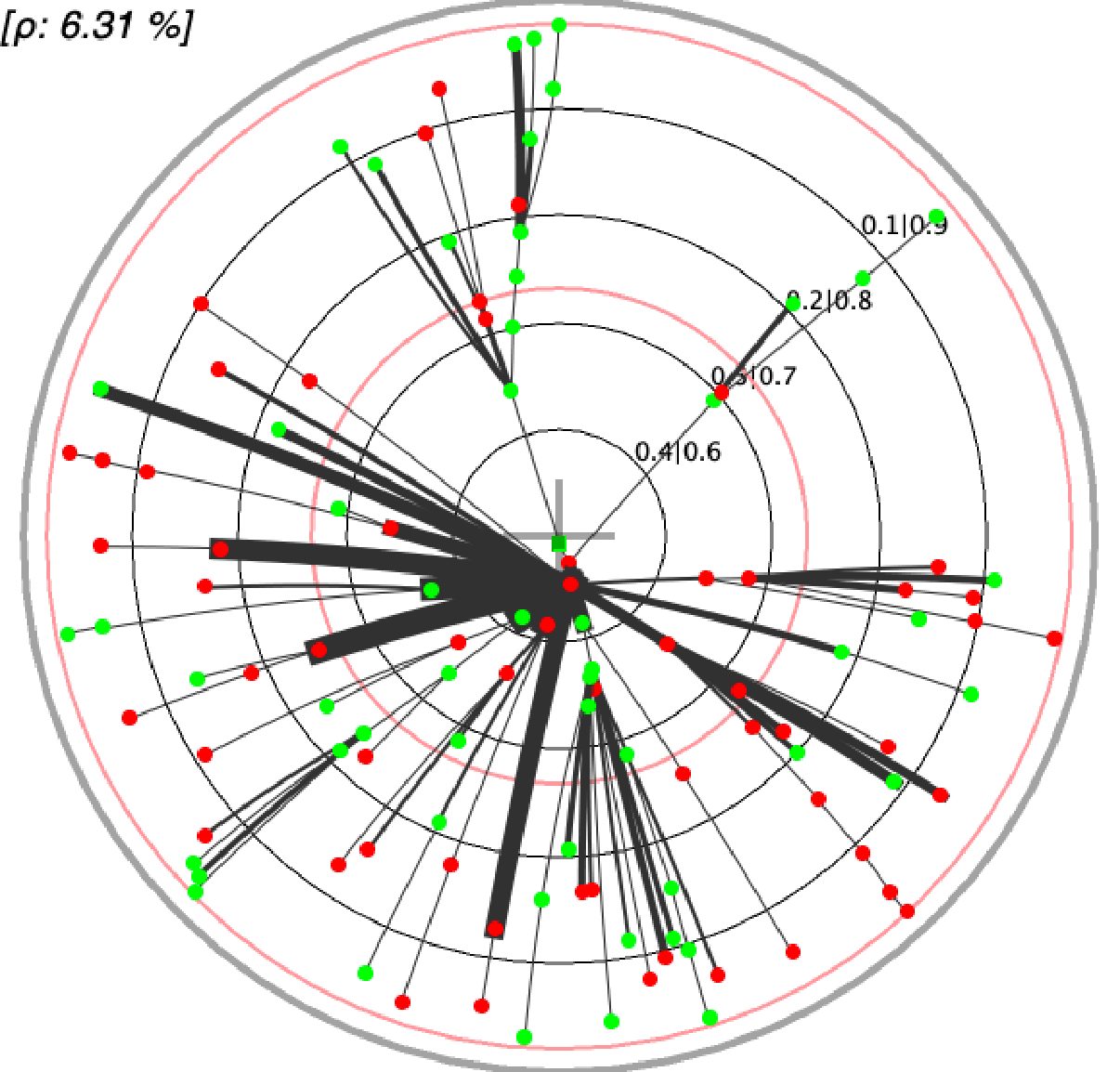}  & \includegraphics[trim=0bp 0bp 0bp 0bp,clip,width=0.3\columnwidth]{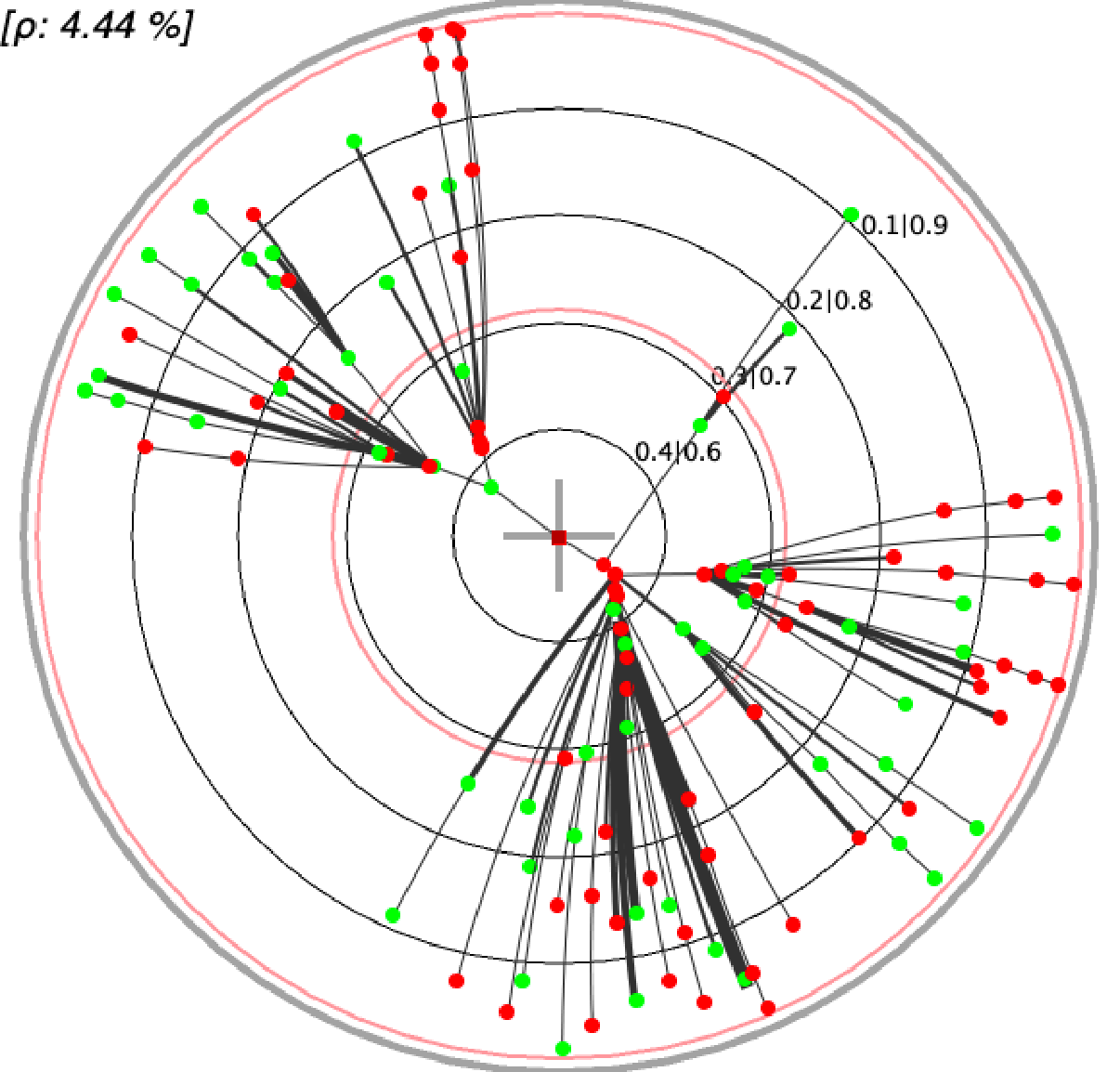} \\
                              MDT $\#1$ & MDT $\#2$& MDT $\#3$& MDT $\#4$& MDT $\#5$\\
                             \includegraphics[trim=0bp 0bp 0bp 0bp,clip,width=0.3\columnwidth]{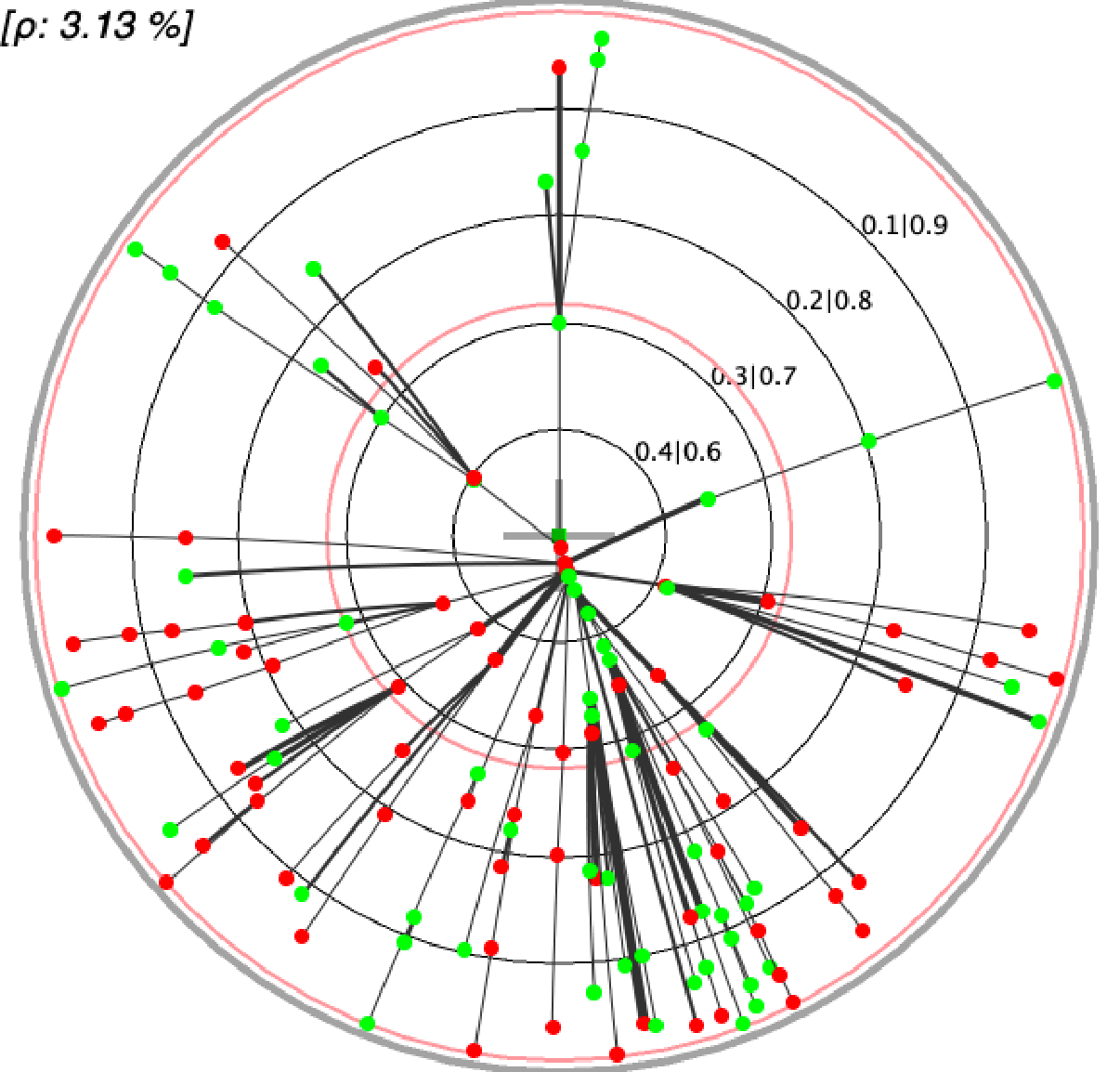} & \includegraphics[trim=0bp 0bp 0bp 0bp,clip,width=0.3\columnwidth]{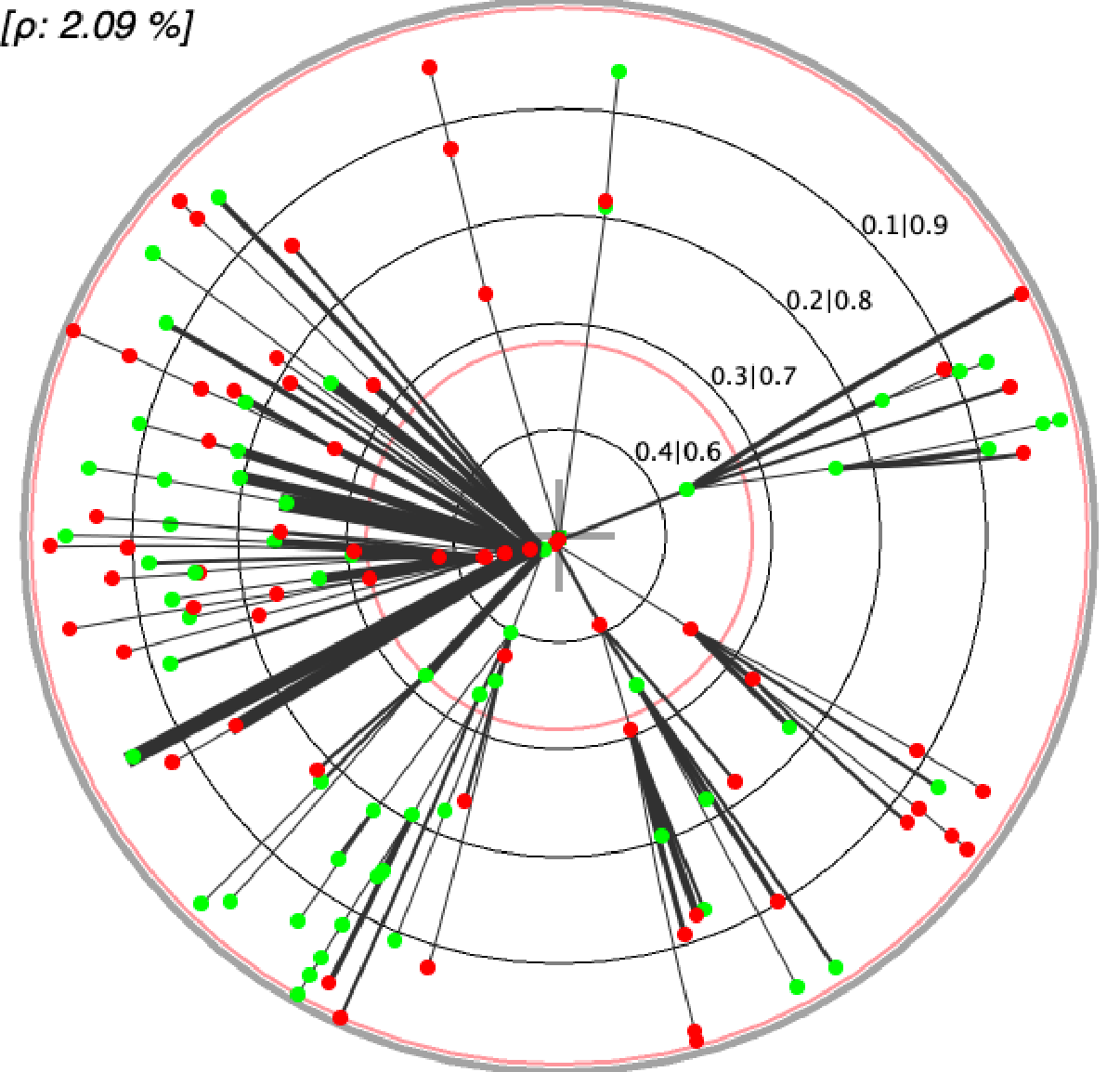} & \includegraphics[trim=0bp 0bp 0bp 0bp,clip,width=0.3\columnwidth]{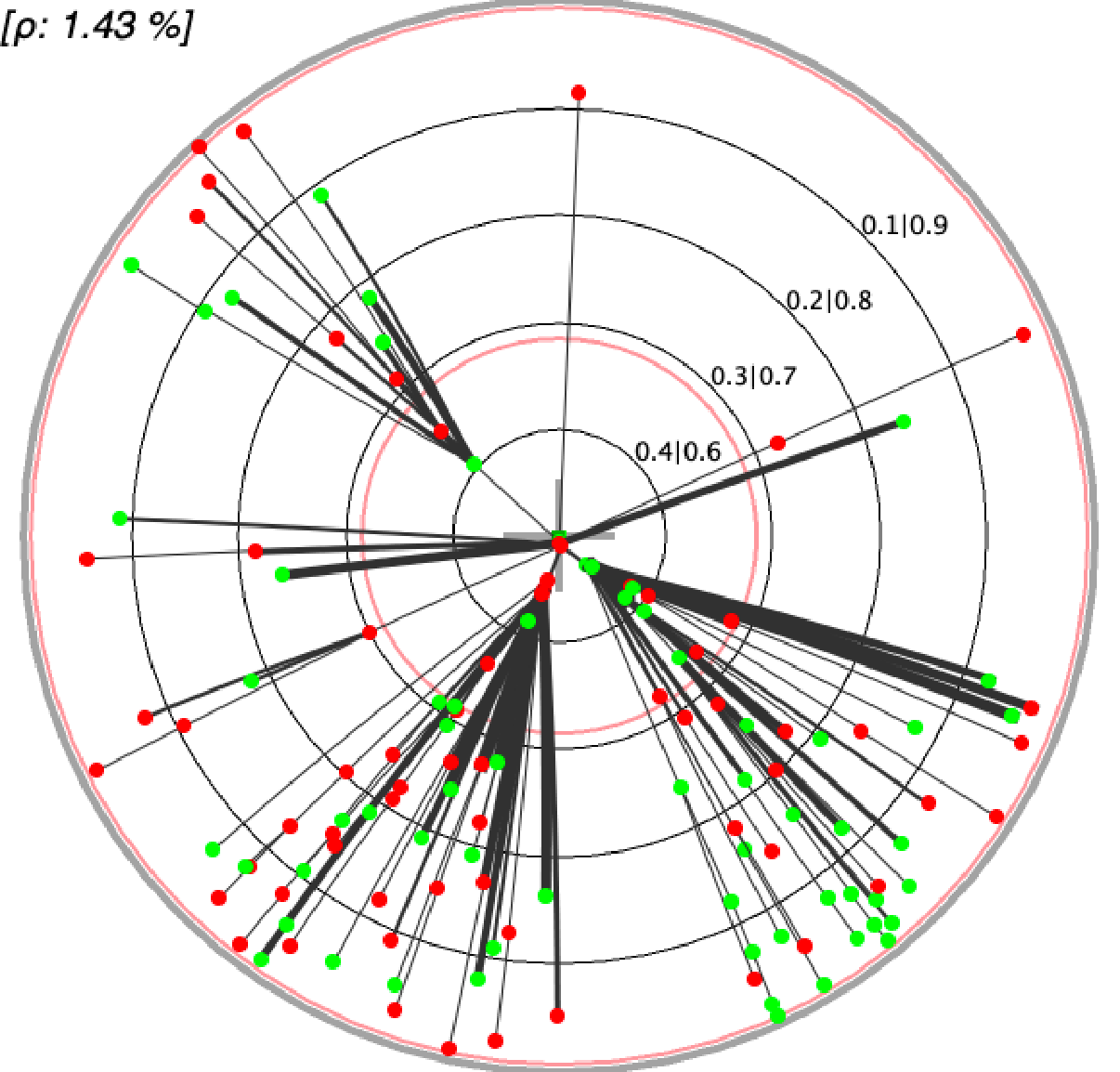}  & \includegraphics[trim=0bp 0bp 0bp 0bp,clip,width=0.3\columnwidth]{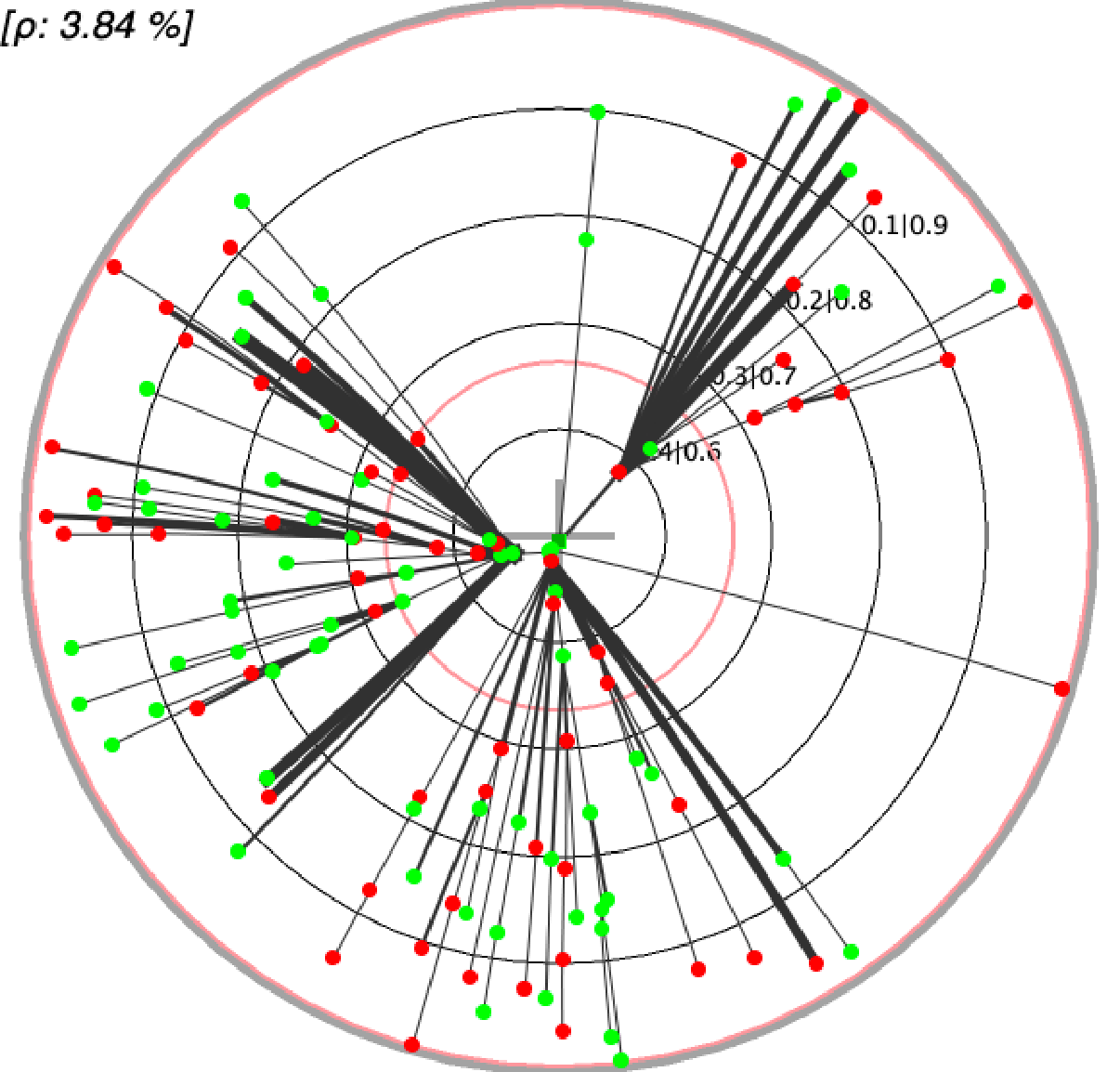}  & \includegraphics[trim=0bp 0bp 0bp 0bp,clip,width=0.3\columnwidth]{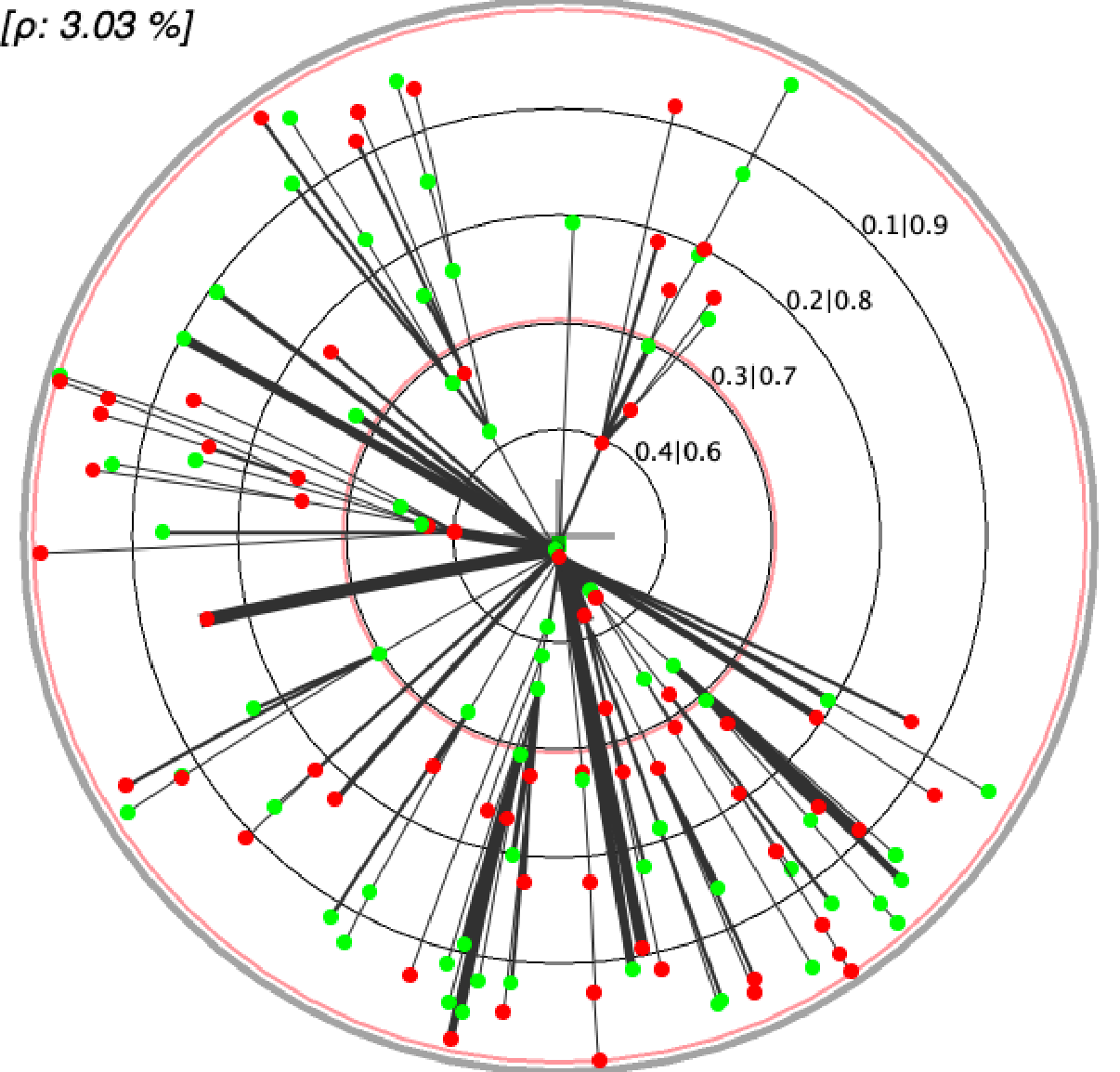} \\
                              MDT $\#6$ & MDT $\#7$& MDT $\#8$& MDT $\#9$& MDT $\#10$\\  \Xhline{2pt} 
  \end{tabular}}
\caption{First 10 MDTs for UCI \domainname{winered}. Convention follows Table \ref{tab:online-shopping-intentions-dt-exerpt}.}
    \label{tab:winered-dt-exerpt}
  \end{table}

   \begin{table}
  \centering
  \resizebox{\textwidth}{!}{\begin{tabular}{ccccc}\Xhline{2pt}
                              \includegraphics[trim=0bp 0bp 0bp 0bp,clip,width=0.3\columnwidth]{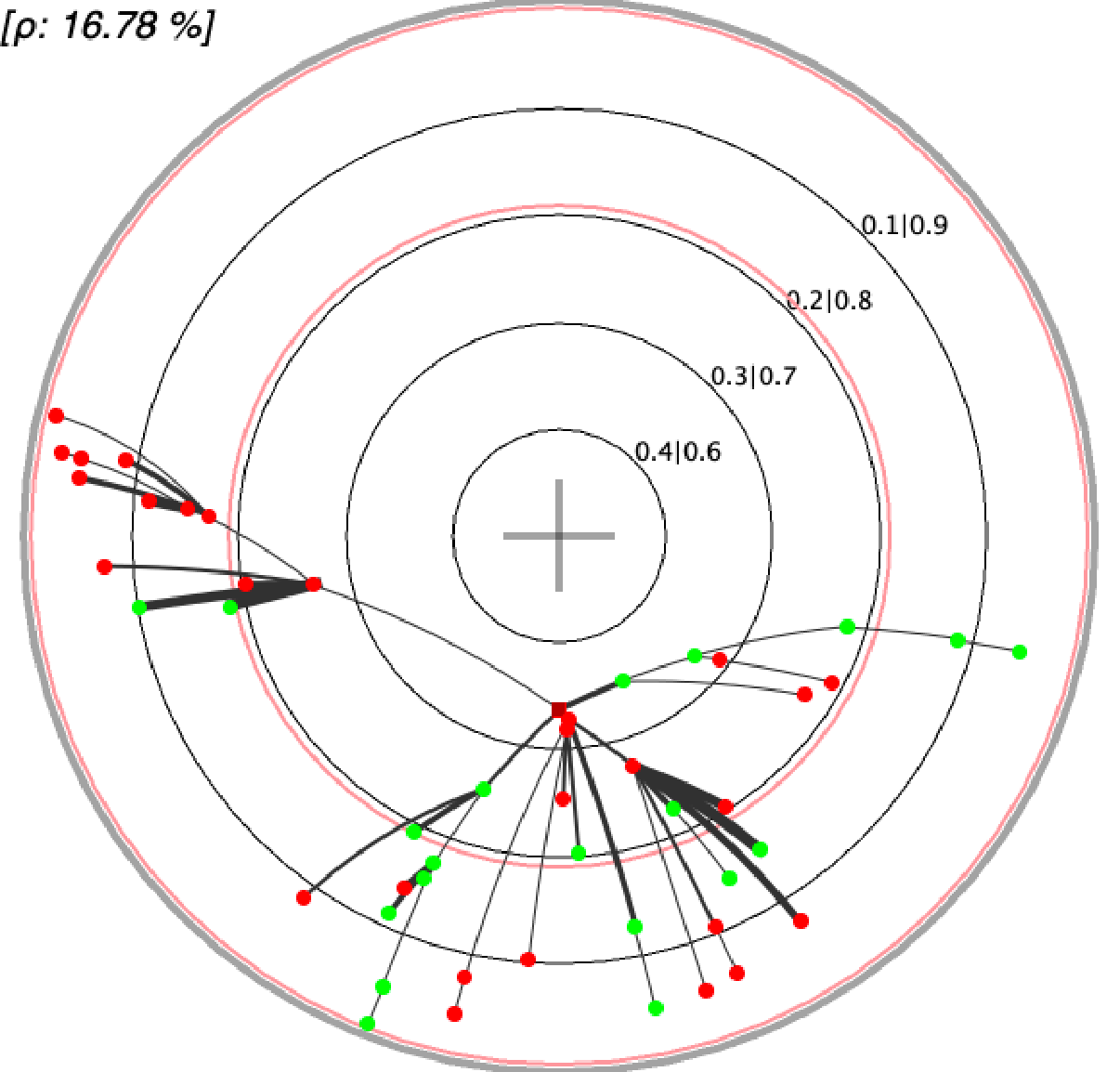} & \includegraphics[trim=0bp 0bp 0bp 0bp,clip,width=0.3\columnwidth]{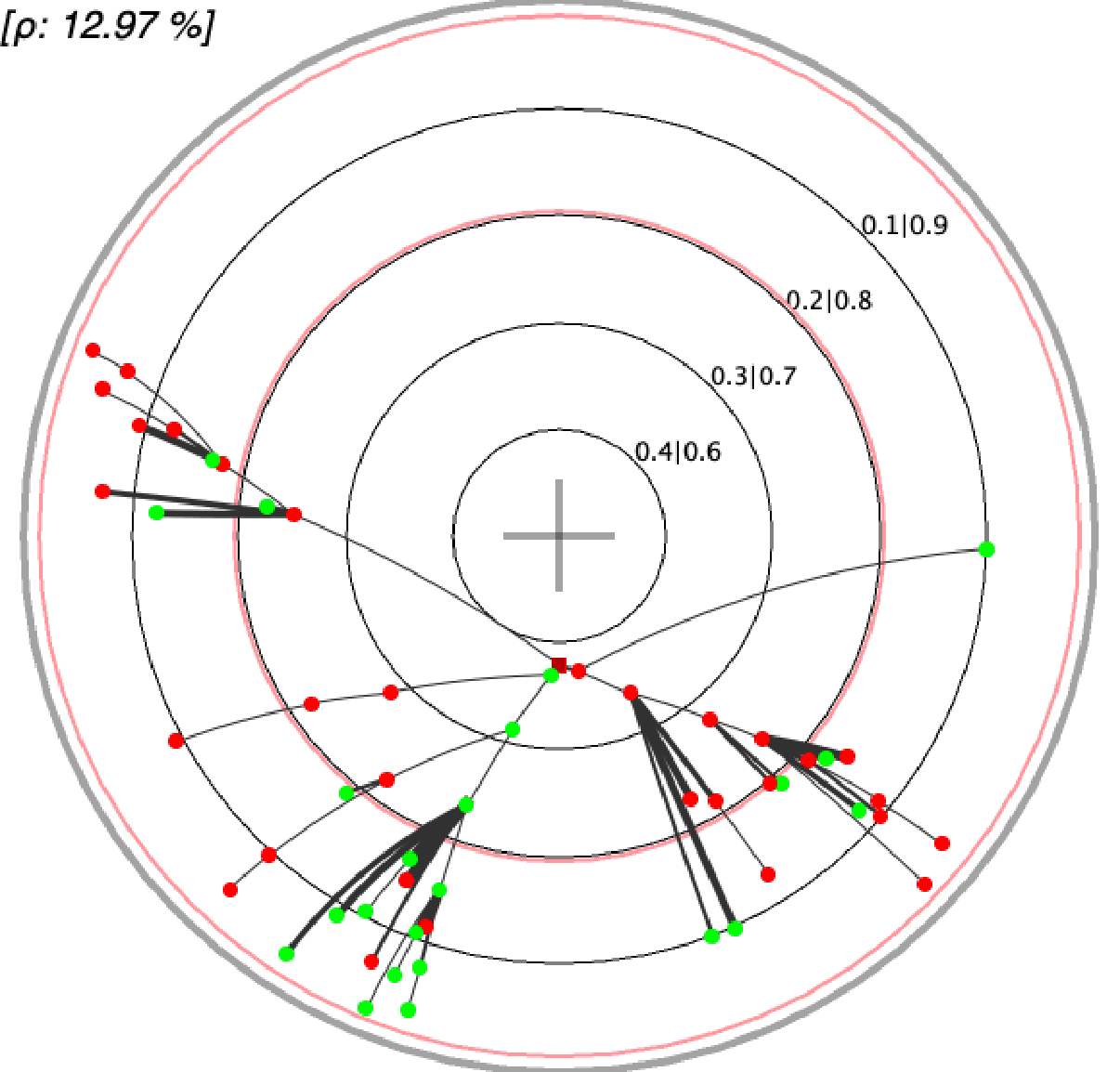} & \includegraphics[trim=0bp 0bp 0bp 0bp,clip,width=0.3\columnwidth]{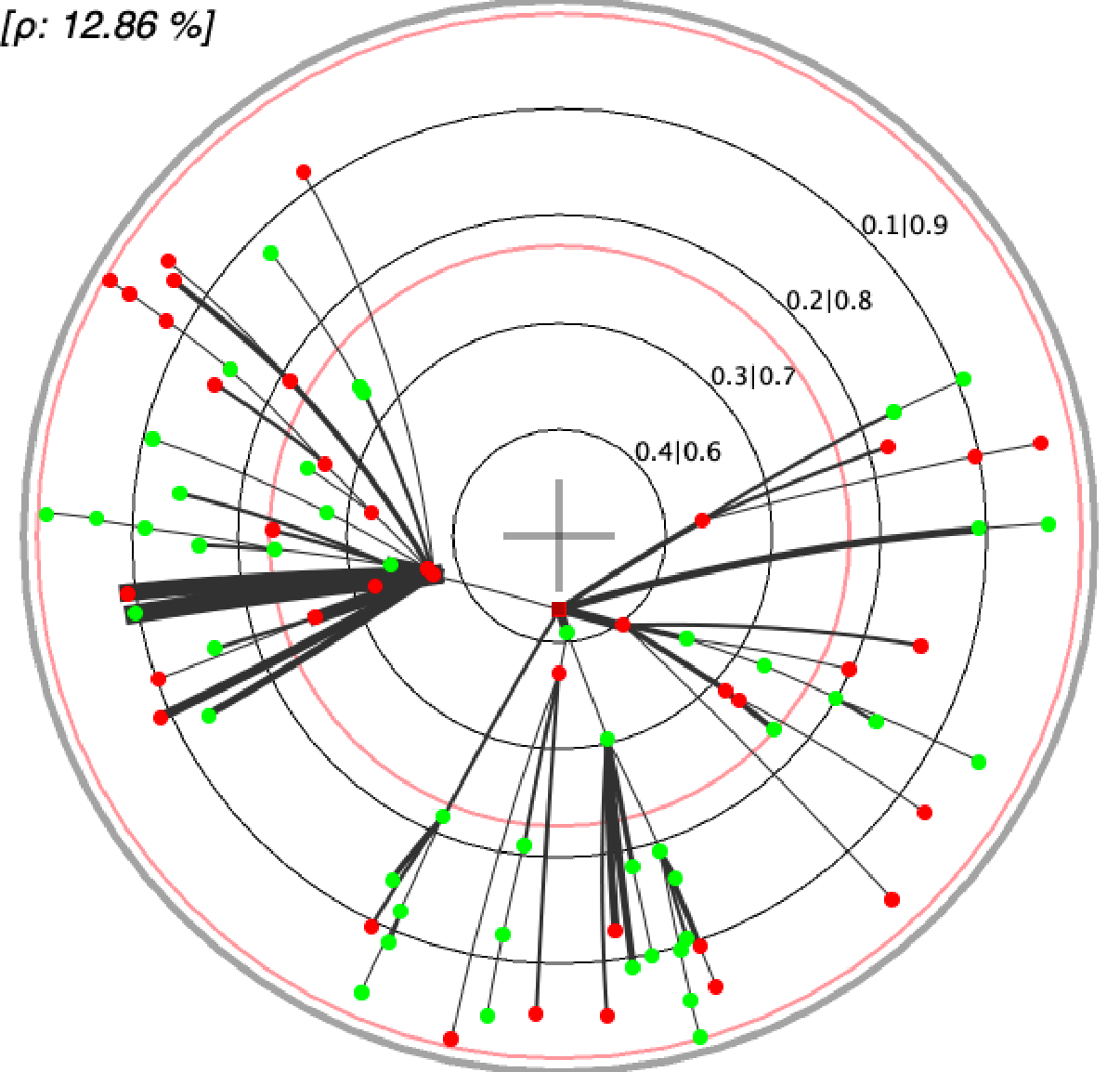}  & \includegraphics[trim=0bp 0bp 0bp 0bp,clip,width=0.3\columnwidth]{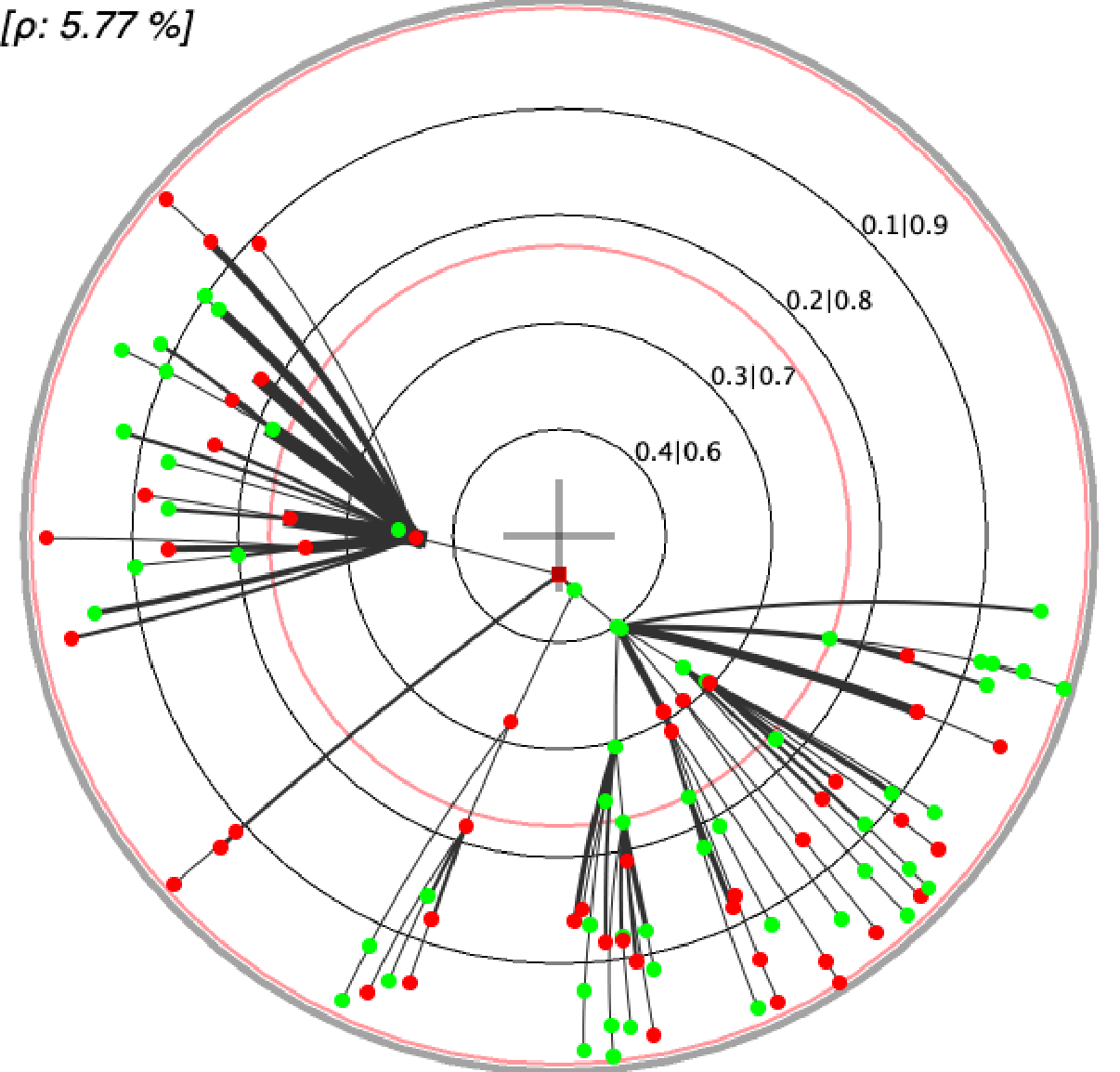}  & \includegraphics[trim=0bp 0bp 0bp 0bp,clip,width=0.3\columnwidth]{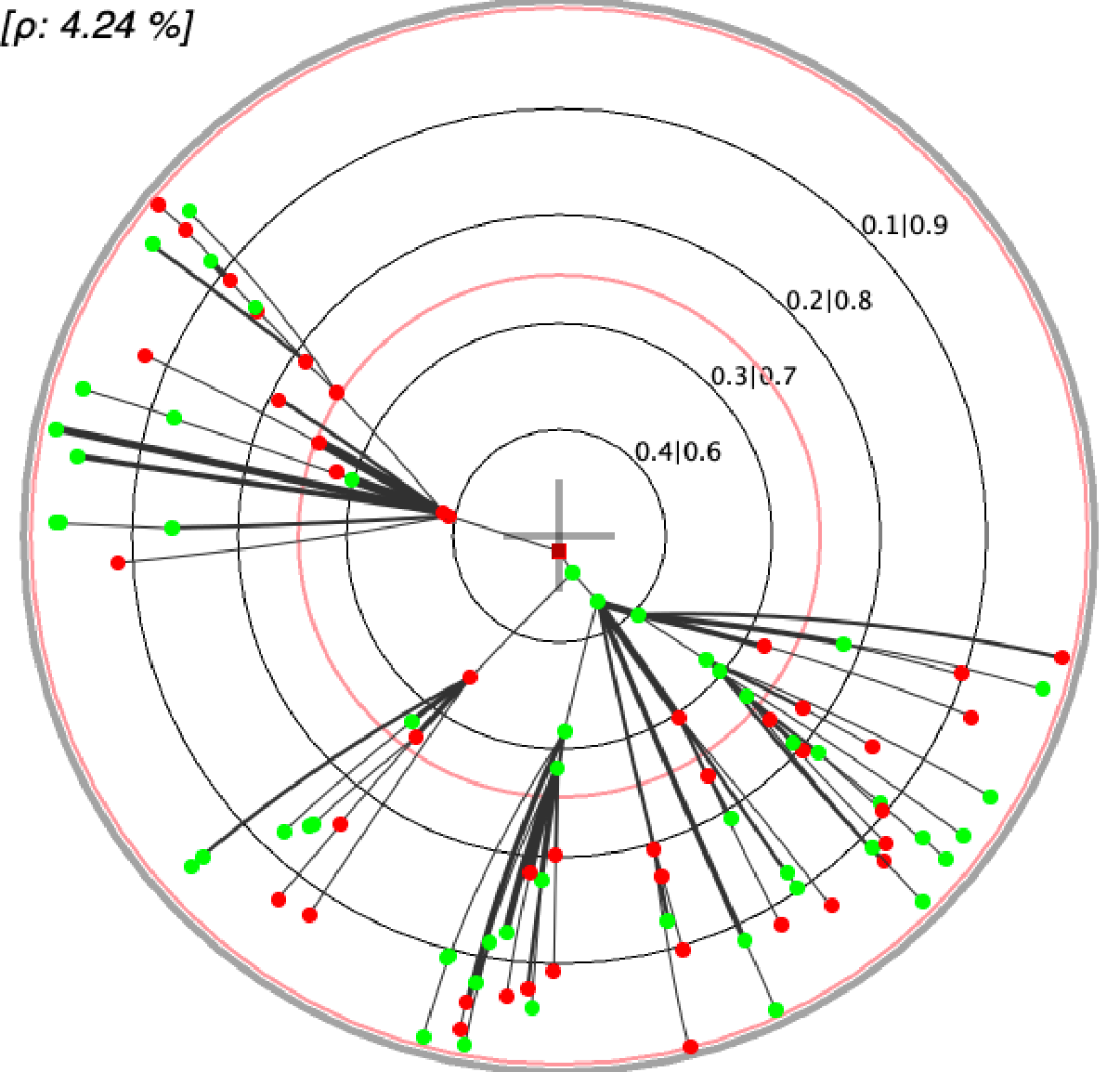} \\
                              MDT $\#1$ & MDT $\#2$& MDT $\#3$& MDT $\#4$& MDT $\#5$\\
                             \includegraphics[trim=0bp 0bp 0bp 0bp,clip,width=0.3\columnwidth]{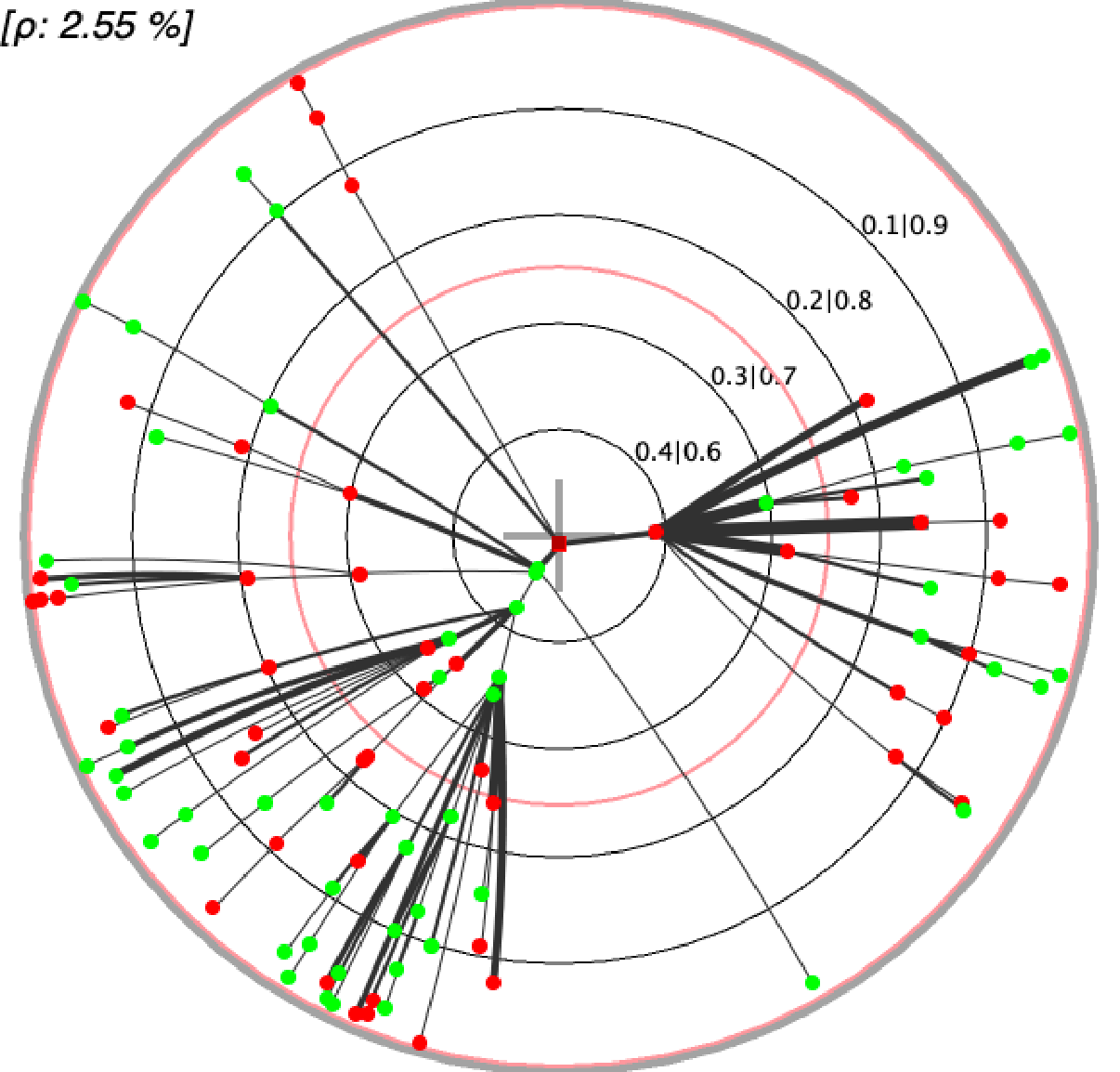} & \includegraphics[trim=0bp 0bp 0bp 0bp,clip,width=0.3\columnwidth]{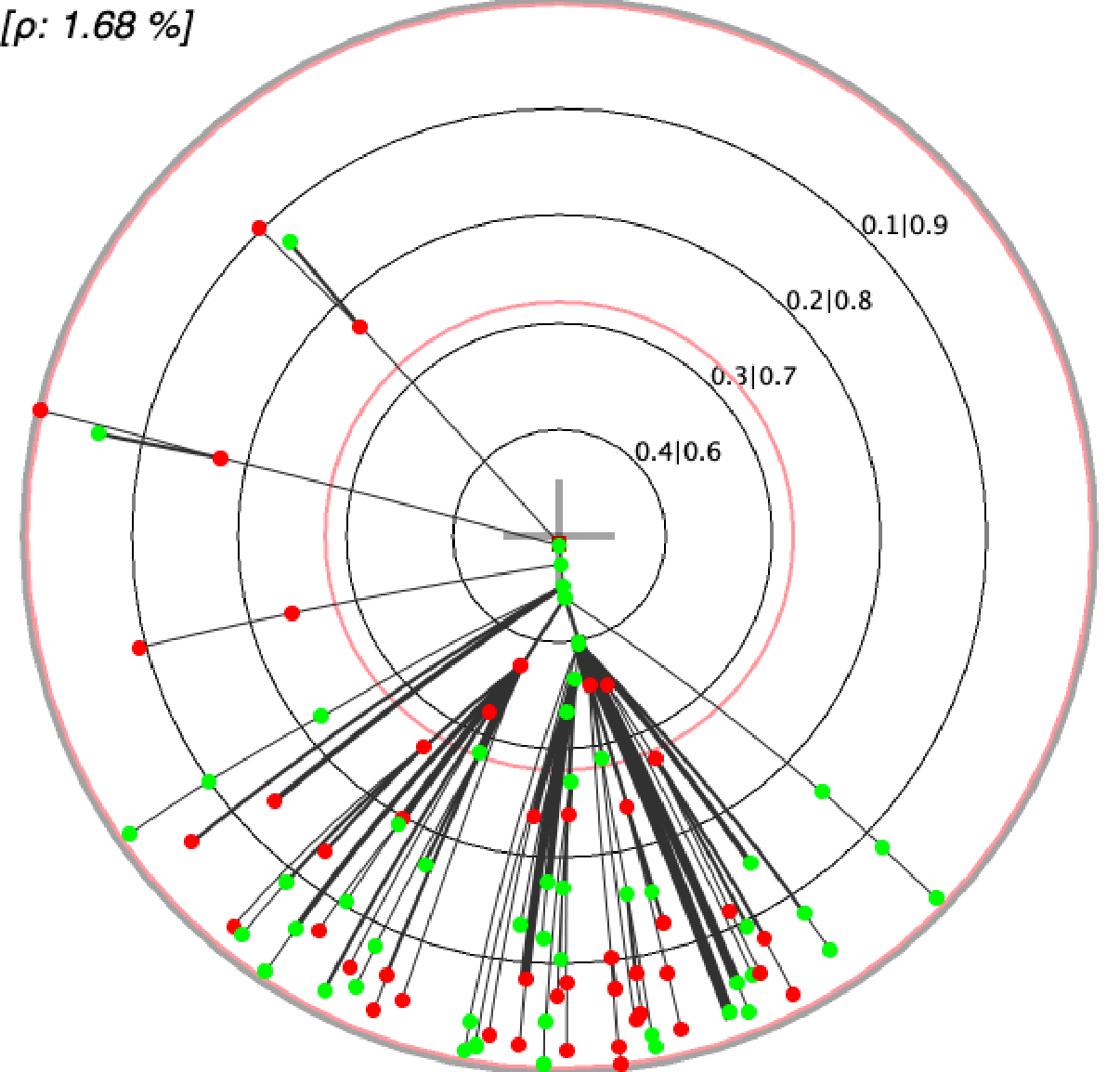} & \includegraphics[trim=0bp 0bp 0bp 0bp,clip,width=0.3\columnwidth]{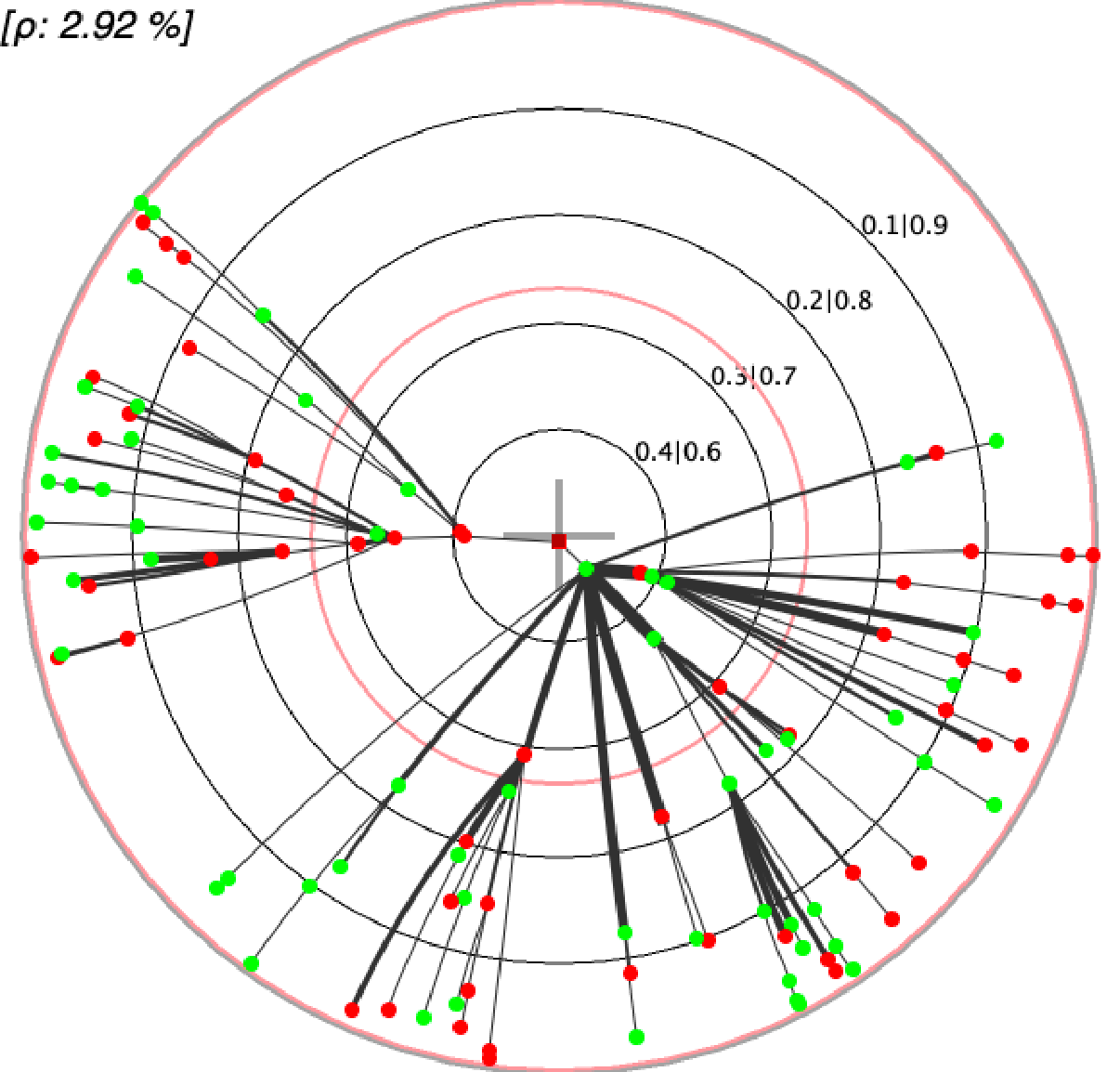}  & \includegraphics[trim=0bp 0bp 0bp 0bp,clip,width=0.3\columnwidth]{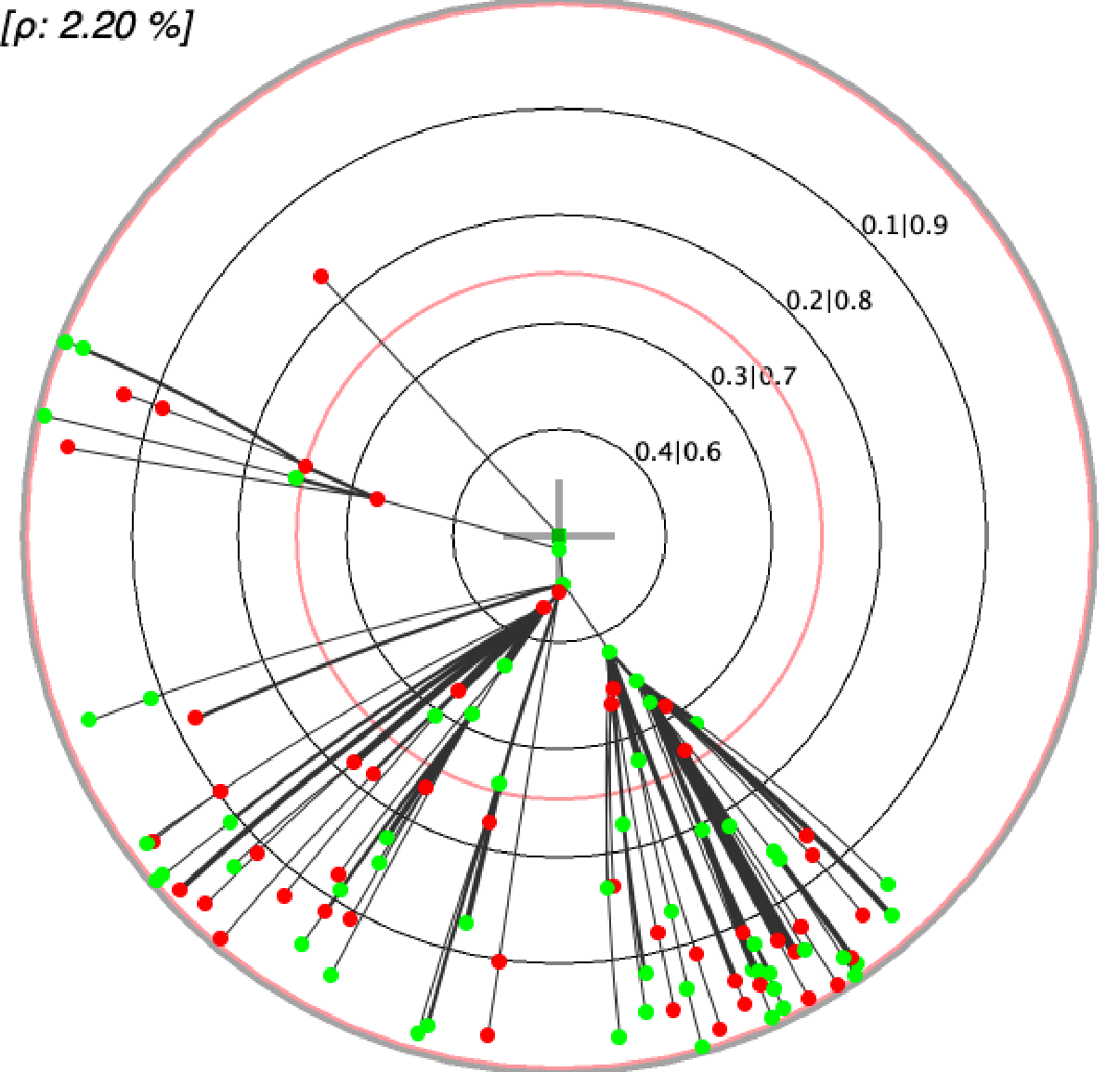}  & \includegraphics[trim=0bp 0bp 0bp 0bp,clip,width=0.3\columnwidth]{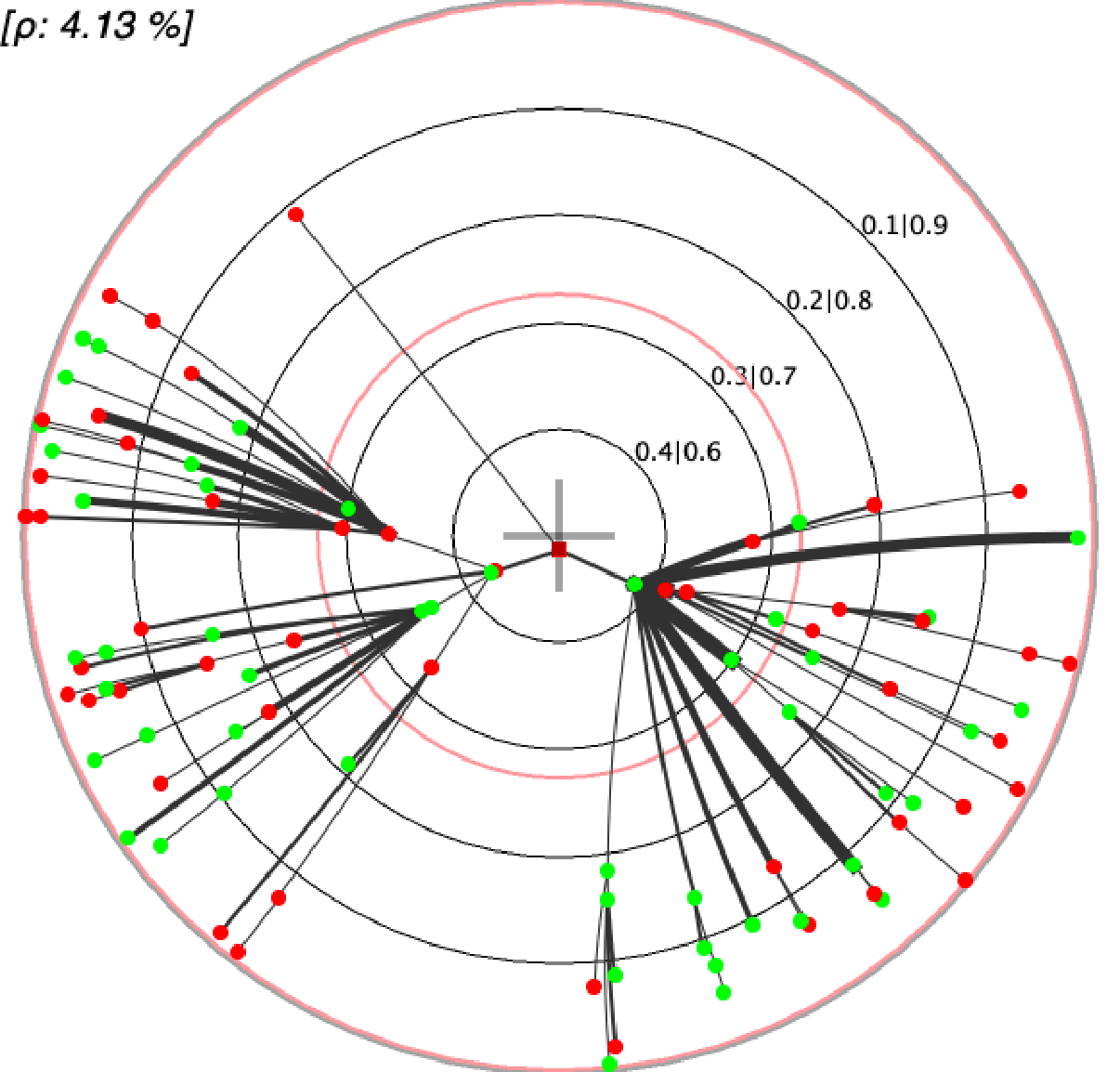} \\
                              MDT $\#6$ & MDT $\#7$& MDT $\#8$& MDT $\#9$& MDT $\#10$\\  \Xhline{2pt} 
  \end{tabular}}
\caption{First 10 MDTs for UCI \domainname{qsar}. Convention follows Table \ref{tab:online-shopping-intentions-dt-exerpt}.}
    \label{tab:qsar-dt-exerpt}
  \end{table}

   \begin{table}
  \centering
  \resizebox{\textwidth}{!}{\begin{tabular}{ccccc}\Xhline{2pt}
                              \includegraphics[trim=0bp 0bp 0bp 0bp,clip,width=0.3\columnwidth]{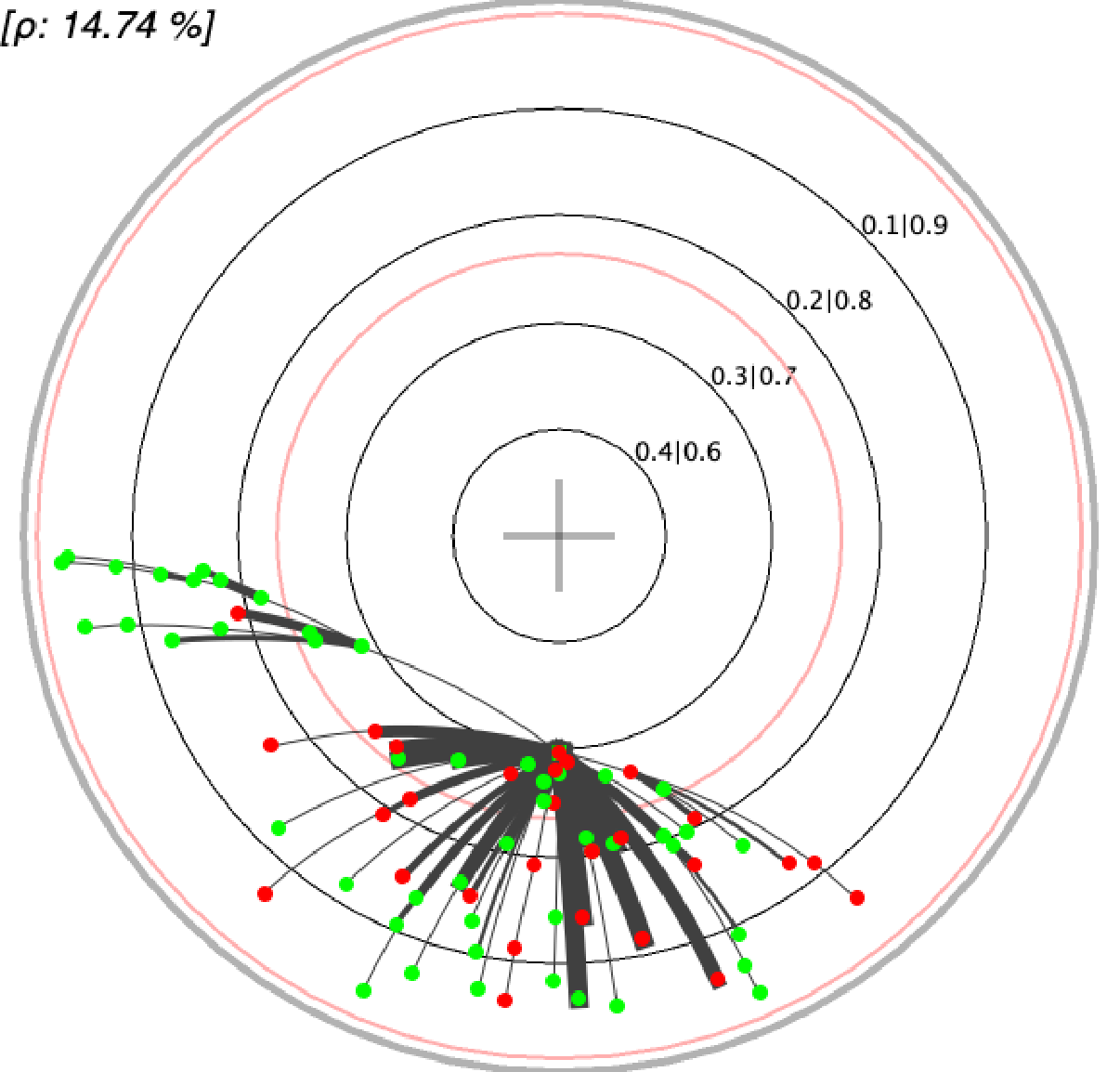} & \includegraphics[trim=0bp 0bp 0bp 0bp,clip,width=0.3\columnwidth]{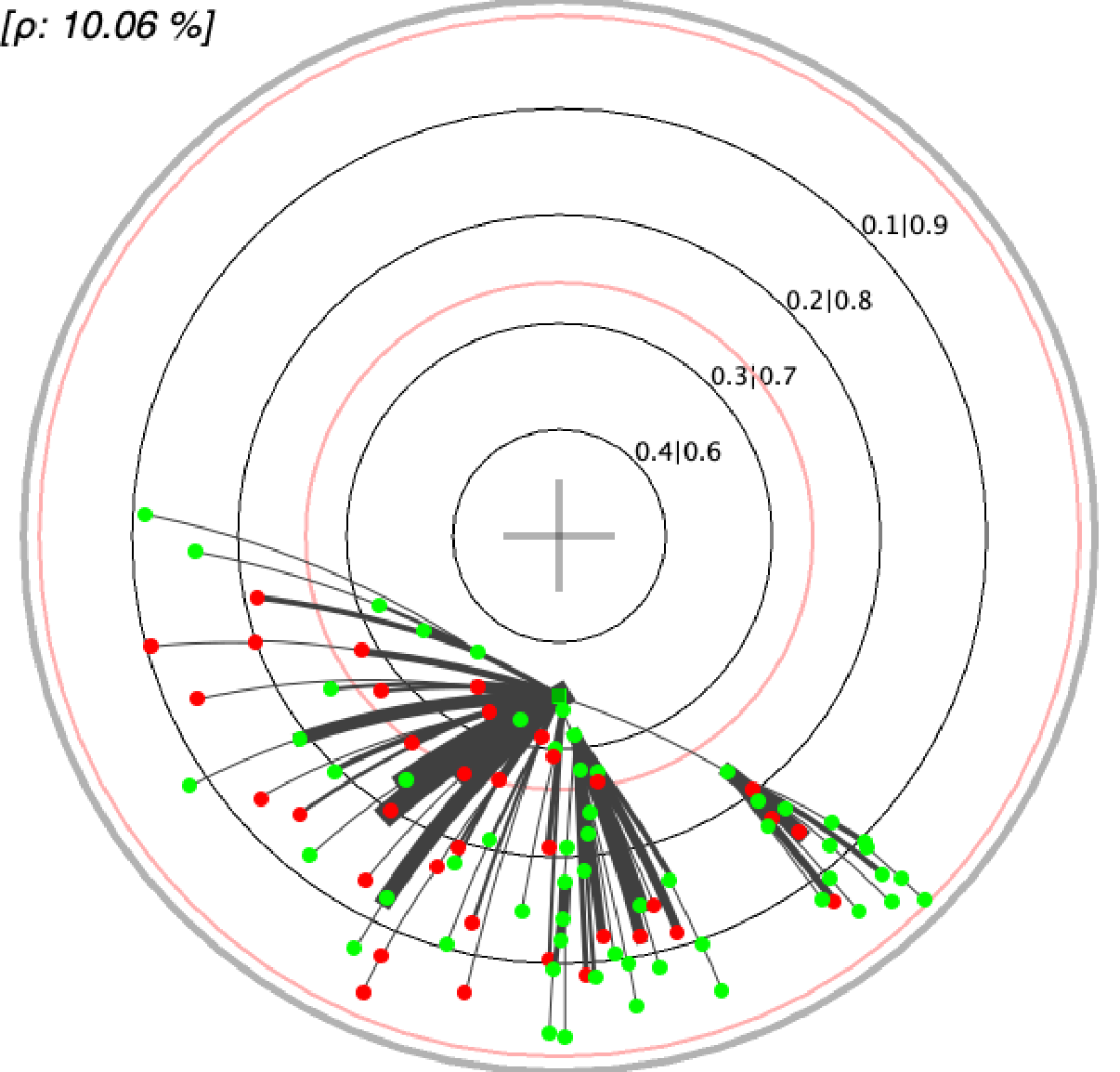} & \includegraphics[trim=0bp 0bp 0bp 0bp,clip,width=0.3\columnwidth]{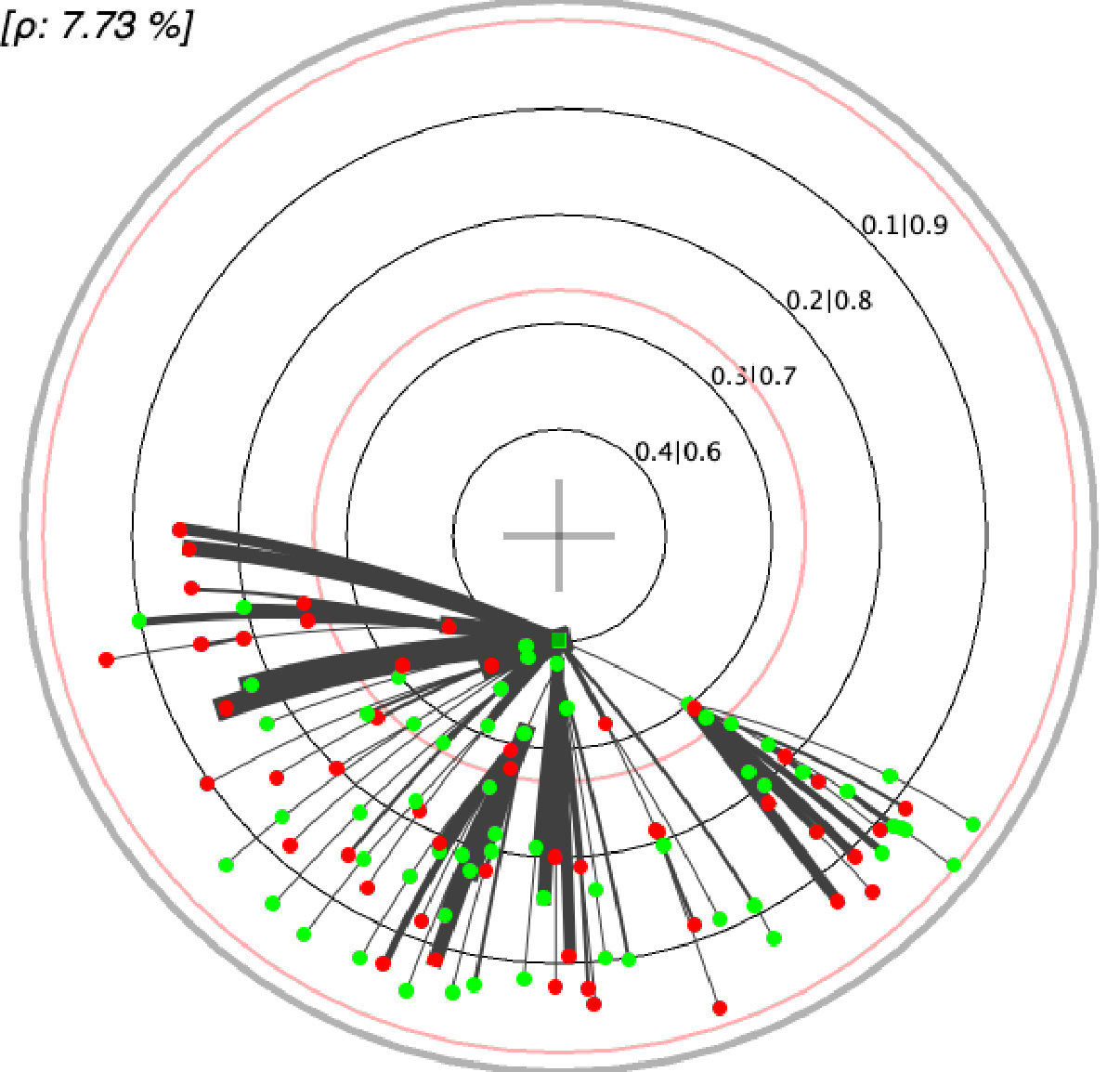}  & \includegraphics[trim=0bp 0bp 0bp 0bp,clip,width=0.3\columnwidth]{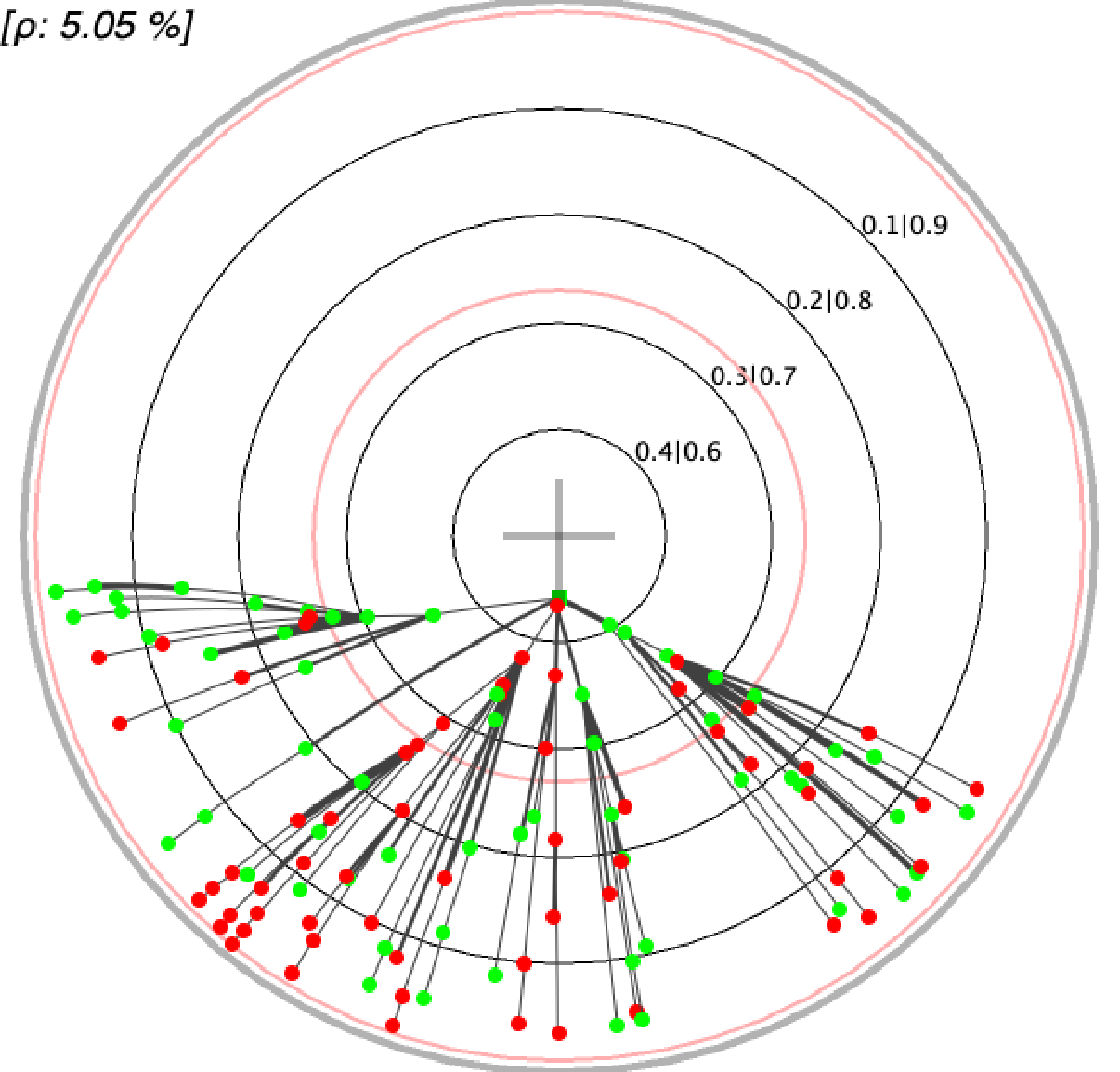}  & \includegraphics[trim=0bp 0bp 0bp 0bp,clip,width=0.3\columnwidth]{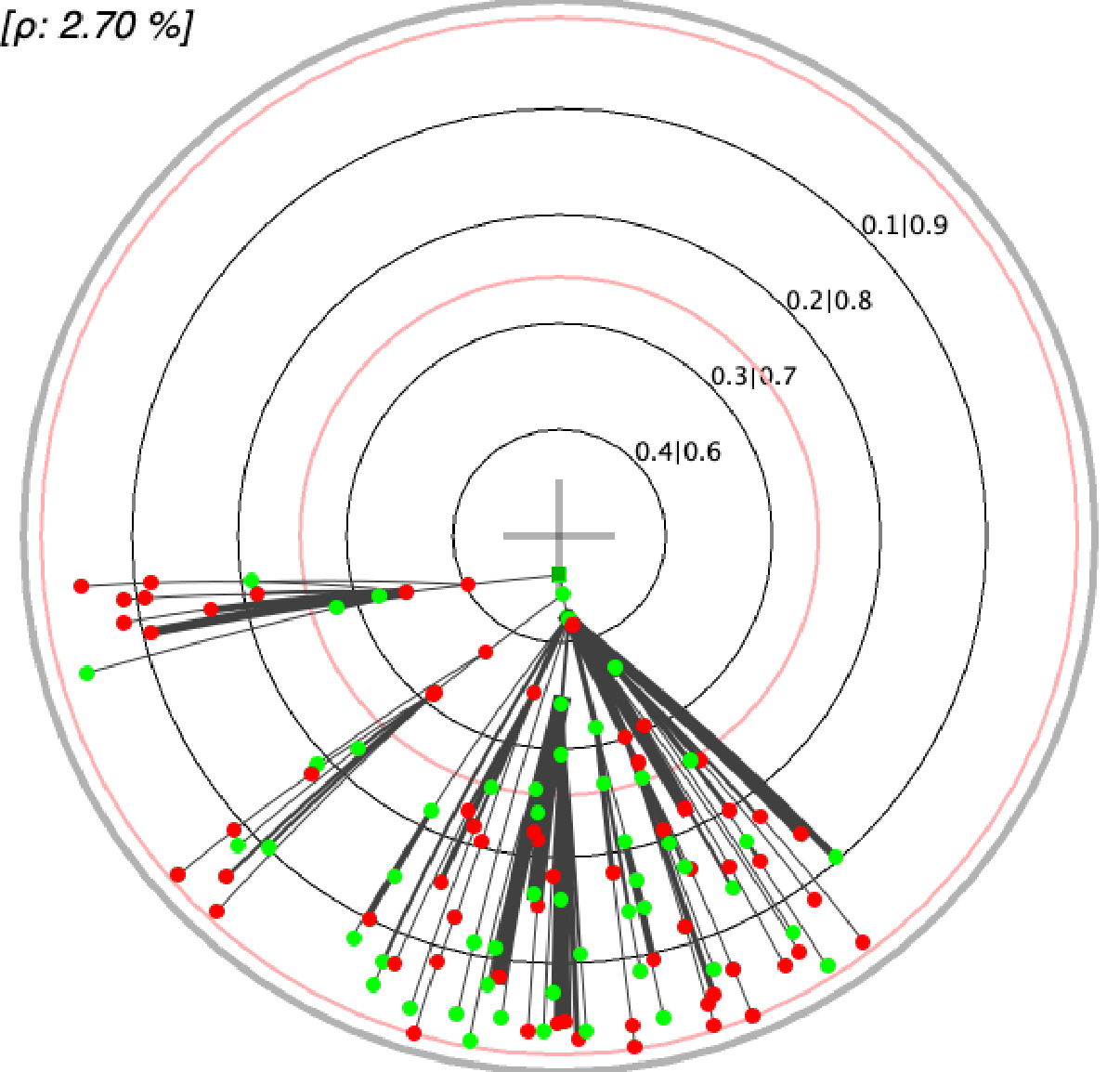} \\
                              MDT $\#1$ & MDT $\#2$& MDT $\#3$& MDT $\#4$& MDT $\#5$\\
                             \includegraphics[trim=0bp 0bp 0bp 0bp,clip,width=0.3\columnwidth]{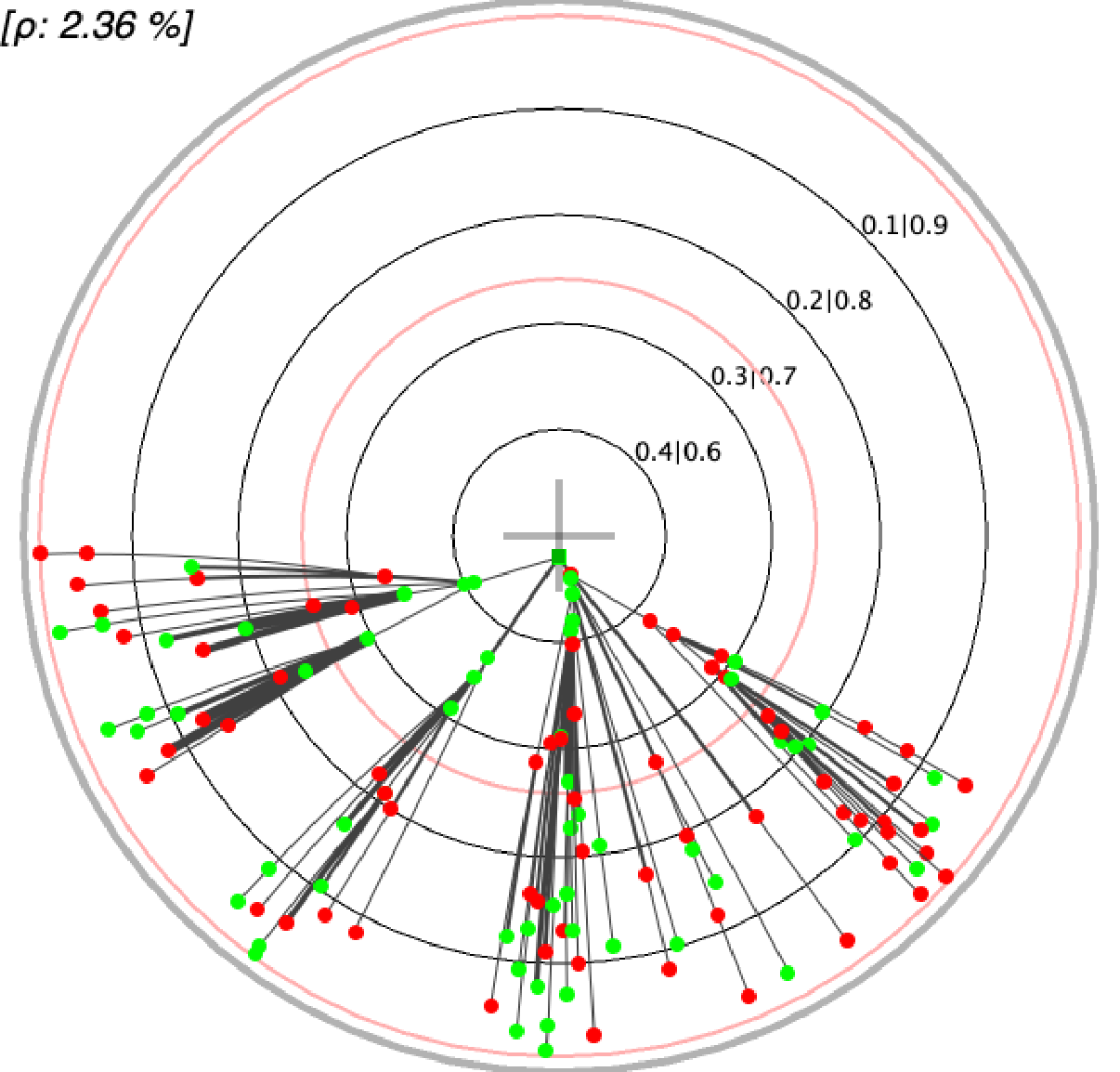} & \includegraphics[trim=0bp 0bp 0bp 0bp,clip,width=0.3\columnwidth]{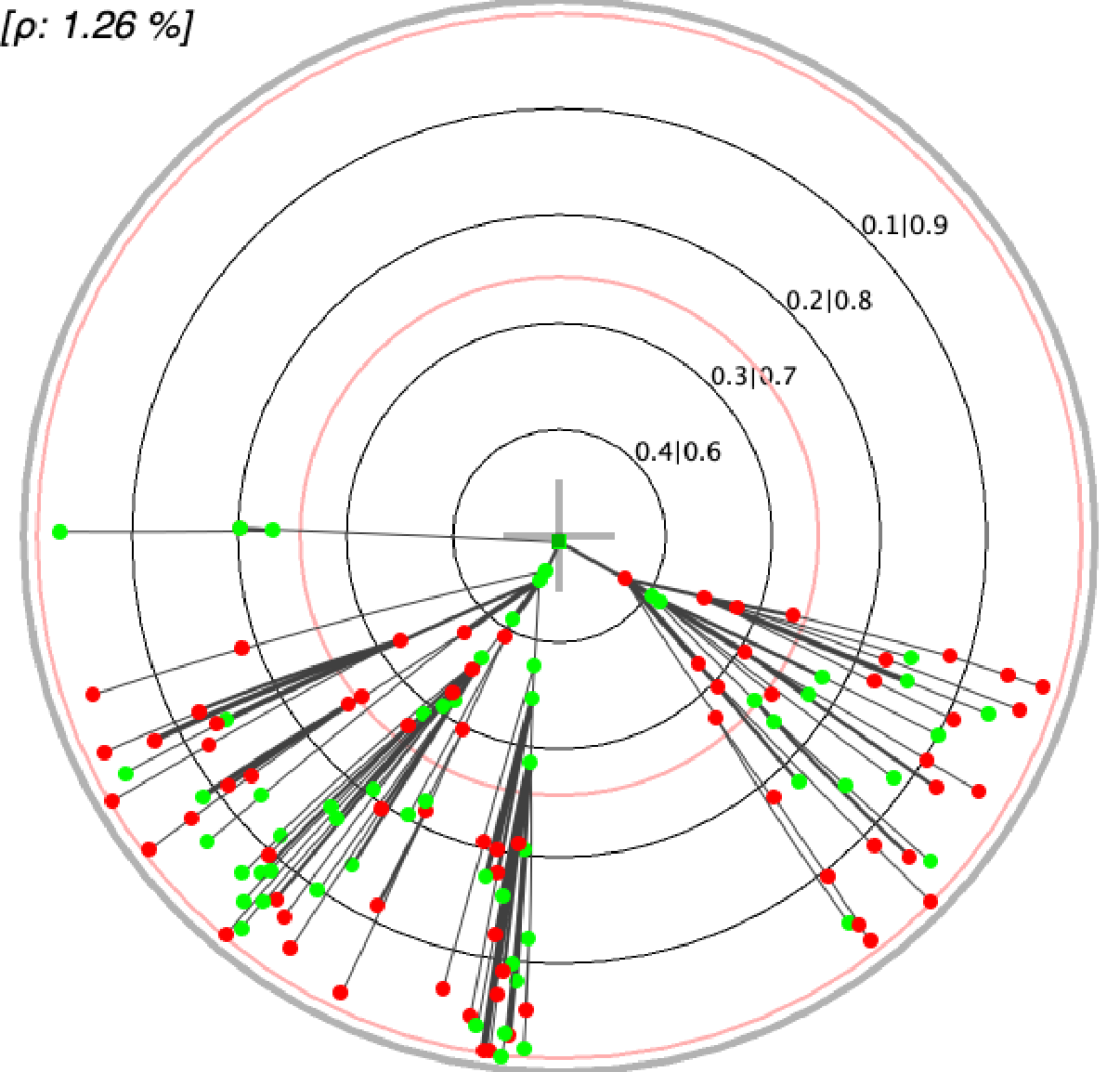} & \includegraphics[trim=0bp 0bp 0bp 0bp,clip,width=0.3\columnwidth]{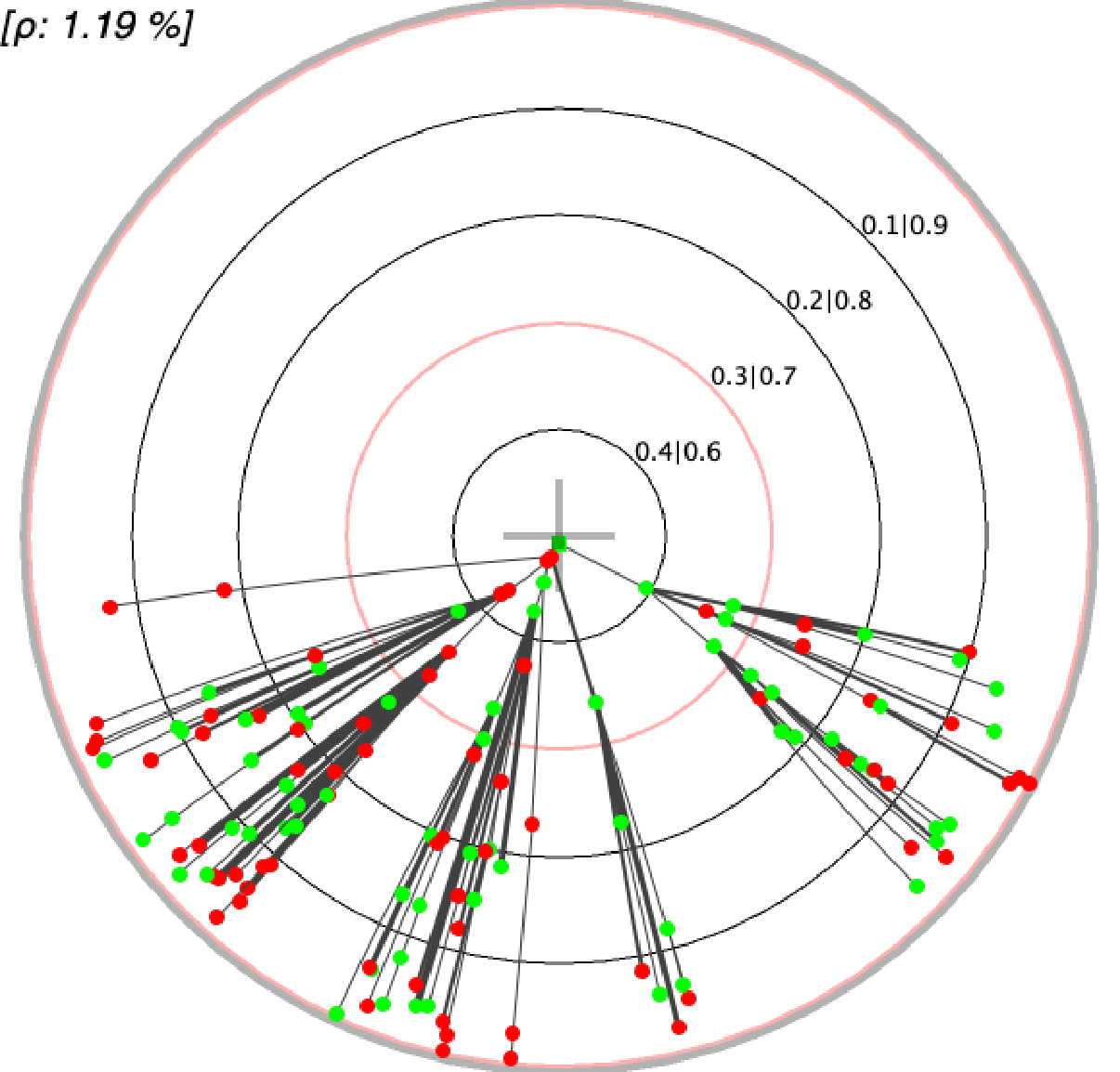}  & \includegraphics[trim=0bp 0bp 0bp 0bp,clip,width=0.3\columnwidth]{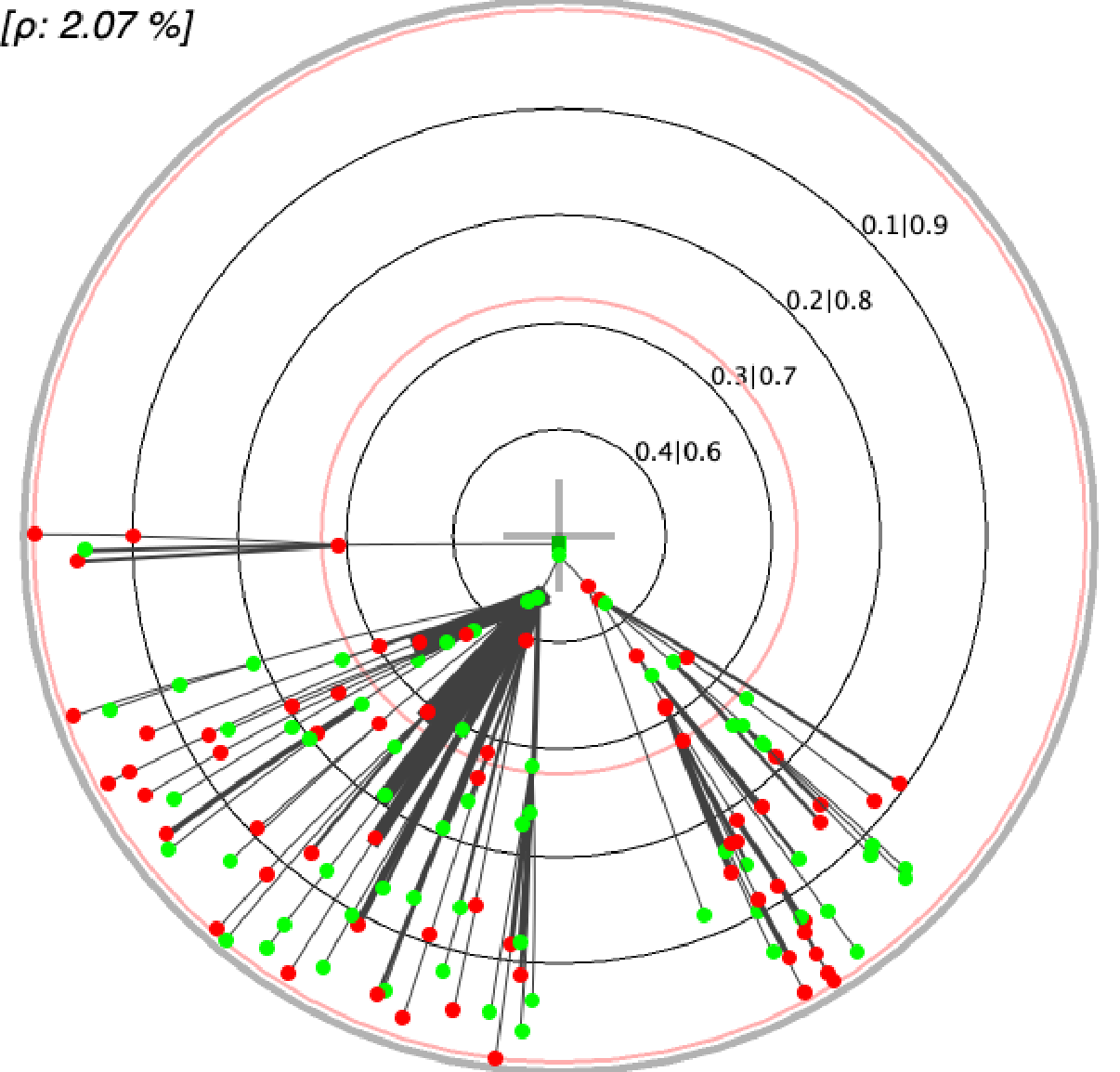}  & \includegraphics[trim=0bp 0bp 0bp 0bp,clip,width=0.3\columnwidth]{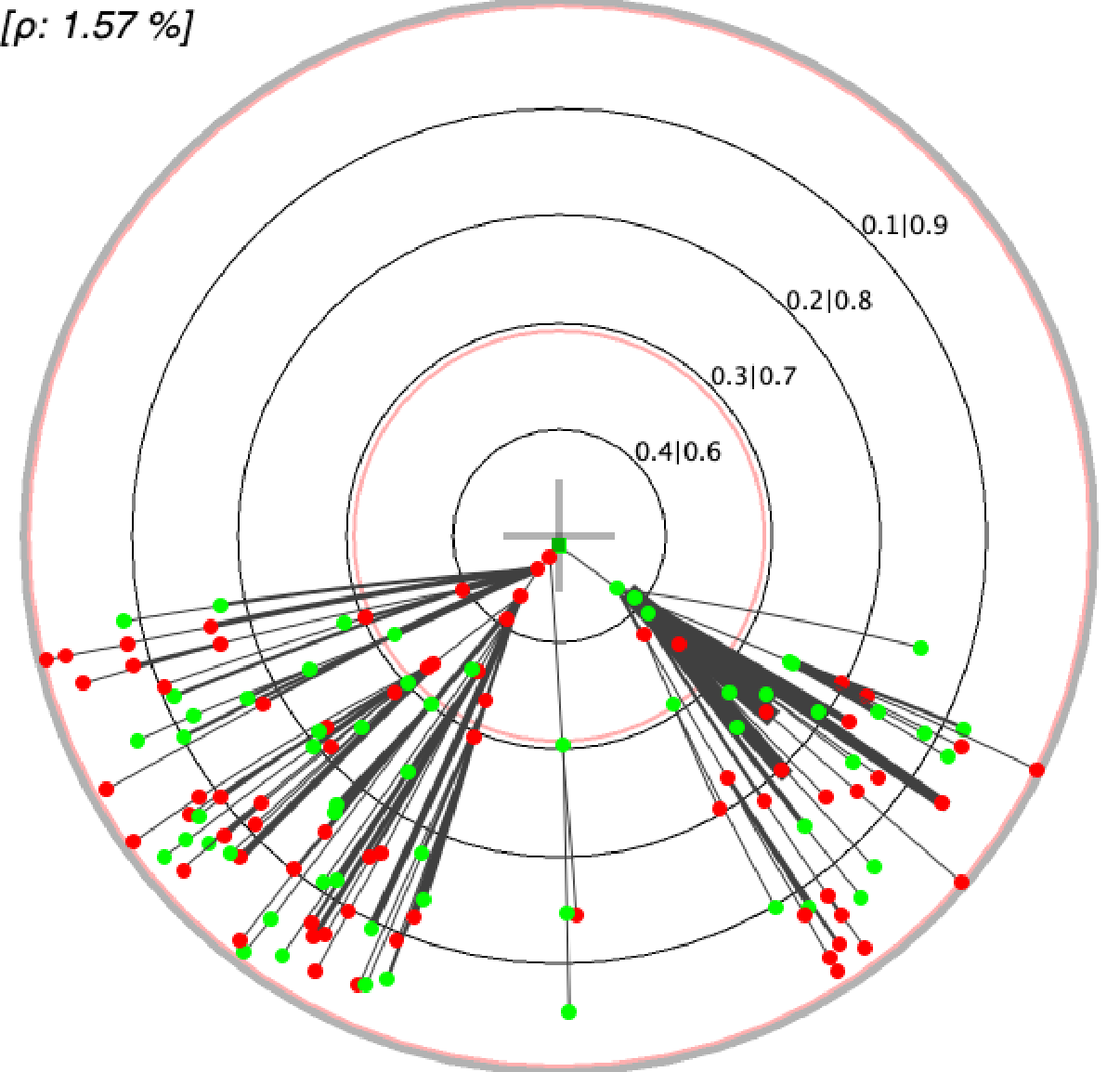} \\
                              MDT $\#6$ & MDT $\#7$& MDT $\#8$& MDT $\#9$& MDT $\#10$\\  \Xhline{2pt} 
  \end{tabular}}
\caption{First 10 MDTs for UCI \domainname{german}. Convention follows Table \ref{tab:online-shopping-intentions-dt-exerpt}.}
    \label{tab:german-dt-exerpt}
  \end{table}

    \begin{table}
  \centering
  \resizebox{\textwidth}{!}{\begin{tabular}{ccccc}\Xhline{2pt}
                              \includegraphics[trim=0bp 0bp 0bp 0bp,clip,width=0.3\columnwidth]{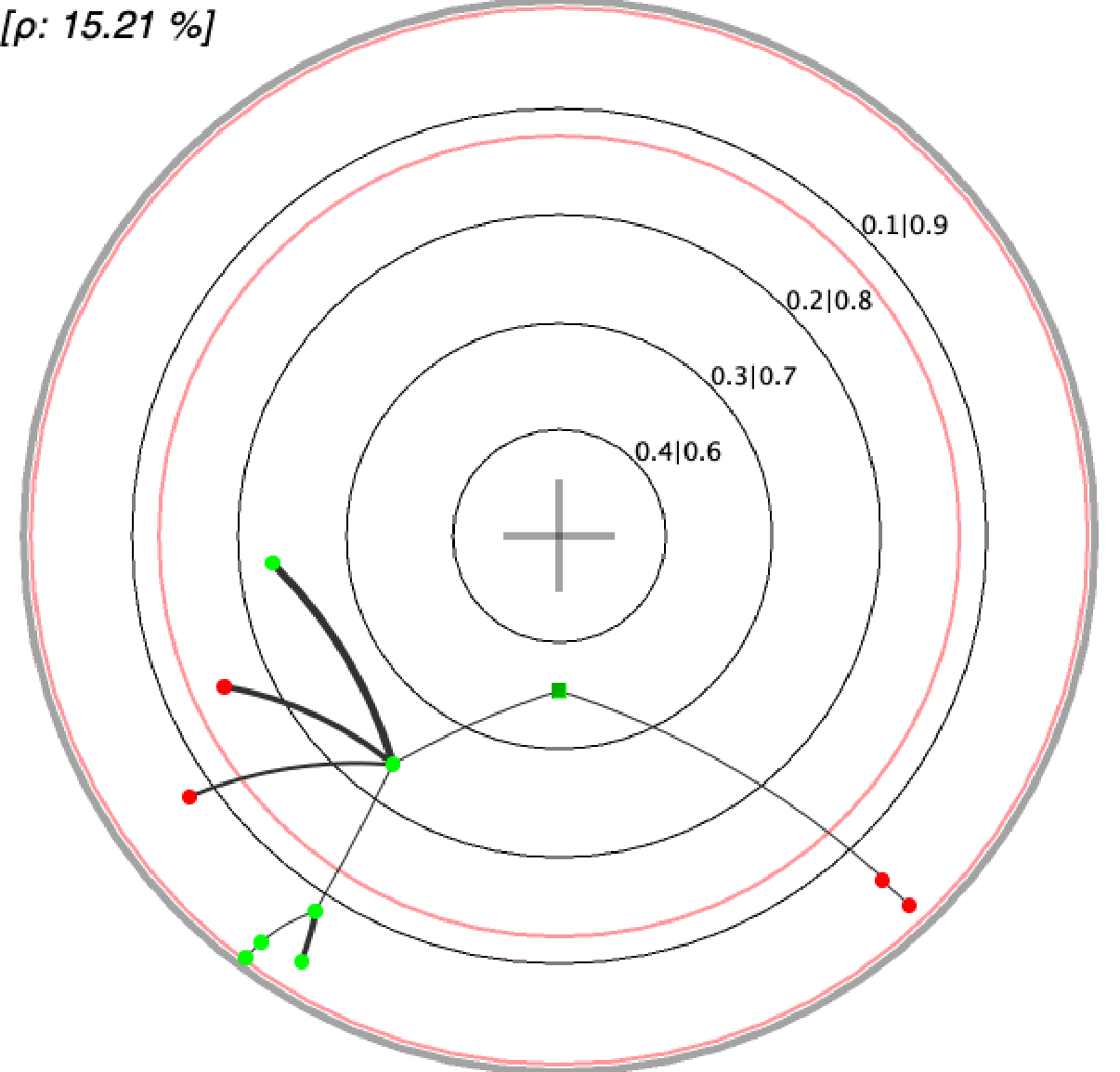} & \includegraphics[trim=0bp 0bp 0bp 0bp,clip,width=0.3\columnwidth]{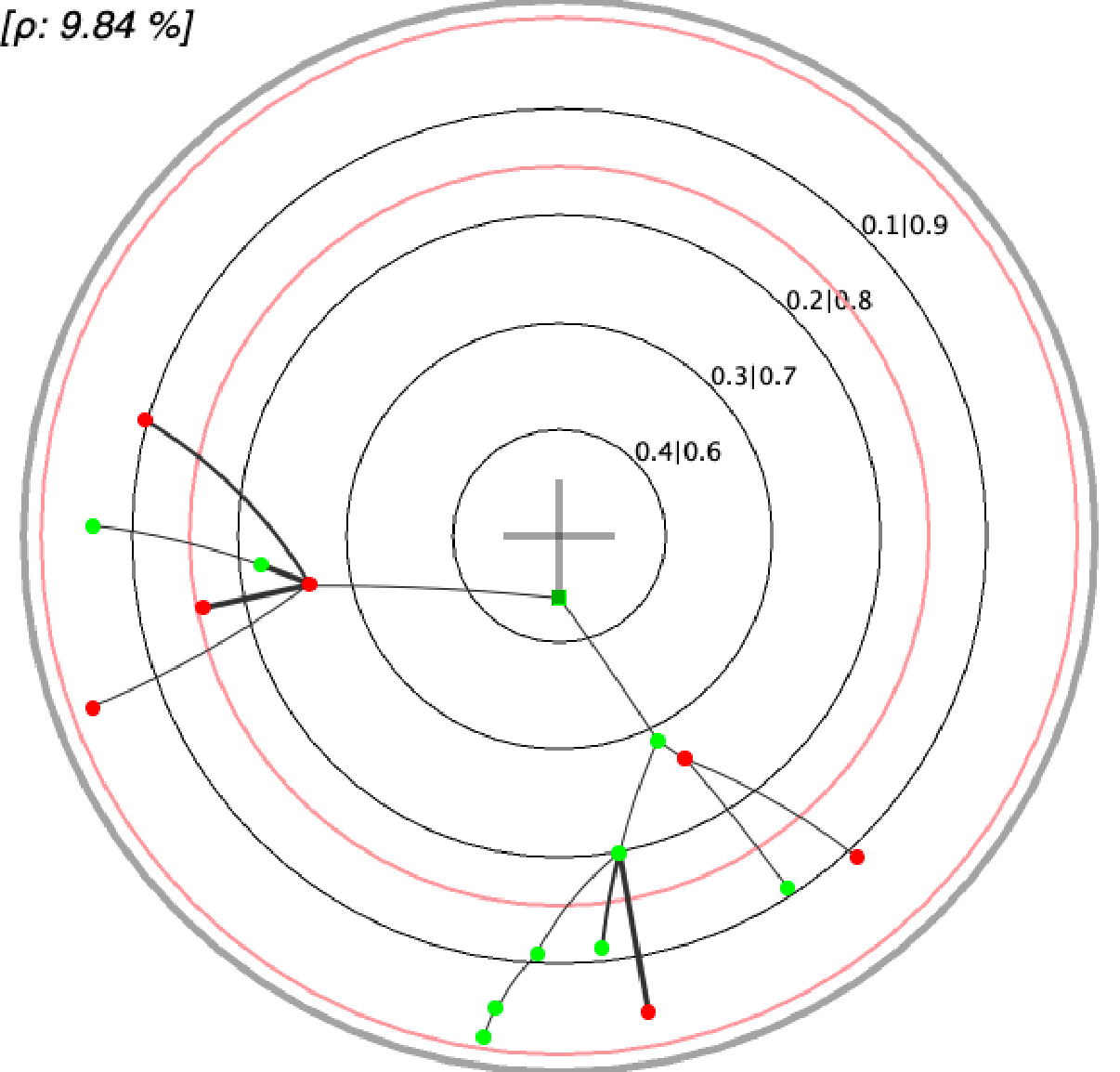} & \includegraphics[trim=0bp 0bp 0bp 0bp,clip,width=0.3\columnwidth]{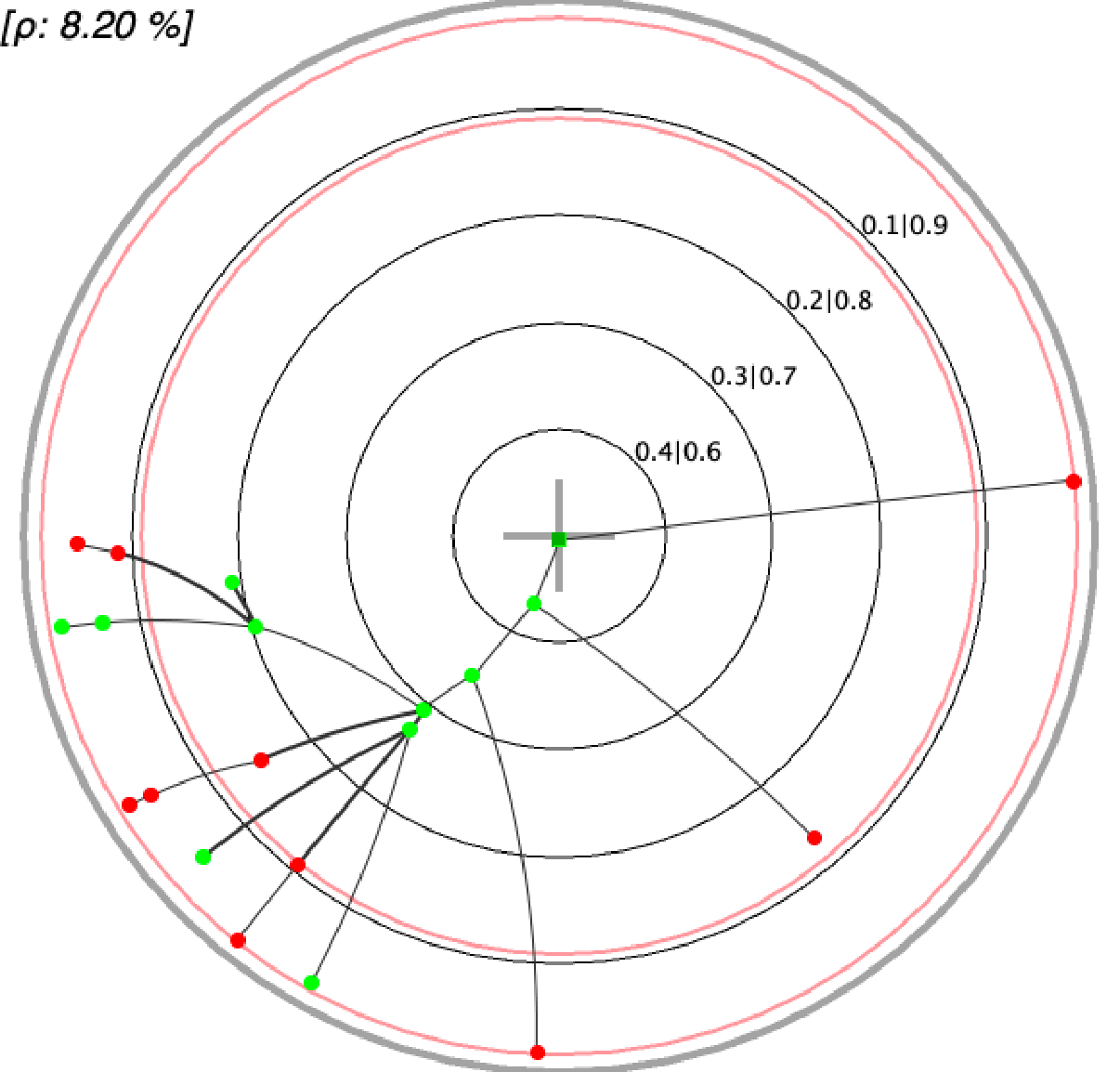}  & \includegraphics[trim=0bp 0bp 0bp 0bp,clip,width=0.3\columnwidth]{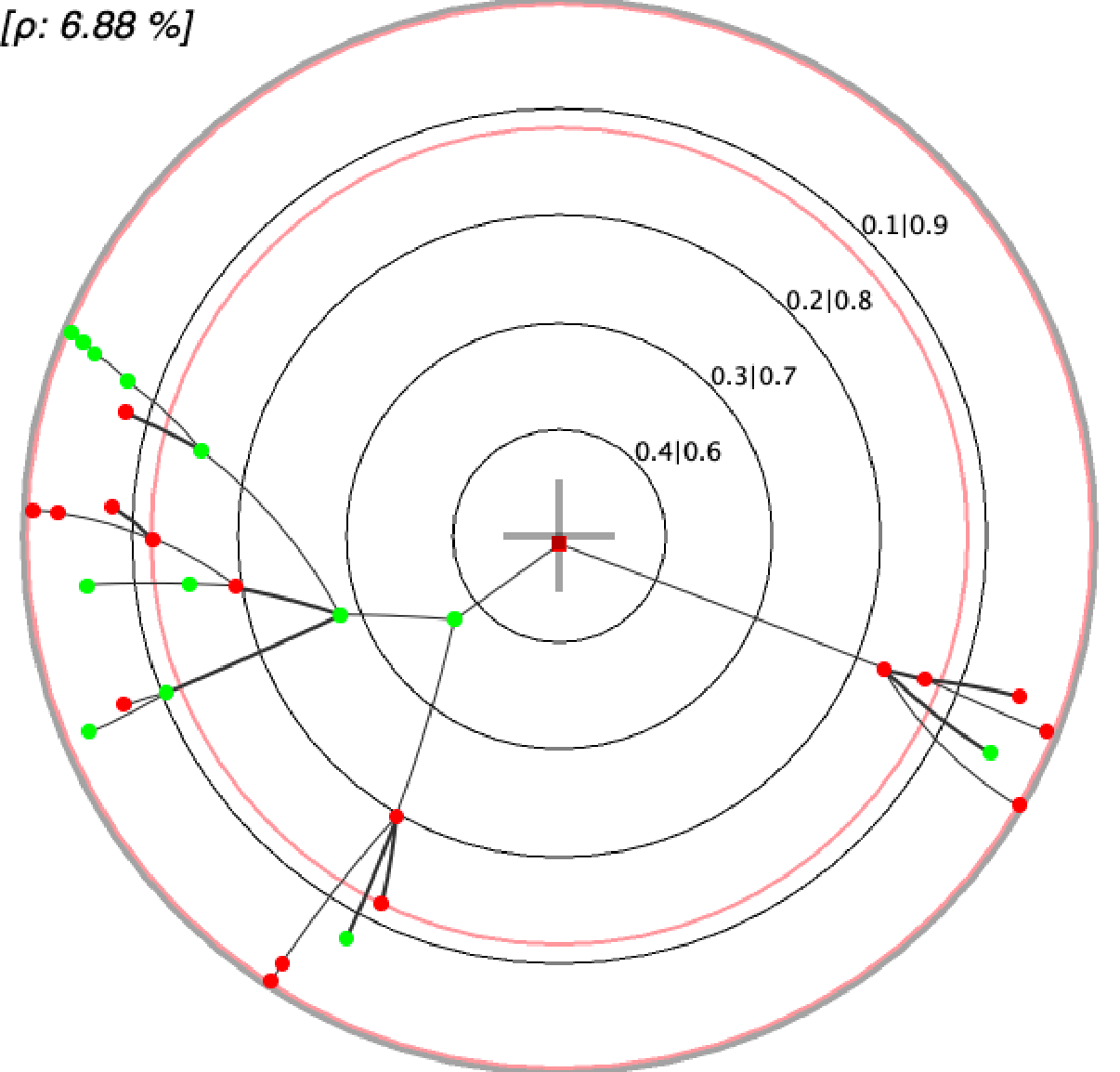}  & \includegraphics[trim=0bp 0bp 0bp 0bp,clip,width=0.3\columnwidth]{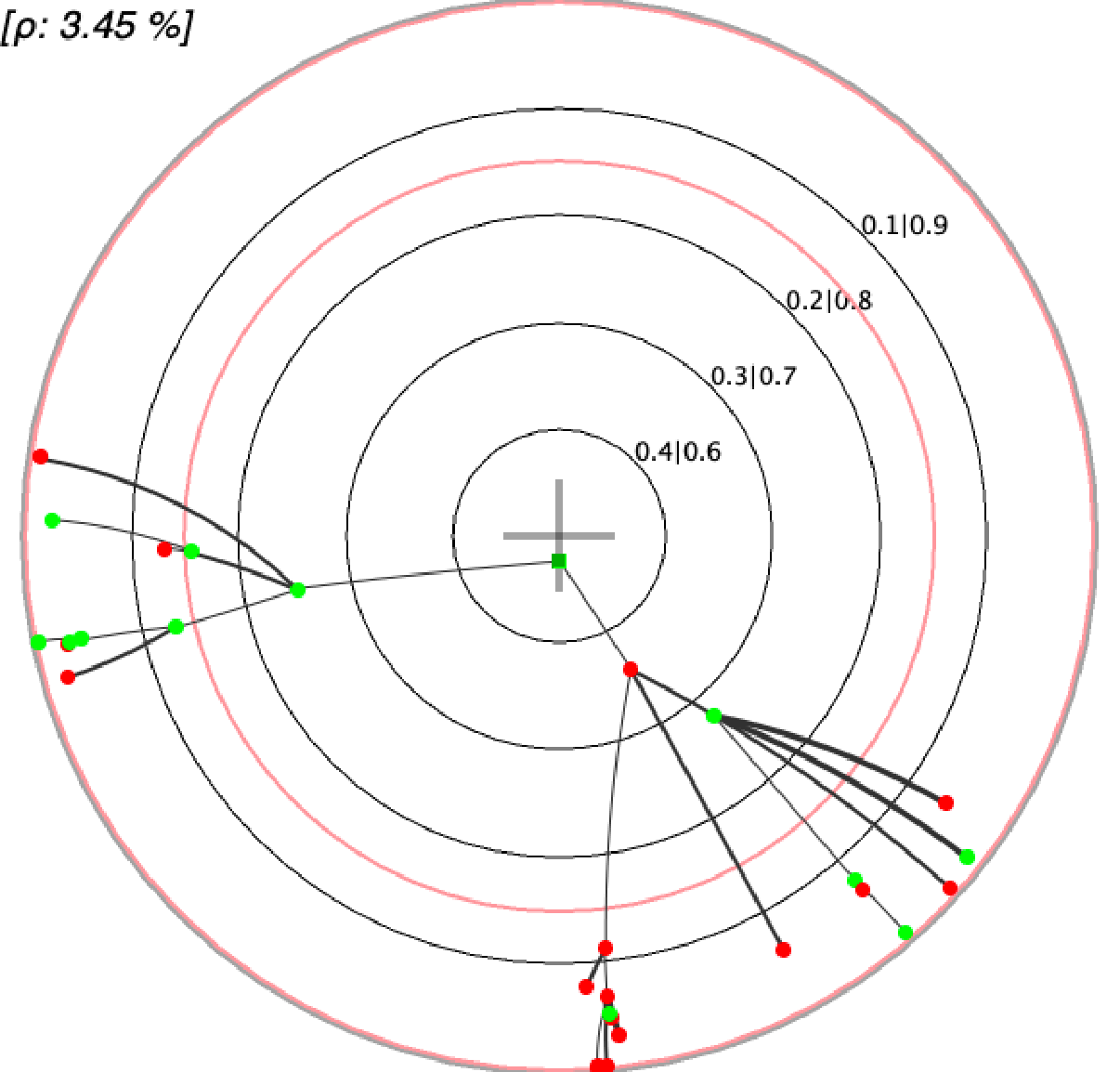} \\
                              MDT $\#1$ & MDT $\#2$& MDT $\#3$& MDT $\#4$& MDT $\#5$\\
                             \includegraphics[trim=0bp 0bp 0bp 0bp,clip,width=0.3\columnwidth]{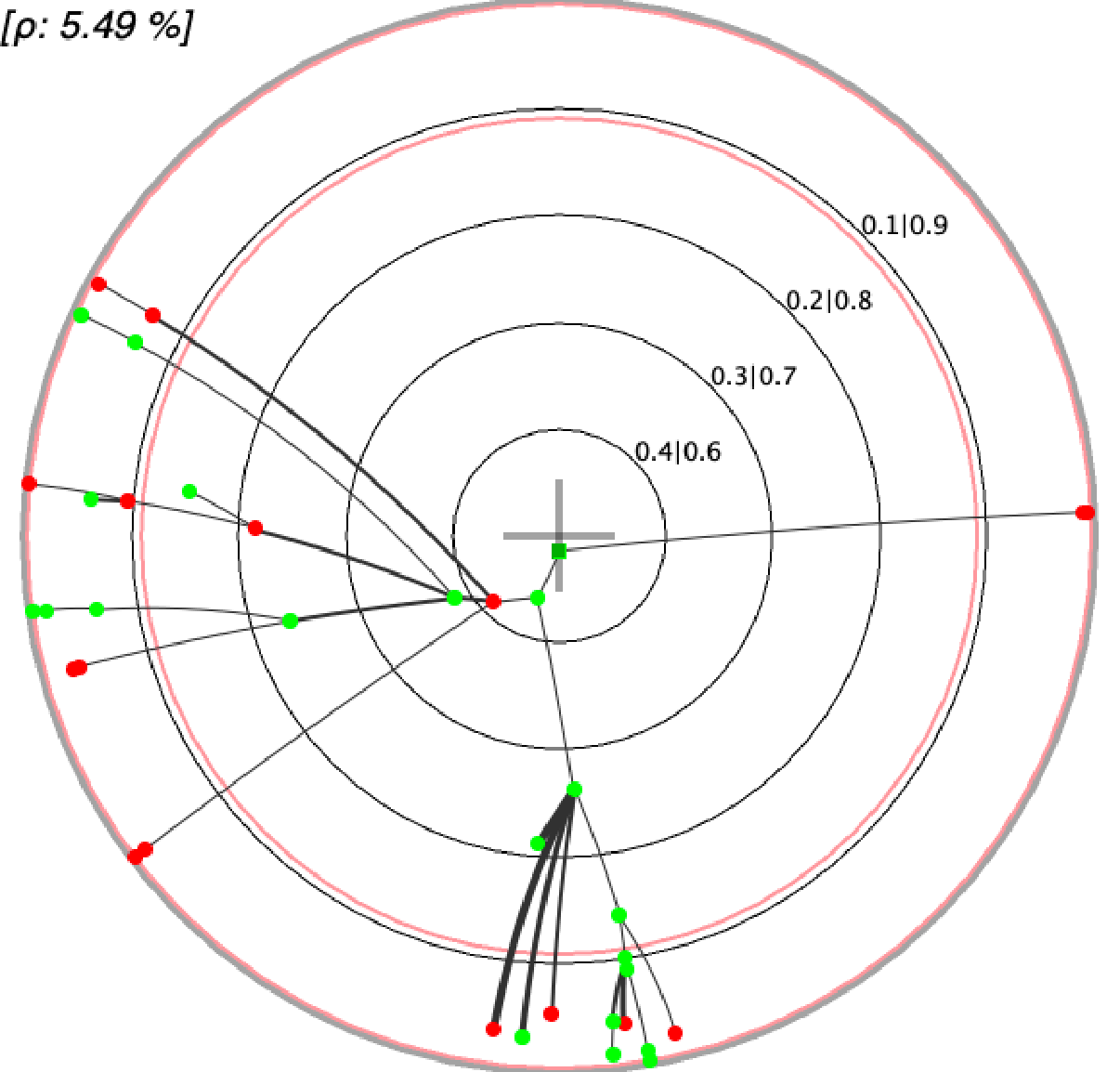} & \includegraphics[trim=0bp 0bp 0bp 0bp,clip,width=0.3\columnwidth]{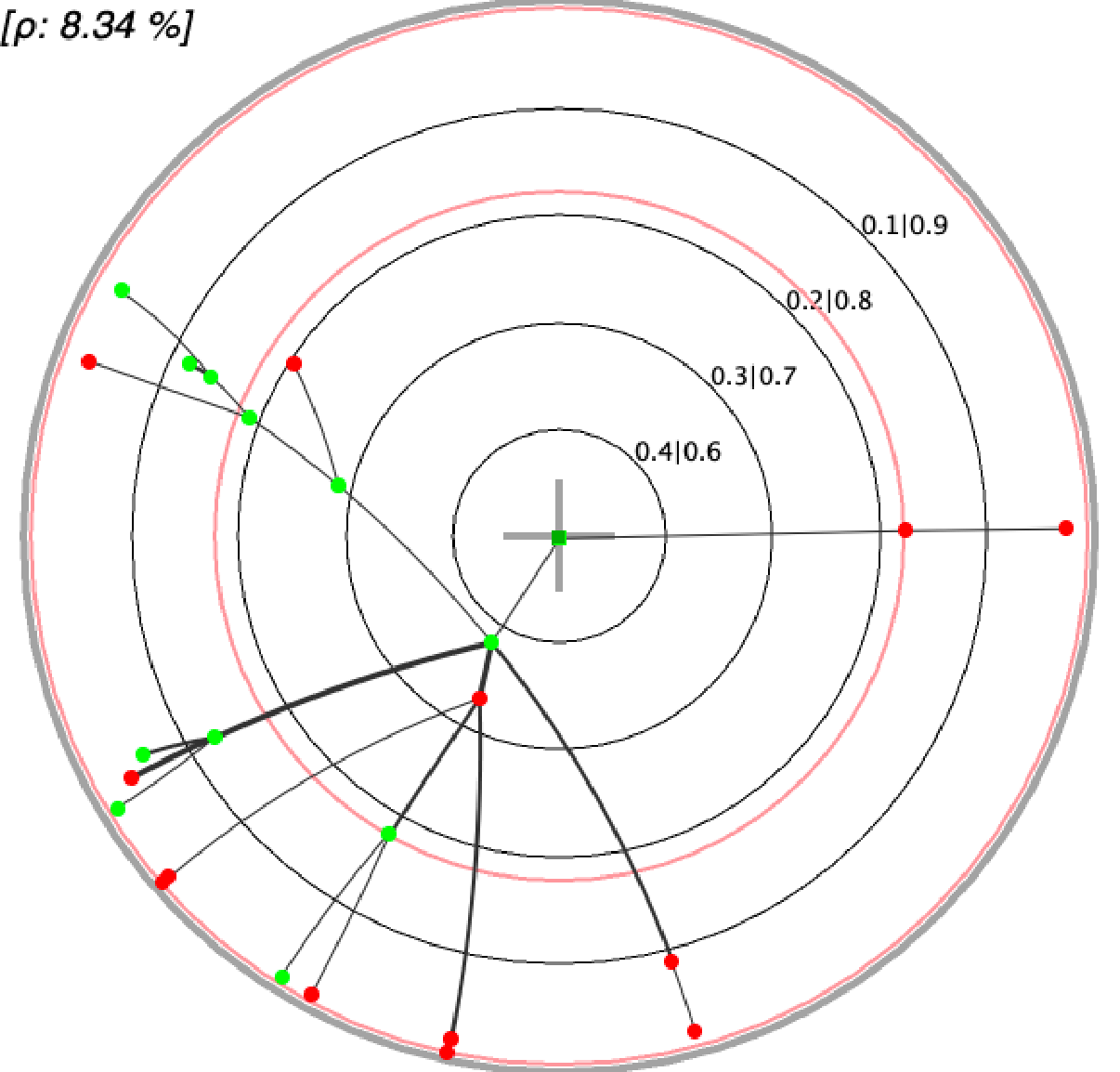} & \includegraphics[trim=0bp 0bp 0bp 0bp,clip,width=0.3\columnwidth]{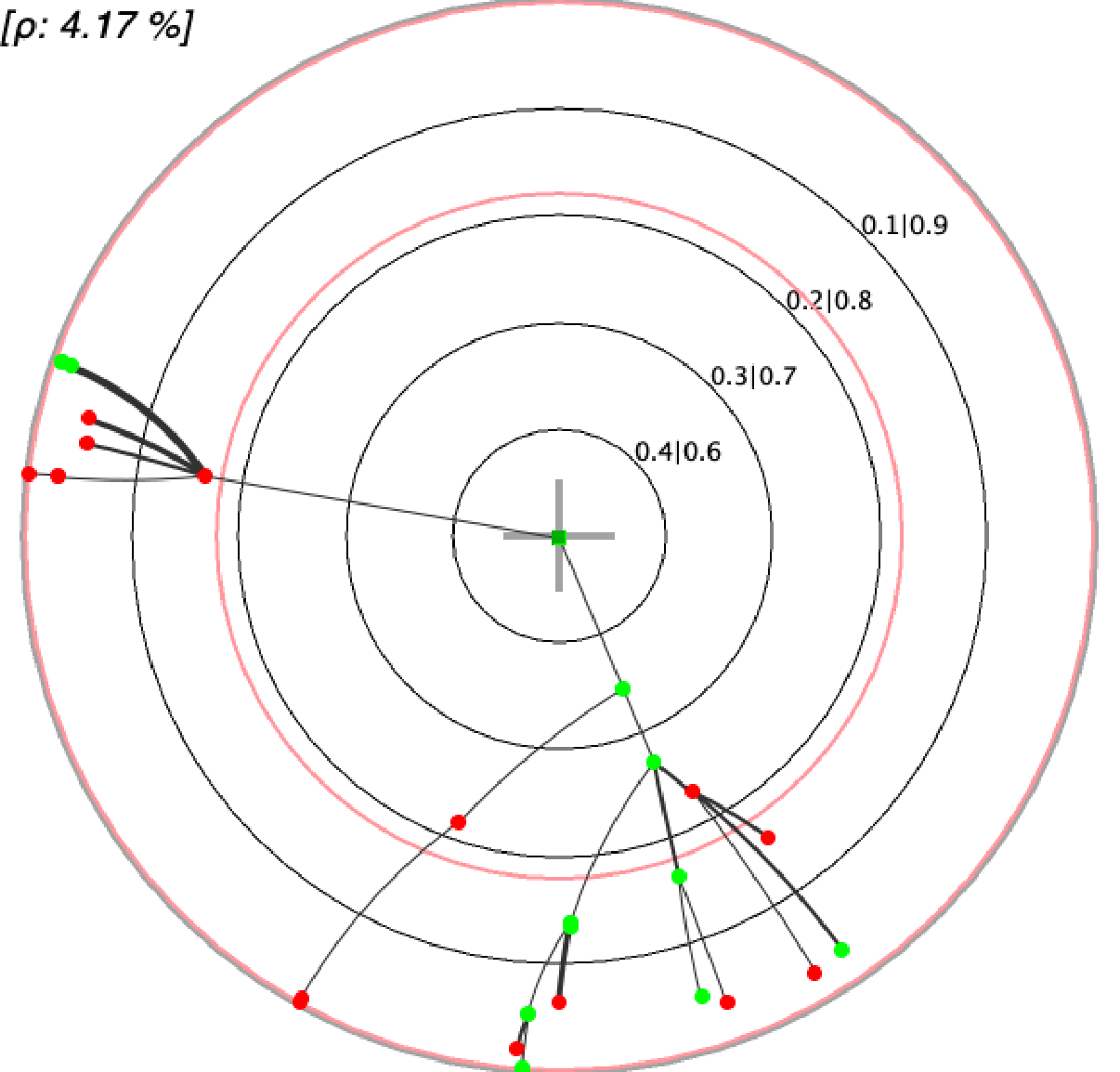}  & \includegraphics[trim=0bp 0bp 0bp 0bp,clip,width=0.3\columnwidth]{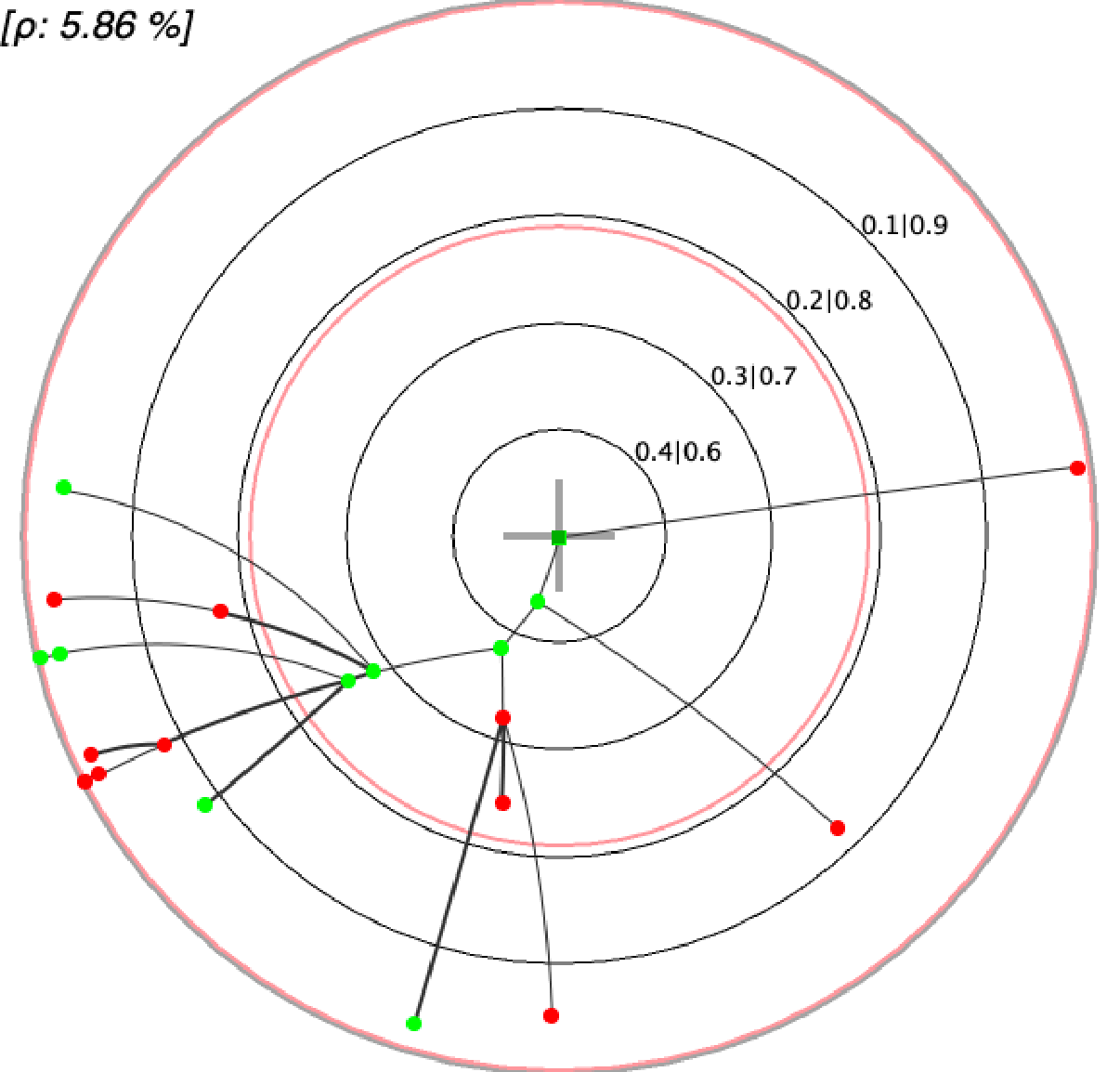}  & \includegraphics[trim=0bp 0bp 0bp 0bp,clip,width=0.3\columnwidth]{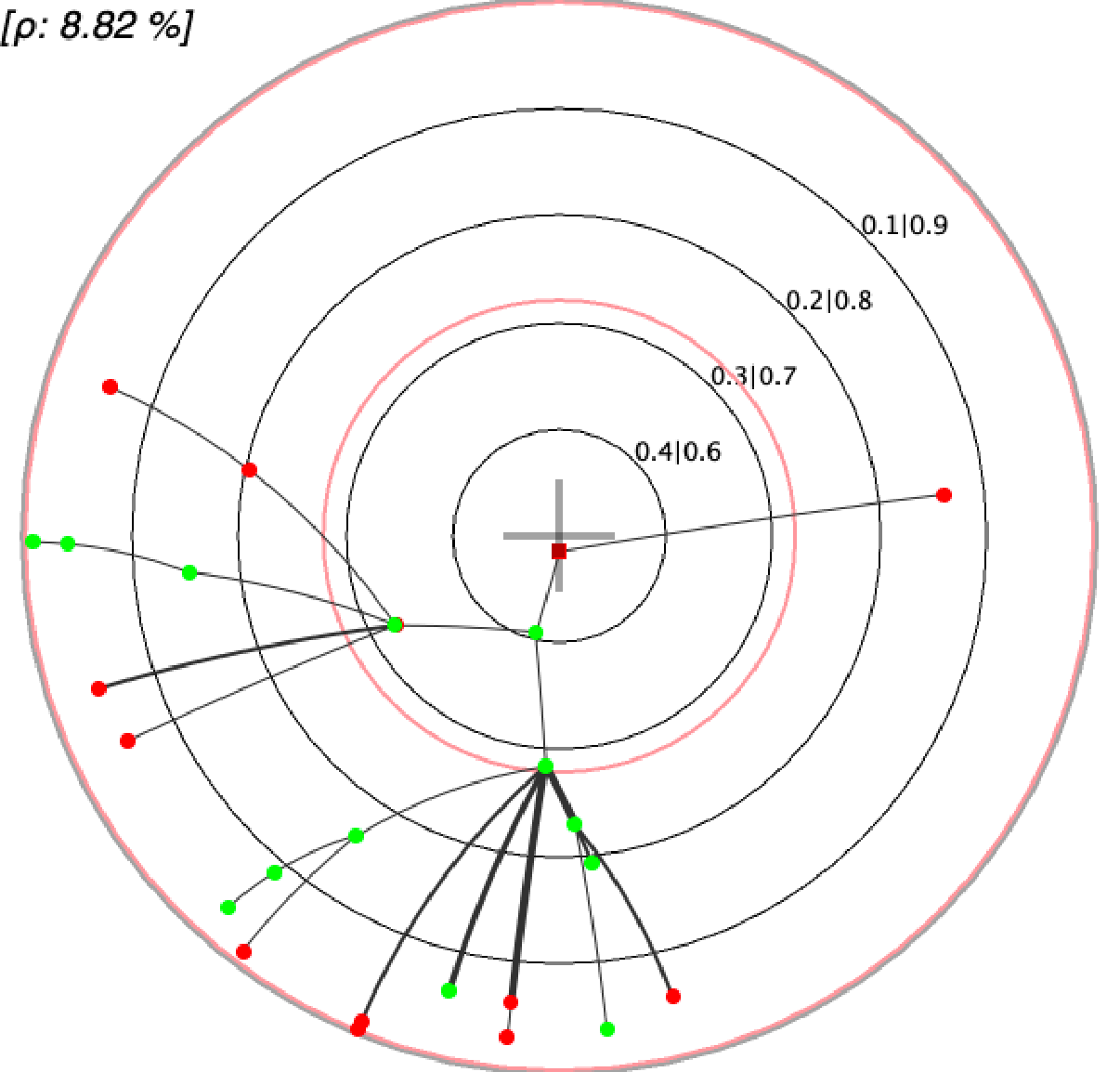} \\
                              MDT $\#6$ & MDT $\#7$& MDT $\#8$& MDT $\#9$& MDT $\#10$\\  \Xhline{2pt} 
  \end{tabular}}
\caption{First 10 MDTs for UCI \domainname{ionosphere}. Convention follows Table \ref{tab:online-shopping-intentions-dt-exerpt}.}
    \label{tab:ionosphere-dt-exerpt}
  \end{table}

    \begin{table}
  \centering
  \resizebox{\textwidth}{!}{\begin{tabular}{ccccc}\Xhline{2pt}
                              \includegraphics[trim=0bp 0bp 0bp 0bp,clip,width=0.3\columnwidth]{Experiments/poincare_DT/hardware/treeplot_Jan_24th__6h_52m_50s_USE_BOOSTING_WEIGHTS_Algo0_SplitCV3_Tree0_T10_10e.eps} & \includegraphics[trim=0bp 0bp 0bp 0bp,clip,width=0.3\columnwidth]{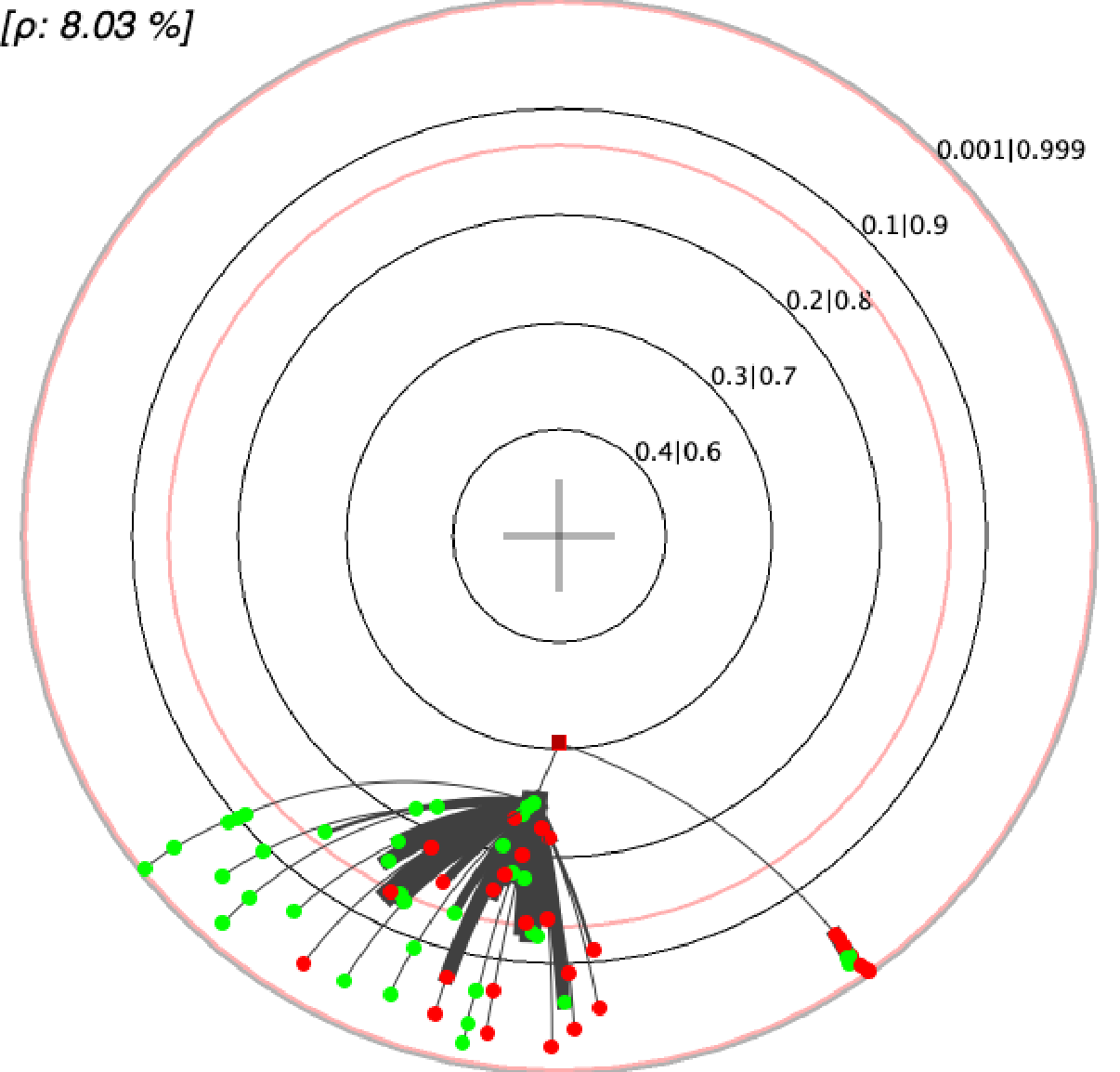} & \includegraphics[trim=0bp 0bp 0bp 0bp,clip,width=0.3\columnwidth]{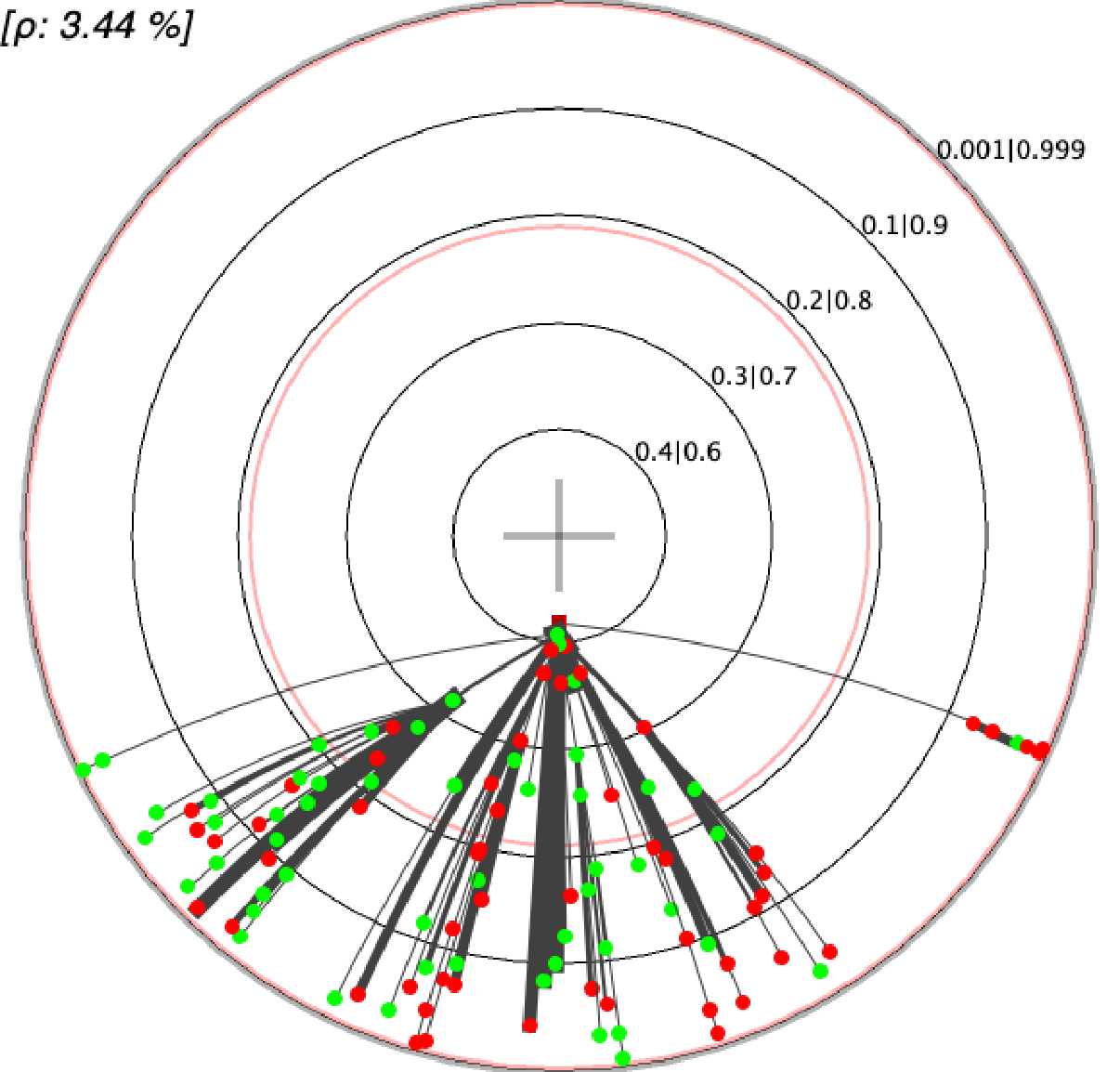}  & \includegraphics[trim=0bp 0bp 0bp 0bp,clip,width=0.3\columnwidth]{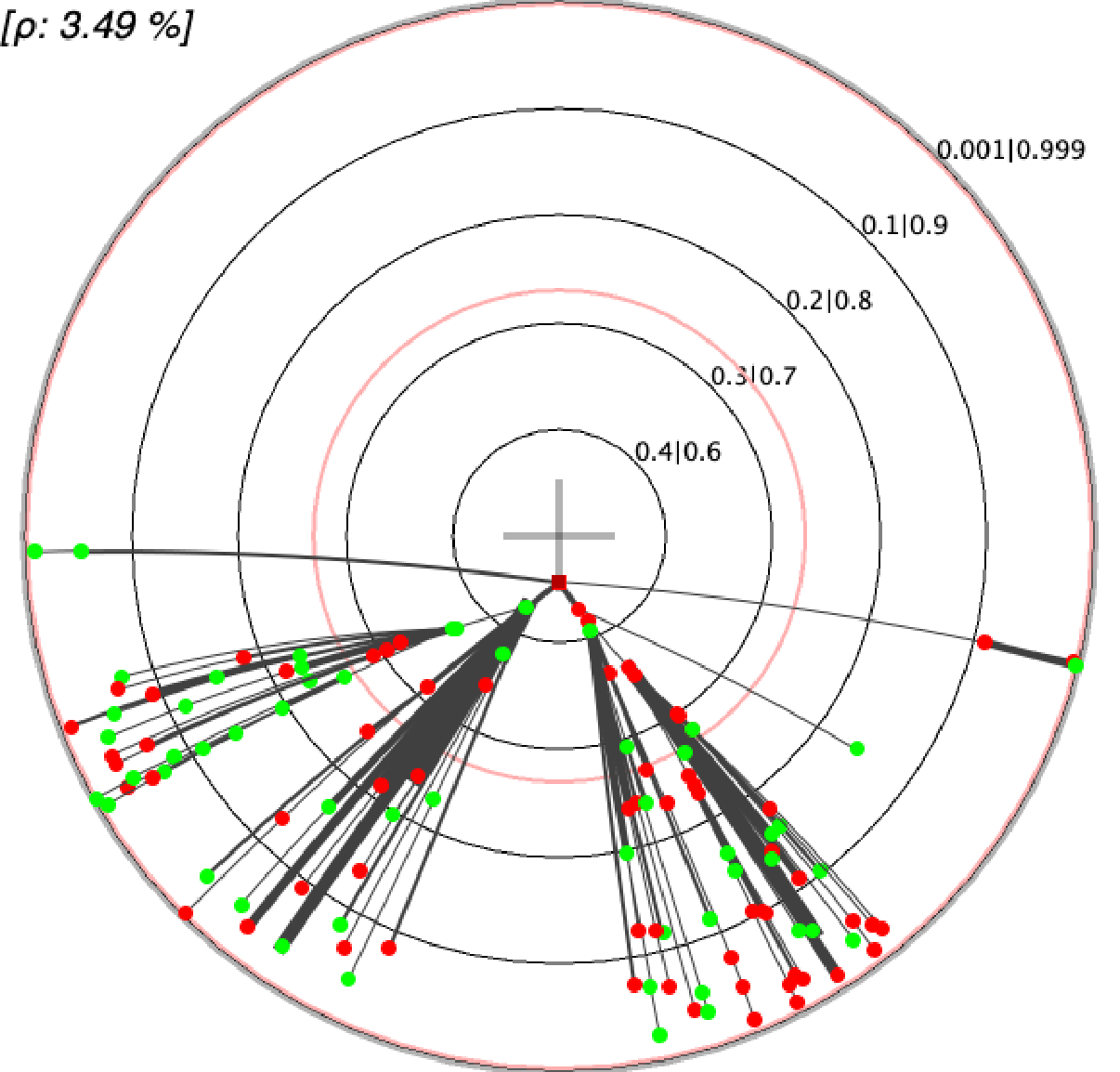}  & \includegraphics[trim=0bp 0bp 0bp 0bp,clip,width=0.3\columnwidth]{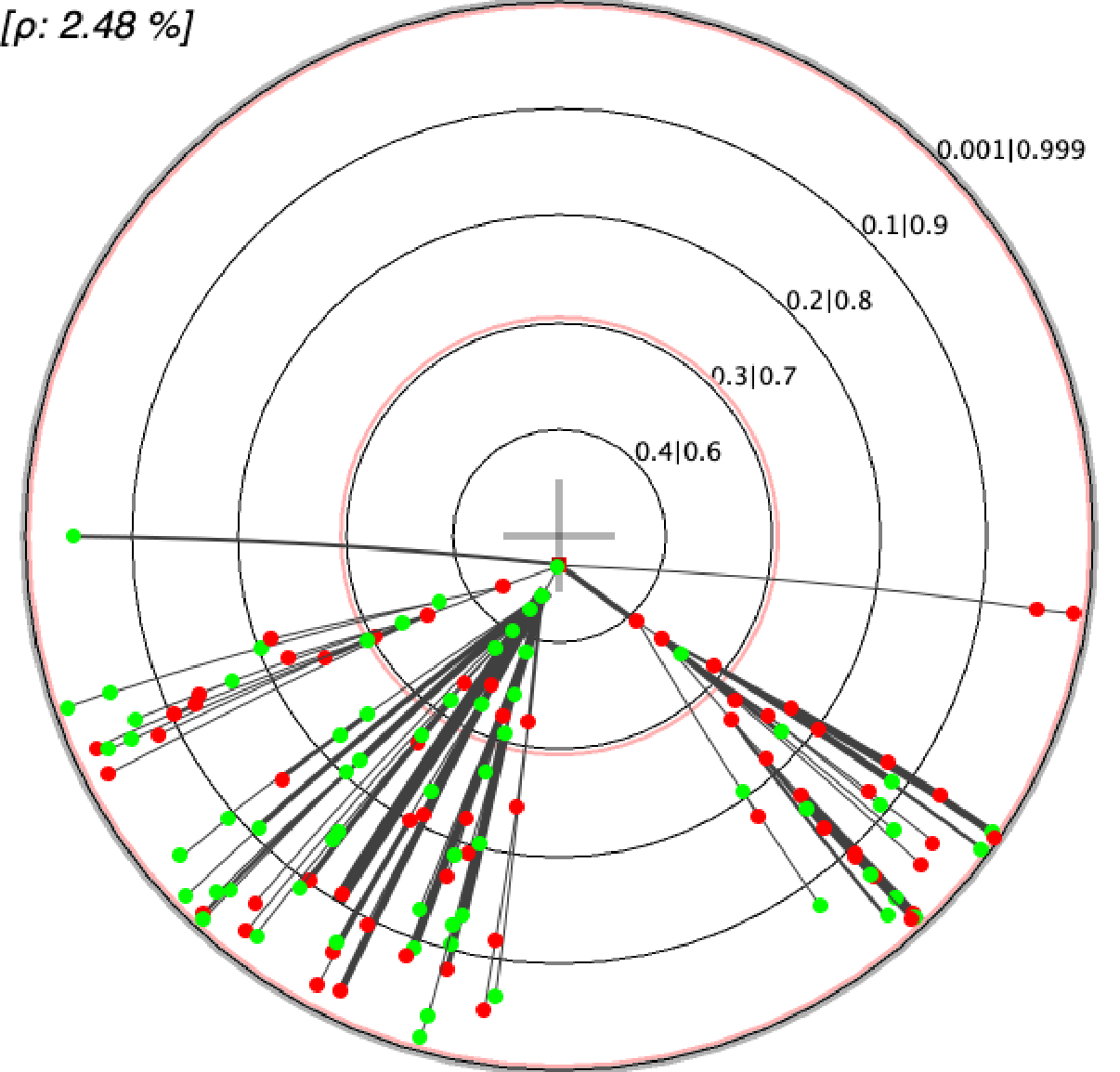} \\
                              MDT $\#1$ & MDT $\#2$& MDT $\#3$& MDT $\#4$& MDT $\#5$\\
                             \includegraphics[trim=0bp 0bp 0bp 0bp,clip,width=0.3\columnwidth]{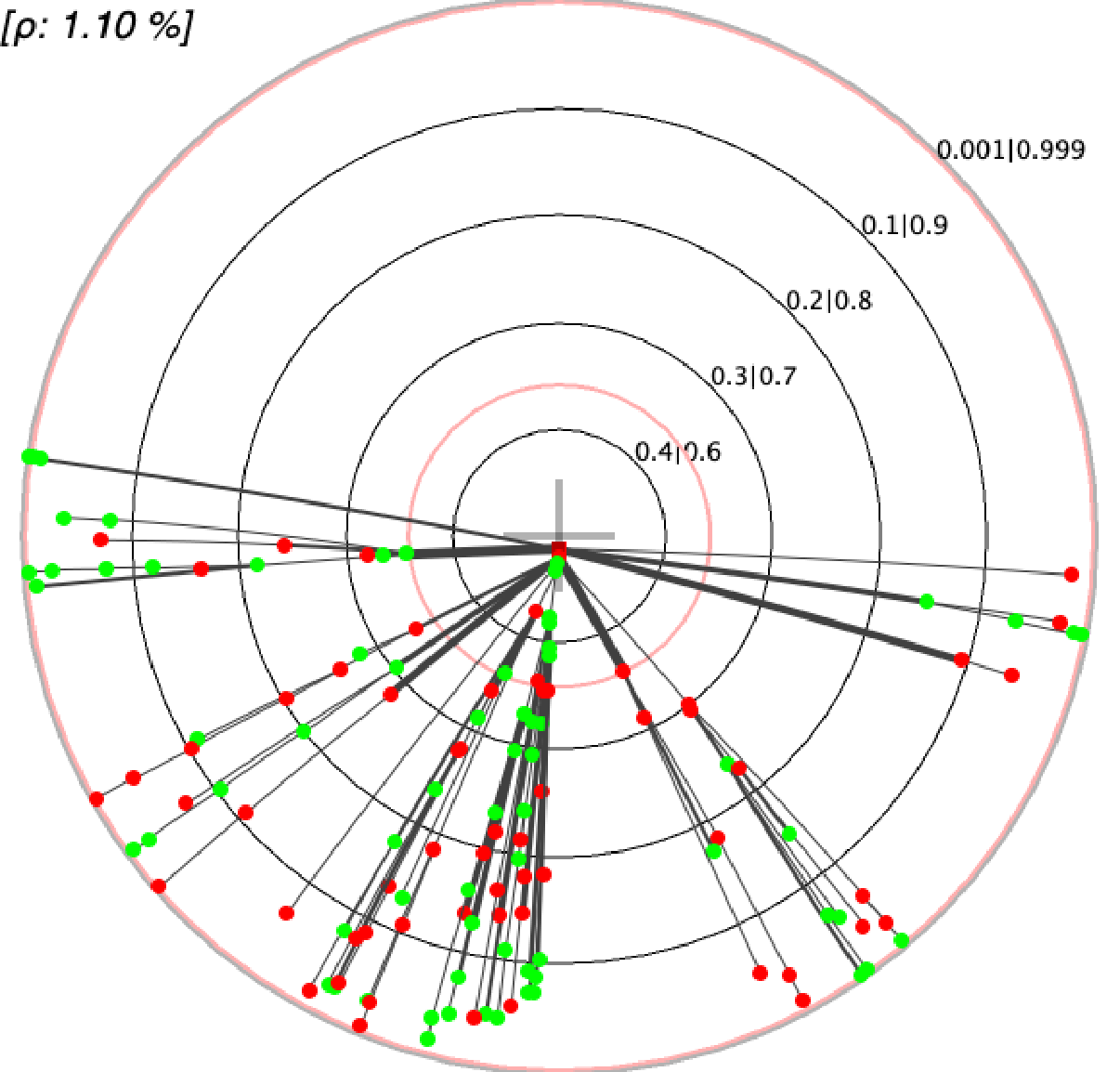} & \includegraphics[trim=0bp 0bp 0bp 0bp,clip,width=0.3\columnwidth]{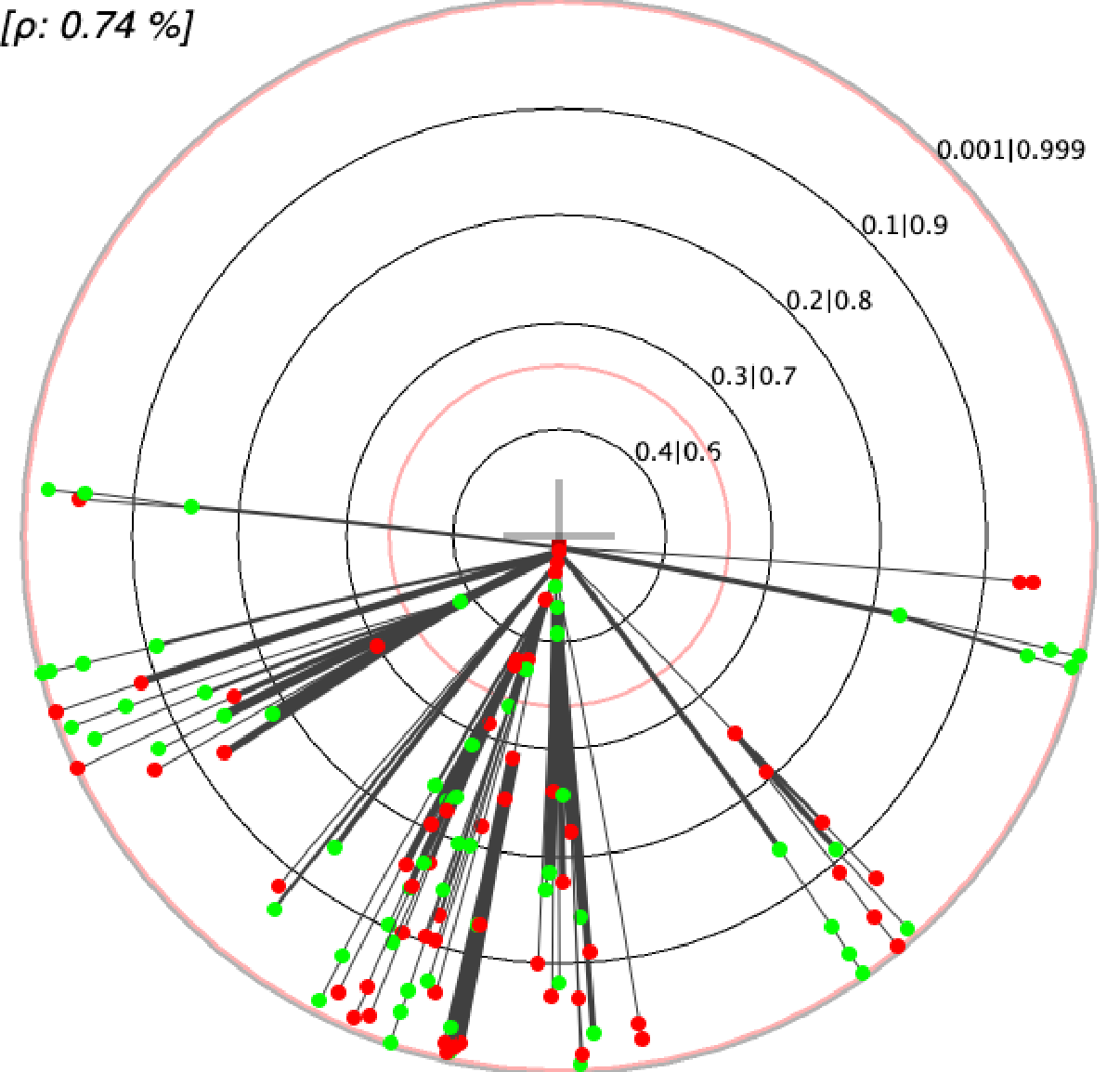} & \includegraphics[trim=0bp 0bp 0bp 0bp,clip,width=0.3\columnwidth]{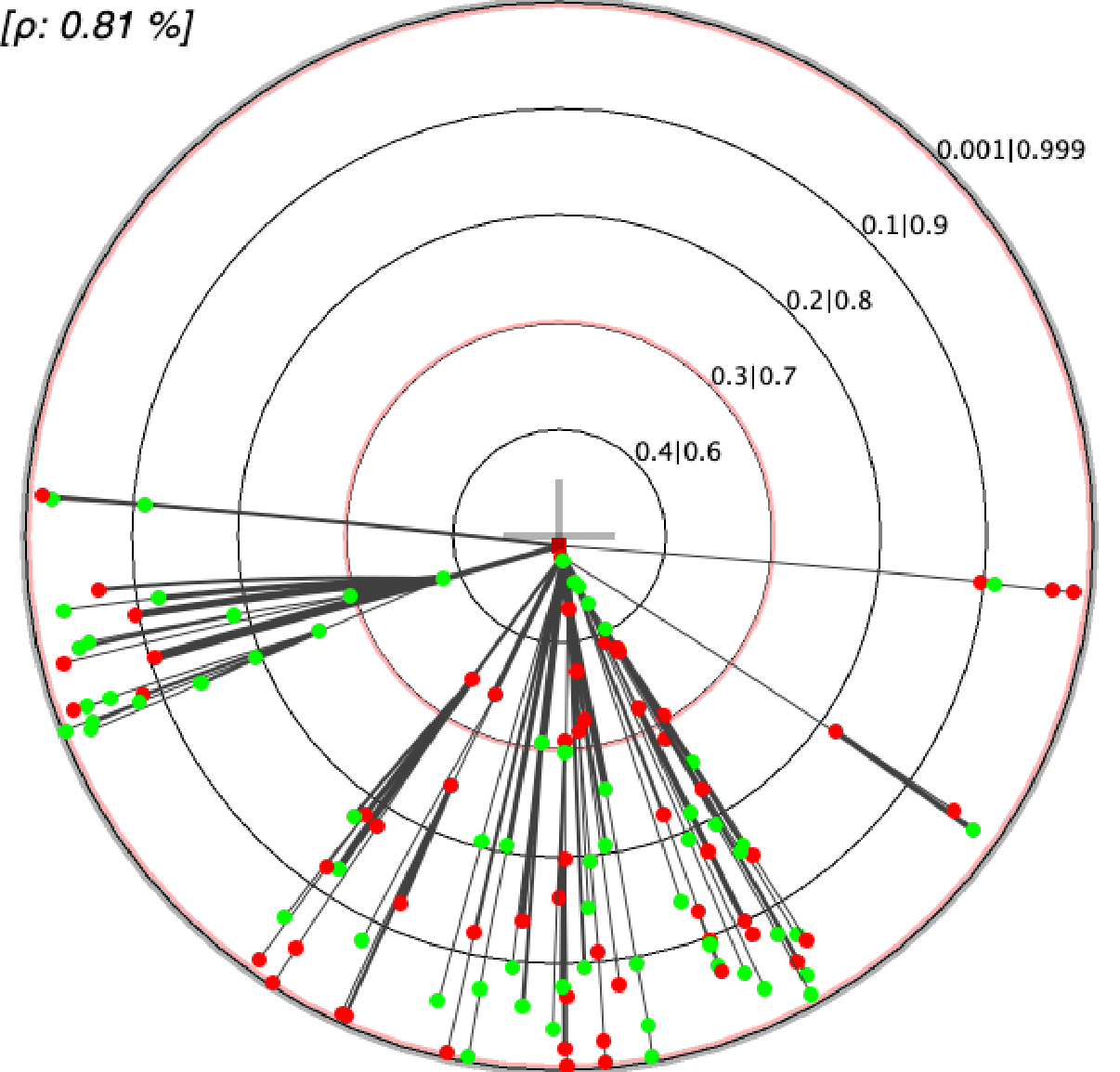}  & \includegraphics[trim=0bp 0bp 0bp 0bp,clip,width=0.3\columnwidth]{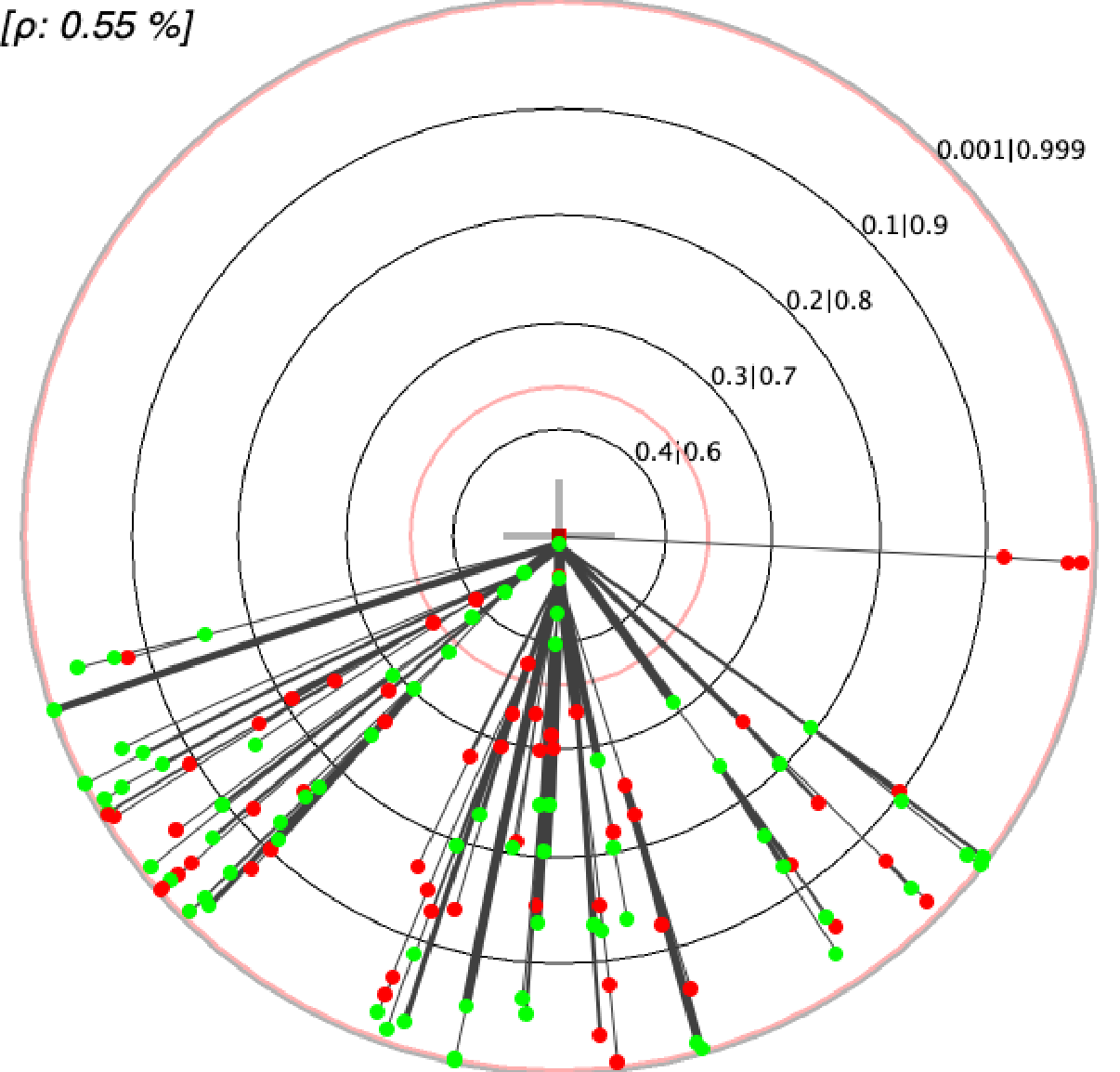}  & \includegraphics[trim=0bp 0bp 0bp 0bp,clip,width=0.3\columnwidth]{Experiments/poincare_DT/hardware/treeplot_Jan_24th__6h_52m_50s_USE_BOOSTING_WEIGHTS_Algo0_SplitCV3_Tree9_T10_10e.eps} \\
                              MDT $\#6$ & MDT $\#7$& MDT $\#8$& MDT $\#9$& MDT $\#10$\\  \Xhline{2pt} 
  \end{tabular}}
\caption{First 10 MDTs for UCI \domainname{hardware}. Convention follows Table \ref{tab:online-shopping-intentions-dt-exerpt}.}
    \label{tab:hardware-dt-exerpt}
  \end{table}

    \begin{table}
  \centering
  \resizebox{\textwidth}{!}{\begin{tabular}{c?ccccc}\Xhline{2pt}
                           \rotatebox[origin=l]{90}{$t=1.0$} & \includegraphics[trim=0bp 0bp 0bp 0bp,clip,width=0.3\columnwidth]{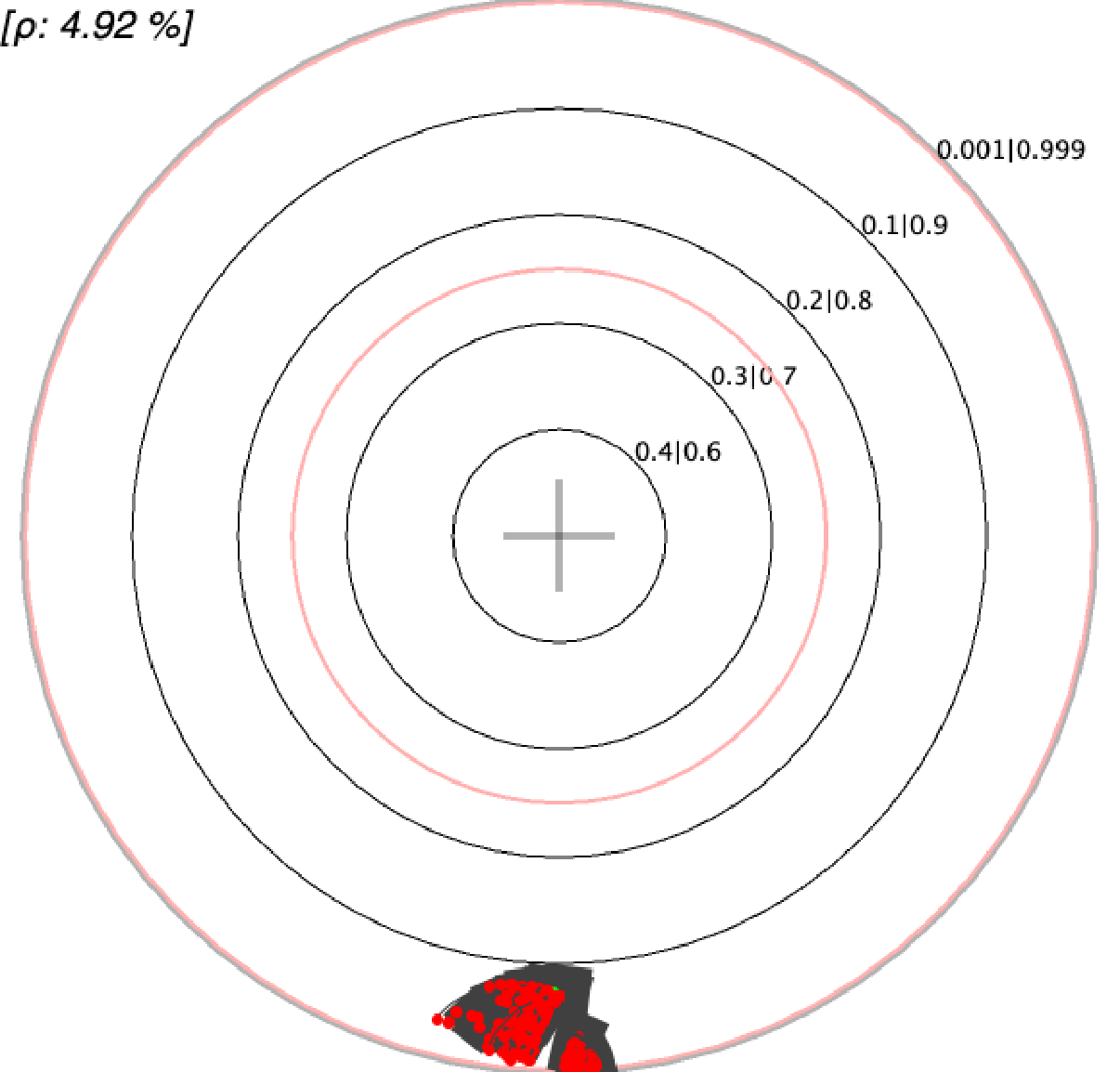} & \includegraphics[trim=0bp 0bp 0bp 0bp,clip,width=0.3\columnwidth]{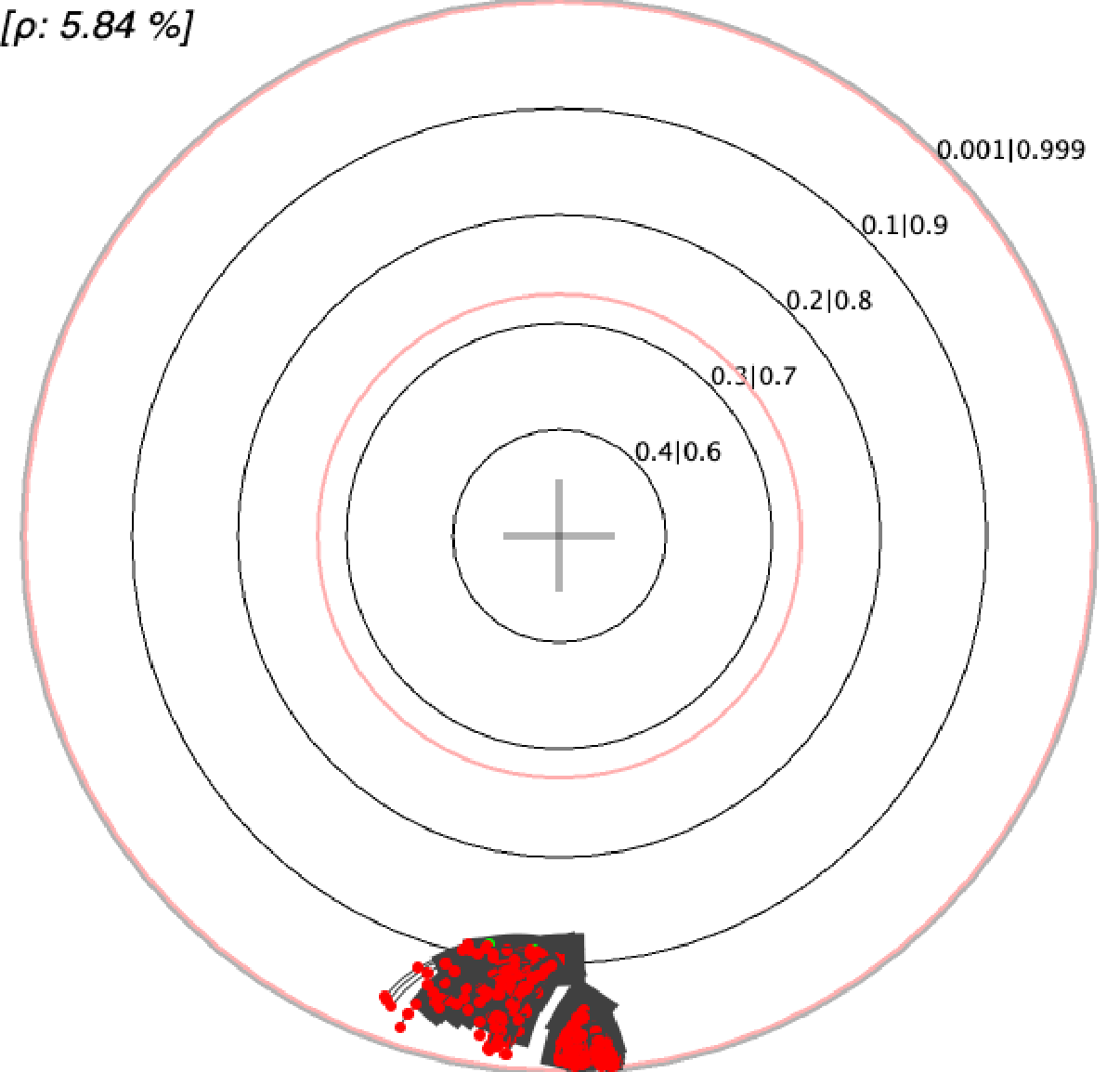} & \includegraphics[trim=0bp 0bp 0bp 0bp,clip,width=0.3\columnwidth]{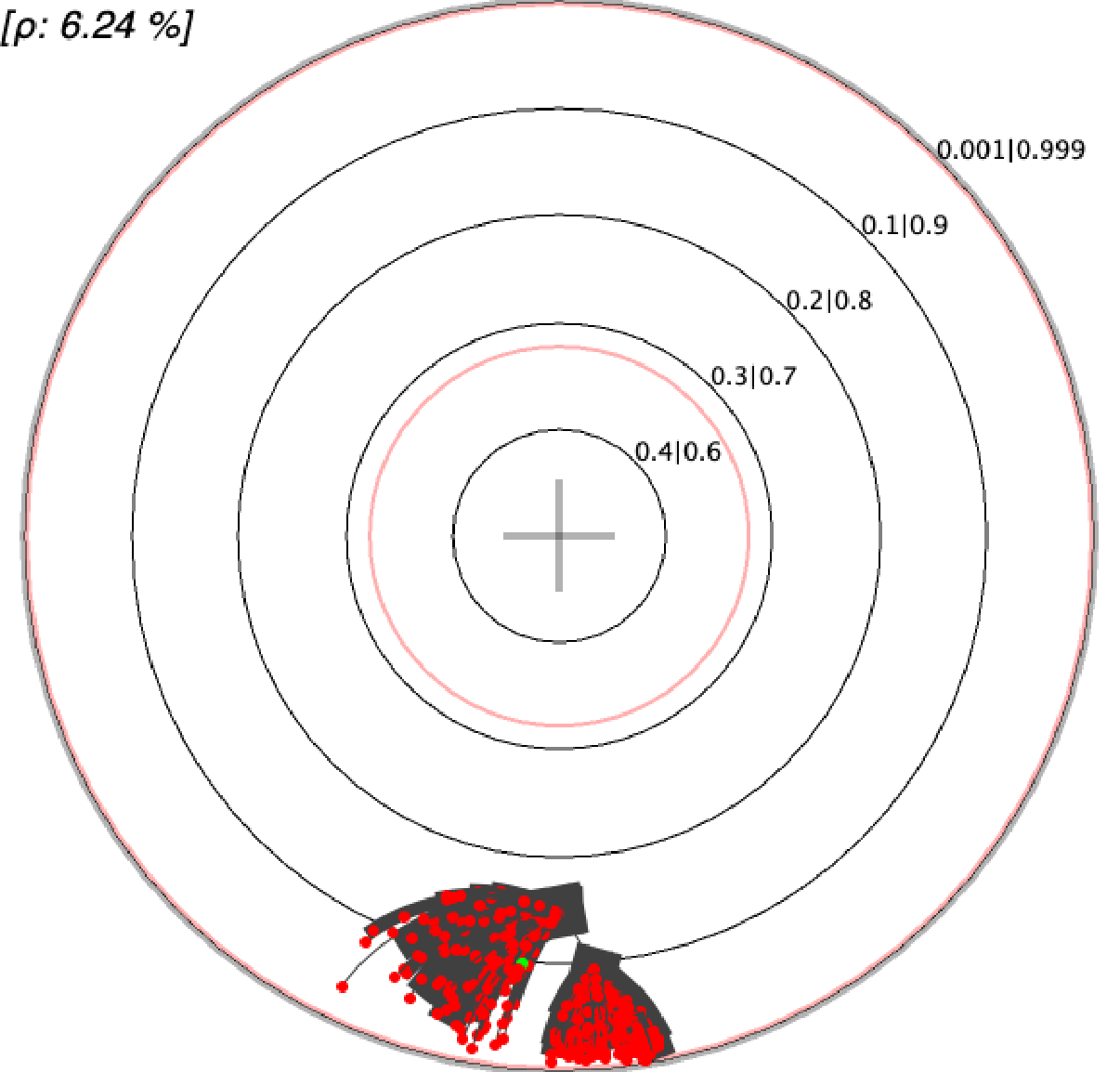} & \includegraphics[trim=0bp 0bp 0bp 0bp,clip,width=0.3\columnwidth]{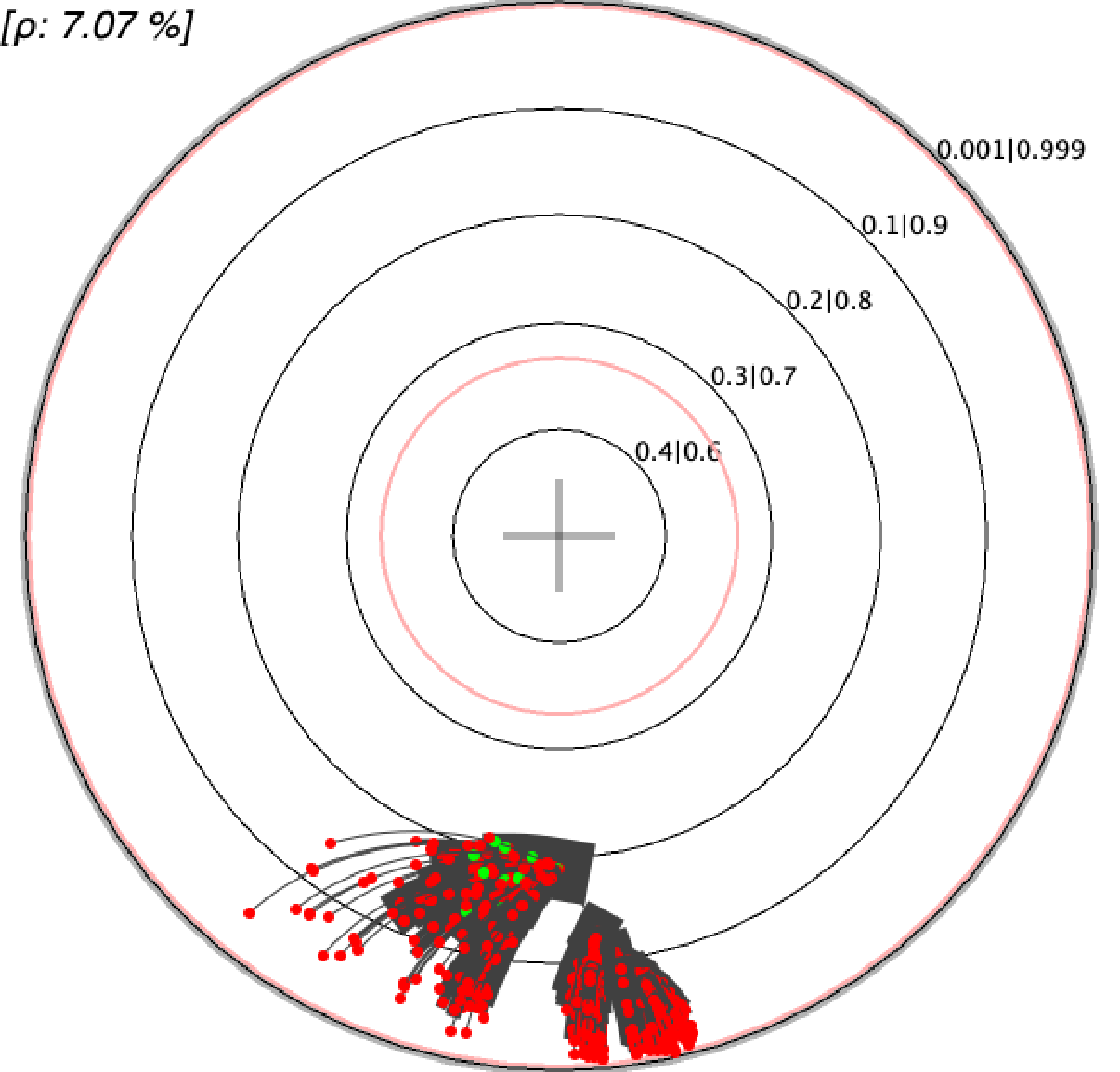} & \includegraphics[trim=0bp 0bp 0bp 0bp,clip,width=0.3\columnwidth]{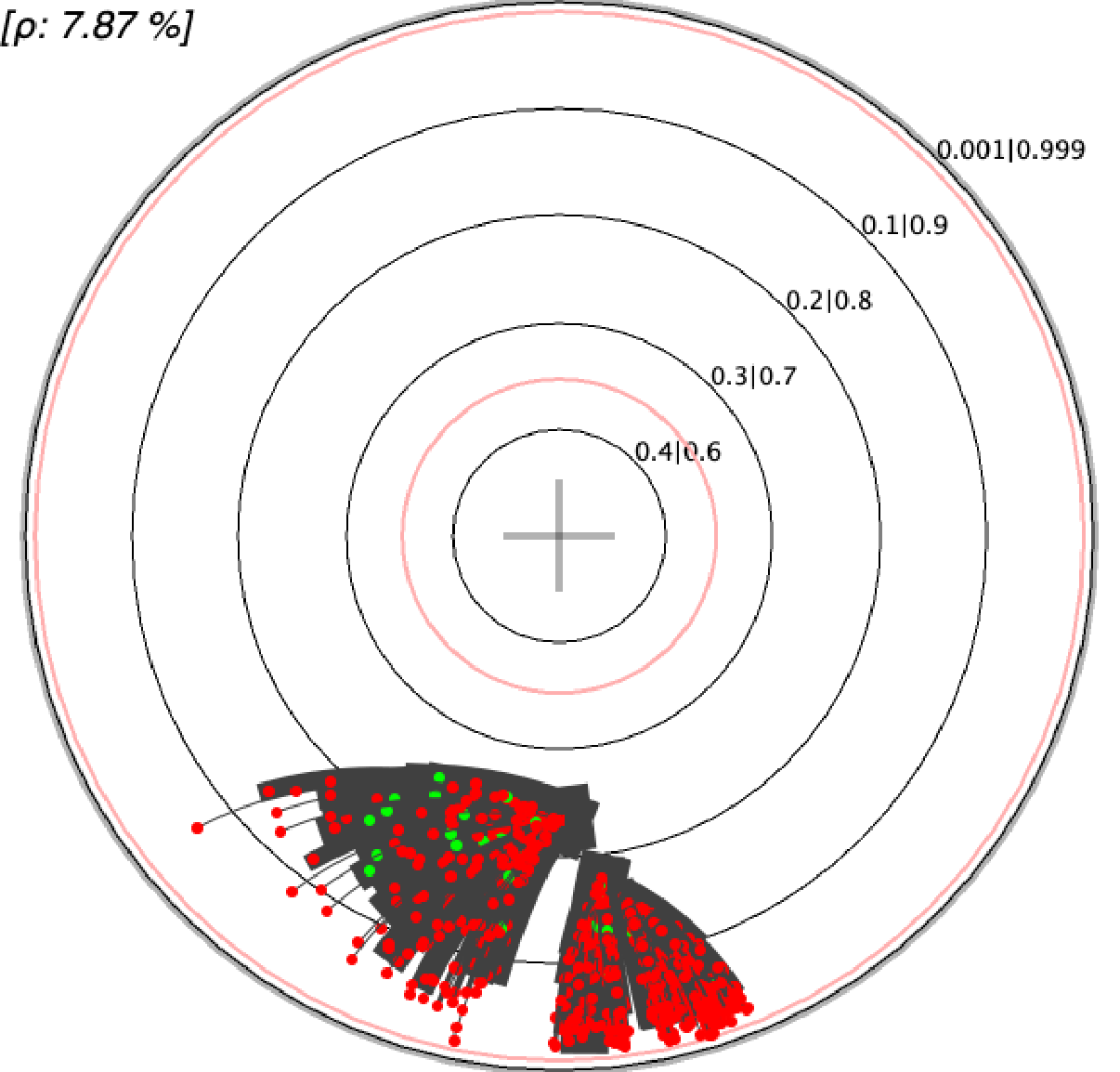}  \\
                             \rotatebox[origin=l]{90}{$t=0.0$} & \includegraphics[trim=0bp 0bp 0bp 0bp,clip,width=0.3\columnwidth]{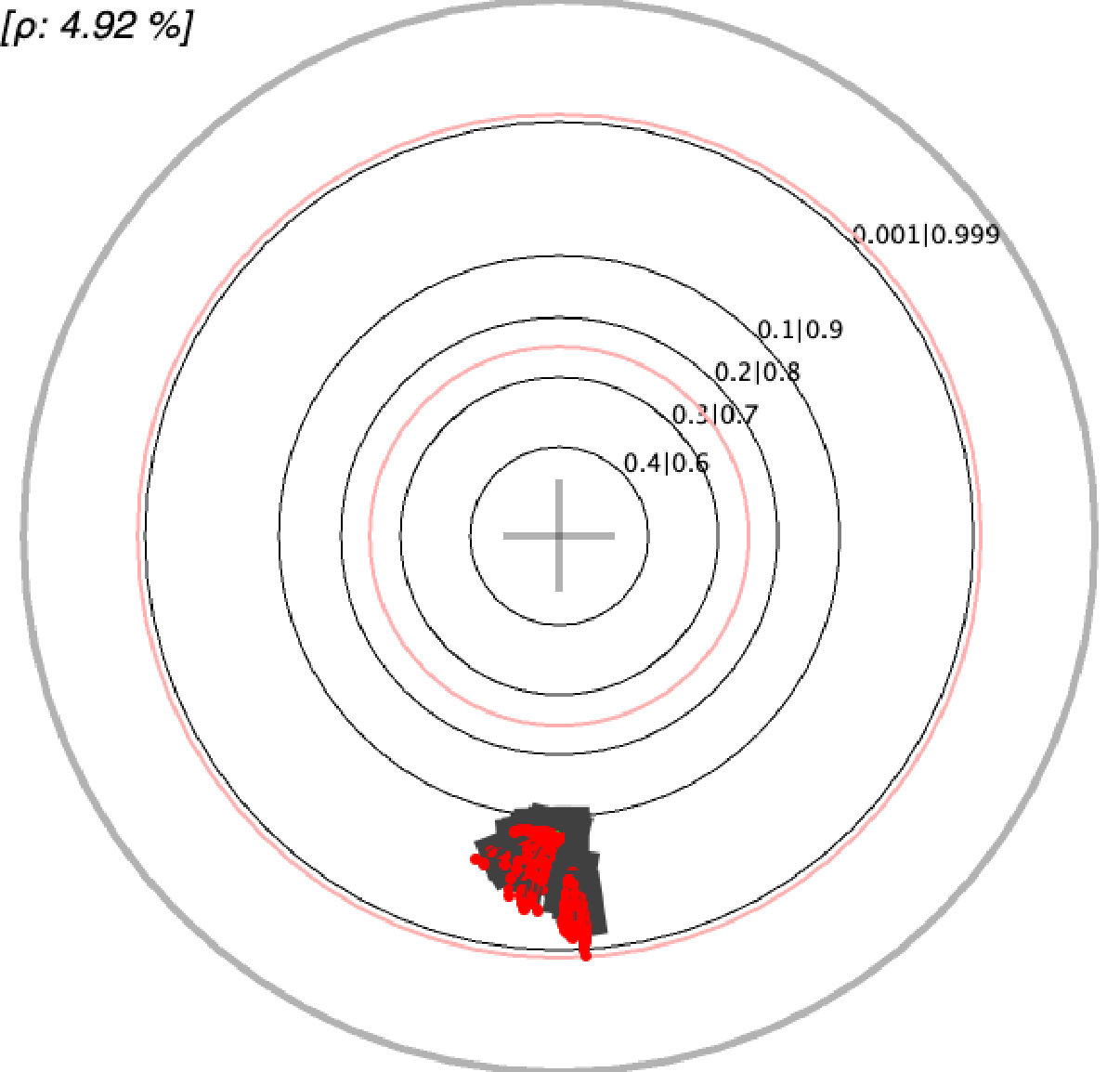} & \includegraphics[trim=0bp 0bp 0bp 0bp,clip,width=0.3\columnwidth]{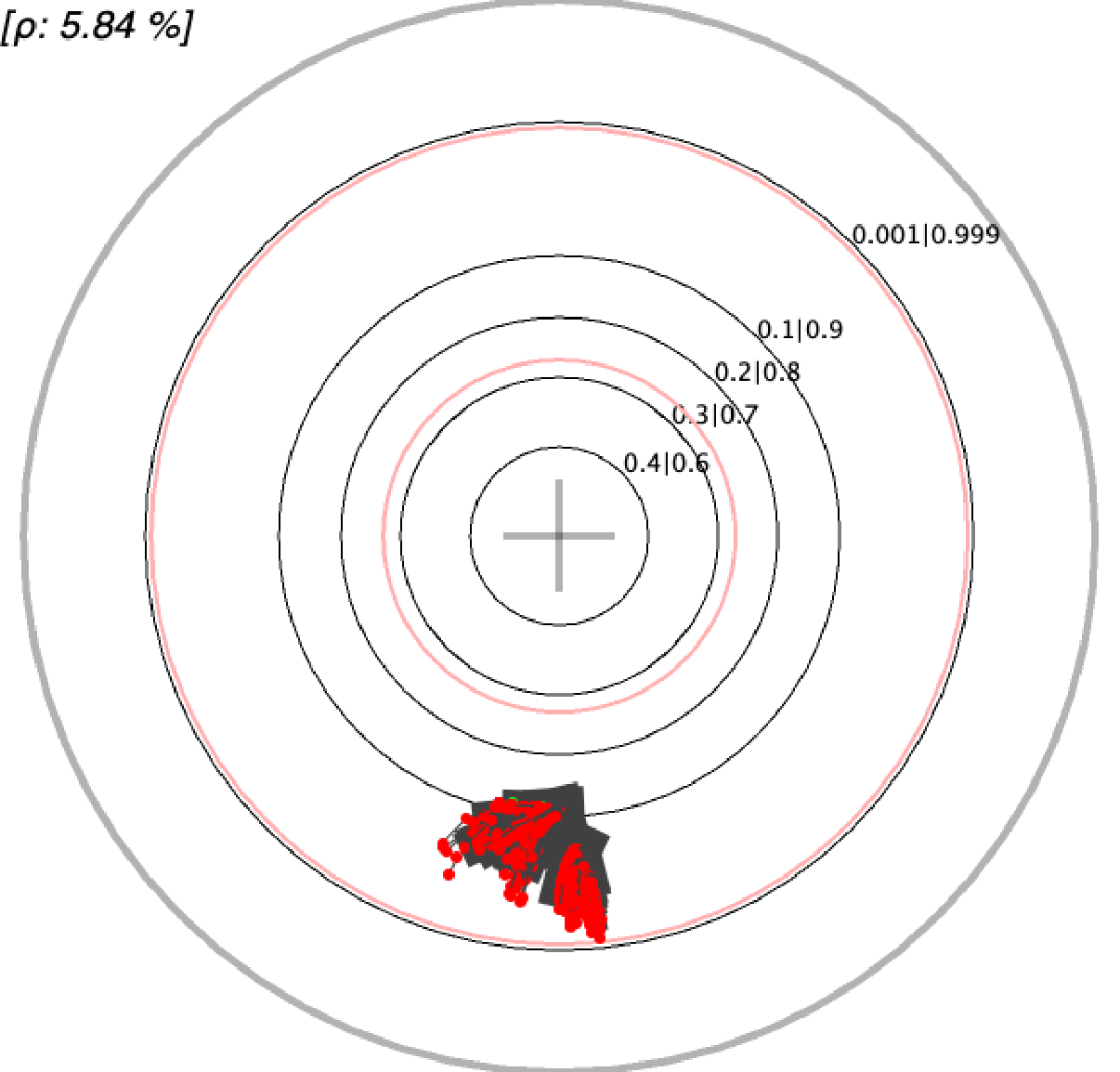} & \includegraphics[trim=0bp 0bp 0bp 0bp,clip,width=0.3\columnwidth]{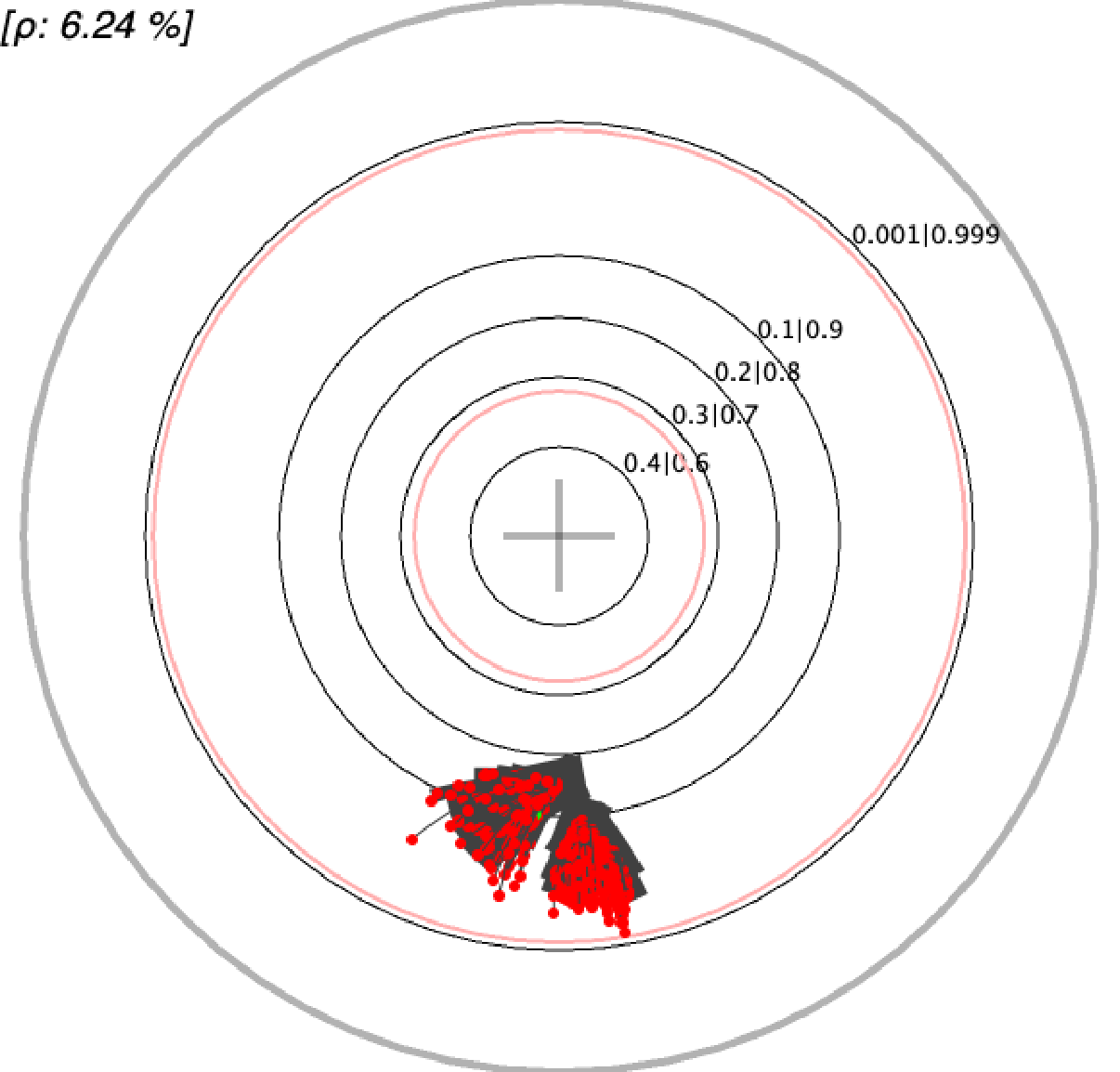} & \includegraphics[trim=0bp 0bp 0bp 0bp,clip,width=0.3\columnwidth]{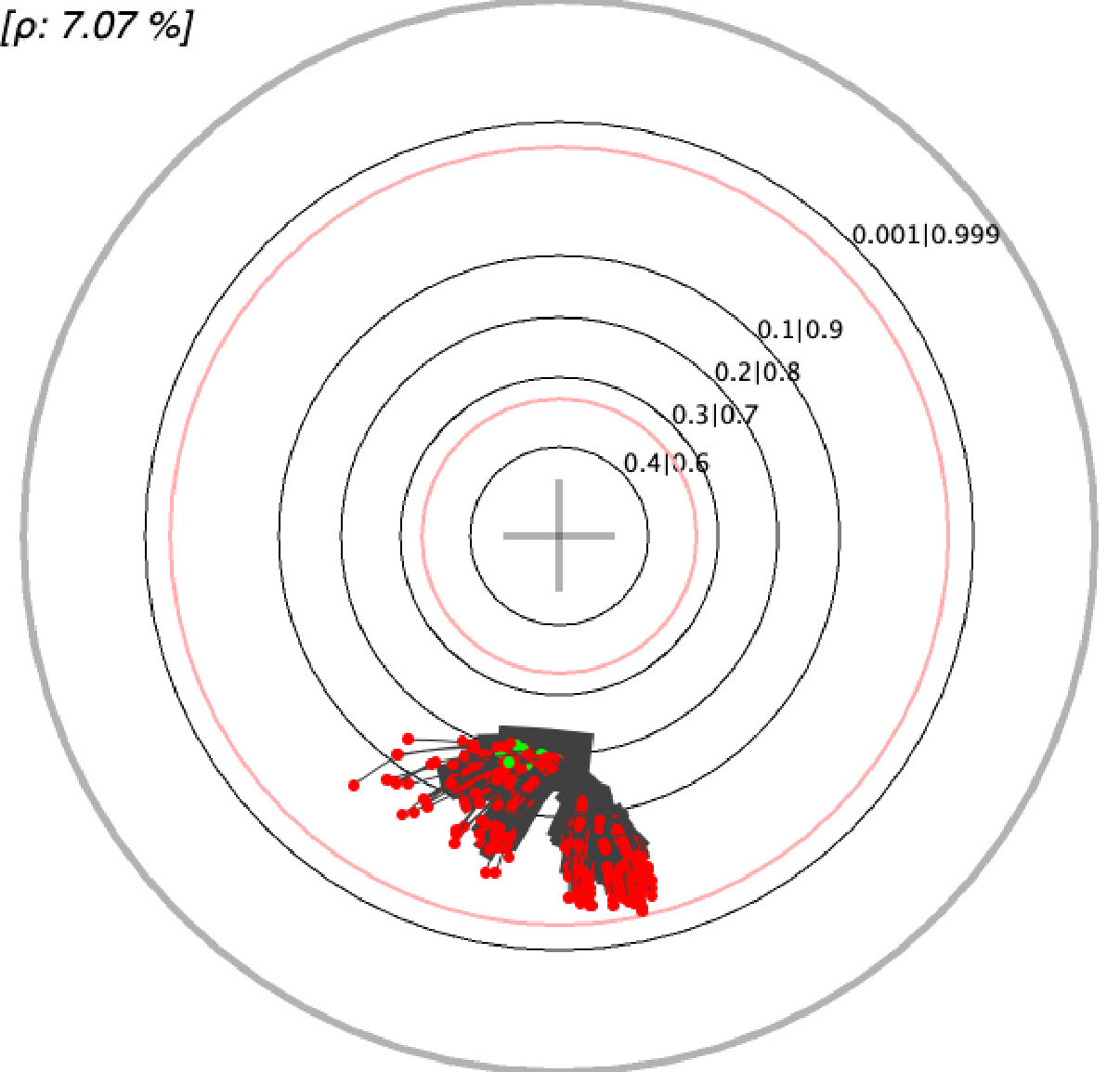} & \includegraphics[trim=0bp 0bp 0bp 0bp,clip,width=0.3\columnwidth]{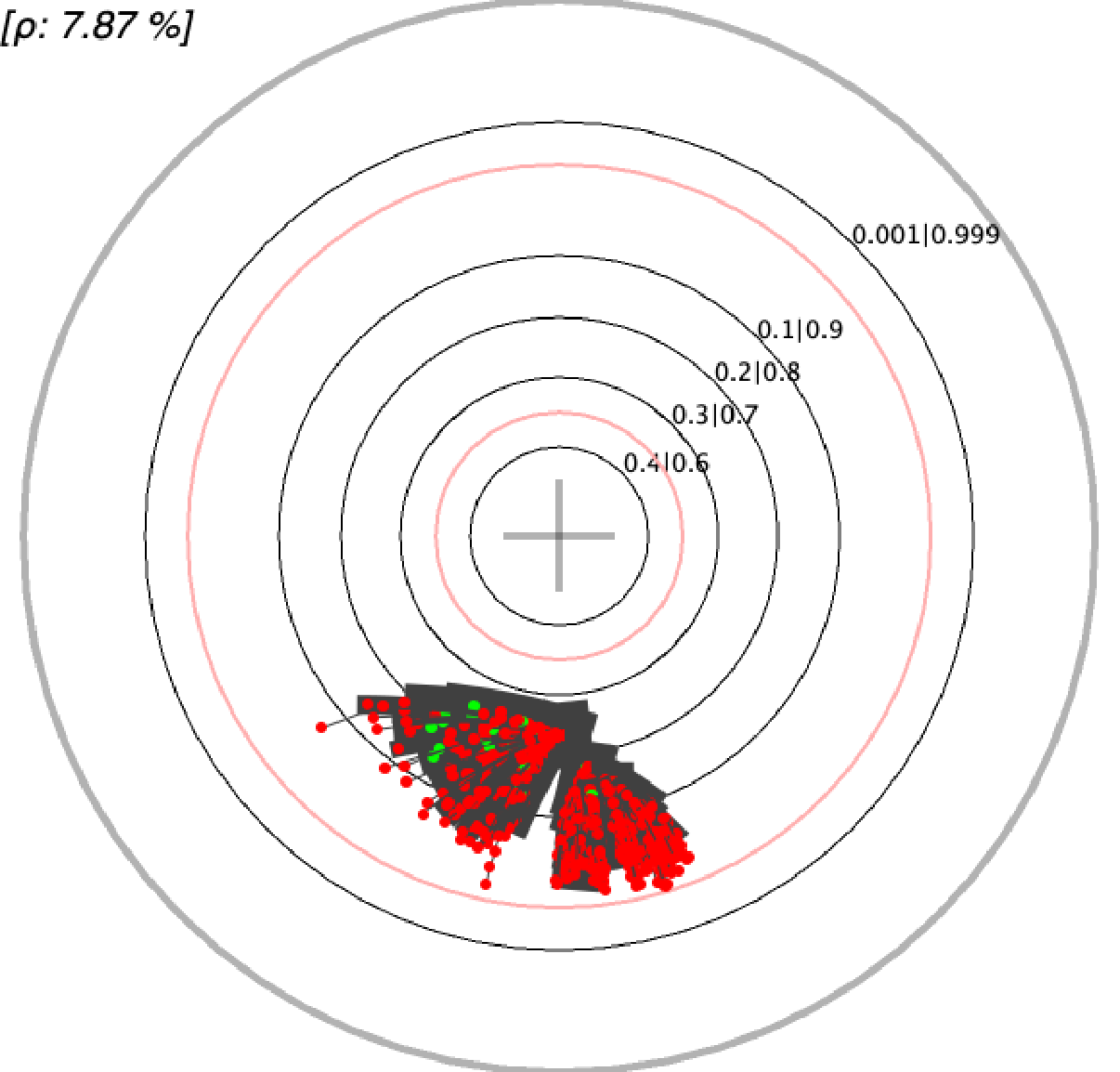} \\
                                                               &     MDT $\#1$ & MDT $\#2$& MDT $\#3$& MDT $\#4$& MDT $\#5$\\  \Xhline{2pt}
                              \rotatebox[origin=l]{90}{$t=1.0$} & \includegraphics[trim=0bp 0bp 0bp 0bp,clip,width=0.3\columnwidth]{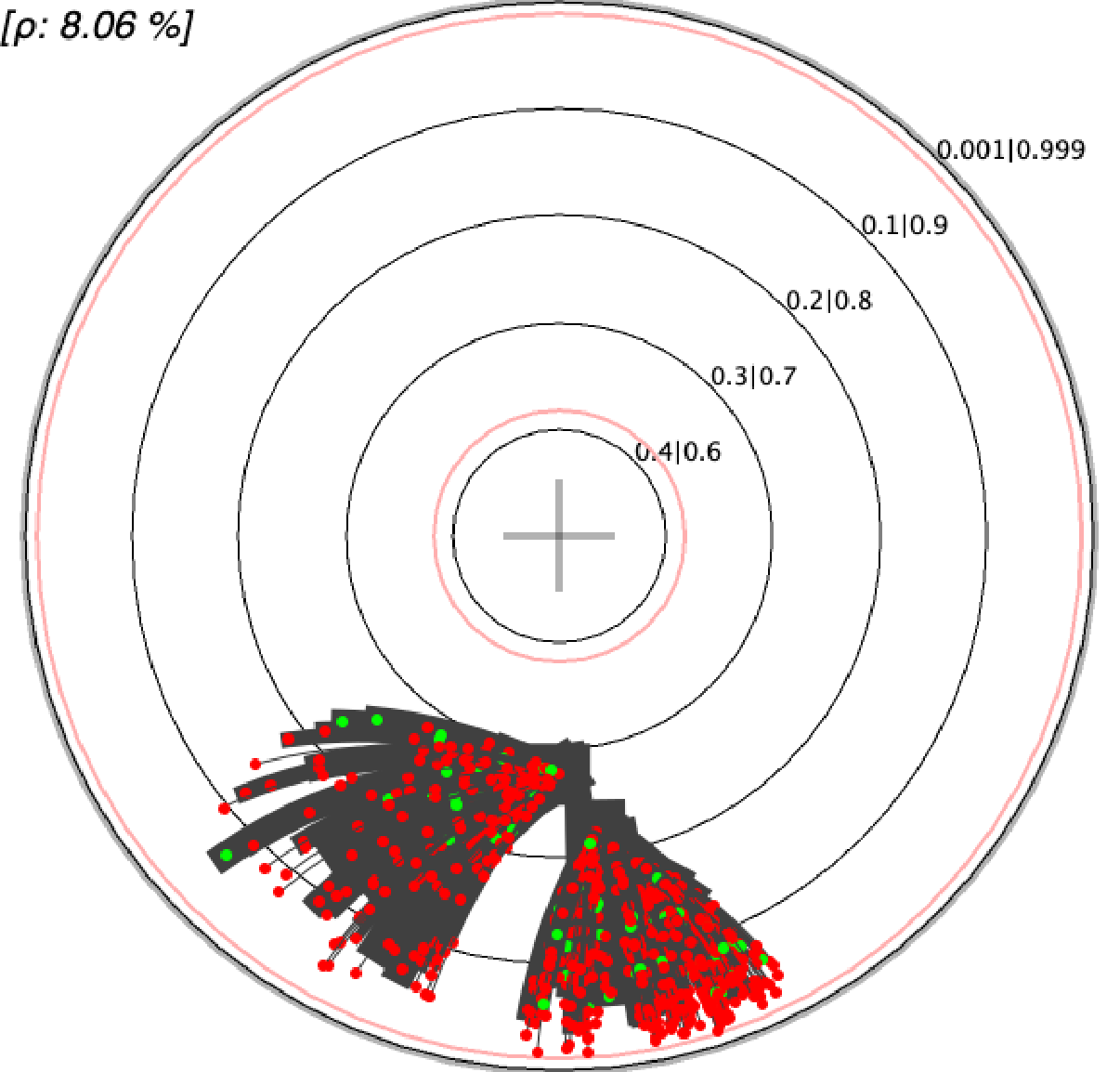} & \includegraphics[trim=0bp 0bp 0bp 0bp,clip,width=0.3\columnwidth]{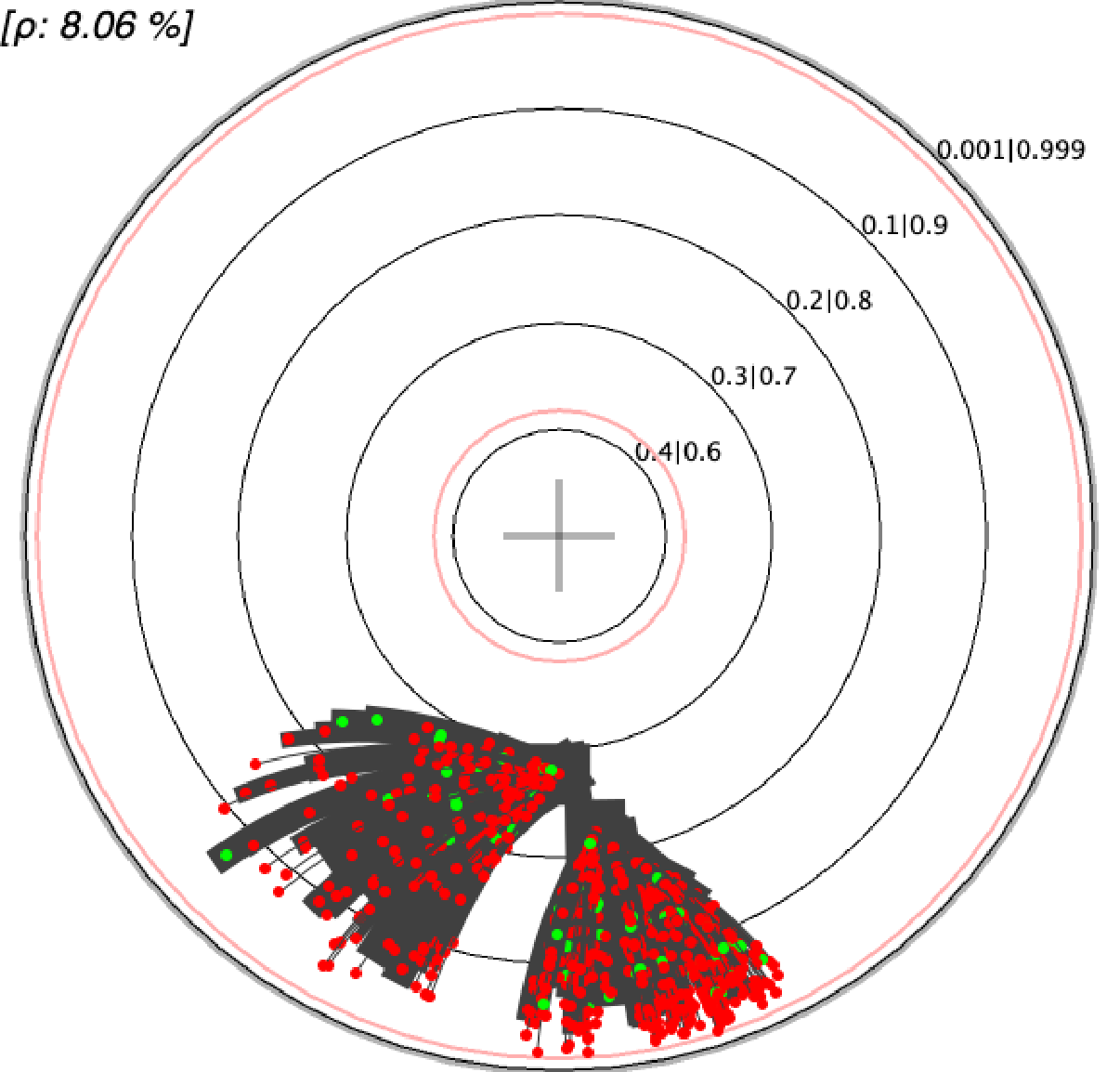} & \includegraphics[trim=0bp 0bp 0bp 0bp,clip,width=0.3\columnwidth]{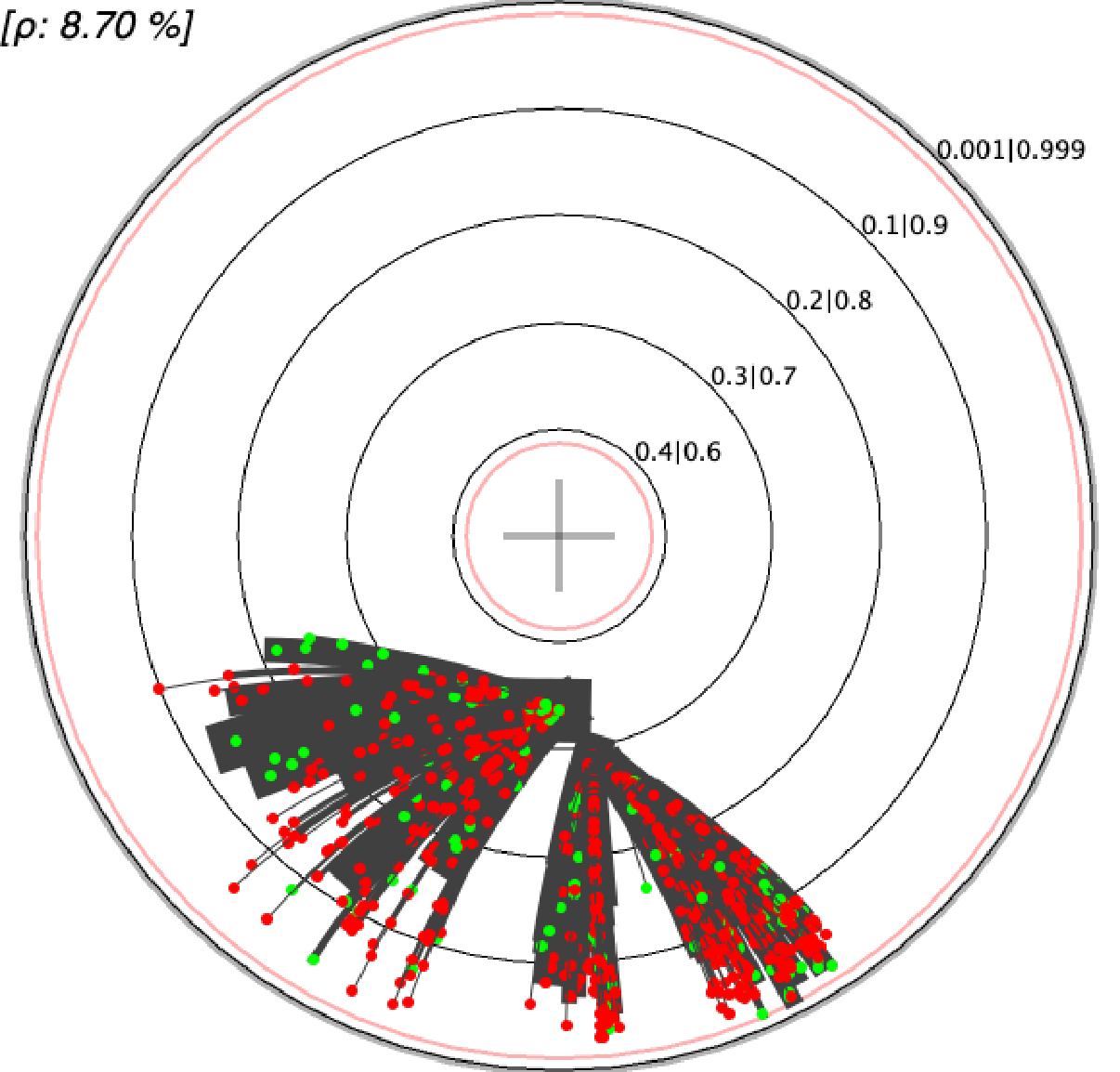} & \includegraphics[trim=0bp 0bp 0bp 0bp,clip,width=0.3\columnwidth]{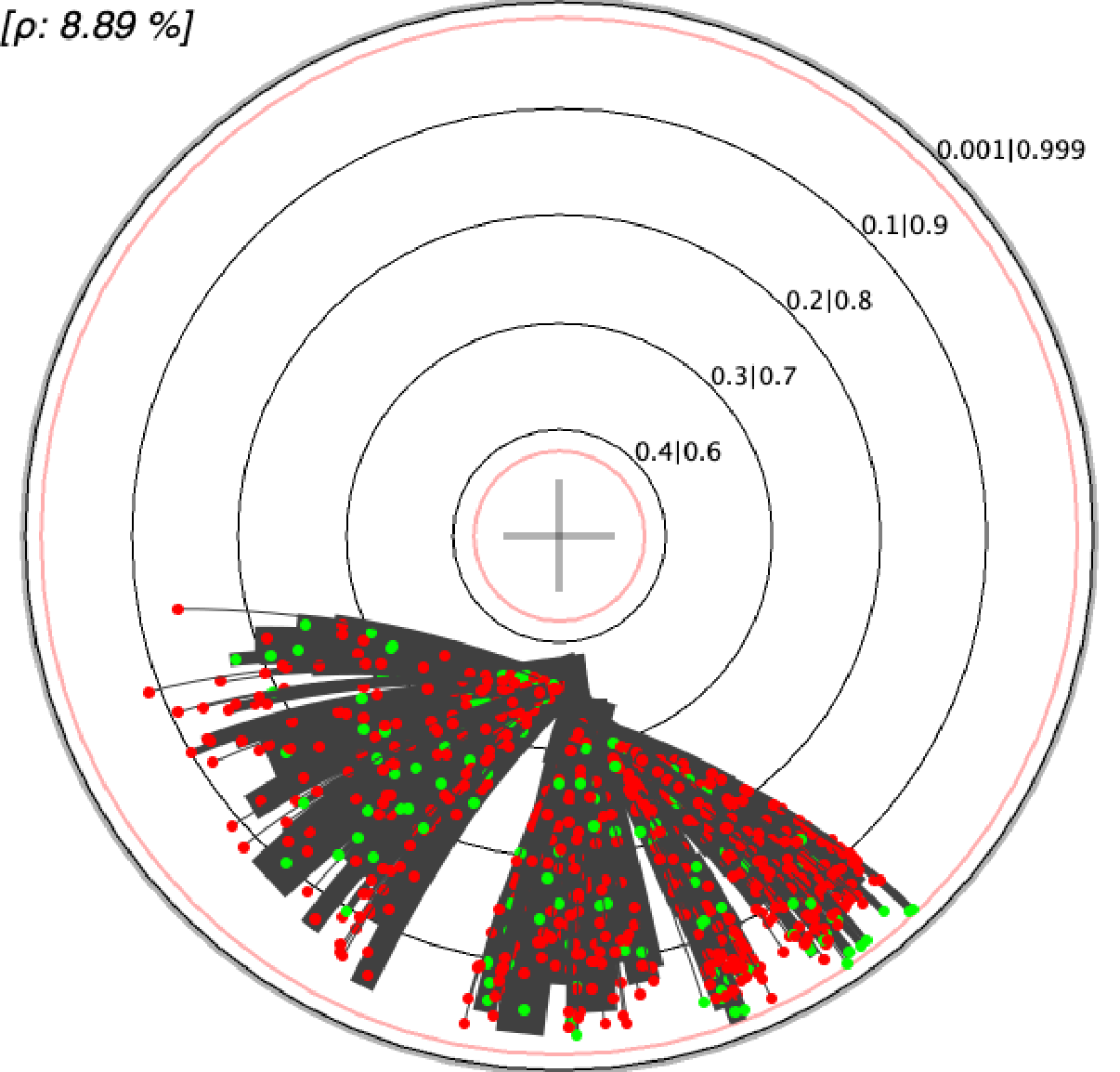} & \includegraphics[trim=0bp 0bp 0bp 0bp,clip,width=0.3\columnwidth]{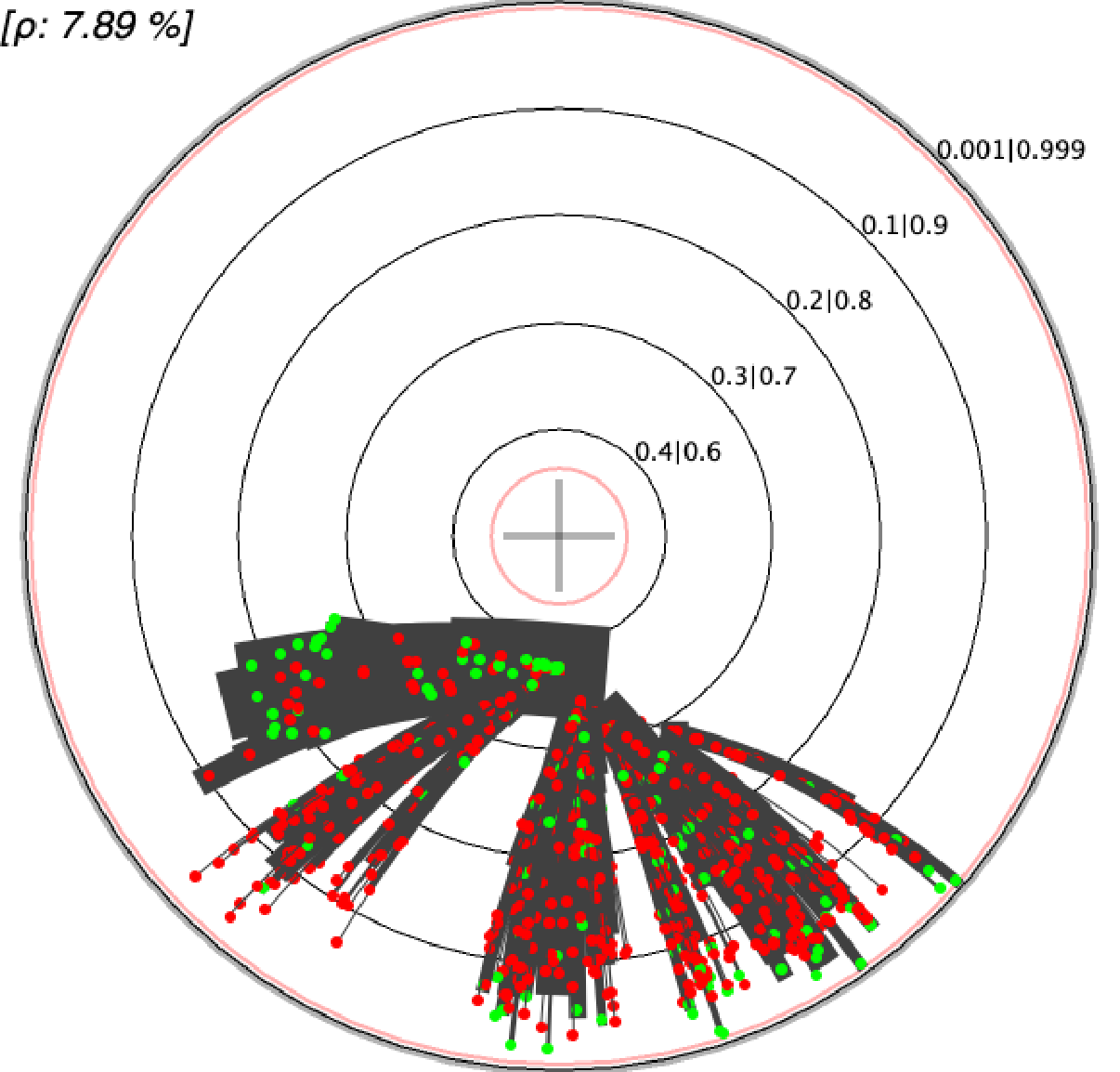}  \\
                              \rotatebox[origin=l]{90}{$t=0.0$} & \includegraphics[trim=0bp 0bp 0bp 0bp,clip,width=0.3\columnwidth]{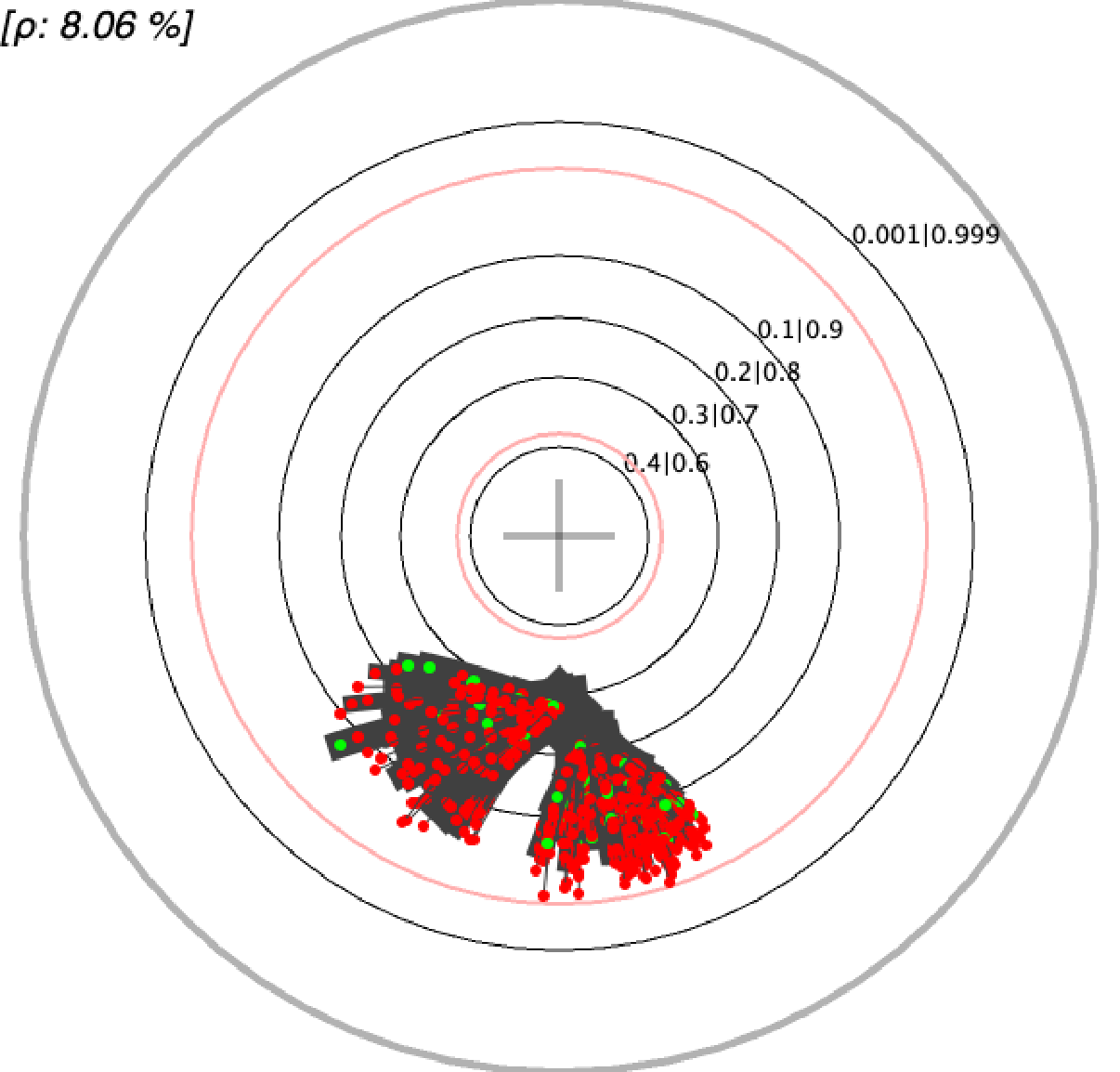} & \includegraphics[trim=0bp 0bp 0bp 0bp,clip,width=0.3\columnwidth]{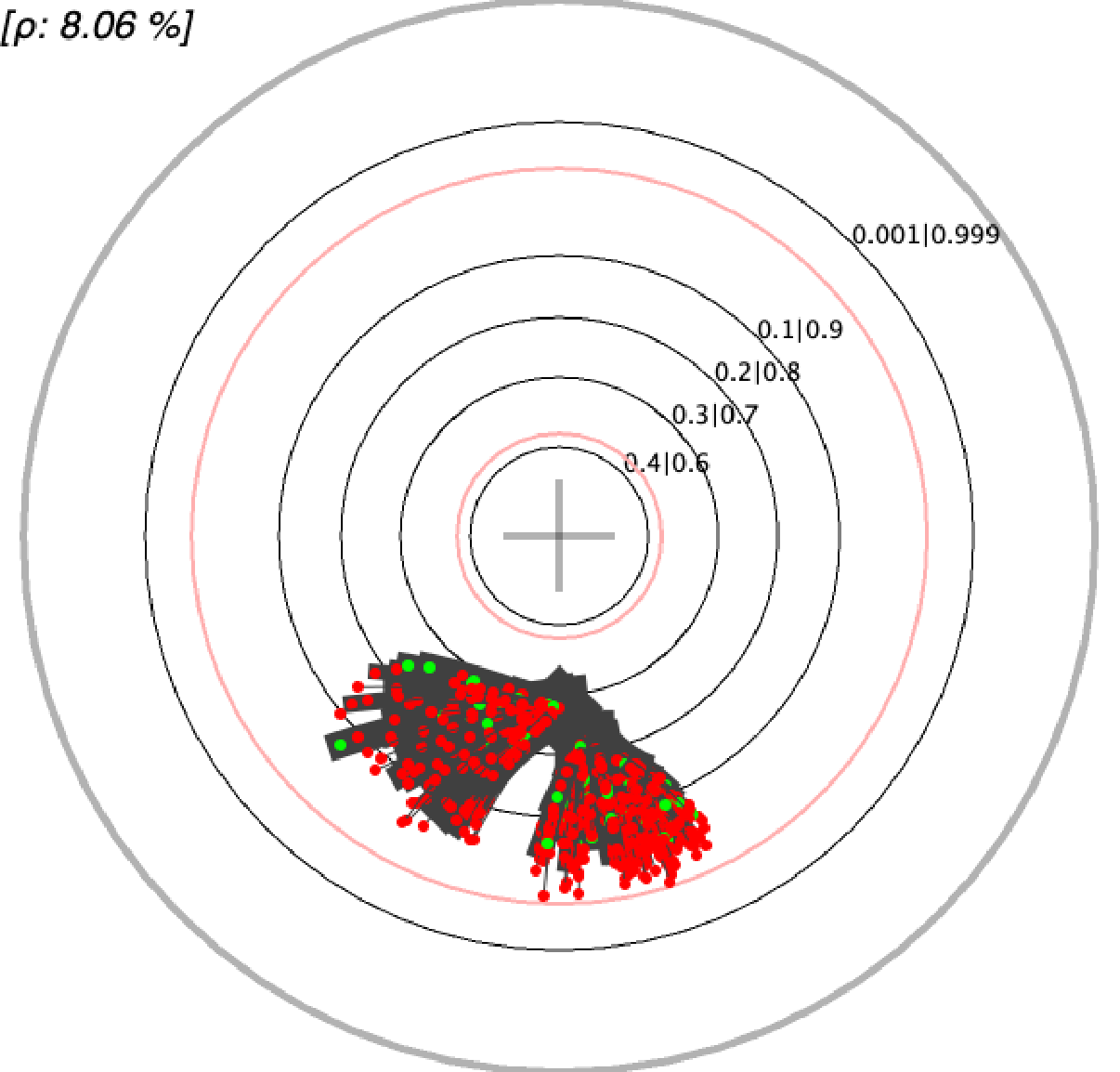} & \includegraphics[trim=0bp 0bp 0bp 0bp,clip,width=0.3\columnwidth]{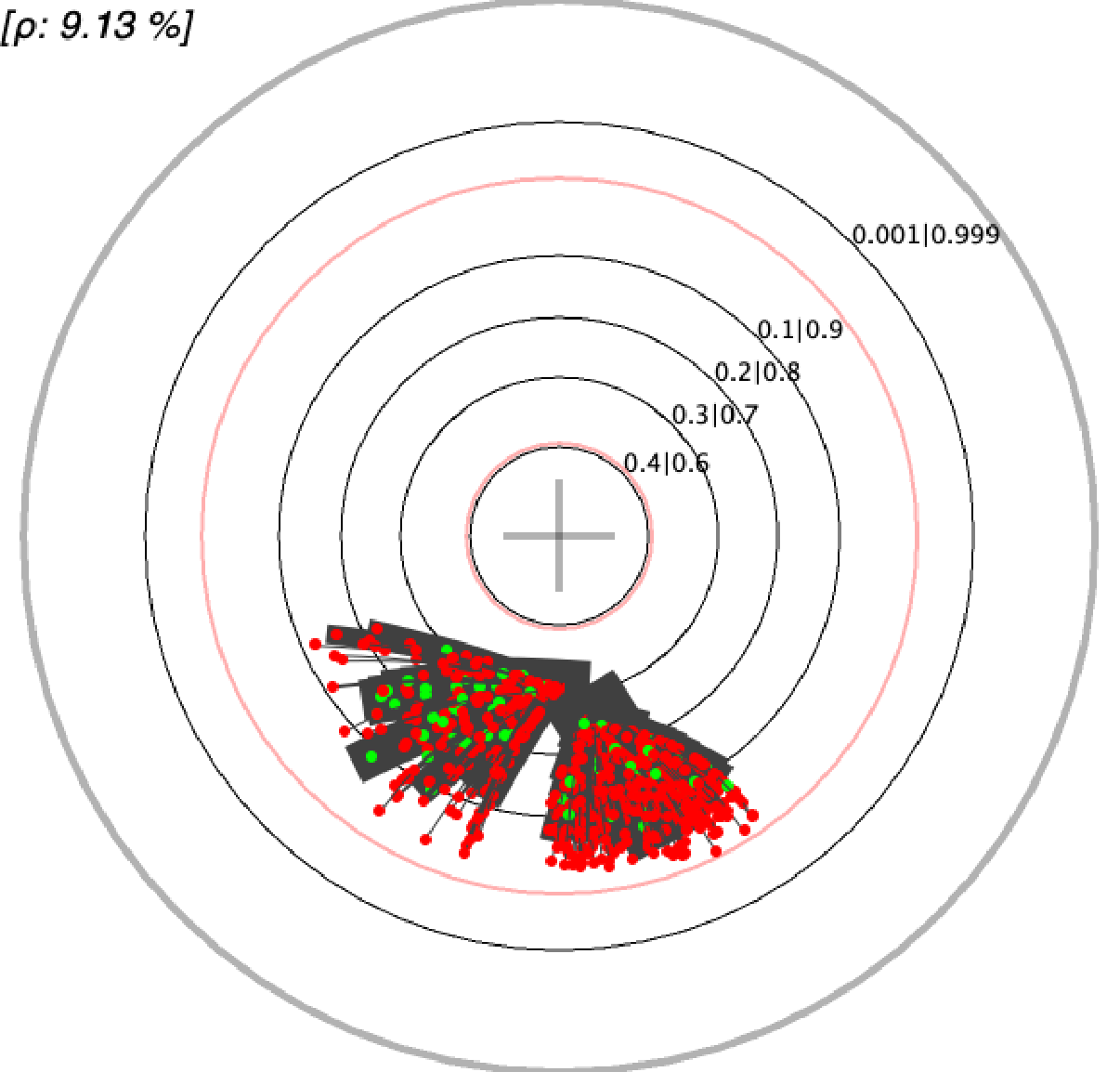} & \includegraphics[trim=0bp 0bp 0bp 0bp,clip,width=0.3\columnwidth]{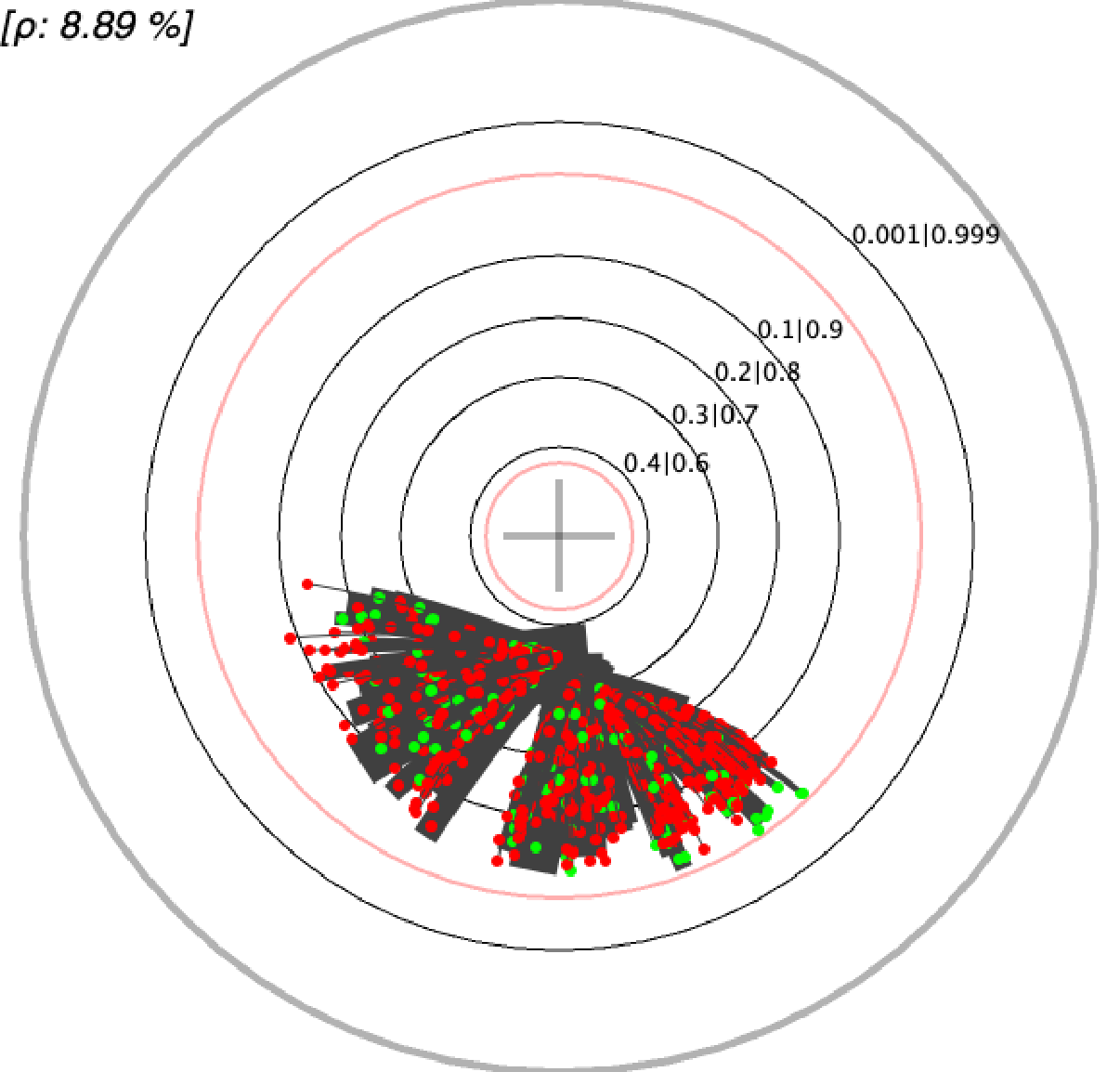} & \includegraphics[trim=0bp 0bp 0bp 0bp,clip,width=0.3\columnwidth]{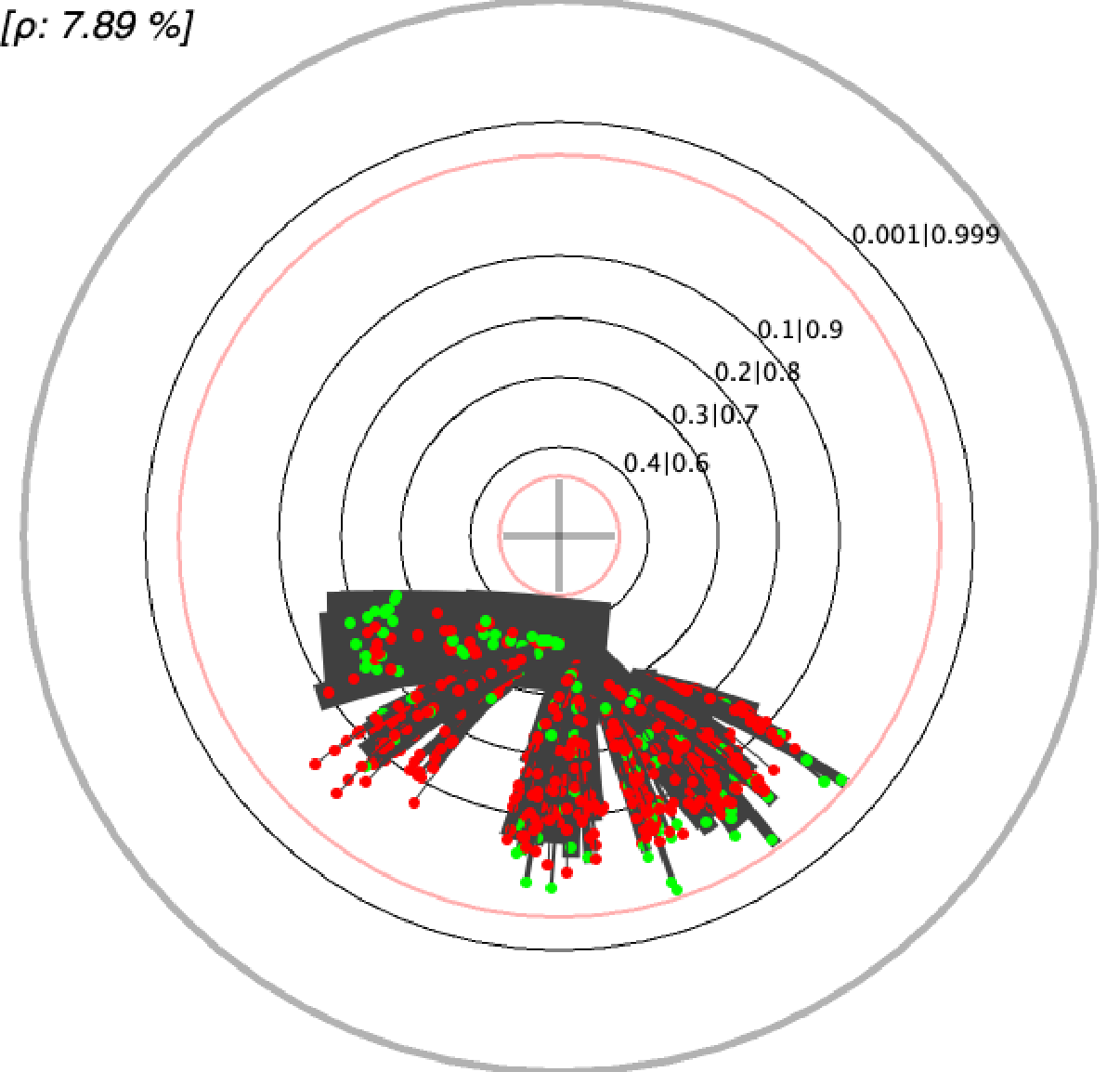} \\
                                                               &     MDT $\#6$ & MDT $\#7$& MDT $\#8$& MDT $\#9$& MDT $\#10$\\  \Xhline{2pt}
  \end{tabular}}
\caption{Plot of the 10 first MDTs learned on kaggle \domainname{give$\_$me$\_$some$\_$credit} (top panel and bottom panel). In each panel, we plot the embedding in $\mathbb{B}_1$ (top row) and the t-self $\mathbb{B}^{(0)}_1$ ($t=0$, bottom row). Remark the ability for the t-self to display a clear difference between the best subtrees, subtrees that otherwise appear quite equivalent in terms of confidence from $\mathbb{B}_1$ alone.}
    \label{tab:kaggle-dt-exerpt}
  \end{table}

\end{document}